%% file: sample-new.tex
\documentclass[twoside,11pt]{article}
 \usepackage[abbrvbib, preprint]{jmlr2e}

 \usepackage{hyperref}
 \usepackage{url}
 \usepackage{amsmath,amssymb,subcaption}
 \usepackage{graphicx}
 \usepackage{epstopdf}
 \usepackage{mathtools}
 \usepackage{paralist}
 \usepackage{cleveref}
 \usepackage{enumitem}
 \usepackage{tikz}
 \usepackage{caption}
 \usepackage{algorithm}
 \usepackage{algpseudocode}
 \usepackage[toc,page]{appendix}
 
 \DeclareMathOperator{\diag}{diag}
 \newtheorem{assumption}{Assumption}

\input{math_commands.tex}

\usepackage{lastpage}

\ShortHeadings{Learning Neural Networks by Neuron Pursuit}{Kumar and Haupt}
\firstpageno{1}

\begin{document}

\title{Learning Neural Networks by Neuron Pursuit}

\author{\name Akshay Kumar \email kumar511@umn.edu \\
       \addr Department of Electrical and Computer Engineering\\
       University of Minnesota\\
       Minneapolis, MN 55455, USA
       \AND
       \name Jarvis Haupt \email jdhaupt@umn.edu \\
       \addr Department of Electrical and Computer Engineering\\
       University of Minnesota\\
       Minneapolis, MN 55455, USA}

\editor{}

\maketitle

\begin{abstract}
	The first part of this paper studies the evolution of gradient flow for homogeneous neural networks near a class of saddle points exhibiting a sparsity structure. The choice of these saddle points is motivated from previous works on homogeneous networks, which identified the first saddle point encountered by gradient flow after escaping the origin. It is shown here that, when initialized sufficiently close to such saddle points, gradient flow remains near the saddle point for  a sufficiently long time, during which the set of weights with small norm remain small but converge in direction. Furthermore, important empirical observations are made on the behavior of gradient descent \emph{after} escaping these saddle points. The second part of the paper, motivated by these results, introduces a greedy algorithm to train deep neural networks called \emph{Neuron Pursuit} (NP). It is an iterative procedure which alternates between expanding the network by adding neuron(s) with carefully chosen weights, and minimizing the training loss using this augmented network. The efficacy of the proposed algorithm is validated using numerical experiments.
\end{abstract}


\section{Introduction}
Deep neural networks trained using gradient-based methods exhibit remarkable generalization performance. Despite this empirical success, our theoretical understanding of why and how these networks perform so well remains limited. A widely held hypothesis is that the \emph{implicit regularization} induced by the training algorithm plays a pivotal role in this success \citep{soudry_ib}. This belief has motivated extensive research into the dynamics of neural network training, yielding several important insights \citep{ntk,chizat_lazy,mei_mean,Lyu_ib,fl_yang}. Nevertheless, a comprehensive theoretical understanding of the training dynamics is still lacking.

An important insight emerging from these works is that the scale of initialization dictates two fundamentally different training regimes. In the large initialization regime\textemdash referred to as the Neural Tangent Kernel (NTK) regime \citep{ntk} or the \emph{lazy} regime \citep{chizat_lazy}\textemdash the weights of the neural network remain close to the initialization throughout the training, and the training dynamics are linear. This is in stark contrast to the small initialization regime, where the weights change significantly during training and the training dynamics are extremely non-linear. This regime is also known as the \emph{feature learning} regime \citep{geiger_feature,fl_yang, mei_mean, srebro_ib}, since the weights adapt to the underlying features present in the data. While the large initialization regime has been analyzed in considerable depth, our theoretical understanding of the small initialization regime remains relatively limited.

Early investigations into the small initialization regime studied the limiting behavior of gradient-based methods. In diagonal linear networks, gradient flow with vanishing initialization converges to a minimum  $\ell_1$-norm solution \citep{srebro_ib,dln_sparse}. In fully-connected linear networks, small initialization has been empirically observed to bias the training dynamics toward low-rank solutions \citep{guna_mtx_fct,cohen_mtx_fct}, with rigorous guarantees established for matrix factorization \citep{liwei_mf,chou_ib} and matrix sensing using two-layer networks \citep{mahdi_ib,lee_saddle,xiong_mat_sen,ma_ms}. For small initialization, empirical evidence indicates that non-linear neural networks also seek low-complexity solutions \citep{chizat_lazy}, and theoretical results have been obtained in many cases involving shallow networks such as linearly separable data \citep{phuong_ib,lyu_simp,wang_saddle}, orthogonal or nearly orthogonal data \citep{gf_orth, frei_orth}, and learning the XOR function \citep{glasgow_xor,alon_xor}. Recent works have also explored the ability of neural networks to learn sparse functions, such as single- or multi-index functions, though these too are largely confined to shallow networks \citep{bietti_ind,abbe_msp,damian_reps,lee_sgd,dandi_multi}. Overall, it is conjectured that deep neural networks trained via gradient descent with small initialization, are implicitly regularized  toward low-complexity solutions. However, our theoretical understanding of this phenomenon remains incomplete.

A more recent line of work has sought to understand the early training dynamics of deep \emph{homogeneous} neural networks in the small initialization regime \citep{kumar_dc,early_dc,maennel_quant,atanasov_align,boursier_early,bantzis_early}. These studies show that, for sufficiently small initialization, the weights remain small in norm and near the origin during the early stages of training but converge in direction\textemdash a phenomenon referred to as early directional convergence. Moreover, in feed-forward deep homogeneous networks, the weights converge to a direction such that the norm of incoming and outgoing weights of each hidden neuron is proportional. Consequently, if the incoming weights of a hidden neuron are zero, its outgoing weights must also be zero, and vice versa. Empirically, \citet{early_dc} observed that many hidden neurons exhibit this behavior, with their incoming and outgoing weights becoming zero in the early stages of training\textemdash resulting in the emergence of a sparsity structure among the weights.

Building on these findings, \citet{kumar_escape} investigate the subsequent phase of training\textemdash the gradient flow dynamics of homogeneous neural networks after the weights escape from the origin. They show that after escaping the origin the weights get close to a saddle point of the training loss, and they characterize this saddle point. Moreover, for feed-forward homogeneous neural networks, the sparsity structure which emerges among the weights before the escape is preserved throughout this phase. In particular, for the set of neurons whose incoming and outgoing weights became zero during the early stages of training, those weights remained zero even after escaping from the origin and until reaching the saddle point.

In the small initialization regime, a complementary line of research has observed an intriguing phenomenon in the trajectory of gradient descent, known as \emph{saddle-to-saddle dynamics} \citep{jacot_sd, lyu_resolving}. Over the course of training, gradient descent passes through a sequence of saddle points, where the network appears to increase its complexity as it moves from one saddle point to another.  This is also reflected in the loss curve, which alternates between long plateaus and sharp drops. This phenomenon has also been referred to as \emph{incremental learning} \citep{gidel_incr, gissin_incr,razin_incr}, as the network learns increasingly complex functions in phases. Formal results on this phenomenon have been established in certain settings, such as linear neural networks \citep{pesme_sd,lee_saddle, abbe_inc,simon_stp} and two-layer nonlinear networks for specific training data \citep{gf_orth, montanari_saddle,wang_saddle,abbe_sgd}.  Notably, the work of \citet{kumar_escape} can be interpreted as characterizing the first saddle point encountered by gradient flow after escaping the origin, for homogeneous neural networks. However, establishing saddle-to-saddle dynamics throughout the entire course of training, especially in deeper neural networks, remains an open problem.

\subsection{Our Contributions}
In the first part of this paper, we study the gradient flow dynamics of homogeneous neural networks near saddle points with a sparsity structure. For feed-forward networks, inspired from \citet{kumar_escape}, we consider saddle points where  a subset of hidden neurons has both incoming and outgoing weights with zero norm. In \Cref{thm_dir_convg}, we show that when initialized sufficiently close to such saddle points, gradient flow remains near the saddle point for a sufficiently long time. During this period, the subset of weights initialized with small norm remain small but converge in direction. Moreover, they converge to a direction such that the norms of the incoming and outgoing weights are proportional. Overall, the training dynamics near these saddle points share many similarities with the training dynamics near the origin, as described previously in \citet{kumar_dc,early_dc}. 

Establishing our results requires overcoming obstacles absent in prior works. The analyses in \citet{kumar_dc,early_dc} establish directional convergence for the \emph{entire} set of weights, by exploiting the homogeneity of the network output with respect to all parameters. In contrast, the setting here requires proving directional convergence for only a \emph{subset} of weights, for which the output is \emph{not} homogeneous\textemdash a key difference from earlier analyses. While \citet{kumar_dc} also considered saddle points with a sparsity structure, they imposed a separability assumption on the network architecture that ensured homogeneity with respect to the relevant weights. Such assumptions do not hold for fully-connected networks, making their techniques inapplicable in our setting. Our key technical contribution is to show that, near the saddle point, the network output can be decomposed into a term homogeneous in the desired subset of weights and another term independent of them. This decomposition, which relies crucially on the sparsity structure, is valid only in the neighborhood of the saddle point and enables us to establish directional convergence.

In \Cref{sec:bey_saddle}, we present empirical observations on the dynamics of gradient descent after escaping these saddle points. For feed-forward neural networks, we observe that, after escaping the saddle point, the trajectory gets close to another saddle point, consistent with the saddle-to-saddle dynamics hypothesis. Notably, the sparsity structure that emerges among the weights with small norm near the saddle point is preserved, even after escaping it and until reaching the next saddle point. Moreover, this new saddle point exhibits a similar sparsity structure, with incoming and outgoing weights of a subset of hidden neurons being zero. While we cannot rigorously establish these empirical observations, they suggest strong similarities between the dynamics of gradient descent after escaping these saddle point and after escaping the origin, as described in \citet{kumar_escape}. 

Combining these insights with our previous results, we argue that neural network training in the small initialization regime unfolds as a sequence of saddle-to-saddle transitions. Moreover, at each successive saddle point, the network systematically activates new subsets of neurons, a phenomenon driven entirely by directional convergence at the saddle point. This mechanistic framework encapsulates the central insight emerging collectively from our prior theoretical analyses and the present work.
 
Drawing on these insights, in \Cref{sec:np}, we present a greedy algorithm to train deep neural networks, which we call \emph{Neuron Pursuit} (NP). At a high level, NP is inspired from the saddle-to-saddle dynamics hypothesis and builds the network by moving from one saddle point of the training loss to another. It further leverages the sparsity structure that emerges at these saddle points and after escaping them, as studied in this paper and previous works. Concretely, the algorithm begins by considering a neural network with a specified (fixed) number of layers and one neuron in each layer, and then trains it to minimize the training loss via gradient descent using specifically chosen initial weights. It then proceeds iteratively: at each iteration, a neuron is added to the network (with location, incoming weights, and outgoing weights prescribed by the algorithm), after which the training loss is minimized via gradient descent using this augmented network. Compared to the traditional back-propagation algorithm, where the set of neurons in neural networks are fixed and it is trained end-to-end, in the NP algorithm the size of the neural networks gradually increases as training progresses. We also conduct experiments to demonstrate the learning capability of the NP algorithm.

Overall, our work is a step towards demystifying neural networks. The theoretical analyses deepen our understanding of their training dynamics, while the NP algorithm provides an alternative lens for understanding how feature learning unfolds in deep networks. Together, these contributions brings us closer to a principled understanding of mechanisms that drive generalization in deep networks.

\subsection{Notation}
\label{subsec:notation}
We use $\sN$ to denote the set of natural numbers, and for any $L\in\sN$, we let $[L] \coloneqq \{1,2,\cdots, L\}$.  For vectors, $\|\cdot\|_2$ denotes the $\ell_2$-norm. For matrices, $\|\cdot\|_F$ and $\|\cdot\|_2$ denote the Frobenius and spectral norms, respectively. For a matrix $\rmW$, $\rmW[:,j]$ and $\rmW[j,:]$ denote its $j$-th column and $j$-th row, respectively. For a vector $\rvp$, $p_i$ denotes its $i$th entry. The $d$-dimensional unit sphere is denoted by $\sS^{d-1}$. We use $\odot$ to denote elementwise multiplication between vectors or matrices. A KKT point of an optimization problem is called a non-negative (positive, zero) KKT point if the objective value at the KKT point is non-negative (positive, zero). 

\section{Background}
\label{sec:background}
In this section we introduce some of the key concepts instrumental to our analyses.\\

\noindent \textbf{Homogeneous neural netwoks.} For a neural network $\mathcal{H}$, $\mathcal{H}(\rvx;\rvw)$ denotes its output, where $\rvx\in \sR^d$ is the input and $\rvw\in \sR^k$ is a vector containing all the weights. A neural network $\mathcal{H}$ is referred to as $L$-\emph{(positively) homogeneous} if  
\begin{equation*}
\mathcal{H}(\rvx;c\rvw) =  c^L\mathcal{H}(\rvx;\rvw), \text{ for all } c\geq 0 \text{ and } \mathbf{w}\in \sR^k.
\end{equation*} 
Suppose $\{\rvx_i,y_i\}_{i=1}^n$ is a collection of training data with $(\rvx_i,y_i) \in \sR^{d}\times \sR$, and let $\rmX = \left[\rvx_1,\cdots,\rvx_n\right] \in \sR^{d \times n}$, $\rvy = \left[y_1,\cdots,y_n\right]^\top\in\sR^{ n}$. Let $\gH(\rmX;\rvw) = \left[\gH(\rvx_1;\rvw), \cdots, \gH(\rvx_n;\rvw)\right]\in\sR^{ n}$ be the vector containing network's outputs, and $\mathcal{J}(\rmX;\rvw)$ denotes the Jacobian of $\gH(\rmX;\rvw)$ with respect to $\mathbf{w}$. Assuming square loss is used for training, the training loss can be written as
\begin{equation}
\mathcal{L}(\rvw) = \frac{1}{2}\sum_{i=1}^n (\mathcal{H}(\rvx_i;\rvw)- y_i)^2 = \frac{1}{2}\|\mathcal{H}(\rmX;\rvw)- \rvy\|_2^2 .
\label{loss_fn}
\end{equation} 
Minimizing the above optimization problem using gradient flow  gives us the following differential equation:
\begin{equation}
\dot{\rvw} = -\nabla \mathcal{L}(\rvw) = -\mathcal{J}(\rmX;\rvw)^\top(\mathcal{H}(\rmX;\rvw)- \rvy).
\label{gf_eq}
\end{equation} 
We will use $\bm{\psi}(t,\rvw(0))$ to denote the solution of above differential equation, where  $\rvw(0)$ is the initialization.\\

\noindent\textbf{Feed-forward neural network.} Suppose $L\geq 2$. Then, the output of an $L$-layer feed-forward neural network $\mathcal{H}$ is defined as
\begin{equation}
\mathcal{H}(\rvx;\rmW_1,\cdots,\rmW_L) = \rmW_L\sigma(\rmW_{L-1}\sigma(\cdots\sigma(\rmW_1\rvx)\cdots)),
\label{L_layer_feed}
\end{equation}
where $\rmW_l\in \sR^{k_{l}\times k_{l-1}}$, $k_0 = d$ and $k_L = 1$, and the activation function $\sigma:\sR\to\sR$ is applied elementwise. Note that, $\rmW_l[j,:]$ contains  the incoming  weights to the $j$-th neuron in the $l$-th layer, and  $\rmW_{l+1}[:,j]$ contains the outgoing weights from the same neuron. Also, if $\sigma(x) = \max(x,\alpha x)^p$, for some $p\in\sN$ and $\alpha\in \sR$, then the above neural network is positively homogeneous with respect to its weights. 

We next briefly review the results of \cite{kumar_dc, early_dc,kumar_escape}, which study the phenomenon of early directional convergence in homogeneous neural networks and the dynamics of gradient flow after the weights escape from the origin. These works are particularly relevant, as our analysis builds upon and extends them. Moreover, the behavior of gradient flow near and beyond the saddle points studied in this paper, share many similarities with these earlier results.
\subsection{Early Directional Convergence}
\label{sec:early_dir}
For neural networks with degree of homogeneity two or higher, the origin is a critical point of the training loss in \cref{loss_fn}. Hence, if initialized near the origin, gradient flow will remain near the origin for some time before eventually escaping it. In \cite{kumar_dc, early_dc}, the authors analyze the training dynamics while the trajectory remains close to the origin. To better describe their results, we introduce some basic concepts.  For a vector $\rvz$ and neural network $\mathcal{H}$, the Neural Correlation Function (NCF) is defined as 
\begin{equation}
\mathcal{N}_{\rvz,\mathcal{H}}(\rvu) =  {\rvz}^\top\gH(\rmX;\rvu).
\label{ncf_gn}
\end{equation}
The NCF measures the correlation between the vector $\rvz$ and the output of the neural network. We omit the dependent variable in the definition of the NCF when it is clear from context. We use $\widetilde{\mathcal{N}}_{\rvz,\mathcal{H}}$ to denote the corresponding constrained NCF problem which is defined as
\begin{equation}
\widetilde{\mathcal{N}}_{\rvz,\mathcal{H}} \coloneqq \max_{\rvu} \mathcal{N}_{\rvz,\mathcal{H}}(\rvu), \text{ s.t. } \|\rvu\|_2^2 = 1.
\label{ncf_gn_const}
\end{equation}
Next, consider the (positive) gradient flow of the NCF:
\begin{equation}
\dot{\rvu} = \nabla\mathcal{N}_{\rvz,\mathcal{H}}(\rvu).
\label{ncf_gf}
\end{equation} 
We use $\bm{\phi}(t,\rvu(0);\mathcal{N}_{\rvz,\mathcal{H}})$ to denote the solution of the above differential equation, where $\rvu(0)$ is the initialization. The following lemma from \cite{kumar_dc, early_dc} describes the limiting dynamics of the gradient flow of the NCF, showing that it either converges to the origin or goes to infinity and converges in direction.
\begin{lemma}
	Suppose $\mathcal{H}$ is $L$-homogeneous, for some $L\geq 2$. For any vector $\rvz$ and initialization $\rvu_0\in \sR^{k}$, either
	\begin{itemize}
		\item $\bm{\phi}(t,\rvu_0;\mathcal{N}_{\rvz,\mathcal{H}})$ converges to the origin,
		\item or $\bm{\phi}(t,\rvu_0;\mathcal{N}_{\rvz,\mathcal{H}})$ goes to infinity and $\frac{\bm{\phi}(t,\rvu_0;\mathcal{N}_{\rvz,\mathcal{H}})}{\|\bm{\phi}(t,\rvu_0;\mathcal{N}_{\rvz,\mathcal{H}})\|_2}$ converges in direction to a non-negative KKT point of $\widetilde{\mathcal{N}}_{\rvz,\mathcal{H}}$.
	\end{itemize}
	\label{lemma:gf_ncf}
\end{lemma}
We next define the notion of stable set for a KKT point of the constrained NCF. 
\begin{definition}
	The stable set $\mathcal{S}(\rvu_*;{\mathcal{N}}_{\rvz,\mathcal{H}})$ of a non-negative KKT point $\rvu_*$ of $\widetilde{\mathcal{N}}_{\rvz,\mathcal{H}}$ is the set of all unit-norm initializations such that  gradient flow of the NCF converges in direction to $\rvu_*$:
	\begin{equation*}
	\mathcal{S}(\rvu_*;{\mathcal{N}}_{\rvz,\mathcal{H}}) \coloneqq \left\{\rvu_0 \in \sS^{k-1}: \frac{\bm{\phi}(t,\rvu_0;{\mathcal{N}}_{\rvz,\mathcal{H}})}{\|\bm{\phi}(t,\rvu_0;{\mathcal{N}}_{\rvz,\mathcal{H}})\|_2} \rightarrow \rvu_*\right\}
	\end{equation*} 
\end{definition}
The next lemma describes the phenomenon of early directional convergence during the early stages of training \citep{kumar_dc, early_dc}. 
\begin{lemma}\label{lemma_early_dc}
	Suppose $\rvw_0$ is a unit-norm vector. For any arbitrarily small $\epsilon>0$, there exists $T$ such that for all sufficiently small $\delta$ we have
	\begin{equation*}
	\|\bm{\psi}(t,\delta\rvw_0)\|_2 = O(\delta), \text{ for all } t\in [0,T/\delta^{L-2}].
	\end{equation*}
	Furthermore, if  $\rvw_0\in \mathcal{S}(\rvw_*;{\mathcal{N}}_{\rvy,\mathcal{H}})$, where $\rvw_*$  is a non-negative KKT point of $\widetilde{\mathcal{N}}_{\rvy,\mathcal{H}}$, then
	\begin{equation*}
	\|\bm{\psi}(T/\delta^{L-2},\delta\rvw_0)\|_2 \geq \delta\eta \text{ and }\frac{\bm{\psi}(T/\delta^{L-2},\delta\rvw_0)^\top\rvw_*}{\|\bm{\psi}(T/\delta^{L-2},\delta\rvw_0)\|_2} = 1- O(\epsilon),
	\end{equation*}
	else, $\|\bm{\psi}(T/\delta^{L-2},\delta\rvw_0)\|_2 = \epsilon \cdot O(\delta)$. Here, $\eta$ is a positive constant independent of $\epsilon$ and $\delta$.
\end{lemma}
The above lemma describes the evolution of weights under gradient flow with initialization $\delta\rvw_0$, where $\delta>0$ is a scalar that controls the scale of initialization. It shows that for small initialization, the weights remain small during the early stages of training. Moreover, if the initial direction $\rvw_0$  belongs to the stable set of a non-negative KKT point of $\widetilde{\mathcal{N}}_{\rvy,\mathcal{H}}$, the constrained NCF defined with respect to $\rvy$ and $\mathcal{H}$, then the weights approximately converge in direction towards that KKT point. Also, if $\rvw_0$ does not belong to the stable set of a non-negative KKT point, then $\bm{\phi}(t,\rvw_0;{\mathcal{N}}_{\rvy,\mathcal{H}})$ converges to the origin (see Lemma \ref{lemma:gf_ncf}). In such cases, instead of directional convergence, the weights approximately become zero, as $\|\rvw_z(T_\epsilon/\delta^{L-2})\|_2 = \epsilon\cdot O(\delta)$, where $\epsilon$ and $\delta$ are both small. In contrast, in the previous case, $\|\rvw_z(T/\delta^{L-2})\|_2 \geq \delta\eta$, where $\eta$ is a constant.

For feed-forward homogeneous neural networks, the next lemma states an important property of \emph{positive} KKT points of the constrained NCF. 
\begin{lemma}\label{bal_r1_kkt}
	Let $\mathcal{H}$ be an $L-$layer feed-forward neural network as in \cref{L_layer_feed}, where $\sigma(x) = \max(x,\alpha x)^p$, for some $p\in\sN$ and $\alpha\in \sR$. Let $\left(\overline{\rmW}_{1},\cdots \overline{\rmW}_{L}\right)$be a positive KKT point of 
	\begin{equation}
	\max_{\rmW_1,\cdots,\rmW_L} \mathcal{N}_{\rvz,\mathcal{H}}(\rmW_1,\cdots,\rmW_L)\coloneqq{\rvz}^\top\mathcal{H}(\rmX;\rmW_1,\cdots,\rmW_L),  \text{ s.t. } \sum_{i=1}^L\left\|\rmW_i\right\|_F^2 = 1.
	\label{ncf_const_ff}
	\end{equation}
	Then, $\left\|\overline{\rmW}_l[j,:]\right\|_2^2 = p\left\|\overline{\rmW}_{l+1}[:,j]\right\|_2^2$, for all $j\in [k_l]$ and $l\in [L-1]$.
\end{lemma}
In the above lemma, the condition $\left\|\overline{\rmW}_l[j,:]\right\|_2^2 = p\left\|\overline{\rmW}_{l+1}[:,j]\right\|_2^2$ implies that the norm of each hidden neuron's incoming weights is proportional to the norm of its outgoing weights. Consequently, if the incoming weights of a neuron have zero norm, its outgoing weights must also have zero norm, and vice-versa. 

Since gradient flow converges in direction to a KKT point of the constrained NCF in the early stages of training, the weights in the early stages will also satisfy this property. Empirically, the weights usually converge to a KKT point where only a few neurons have non-zero incoming and outgoing weights, leading to the emergence of a sparsity structure among the weights in the early stages of training.
In fact, for $p\geq 2$, \citet{early_dc} observed that typically only a single neuron in each layer had non-zero incoming and outgoing weights. For $p=1$, that is for ReLU and Leaky ReLU, multiple neurons in each layer had non-zero incoming and outgoing weights; however, the rank of all the weight matrices were typically one, and the resulting network output could be expressed using one neuron per layer. A similar observation for ReLU networks in the early stages was also made in \citet{bantzis_early}.  

Note that, if the incoming and outgoing weights of a hidden neuron are zero, then the output of that neuron is zero. Such a neuron can be considered inactive, as it does not contribute to the overall output of the network. Hence, in the early stages of training, certain hidden neurons become approximately inactive because the weights converge towards a KKT point of the constrained NCF, and remaining neurons can be considered active.
\subsection{Gradient Flow Dynamics Beyond the Origin} 
\label{sec:bey_origin}
We now discuss the results of \citet{kumar_escape}, which studies the gradient flow dynamics of homogeneous neural networks after escaping the origin and characterizes the first saddle point encountered by gradient flow after escaping the origin.
\begin{lemma}\label{main_lemma_escape}
	Suppose $\mathcal{H}$ is an $L$-homogeneous neural network, for some $L\geq 2$, and $\rvw_0\in\mathcal{S}(\rvw_*;\mathcal{N}_{\rvy, \mathcal{H}})$, where $\rvw_*$ is a second-order positive KKT point of  $\widetilde{\mathcal{N}}_{\rvy, \mathcal{H}}$. Let $\widetilde{T}\in [0,\infty)$ be arbitrarily large, then for all sufficiently small $\delta>0$,
	\begin{equation}
	\left\|\bm{\psi}\left(t+{T}_\delta,\delta\rvw_0\right) - \rvp(t)\right\|_2 = O(\delta^\beta), \text{ for all } t\in [-\widetilde{T},\widetilde{T}],
	\label{after_esc} 
	\end{equation}
	where $\beta>0$, ${T}_\delta$ is some function of $\delta$, and $\rvp(t)$ is defined as:
	\begin{align*}
	&\rvp(t) \coloneqq\lim_{\delta\rightarrow 0} \bm{\psi}\left(t+\frac{\ln\left({1}/{\delta}\right)}{2\mathcal{N}_{\rvy, \mathcal{H}}(\rvw_*)},\delta\rvw_*\right), \text{ if }L=2, \text{ and }\\
	&\rvp(t) \coloneqq\lim_{\delta\rightarrow 0} \bm{\psi}\left(t+\frac{{1}/{\delta^{L-2}}}{L(L-2)\mathcal{N}_{\rvy, \mathcal{H}}(\rvw_*)},\delta\rvw_*\right), \text{ if }L>2. 
	\end{align*}	
	Also, let $\epsilon>0$ be arbitrarily small, and let $\rvp^* = \lim_{t\rightarrow\infty} \rvp(t)$, where $\rvp^*$ is a saddle point of the training loss. Then, there exists a $T_\epsilon$ such that for all  for all sufficiently small $\delta>0$,
	\begin{equation}
	\|\bm{\psi}\left(T_\epsilon+{T}_\delta,\delta\rvw_0\right) - \rvp^*\|_2\leq \epsilon.
	\label{near_saddle}
	\end{equation}
\end{lemma}
To interpret this result, recall that $\bm{\psi}\left(t,\delta\rvw_*\right)$ denotes the gradient flow trajectory initialized at $\delta\rvw_*$. For small $\delta$, the trajectory remains near the origin initially. Roughly speaking, for $L=2$, $\bm{\psi}\left(t,\delta\rvw_*\right)$ escapes from the origin after $\frac{\ln\left({1}/{\delta}\right)}{2\mathcal{N}_{\rvy, \mathcal{H}}(\rvw_*)}$ time has elapsed; for $L > 2$, this time scales as $\frac{{1}/{\delta^{L-2}}}{L(L-2)\mathcal{N}_{\rvy, \mathcal{H}}(\rvw_*)}$. Therefore, $\rvp(t)$ is the limiting path taken by $\bm{\psi}\left(t,\delta\rvw_*\right)$ after escaping from the origin, as scale of initialization approaches zero.

According to \cref{after_esc}, if $\rvw_0$ lies in the stable set of $\rvw_*$, then the gradient flow trajectory $\bm{\psi}\left(t+T_\delta,\delta\rvw_0\right)$ remains close to $\rvp(t)$ for all $t\in [-\widetilde{T},\widetilde{T}]$ and for all sufficiently small $\delta$. That is, the trajectory of $\bm{\psi}\left(t+T_\delta,\delta\rvw_0\right)$ after escaping the origin is close to the limiting trajectory $\rvp(t)$ for an arbitrarily long time and for sufficiently small initialization. Therefore, using the definition of  $\rvp(t)$, the gradient flow trajectory initialized at $\delta\rvw_0$ escapes from the origin along the same path as if initialized at $\delta\rvw_*$, for all small $\delta$. Thus, the escape dynamics are determined by the KKT point $\rvw_*$, regardless of the specific choice of $\rvw_0$ in its stable set.

Eventually, $\rvp(t)$ converges to a saddle point $\rvp^*$ of the training loss, and according to \cref{near_saddle}, the gradient flow $\bm{\psi}\left(t,\delta\rvw_0\right)$ also gets arbitrarily close to this saddle point at some time. However, the lemma does not assert that gradient flow converges to $\rvp^*$; it may eventually escape from $\rvp^*$. The work of \citet{kumar_escape} does not address the behavior of gradient flow near or beyond this saddle point.  

For homogeneous feed-forward neural networks, the next lemma describes the sparsity structure present in the saddle point reached by gradient flow after escaping the origin. For brevity, we use $\rmW_{1:L}$ to denote the weight matrices of an $L$-layer network, that is, $\rmW_{1:L}\coloneqq(\rmW_1,\cdots,\rmW_L)$.

\begin{lemma}
	Suppose $\mathcal{H}$ is an $L-$layer feed-forward neural network as defined in \cref{L_layer_feed}, where $L\geq 2$, $\sigma(x) = \max(x,\alpha x)^p$, for $p=1, \alpha = 1$ or $p\in\sN, p\geq 2$ and $\alpha\in \sR$. Let $\rmW_{1:L}^0 \in  \mathcal{S}(\overline{\rmW}_{1:L};\mathcal{N}_{\rvy,\mathcal{H}}) $, where $\overline{\rmW}_{1:L}$ is a second-order positive KKT point of $\widetilde{\mathcal{N}}_{\rvy,\mathcal{H}}$. Define a subset of the weights $\rvw_z$ as:
	\begin{equation*}
	\rvw_z \coloneqq \bigcup\limits_{l=1}^{L-1}\bigg(\left\{\rmW_l[j,:]: \left\|\overline{\rmW}_l[j,:]\right\|_2 = 0, j\in [k_l]\right\} \cup \left\{\rmW_{l+1}[:,j]: \left\|\overline{\rmW}_{l+1}[:,j]\right\|_2 = 0, j\in [k_l]\right\}\bigg).
	\end{equation*} 
	Let $\widetilde{T}\in [0,\infty)$ be arbitrarily large, then for all sufficiently small $\delta>0$,
	\begin{equation}
	\left\|\bm{\psi}\left(t+{T}_\delta,\delta\rmW_{1:L}^0\right) - \rvp(t)\right\|_2 = O(\delta^\beta), \text{ for all } t\in [-\widetilde{T},\widetilde{T}],
	\label{after_esc_ff} 
	\end{equation}
	where $\beta>0$, ${T}_\delta$ is some function of $\delta$, and $\rvp(t)$ is defined in the same way as in Lemma \ref{main_lemma_escape}. Furthermore, let $\epsilon>0$ be arbitrarily small. Then, there exists a $T_\epsilon$ such that for all sufficiently small $\delta>0$,
	\begin{equation}
	\|\bm{\psi}\left(T_\epsilon+{T}_\delta,\delta\rmW_{1:L}^0\right) - \rvp^*\|_2\leq \epsilon,
	\label{near_saddle_ff}
	\end{equation}
	where $\rvp^*$ is a saddle point of the training loss function and is defined in the way as in Lemma \ref{main_lemma_escape}. Finally, 
	\begin{itemize}
		\item $\|\rvp_{\rvw_z}(t)\|_2 = 0, \text{ for all } t\in (-\infty,\infty), \text{ and } $
		\item $\left\|\bm{\psi}_{\rvw_z}\left(t+{T}_\delta,\delta\rmW_{1:L}^0\right)\right\|_2 = O(\delta^\beta), \text{ for all } t\in [-\widetilde{T},\widetilde{T}],$	
		\label{wz_saddle}
	\end{itemize}
	which implies $\|\rvp^*_{\rvw_z}\|_2 = 0 $ and $\|\bm{\psi}_{\rvw_z}\left(T_\epsilon+T_\delta,\delta\rmW_{1:L}^0\right) \|_2\leq \epsilon,$ where $(\cdot)_{\rvw_z}$ denotes the sub-vector corresponding to weights in $\rvw_z$.
	\label{lemma_small_wt}
\end{lemma}

In this lemma, $\overline{\rmW}_{1:L}$ is a positive KKT point of the constrained NCF defined with respect to $\rvy$ and $\mathcal{H}$. The set $\rvw_z$ contains all rows and columns of the weight matrices where the corresponding rows and columns of  $\overline{\rmW}_{1:L}$ have zero norm. As shown in Lemma \ref{bal_r1_kkt}, this construction ensures that if  $\rmW_{l+1}[:,j] \in \rvw_z$, then $\rmW_{l}[j,:]\in\rvw_z$, and vice-versa. Thus, the set $\rvw_z$ will contain the incoming and outgoing weights of certain subset of hidden neurons. 

If $\rmW_{1:L}^0$ lies in the stable set of $\overline{\rmW}_{1:L}$, then\textemdash as in the previous lemma\textemdash the gradient flow trajectory stays close to the limiting path $\rvp(t)$ after escaping the origin, eventually getting close to the saddle point $\rvp^*$, as stated in \cref{near_saddle_ff}. More importantly, $\|\rvp_{\rvw_z}(t)\|_2 = 0$,  implying that along the trajectory of $\rvp(t)$, the weights belonging to $\rvw_z$ have zero norm. Consequently, $\left\|\bm{\psi}_{\rvw_z}\left(t+{T}_\delta,\delta\rmW_{1:L}^0\right)\right\|_2$  is small, $\|\rvp^*_{\rvw_z}\|_2 = 0$ and $\|\bm{\psi}_{\rvw_z}\left(T_\epsilon+T_\delta,\delta\rmW_{1:L}^0\right) \|_2$ is small. Thus, the weights belonging to $\rvw_z$ remain small throughout the time interval during which the gradient flow $\bm{\psi}\left(t,\delta\rmW_{1:L}^0\right)$ escapes from the origin and gets close to the saddle point $\rvp^*$.  

To summarize, for homogeneous feed-forward neural networks, Lemma \ref{lemma_early_dc} shows that during the early stages of training, weights remain small in norm but converge in direction towards $\overline{\rmW}_{1:L}$. Since $\overline{\rmW}_{1:L}$ exhibits a sparsity structure\textemdash specifically, the weights belonging to $\rvw_z$  are zero\textemdash the same sparsity structure would also emerge among the weights during the early stages of training. The above lemma further says that this sparsity structure is preserved even after gradient flow escapes from the origin and until it reaches a saddle point, since weights belonging to $\rvw_z$ remain small during this time interval.

An alternative perspective on this behavior is through the lens of hidden neurons. In the early stages of training, directional convergence toward $\overline{\rmW}_{1:L}$ causes a subset of hidden neurons to become inactive, as their incoming and outgoing weights become small, and the remaining neurons can be considered active. Moreover, these inactive neurons remain inactive even after gradient flow escapes from the origin and reaches the next saddle point. As a result, until reaching the first saddle point, the training loss is essentially minimized using only the active neurons, through gradient flow initialized with weights that are small in norm and aligned in direction with a KKT point of the constrained NCF.

Another key takeaway is that gradient flow reaches a saddle point where the incoming and outgoing weights of a certain subset of hidden neurons are zero. This structural property of the saddle point will be crucial in our analysis of gradient flow dynamics near saddle points in homogeneous neural networks.

\section{Gradient Flow Dynamics Near Saddle Points}
\label{sec:near_saddle}
In this section, we study the gradient flow dynamics near saddle points of the training loss for homogeneous neural networks. Our focus is on a specific class of saddle points where a subset of the weights is zero. More precisely, we assume that the weights of the neural network can be divided into two sets, $\rvw= (\rvw_n,\rvw_z)$, such that $(\overline{\rvw}_n,\mathbf{0})$ is the saddle point of the training loss. Under this setup, the training loss can be expressed as
\begin{align}
\mathcal{L}(\rvw_n,\rvw_z) = \frac{1}{2}\|\mathcal{H}(\rmX;\rvw_n,\rvw_z)- \rvy\|_2^2 .
\label{loss_fn_sep}
\end{align}
Minimizing the above optimization problem using gradient flow with initialization near the saddle point $(\overline{\rvw}_n,\mathbf{0})$ gives us the following differential equation:
\begin{align}
&\dot{\rvw}_n = -\mathcal{J}_n(\rmX;\rvw_n,\rvw_z)^\top(\mathcal{H}(\rmX;\rvw_n,\rvw_z)- \rvy), \rvw_n(0) = \overline{\rvw}_n + \delta\rvn \label{gf_sep_wn}\\
&\dot{\rvw}_z = -\mathcal{J}_z(\rmX;\rvw_n,\rvw_z)^\top(\mathcal{H}(\rmX;\rvw_n, \rvw_z)- \rvy), \rvw_z(0) =  \delta\rvz, \label{gf_sep_wz}
\end{align}
where $\delta>0$ is a scalar that controls how close the initialization is to the saddle point, and  $\|\rvn\|_2 = \|\rvz\|_2 = 1.$ Here, $\mathcal{J}_n(\rmX;\rvw_n,\rvw_z)$ and $\mathcal{J}_z(\rmX;\rvw_n,\rvw_z)$ denote the Jacobian of $\mathcal{H}(\rmX;\rvw_n,\rvw_z)$ with respect to $\rvw_n$ and $\rvw_z$, respectively. 

Next,  we will make certain assumptions on output of the neural network. These assumptions are motivated from previous works which describe the saddle points encountered by gradient flow during training. To provide intuition for these assumptions and to outline our main results, we begin with an informal analysis of gradient flow dynamics of a homogeneous feed-forward neural networks near certain saddle points.
\subsection{An Informal Analysis }
\label{inf_analysis}
Let $\mathcal{H}$ be a three-layer neural network  with activation function $\sigma(x) = x^2$ and its output is
\begin{equation*}
\mathcal{H}(\rvx;\rmW_1,\rmW_2,\rmW_3) = \rmW_3\sigma(\rmW_2\sigma(\rmW_1\rvx)),
\end{equation*} 
where $\rvx\in \sR^d $, $\rmW_1\in \sR^{k_1\times d}$,  $\rmW_2\in \sR^{k_2\times k_1}$ and   $\rmW_3\in \sR^{1\times k_2}$. Thus, $\mathcal{H}$ has two hidden layers, with first and second layer containing $k_1$ and $k_2$ neurons, respectively. Next, we are going to divide the weights in the following way: 
\begin{equation*}
\rmW_1 = \begin{bmatrix}
{{\rmN}_{1}}\\\rmA_1
\end{bmatrix}, 
\rmW_2 = \begin{bmatrix}
{{\rmN}_{2}}&\rmB_2  \\
\rmA_2& \rmC_2 \\
\end{bmatrix}, \rmW_3 = \begin{bmatrix}
{{\rmN}_{3}}&\rmB_3
\end{bmatrix},
\end{equation*}
where $\rmN_1 \in \sR^{p_1\times d}$, $\rmA_1 \in \sR^{(k_1 - p_1)\times d}$,  $\rmN_2 \in \sR^{p_2\times p_1}$, $\rmC_2 \in \sR^{(k_2 - p_2)\times (k_1-p_1)}$,  $\rmN_3 \in \sR^{1 \times p_2}$, $\rmB_3 \in \sR^{1 \times (k_2 - p_2)}$, and $\rmB_2$ and $\rmA_2$ are defined in a consistent way. In the above division, the matrix $\rmA_1$ ($\rmB_2,\rmC_2$) contains all the incoming (outgoing) weights of the last $k_1-p_1$ neurons of the first hidden layer. Similarly, $\rmB_3$ ($\rmA_2$, $\rmC_2$) contains all the outgoing (incoming) weights of the last $k_2-p_2$ neurons of the second hidden layer. Let $(\overline{\rmW}_1,\overline{\rmW}_2,\overline{\rmW}_3)$ be a saddle point of the training loss such that all the incoming and outgoing weights of the last $k_1-p_1$ neurons of the first layer and last $k_2-p_2$ neurons of the second layer are zero. Therefore,
\begin{equation*}
\overline{\rmW}_1 = \begin{bmatrix}
{\overline{\rmN}_{1}}\\\mathbf{0} 
\end{bmatrix}, 
\overline{\rmW}_2 = \begin{bmatrix}
{\overline{\rmN}_{2}}&\mathbf{0}  \\
\mathbf{0} & \mathbf{0}  \\
\end{bmatrix}, \overline{\rmW}_3 = \begin{bmatrix}
{\overline{\rmN}_{3}}&\mathbf{0} 
\end{bmatrix}.
\end{equation*}
As discussed in \Cref{sec:bey_origin}, the choice of this saddle point is motivated from the result of \citet{kumar_escape}, which showed that after escaping from the origin, gradient flow reaches a saddle point where the incoming and outgoing weights of a subset of hidden neurons have zero norm.
Our goal here is to (informally) analyze the gradient flow dynamics of the training loss when initialized near the above saddle point. 

Let $\rvw_z$ be the subset of the weights containing all the incoming and outgoing weights of the last $k_1-p_1$ neurons of the first layer and last $k_2-p_2$ neurons of the second layer, and $\rvw_n$ contains the remaining weights. With this notation, the full set of network weights can be written as $(\rvw_n, \rvw_z)$, and the saddle point $(\overline{\rmW}_1, \overline{\rmW}_2, \overline{\rmW}_3)$ corresponds to the point $(\overline{\rvw}_n, \mathbf{0})$, for some $\overline{\rvw}_n$.

Suppose  $\rvw_n$ and $\rvw_z$ are evolving according to \cref{gf_sep_wn} and \cref{gf_sep_wz}, respectively, where $\delta$ is small. Since $(\overline{\rvw}_n,\mathbf{0})$ is a saddle point, the weights will remain near the saddle point for some time after training begins. Thus, in the initial stages of training, we can assume $\rvw_n \approx \overline{\rvw}_n$ and $\rvw_z \approx \mathbf{0}$. Define $\overline{\rvy}\coloneqq  \rvy - \mathcal{H}(\rmX;\overline{\rvw}_n, \mathbf{0})$, then the dynamics of $\rvw_z$ can approximately be written as
\begin{align}
\dot{\rvw}_z \approx -\mathcal{J}_z(\rmX;\overline{\rvw}_n,\rvw_z)^\top(\mathcal{H}(\rmX;\overline{\rvw}_n, \rvw_z)- \rvy) &\approx -\mathcal{J}_z(\rmX;\overline{\rvw}_n,\rvw_z)^\top(\mathcal{H}(\rmX;\overline{\rvw}_n, \mathbf{0})- \rvy)\nonumber\\
& = \mathcal{J}_z(\rmX;\overline{\rvw}_n,\rvw_z)^\top\overline{\rvy}.
\label{ncf_approx}
\end{align}
To understand the behavior of $\mathcal{J}_z(\rmX;\overline{\rvw}_n,\rvw_z)$, we will simplify $\mathcal{H}(\rmX;\overline{\rvw}_n,\rvw_z)$. Since $\sigma(x) = x^2$, we get 
\begin{align*}
\mathcal{H}(\rvx;\overline{\rvw}_n,\rvw_z) &=  \begin{bmatrix}
{\overline{\rmN}_{3}}&\rmB_3
\end{bmatrix}\sigma\left(\begin{bmatrix}
{\overline{\rmN}_{2}}&\rmB_2  \\
\rmA_2& \rmC_2 \\
\end{bmatrix}\begin{bmatrix}
({\overline{\rmN}_{1}}\rvx)^2\\(\rmA_1\rvx)^2
\end{bmatrix}\right)\\
&=  \begin{bmatrix}
{\overline{\rmN}_{3}}&\rmB_3
\end{bmatrix}\sigma\left(\begin{bmatrix}
{\overline{\rmN}_{2}}( {\overline{\rmN}_{1}}\rvx)^2+\rmB_2(\rmA_1\rvx)^2  \\
\rmA_2( {\overline{\rmN}_{1}}\rvx)^2+ \rmC_2(\rmA_1\rvx)^2 \\
\end{bmatrix}\right)\\
&=  \begin{bmatrix}
{\overline{\rmN}_{3}}&\rmB_3
\end{bmatrix}\begin{bmatrix}
( {\overline{\rmN}_{2}}( {\overline{\rmN}_{1}}\rvx)^2)^2+(\rmB_2(\rmA_1\rvx)^2)^2 + 2( {\overline{\rmN}_{2}}( {\overline{\rmN}_{1}}\rvx))\odot(\rmB_2(\rmA_1\rvx)^2) \\
(\rmA_2( {\overline{\rmN}_{1}}\rvx)^2)^2+( \rmC_2(\rmA_1\rvx)^2 )^2 + 2(\rmA_2( {\overline{\rmN}_{1}}\rvx)^2)\odot( \rmC_2(\rmA_1\rvx)^2 )\\
\end{bmatrix}\\
&=  {\overline{\rmN}_{3}}( {\overline{\rmN}_{2}}( {\overline{\rmN}_{1}}\rvx)^2)^2 + 2 {\overline{\rmN}_{3}}\left(( {\overline{\rmN}_{2}}( {\overline{\rmN}_{1}}\rvx))\odot(\rmB_2(\rmA_1\rvx)^2)\right) +\rmB_3(\rmA_2( {\overline{\rmN}_{1}}\rvx)^2)^2\\
& + 2\rmB_3(\rmA_2( {\overline{\rmN}_{1}}\rvx)^2)\odot( \rmC_2(\rmA_1\rvx)^2 ) +  {\overline{\rmN}_{3}}(\rmB_2(\rmA_1\rvx)^2)^2 + \rmB_3( \rmC_2(\rmA_1\rvx)^2 )^2. 
\end{align*}
In the last term, the expression $ {\overline{\rmN}_{3}}( {\overline{\rmN}_{2}}( {\overline{\rmN}_{1}}\rvx)^2)^2$ is independent of $\rvw_z$. The remaining terms, however, do depend on $\rvw_z$ and are homogeneous functions of $\rvw_z$  with different degrees of homogeneity.  Since we have assumed $\rvw_z$ to be small, we can just keep the term with lowest degree of homogeneity and ignore the higher order terms. Therefore,
\begin{equation*}
\mathcal{H}(\rvx;\overline{\rvw}_n,\rvw_z) \approx   {\overline{\rmN}_{3}}( {\overline{\rmN}_{2}}( {\overline{\rmN}_{1}}\rvx)^2)^2 + 2 {\overline{\rmN}_{3}}\left(( {\overline{\rmN}_{2}}( {\overline{\rmN}_{1}}\rvx))\odot(\rmB_2(\rmA_1\rvx)^2)\right) +\rmB_3(\rmA_2( {\overline{\rmN}_{1}}\rvx)^2)^2.
\end{equation*}
\sloppypar \noindent Define $\mathcal{H}_1(\rvx;\overline{\rvw}_n,\rvw_z) \coloneqq 2 {\overline{\rmN}_{3}}\left(( {\overline{\rmN}_{2}}( {\overline{\rmN}_{1}}\rvx))\odot(\rmB_2(\rmA_1\rvx)^2)\right) +\rmB_3(\rmA_2( {\overline{\rmN}_{1}}\rvx)^2)^2$, and let $\mathcal{J}_1(\rmX;\overline{\rvw}_n,\rvw_z) $ denote the Jacobian of $\mathcal{H}_1(\rmX;\overline{\rvw}_n,\rvw_z)$ with respect to $\rvw_z$. Then, $\mathcal{H}_1(\rvx;\overline{\rvw}_n,\rvw_z)$ is $3-$homogeneous in $\rvw_z$. Thus, \cref{ncf_approx} can be written as
\begin{equation}
\dot{\rvw}_z \approx  \mathcal{J}_1(\rmX;\overline{\rvw}_n,\rvw_z) ^\top\overline{\rvy}, \rvw_z(0) = \delta\rvz.
\label{inf_gf_wz}
\end{equation} 
Hence, near the saddle point, the evolution of $\rvw_z$ is governed by the gradient flow that maximizes a homogeneous function. From Lemma \ref{lemma:gf_ncf}, we know that in such a scenario $\rvw_z$ will converge in direction towards a KKT point of an appropriate constrained NCF. Thus, near the saddle point, the weights belonging to $\rvw_z$ remain small in norm but converge in direction. We will follow a similar approach towards formally establishing directional convergence among the weights with small magnitude near the saddle points. 

Although we have assumed $\sigma(x) = x^2$, similar decomposition holds for $\sigma(x) = x^p$, as shown later. For ReLU-type activation functions such as $\sigma(x) = \max(0, x)^p$, an analogous decomposition requires Taylor approximations of $\sigma(x)$.   More specifically, in that case,
\begin{align*}
\mathcal{H}(\rvx;\overline{\rvw}_n,\rvw_z) &=  \begin{bmatrix}
{\overline{\rmN}_{3}}&\rmB_3
\end{bmatrix}\sigma\left(\begin{bmatrix}
{\overline{\rmN}_{2}}&\rmB_2  \\
\rmA_2& \rmC_2 \\
\end{bmatrix}\begin{bmatrix}
\sigma( {\overline{\rmN}_{1}}\rvx)\\\sigma(\rmA_1\rvx)
\end{bmatrix}\right)\\
&=  \begin{bmatrix}
{\overline{\rmN}_{3}}&\rmB_3
\end{bmatrix}\sigma\left(\begin{bmatrix}
{\overline{\rmN}_{2}}\sigma( {\overline{\rmN}_{1}}\rvx) +\rmB_2\sigma(\rmA_1\rvx)  \\
\rmA_2\sigma( {\overline{\rmN}_{1}}\rvx)+ \rmC_2\sigma(\rmA_1\rvx) \\
\end{bmatrix}\right)\\
&=  \begin{bmatrix}
\overline{\rmN}_{3}&\rmB_3
\end{bmatrix}\begin{bmatrix}
\sigma( {\overline{\rmN}_{2}}\sigma( {\overline{\rmN}_{1}}\rvx) +\rmB_2\sigma(\rmA_1\rvx)) \\
\sigma(\rmA_2\sigma( {\overline{\rmN}_{1}}\rvx)+ \rmC_2\sigma(\rmA_1\rvx) )\\
\end{bmatrix}\\
&=  {\overline{\rmN}_{3}}\sigma\left( {\overline{\rmN}_{2}}\sigma( {\overline{\rmN}_{1}}\rvx) +\rmB_2\sigma(\rmA_1\rvx)\right) + \rmB_3\sigma\left(\rmA_2\sigma( {\overline{\rmN}_{1}}\rvx)+ \rmC_2\sigma(\rmA_1\rvx) \right)\\
&\approx   {\overline{\rmN}_{3}}\sigma\left( {\overline{\rmN}_{2}}\sigma( {\overline{\rmN}_{1}}\rvx) \right)+  {\overline{\rmN}_{3}}\sigma'( {\overline{\rmN}_{2}}\sigma( {\overline{\rmN}_{1}}\rvx))\odot(\rmB_2\sigma(\rmA_1\rvx)) + \rmB_3\sigma(\rmA_2\sigma( {\overline{\rmN}_{1}}\rvx))\\
&+ \rmB_3\sigma'(\rmA_2\sigma( {\overline{\rmN}_{1}}\rvx))\odot(\rmC_2\sigma(\rmA_1\rvx) ),
\end{align*}
where the final equality follows from Taylor's approximation of $\sigma(\cdot)$ and $\rvw_z$ being small. As before, $\mathcal{H}(\rvx;\overline{\rvw}_n,\rvw_z) $ decomposes into a term that only depends on $\rvw_n$ and other terms that are homogeneous in $\rvw_z$, with different degrees of homogeneity. It is important to note, however, that applying Taylor approximations requires assuming sufficient smoothness of $\sigma(\cdot)$\textemdash an assumption we will explicitly state in our results.

\subsection{Main Results}
\label{sec:main_results}
In this subsection, we present our main results describing the gradient flow dynamics of homogeneous neural networks near saddle points. We begin by stating an assumption on the output of neural network that plays a central role in our analysis.
\begin{assumption}
	Suppose $\rvw_n$ is fixed and $\|\rvw_z\|_2= O(\delta)$. Then, for all sufficiently small $\delta>0$, we have
	\begin{itemize}[label={}]
		\item $(i)$ $\mathcal{H}(\rvx;\rvw_n,\rvw_z) = \mathcal{H}(\rvx;\rvw_n,\mathbf{0}) + \mathcal{H}_1(\rvx;\rvw_n,\rvw_z) + O(\delta^K),$
		\item $(ii)$ $\nabla_{\rvw_z}\mathcal{H}(\rvx;\rvw_n,\rvw_z) =  \nabla_{\rvw_z}\mathcal{H}_1(\rvx;\rvw_n,\rvw_z) + O(\delta^{K-1}),$
		\item $(iii)$ $\nabla_{\rvw_n}\mathcal{H}(\rvx;\rvw_n,\rvw_z) = \nabla_{\rvw_n}\mathcal{H}(\rvx;\rvw_n,\mathbf{0}) + \nabla_{\rvw_n}\mathcal{H}_1(\rvx;\rvw_n,\rvw_z) + O(\delta^{K}),$
	\end{itemize}
	where $\mathcal{H}_1(\rvx;\rvw_n,\rvw_z)$ is $L$-homogeneous in $\rvw_z$ for some $L\geq2$, $\mathcal{H}_1(\rvx;\rvw_n,\rvw_z)$ has locally Lipschitz gradients in both $\rvw_n$ and $\rvw_z$, and $K>L$.
	\label{ass_str}
\end{assumption}
The first condition of the above assumption states that, when $\|\rvw_z\|_2= O(\delta)$,  the output of the network can be decomposed into a leading term independent of $\rvw_z$, an $L$-homogeneous term in $\rvw_z$, and a higher-order remainder term whose magnitude is of $O(\delta^K)$, where $K>L$. The other two conditions require the gradient of the output with respect to $\rvw_n$ and $\rvw_z$ to also behave consistently with the first. This assumption is inspired by the behavior of feed-forward neural networks near certain saddle points, as discussed in \Cref{inf_analysis}. Later, we will show that this assumption is indeed satisfied by feed-forward neural networks when $\rvw_z$ consists of the incoming and outgoing weights of a subset of hidden neurons, and $\rvw_n$ contains the rest.

We next state a lemma describing how the scale of initialization affects gradient flow trajectories of homogeneous functions. The proof is in \Cref{pf_traj_eq}.
\begin{lemma}
	Suppose $g(\rvs)$ is an $L$-homogeneous function in $\rvs$ for some $L\geq2$. Let $\rvs_0$ be a non-zero vector and $\rvs(t)$ be the solution of the following differential equation$\colon$
	\begin{equation*}
	\dot{\rvs} = \nabla _\rvs g(\rvs), \rvs(0) = \rvs_0.
	\end{equation*}
	For any scalar $\delta > 0$, let $\rvs_{\delta}(t)$ be the solution of the following differential equation$\colon$
	\begin{equation*}
	\dot{\rvs} = \nabla _\rvs g(\rvs), \rvs(0) = \delta\rvs_0.
	\end{equation*}
	Then,
	\begin{equation}
	\rvs(t) = \frac{1}{\delta}\rvs_{\delta}\left(\frac{t}{\delta^{L-2}}\right).
	\end{equation}
	\label{traj_init_eq}
\end{lemma}
For gradient flow of homogeneous functions, this result implies that scaling the initialization only leads to the scaling of magnitude and time of the gradient flow trajectory. Consequently, the limiting direction of the gradient flow will not be affected, however, the convergence time will be scaled by $1/\delta^{L-2}$. This fact is directly relevant to us because near the saddle point, as shown in \cref{inf_gf_wz}, $\rvw_z$ approximately evolves according to gradient flow of a homogeneous function, but its initialization is scaled by $\delta$. Therefore, $\rvw_z$  will require $O(1/\delta^{L-2})$ time to converge in direction. However, we arrived at \cref{inf_gf_wz} under the assumption that $\rvw_n(t) \approx \overline{\rvw}_n$. Therefore, we have to ensure that $\rvw_n(t) \approx \overline{\rvw}_n$ holds true for $O(1/\delta^{L-2})$ amount of time. To ensure this, we make the following assumption on the saddle point $(\overline{\rvw}_n,\mathbf{0})$.
\begin{assumption} 
	\label{loj_ass}
	We assume that $(\overline{\rvw}_n,\mathbf{0})$ is a saddle point of the training loss in \cref{loss_fn_sep} such that $\overline{\rvw}_n$ is a local minimum of 
	\begin{equation*}
	\widetilde{\mathcal{L}}(\rvw_n) \coloneqq {\mathcal{L}}(\rvw_n,\mathbf{0})= \frac{1}{2}\|\mathcal{H}(\rmX;\rvw_n,\mathbf{0}) - \rvy\|_2^2,
	\end{equation*}
	and Lojasiewicz's inequality is satisfied in a neighborhood of $\overline{\rvw}_n$:  there exists $\mu_1,\gamma>0$ and $\alpha \in \left(0, \frac{L}{2(L-1)}\right)$ such that
	\begin{equation*}
	\left\|	\nabla\widetilde{\mathcal{L}}(\rvw_n)\right\|_2\geq \mu_1\left(\widetilde{\mathcal{L}}(\rvw_n) - \widetilde{\mathcal{L}}(\overline{\rvw}_n)\right)^{\alpha}, \text{ if } \|\rvw_n - \overline{\rvw}_n\|_2\leq \gamma.
	\end{equation*}
	\label{ass_loj}
\end{assumption}
Here, $\widetilde{\mathcal{L}}(\rvw_n)$ is defined by fixing $\rvw_z = \mathbf{0}$ in the training loss ${\mathcal{L}}(\rvw_n,\rvw_z)$. Since  $(\overline{\rvw}_n,\mathbf{0})$ is a saddle point of \cref{loss_fn_sep}, it follows that $\overline{\rvw}_n$ is a stationary point  of $\widetilde{\mathcal{L}}(\rvw_n)$. By further assuming it is a local minimum of $\widetilde{\mathcal{L}}(\rvw_n)$, we can ensure that $\rvw_n(t)$ remains close to $\overline{\rvw}_n$ after training begins. However, we require  $\rvw_n(t)$ to remain close to $\overline{\rvw}_n $ for $O(1/\delta^{L-2})$ amount of time. This is where the Lojasiewicz's inequality plays a key role. We will use the path length bounds obtained via Lojasiewicz's inequality to ensure $\rvw_n(t)$ remains close to $\overline{\rvw}_n$ for the required duration.

Regarding the validity of the above assumption, Lojasiewicz's inequality is known to hold near local minima of real-analytic and subanalytic functions for some $\alpha\in (0,1)$ \citep{loj_ineq,bolte_loj}. This includes feed-forward neural networks with activation functions such as $x^p$ and $\max(x,0)^p$. But, we require $\alpha \in \left(0,\frac{L}{2(L-1)}\right)$, which is a stricter condition for all $L>2$. We provide a simple instance in \Cref{loj_ineq_ex}  where the local minimum of a feed-forward homogeneous neural network satisfies Lojasiewicz's inequality with $\alpha = \frac{1}{2}$, validating \Cref{ass_loj} in those cases. Establishing the validity of \Cref{ass_loj} more generally, or even relaxing it, is an important direction for future work. 

It is also worth noting that $\alpha = \frac{1}{2}$, which corresponds to the Polyak-Lojasiewicz (PL) inequality \citep{polyak_pl, karimi_pl}, is always contained in $\left(0,\frac{L}{2(L-1)}\right)$. The PL inequality is widely studied and holds near \emph{global minima} of many optimization problems, including those involving feed-forward networks \citep{liu_pl, chatterjee_pl}. Perhaps these results could be useful in establishing the validity of \Cref{ass_loj}.

We now present the main result of this subsection.
\begin{theorem}
	Suppose  \Cref{ass_str} is satisfied, and $(\overline{\rvw}_n,\mathbf{0})$ is a saddle point of the training loss in \cref{loss_fn_sep} such that \Cref{ass_loj} is satisfied. Let $(\rvw_n(t),\rvw_z(t))$ evolve according to \cref{gf_sep_wn} and \cref{gf_sep_wz}, respectively. Define $\overline{\rvy} \coloneqq \rvy - \mathcal{H}(\rmX;\overline{\rvw}_n,\mathbf{0})$ and $\overline{\mathcal{H}}_1(\rvx;\rvw_z) \coloneqq {\mathcal{H}}_1(\rvx;\overline{\rvw}_n,\rvw_z)$. Then, for any arbitrarily small $\epsilon>0$, there exists $T_\epsilon$ such that for all sufficiently small $\delta>0$ the following holds:
	\begin{equation*}
	\|\rvw_n(t)-\overline{\rvw}_n\|_2 = O(\delta^{\beta_1}) \text{ and }	\|\rvw_z(t)\|_2 = O(\delta),  \text{ for all } t\in \left[0,\frac{T_\epsilon}{\delta^{L-2}}\right], 
	\end{equation*}
	where $\beta_1>0$. Moreover, if $\rvz \in \mathcal{S}(\rvz_*;\mathcal{N}_{\overline{\rvy},\overline{\mathcal{H}}_1})$, where $\rvz_*$ is a non-negative KKT point of $\widetilde{\mathcal{N}}_{\overline{\rvy},\overline{\mathcal{H}}_1}$, then 
	\begin{equation*}
	\|\rvw_z(T_\epsilon/\delta^{L-2})\|_2\geq \delta\eta_1 \text{ and }\frac{\rvz_*^\top\rvw_z(T_\epsilon/\delta^{L-2})}{\|\rvw_z(T_\epsilon/\delta^{L-2})\|_2} = 1-O(\epsilon),
	\end{equation*}
	else, $\|\rvw_z(T_\epsilon/\delta^{L-2})\|_2 = \epsilon\cdot O(\delta)$. Here, $\eta_1$ is a positive constant independent of $\epsilon$ and $\delta$.
	\label{thm_dir_convg}
\end{theorem}
In the above theorem, $\overline{\rvy}$ denotes the residual error at the saddle point $(\overline{\rvw}_n,\mathbf{0})$, and $\mathcal{N}_{\overline{\rvy},\overline{\mathcal{H}}_1}$ is the NCF defined with respect to $\overline{\rvy}$ and $\overline{\mathcal{H}}_1(\rvx;\rvw_z)$.  The theorem states that, under \Cref{ass_str} and \Cref{ass_loj}, the gradient flow initialized near the saddle point $(\overline{\rvw}_n,\mathbf{0})$ evolves such that, during the initial stages of training, $\rvw_n$ remains close to $\overline{\rvw}_n$ and $\rvw_z$ remains small. Moreover, if the initial direction of $\rvw_z$, denoted by $\rvz$, lies in a stable set of a non-negative KKT point of $\widetilde{\mathcal{N}}_{\overline{\rvy},\overline{\mathcal{H}}_1}$, then $\rvw_z$ approximately converge in direction towards that KKT point. Note that, $\widetilde{\mathcal{N}}_{\overline{\rvy},\overline{\mathcal{H}}_1}$ has the following form:
\begin{equation}
\widetilde{\mathcal{N}}_{\overline{\rvy},\overline{\mathcal{H}}_1} \coloneqq \max_{\|\rvw_z\|_2^2 = 1} \overline{\rvy}^\top\overline{\mathcal{H}}_1(\rmX;\rvw_z) =  \max_{\|\rvw_z\|_2^2 = 1} \overline{\rvy}^\top\mathcal{H}_1(\rmX;\overline{\rvw}_n,\rvw_z).
\label{const_ncf_main}
\end{equation}
If $\rvz$ does not lie in a stable set of a non-negative KKT point, then $\rvw_z$ approximately becomes zero, as $\|\rvw_z(T_\epsilon/\delta^{L-2})\|_2 = \epsilon\cdot O(\delta)$, where $\epsilon$ and $\delta$ are both small. In contrast, in the previous case, $\|\rvw_z(T_\epsilon/\delta^{L-2})\|_2 \geq \delta\eta_1$, where $\eta_1$ is a positive constant. This essentially happens because the gradient flow of $\mathcal{N}_{\overline{\rvy},\overline{\mathcal{H}}_1}$ converges to the origin in this case.   

Our proof technique is similar to the discussion in \Cref{inf_analysis}. We show that, in the initial stages of training, the dynamics of $\rvw_z$ stated in \cref{gf_sep_wz} are close to the gradient flow dynamics of the NCF $\mathcal{N}_{\overline{\rvy},\overline{\mathcal{H}}_1}$ with initialization $\delta\rvz$. According to Lemma \ref{lemma:gf_ncf}, the latter dynamics would either converge in direction to a KKT point of  $\widetilde{\mathcal{N}}_{\overline{\rvy},\overline{\mathcal{H}}_1}$ or converge to the origin. This implies $\rvw_z$ also either converges in direction to the KKT point or goes towards the origin. The detailed proof is in \Cref{pf_main_thm}. 
\begin{remark}
	In \Cref{thm_dir_convg}, the KKT point to which $\rvw_z$ converges in direction depends on $\rvz$, the initial direction of $\rvw_z$. For homogeneous feed-forward neural network, as stated in Lemma \ref{lemma_small_wt}, after escaping from the origin, the gradient flow reaches a saddle point where a certain subset of weights are small. In \Cref{thm_dir_convg}, $\rvw_z$ corresponds to this subset of small weights. However, while Lemma \ref{lemma_small_wt} characterizes the magnitude of these weights, it does not provide any insight into their \emph{direction}. On the other hand, \Cref{thm_dir_convg} suggests that information about the direction of these small weights is also crucial for a better understanding of the dynamics of gradient flow near such saddle points. Determining this directions appears to be difficult and is an interesting direction for future works.
	\label{rem_dir_saddle}
\end{remark}
The above theorem crucially relies on \Cref{ass_str}. We now show that \Cref{ass_str} is satisfied for feed-forward neural networks with homogeneous activation functions, under the condition that $\rvw_z$ contains all the incoming and outgoing weights of a subset of hidden neurons, while $\rvw_n$ contains the remaining weights.
\begin{lemma}
	\label{h1_proof_poly}
	Let $\mathcal{H}$ be an $L$-layer feed-forward neural network as described in \cref{L_layer_feed}, for some $L\geq2$, with activation function $\sigma(x) = \max(x,\alpha x)^p$, where $p\in \sN$. Let $\mathcal{G}_z$ denote the subset of hidden neurons containing last $k_l-p_l$ neurons of the $l$th hidden layer, for all $l\in [L-1]$. Let $\rvw_z$ be the subset of the weights containing all outgoing and incoming weights of hidden neurons in $\mathcal{G}_z$, and $\rvw_n$ contains the remaining weights. More specifically, let 
	\begin{equation*}
	\rmW_1 = \begin{bmatrix}
	{{\rmN}_{1}}\\\rmA_1
	\end{bmatrix}, 
	\rmW_l = \begin{bmatrix}
	{{\rmN}_{l}}&\rmB_l \\
	\rmA_l& \rmC_l \\
	\end{bmatrix} \text{ for } 2\leq l\leq L-1, \text{ and } \rmW_L = \begin{bmatrix}
	{{\rmN}_{L}}&\rmB_L
	\end{bmatrix},
	\end{equation*}
	where $\rmN_1 \in \sR^{p_1\times d}$, $\rmA_1 \in \sR^{(k_1 - p_1)\times d}$,  $\rmN_l \in \sR^{p_l\times p_{l-1}}$, $\rmC_2 \in \sR^{(k_l - p_l)\times (k_{l-1}-p_{l-1})}$,  $\rmN_L \in \sR^{1 \times p_{L-1}}$, $\rmB_L \in \sR^{1 \times (k_{L-1} - p_{L-1})}$ , and $\rmB_l$ and $\rmA_l$ are defined in a consistent way. Then $\rvw_n$ will contain all the weights belonging to $ \{\rmN_l\}_{l=1}^L$, and $\rvw_z$ contains the remaining weights. 
	\begin{enumerate}[label=(\roman*)]
		\item Suppose $\alpha\neq 1$ and $p\geq 4$. Let $\rvw_n$ be fixed and $\|\rvw_z\|_2 = O(\delta)$, then for all sufficiently small $\delta>0$, we have
		\begin{enumerate}
			\item $\mathcal{H}(\rvx;\rvw_n,\rvw_z) = \mathcal{H}(\rvx;\rvw_n,\mathbf{0}) + \mathcal{H}_1(\rvx;\rvw_n,\rvw_z) + O(\delta^K), $
			\item $\nabla_{\rvw_z}\mathcal{H}(\rvx;\rvw_n,\rvw_z) =  \nabla_{\rvw_z}\mathcal{H}_1(\rvx;\rvw_n,\rvw_z) + O(\delta^{K-1})$,
			\item 	$\nabla_{\rvw_n}\mathcal{H}(\rvx;\rvw_n,\rvw_z) = \nabla_{\rvw_n}\mathcal{H}(\rvx;\rvw_n,\mathbf{0}) + \nabla_{\rvw_n}\mathcal{H}_1(\rvx;\rvw_n,\rvw_z) + O(\delta^{K})$,
		\end{enumerate}
		where $\mathcal{H}_1(\rvx;\rvw_n,\rvw_z)$ is $(p+1)$-homogeneous in $\rvw_z$, and $K>p+1$.
		\item Suppose $\alpha= 1$ and $p\geq 1$. For any $\rvw_n$ and $\rvw_z$, we have
		\begin{equation}
		\mathcal{H}(\rvx;\rvw_n,\rvw_z) = \mathcal{H}(\rvx;\rvw_n,\mathbf{0}) + \sum_{i=1}^m\mathcal{H}_i(\rvx;\rvw_n,\rvw_z), 
		\label{decomp_poly}
		\end{equation} 
		for some $m\geq 1$, where $\mathcal{H}_1(\rvx;\rvw_n,\rvw_z)$ is $(p+1)-$homogeneous in $\rvw_z$. For all $i\geq 2$, $\mathcal{H}_i(\rvx;\rvw_n,\rvw_z)$ is also homogeneous in $\rvw_z$ with degree of homogeneity strictly greater than $p+1$. Also, for all $i\geq 1$, $\mathcal{H}_i(\rvx;\rvw_n,\rvw_z)$ is a polynomial in $\rvw_z$ and $\rvw_n$.
	\end{enumerate}	
	Finally, in both of the above cases, $\mathcal{H}_1(\rvx;\rvw_n,\rvw_z)$ can be expressed as:
	\begin{equation}
	\mathcal{H}_1(\rvx;\rvw_n,\rvw_z)  =  \sum_{l=1}^{L-2} \nabla_\rvs f_{L,l+2}^\top(\rmN_{l+1}\rvg_{l}(\rvx))\rmB_{l+1}\sigma(\rmA_{l}\rvg_{l-1}(\rvx)) +  \rmB_L\sigma(\rmA_{L-1}g_{L-2}(\rvx)),
	\label{h1_exp_poly}
	\end{equation}
	where, for $1\leq l\leq L-1$,
	\begin{equation*}
	g_0(\rvx) = \rvx, g_l(\rvx) = \sigma(\rmN_lg_{l-1}(\rvx)), \text{ and } f_{L,l+1}(\rvs) =  {{\rmN}_{L}}\sigma\left(\rmN_{L-1}\sigma\left(\cdots\sigma(\rmN_{l+1}\sigma(\rvs))\right)\right).
	\end{equation*}
	\label{lemma_poly_str}
\end{lemma}
In the above lemma, $\mathcal{G}_z$ contains the last $k_l-p_l$ neurons of each hidden layer, and the incoming and outgoing weights of these neurons belong to $\rvw_z$, while the remaining weights belong to $\rvw_n$. In the first case, where $\alpha\neq1$ and $p\geq 4$, the output of the neural network satisfies  the three conditions of \Cref{ass_str}, when $\rvw_z$ is small. We had to assume $p\geq 4$ to ensure the Taylor's approximation of $\mathcal{H}(\rvx;\rvw_n,\rvw_z)$ and its gradient are sufficiently smooth.

In the second case, where $\alpha = 1$ and $p\geq 1$, the output of the neural network can be decomposed into a leading term that is independent of $\rvw_z$, along with additional terms that are homogeneous in $\rvw_z$ with different degrees of homogeneity. The term with lowest degree of homogeneity is denoted by $\mathcal{H}_1(\rvx;\rvw_n,\rvw_z)$. Using the decomposition in \cref{decomp_poly}, we now verify that all three conditions of \Cref{ass_str} are satisfied. Suppose $\|\rvw_z\|_2 = O(\delta)$. Since $\mathcal{H}_i(\rvx;\rvw_n,\rvw_z)$ is homogeneous in $\rvw_z$ with degree of homogeneity strictly greater than $p+1$, for all $i\geq 2$, it follows that
\begin{align*}
\mathcal{H}(\rvx;\rvw_n,\rvw_z) &= \mathcal{H}(\rvx;\rvw_n,\mathbf{0}) + \mathcal{H}_1(\rvx;\rvw_n,\rvw_z) + \sum_{i=2}^m\mathcal{H}_i(\rvx;\rvw_n,\rvw_z)\\
&  = \mathcal{H}(\rvx;\rvw_n,\mathbf{0}) + \mathcal{H}_1(\rvx;\rvw_n,\rvw_z) + O(\delta^K),
\end{align*}
for some $K>p+1$. Next,  from Lemma \ref{euler_thm}, $\nabla_{\rvw_z}\mathcal{H}_i(\rvx;\rvw_n,\rvw_z)$ is homogeneous in $\rvw_z$ with degree of homogeneity strictly greater than $p$, for all $i\geq 2$. Hence,
\begin{align*}
\nabla_{\rvw_z}\mathcal{H}(\rvx;\rvw_n,\rvw_z) &= \nabla_{\rvw_z}\mathcal{H}_1(\rvx;\rvw_n,\rvw_z) + \sum_{i=2}^m\nabla_{\rvw_z}\mathcal{H}_i(\rvx;\rvw_n,\rvw_z)\\
&  =\nabla_{\rvw_z}\mathcal{H}_1(\rvx;\rvw_n,\rvw_z) + O(\delta^{K-1}).
\end{align*}
Finally,  from Lemma \ref{subset_homog}, $\nabla_{\rvw_n}\mathcal{H}_i(\rvx;\rvw_n,\rvw_z) $ will be homogeneous in $\rvw_z$ with same degree of homogeneity as  $\mathcal{H}_i(\rvx;\rvw_n,\rvw_z) $. Therefore,  $\|\nabla_{\rvw_n}\mathcal{H}_i(\rvx;\rvw_n,\rvw_z)\|_2 = O(\delta^K)$, for $K>p+1$ and $i\geq 2$. Hence,
\begin{align*}
\nabla_{\rvw_n}\mathcal{H}(\rvx;\rvw_n,\rvw_z)  = \nabla_{\rvw_n}\mathcal{H}(\rvx;\rvw_n,\mathbf{0}) + \nabla_{\rvw_n}\mathcal{H}_1(\rvx;\rvw_n,\rvw_z) + O(\delta^K).
\end{align*}
Thus, all three conditions of \Cref{ass_str} are satisfied. 

The expression of $\mathcal{H}_1(\rvx;\rvw_n,\rvw_z)$ in both the cases is stated in \cref{h1_exp_poly}. In the first case, since $p\geq 4$, $\mathcal{H}_1(\rvx;\rvw_n,\rvw_z)$  has locally Lipschitz gradients with respect to both $\rvw_n$ and $\rvw_z$. In the second case, $\mathcal{H}_1(\rvx;\rvw_n,\rvw_z)$ is a polynomial with respect to $\rvw_n$ and $\rvw_z$, thus, it has locally Lipschitz gradients with respect to both $\rvw_n$ and $\rvw_z$. The proof of the lemma is in \Cref{pf_decomp_poly}.
\begin{remark}
	In 	Lemma \ref{lemma_poly_str}, the set $\rvw_z$ contains the incoming and outgoing weights of the last few neurons in each layer. This choice is made purely for notational simplicity and is without loss of generality. In feed-forward neural networks, neurons within a layer can be arbitrarily reordered without affecting the network's output. Therefore, any subset of neurons could be grouped into $\rvw_z$ via an appropriate re-indexing.
\end{remark}
\textbf{What happens if ${\mathcal{N}}_{\overline{\rvy},\overline{\mathcal{H}}_1}(\rvw_z) = 0$?} The proof of \Cref{thm_dir_convg} relies on showing that, in the initial stages of training, the dynamics of $\rvw_z$ stated in \cref{gf_sep_wz} are close to the gradient flow dynamics of the NCF $\mathcal{N}_{\overline{\rvy},\overline{\mathcal{H}}_1}$ with initialization $\delta\rvz$. But, if  ${\mathcal{N}}_{\overline{\rvy},\overline{\mathcal{H}}_1}(\rvw_z) = 0$, for all $\rvw_z$, then the gradient flow of $\mathcal{N}_{\overline{\rvy},\overline{\mathcal{H}}_1}$ is not meaningful. We explore this situation further in \Cref{zero_NCF}. Our experiments suggest that in such cases, the higher-order remainder terms that arise in the decomposition of the output of the neural network start becoming crucial, and they determine the directional convergence among the weights. This situation arises in deep matrix factorization problems. For certain problems, it also seems to arise when a ReLU-type activation function ($\max(x,0)^p$) is used. We believe this happens because for ReLU-type activation functions, the function and its gradient both are zero for negative values. We do not encounter this behavior if Leaky ReLU-type activation functions ($\max(x,\alpha x)^p$) are used.
\subsection{Additional Discussion}
\label{sec:prop_kkt}
In this subsection, we discuss certain properties of the KKT points of $\widetilde{\mathcal{N}}_{{\rvp},\mathcal{H}_1}$, when $\mathcal{H}_1$ is of the form as in  \cref{h1_exp_poly} and $\rvp$ is a non-zero vector. \\

\noindent\textbf{Proportionality of weights at KKT point.} We show that at a positive KKT point of $\widetilde{\mathcal{N}}_{{\rvp},\mathcal{H}_1}$, the norm of incoming and outgoing weights of hidden neurons belonging to $\mathcal{G}_z$ are proportional.
\begin{lemma}
	\label{lemma_bal_wt}
	Let $\mathcal{H}$ be an $L$-layer feed-forward neural network as described in \cref{L_layer_feed}, for some $L\geq2$, with activation function $\sigma(x) = \max(x,\alpha x)^p$, where $p\in \sN$ and $p\geq 1$. Suppose its weights are partitioned into two sets $\rvw_n$ and $\rvw_z$ as done in \Cref{h1_proof_poly}, and let $\mathcal{H}_1$ be as defined in Lemma \ref{h1_exp_poly}. Suppose $\rvw_n$ is fixed, $\rvp$ is a non-zero vector, and $\overline{\rvw}_z$ is a positive KKT point of
	\begin{equation*}
	\widetilde{\mathcal{N}}_{{\rvp},\mathcal{H}_1} \coloneqq \max_{\rvw_z} {\rvp}^\top\mathcal{H}_1(\rmX;{\rvw}_n,\rvw_z), \text{ s.t. } \|\rvw_z\|_2^2 = 1.
	\end{equation*}
	Then, $\overline{\rmC}_{l} = \mathbf{0}$, for all $2\leq l\leq L-1$, and
	\begin{equation*}
	p\|\overline{\rmB}_{l+1}[:,j]\|_2^2 = \|\overline{\rmA}_{l}[j,:]\|_2^2, \text{ for all }l\in [L-1] \text{ and } 1 \leq j\leq k_l -p_l.
	\end{equation*}
	Consequently, 
	\begin{equation*}
	p\|\overline{\rmW}_{l+1}[:,j]\|_2^2 = \|\overline{\rmW}_{l}[j,:]\|_2^2, \text{ for all }l\in [L-1] \text{ and }  p_l+1 \leq j\leq k_l.
	\end{equation*}
\end{lemma}
Recall that $\rvw_z$ consists of $\{\rmB_{l+1}, \rmA_l\}_{l=1}^{L-1}$ and $\{\rmC_l\}_{l=2}^{L-1}$, and we denote their values at the KKT point $\overline{\rvw}_z$ by $\{\overline{\rmB}_{l+1}, \overline{\rmA}_l\}_{l=1}^{L-1}$ and $\{\overline{\rmC}_l\}_{l=2}^{L-1}$. The above lemma shows that at any positive KKT point, the norm of $\overline{\rmB}_{l+1}[:,j]$ is proportional to $\overline{\rmA}_{l}[j,:]$, and $\overline{\rmC}_{l} = \mathbf{0}$. Since, for all $l\in [L-1]$ and $p_l+1 \leq j\leq k_l $, ${\rmW}_{l+1}[:,j]$ is a concatenation of ${\rmB}_{l+1}[:,j]$ and ${\rmC}_{l+1}[:,j]$, and ${\rmW}_{l}[j,:]$ is a concatenation of ${\rmA}_{l}[j,:]$ and ${\rmC}_{l}[j,:]$, we get that $p\|\overline{\rmW}_{l+1}[:,j]\|_2^2 = \|\overline{\rmW}_{l}[j,:]\|_2^2$. The proof is in \Cref{pf_bal_lemma}. 

Recall that in Lemma \ref{h1_proof_poly}, $\mathcal{G}_z$ denotes the set containing the last $k_l-p_l$ neurons of each hidden layer.  The incoming and outgoing weights of neurons in $\mathcal{G}_z$ belong to $\rvw_z$. Thus, the above lemma establishes that at a positive KKT point, the norm of incoming and outgoing weights of hidden neurons belonging to $\mathcal{G}_z$ are proportional, implying that if the incoming weight is zero, then outgoing weight would also be zero, and vice-versa. We also observe this behavior in our numerical experiments discussed in \Cref{num_exp}. 

Note that when $\alpha\neq 1$, Lemma \ref{h1_proof_poly} holds for $p\geq4$, whereas the above lemma holds for $p\geq 1$, which includes ReLU activation. However, for ReLU activation, $f_{L,l}(\cdot)$ is not differentiable everywhere, so $\nabla_\rvs f_{L,l}(\cdot)$  can instead be replaced by any element of the Clarke subdifferential of  $f_{L,l}(\cdot)$.\\

\noindent\textbf{Parallel computation in maximizing NCF.} We next show that ${\mathcal{N}}_{{\rvp},\mathcal{H}_1}$ can be written as a sum of homogeneous functions, each depending on a distinct and disjoint subset of variables. This decomposition has important implications on the KKT points of $\widetilde{\mathcal{N}}_{{\rvp},\mathcal{H}_1}$.

Recall that $\rmB_{l}[:,j]$ and $\rmA_{l}[j,:]$ denote the $j$-th column of $\rmB_{l}$ and $j$-th row  of $\rmA_{l}$, respectively. For any $1\leq l\leq L-1$, we can write:
\begin{equation*}
\rmB_{l+1}\sigma\left(\rmA_{l}\rvg_{l-1}(\rvx)\right) = \sum_{j=1}^{\Delta_l}\rmB_{l+1}[:,j]\sigma\left(\rmA_{l}[j,:]\rvg_{l-1}(\rvx)\right),
\end{equation*}
where $\Delta_{l} = k_l-p_l$ is the number of neurons in layer $l$ belonging to $\mathcal{G}_z$. Hence, ${\mathcal{N}}_{{\rvp},\mathcal{H}_1}(\rvw_z)$ can be written as	
\begin{align*}
{\mathcal{N}}_{{\rvp},\mathcal{H}_1}(\rvw_z) &= \rvp^\top\mathcal{H}_1(\rmX;\rvw_n,\rvw_z)\\
&= \sum_{l=1}^{L-2}\sum_{j=1}^{\Delta_l}\sum_{i=1}^n p_i\nabla_\rvs f_{L,l+2}^\top\left(\rmN_{l+1}\rvg_{l}(\rvx_i)\right)\rmB_{l+1}[:,j]\sigma\left(\rmA_{l}[j,:]\rvg_{l-1}(\rvx_i)\right) \\
&+  \sum_{j=1}^{\Delta_{L-1}}\sum_{i=1}^n p_i\rmB_L[:,j]\sigma\left(\rmA_{L-1}[j,:]g_{L-2}(\rvx_i)\right)\\
&= \sum_{l=1}^{L-1}\sum_{j=1}^{\Delta_l}\sum_{i=1}^n p_i\mathcal{H}_{1,l}(\rvx_i;\rmB_{l+1}[:,j], \rmA_{l}[j,:]) \\
& = \sum_{l=1}^{L-1}\sum_{j=1}^{\Delta_l}\rvp^\top \mathcal{H}_{1,l}(\rmX;\rmB_{l+1}[:,j], \rmA_{l}[j,:])  = \sum_{l=1}^{L-1}\sum_{j=1}^{\Delta_{l}} \mathcal{N}_{\rvp,\mathcal{H}_{1,l}}(\rmB_{l+1}[:,j], \rmA_{l}[j,:]),
\end{align*}
where we define
\begin{align*}
&\mathcal{H}_{1,l}(\rvx;\rvb, \rva^\top) = \nabla_\rvs f_{L,l+2}^\top\left(\rmN_{l+1}\rvg_{l}(\rvx)\right)\rvb\sigma\left(\rva^\top\rvg_{l-1}(\rvx)\right),\text{ for all }l\in[L-2],\\
& \mathcal{H}_{1,L-1}(\rvx;\rvb, \rva^\top) = \rvb\sigma\left(\rva^\top\rvg_{L-2}(\rvx)\right).
\end{align*}
Note that each term in the above decomposition is $(p+1)$-homogeneous in $\rvw_z$. Moreover, each individual term is associated with exactly one hidden neuron in $\mathcal{G}_z$: it depends only on that neuron's incoming and outgoing weights. In other words, the contribution of each neuron in $\mathcal{G}_z$ to the NCF is isolated, and the NCF ${\mathcal{N}}_{{\rvp},\mathcal{H}_1}(\rvw_z)$  separates into neuron-specific, homogeneous components that are mutually independent in terms of the variables they involve. It is also worth noting that the functional form of the terms corresponding to neurons in a particular layer is the same but they depend on different sets of variables. Hence, the above decomposition of the NCF contains $(L-1)$ distinct functions $\{\mathcal{N}_{\rvp,\mathcal{H}_{1,l}}\}_{l=1}^{L-1}$ and their multiple copies.

This decomposition and the presence of multiple copies means that maximizing ${\mathcal{N}}_{{\rvp},\mathcal{H}_1}(\rvw_z)$ via gradient flow is equivalent to simultaneously maximizing each of these $(L-1)$ functions via gradient flow, with multiple independent initializations. Since each $\mathcal{N}_{\rvp,\mathcal{H}_{1,l}}$ is homogeneous, Lemma \ref{lemma:gf_ncf} implies that maximizing $\mathcal{N}_{\rvp,\mathcal{H}_{1,l}}$ via gradient flow will cause the associated variables to diverge in norm to infinity but converge in direction to a KKT point of $\widetilde{\mathcal{N}}_{\rvp,\mathcal{H}_{1,l}}$. However, our interest lies in the limiting direction of all the variables together, which corresponds to a KKT point of $\widetilde{\mathcal{N}}_{{\rvp},\mathcal{H}_1}(\rvw_z)$. The following lemma examines this scenario.
\begin{lemma}
	Consider maximizing $\sum_{i=1}^m\mathcal{G}_i(\rvw_i)$ via gradient flow, where each $\mathcal{G}_i(\rvw_i)$ is a a two-homogeneous functions in $\rvw_i$, for all $i\in [m]$. Let $(\rvw_1(t), \cdots,\rvw_m(t))$ be the corresponding gradient flow trajectories:
	\begin{equation}
	\dot{\rvw}_i = \nabla_{\rvw_i}\left(\sum_{i=1}^m\mathcal{G}_i(\rvw_i)\right) =  \nabla_{\rvw_i}\mathcal{G}_i(\rvw_i), \rvw_i(0) = \rvw_{i0},
	\label{gf_ind}
	\end{equation}
	where $(\rvw_{10}, \cdots,\rvw_{m0})$ is the initialization. For all $i\in [m]$, let $\rvw_{i0}\in \mathcal{S}(\rvw_i^*;\mathcal{G}_i(\rvw_i))$, that is,
	\begin{equation*}
	\lim_{t\rightarrow \infty} \frac{\rvw_i(t)}{\|\rvw_i(t)\|_2} = \rvw_i^*,
	\end{equation*}
	where $\rvw_i^*$ is a second-order positive KKT point of 
	\begin{equation}
	\widetilde{\mathcal{G}}_i(\rvw_i) \coloneqq \max_{\|\rvw_i\|_2^2 = 1} \mathcal{G}_i(\rvw_i).
	\label{g_NCF}
	\end{equation} 
	Then,
	\begin{equation}
	\lim_{t\rightarrow\infty} \frac{\|\rvw_i(t)\|_2}{\sqrt{\sum_{j=1}^m\|\rvw_j(t)\|_2^2}} \left\{\begin{matrix}
	>0, \text{ if } \mathcal{G}_i(\rvw_i^*) = \max_{j\in [m]} \mathcal{G}_j(\rvw_j^*)\\ = 0, \text{ if } \mathcal{G}_i(\rvw_i^*) < \max_{j\in [m]} \mathcal{G}_j(\rvw_j^*)
	\end{matrix}\right..
	\label{lim_dir}
	\end{equation}
	\label{2hm_gf}
\end{lemma} 
The above lemma analyzes the behavior of gradient flow when maximizing a sum of two-homogeneous functions, each depending on a different variable $\rvw_i$. Note that, we have not assumed $\mathcal{G}_i$'s to be distinct. This setup leads to decoupled gradient flows for each $\rvw_i$, as described in \cref{gf_ind}. Here, the norm of $\rvw_i(t)$ diverges and its direction converges to a KKT point of $\widetilde{\mathcal{G}}_i(\rvw_i)$, for all $i\in [m]$. However, when we examine the direction of all the variables combined, \cref{lim_dir} reveals that in the limit, only those set of variables corresponding to the most dominant KKT points remain nonzero\textemdash that is, those KKT points $\rvw_i^*$ for which $\mathcal{G}_i(\rvw_i)$ achieves the largest value among all $\{\mathcal{G}_j(\rvw_j)\}_{j=1}^m$.

The next lemma considers the case where the degree of homogeneity is greater than two.
\begin{lemma}
	Consider maximizing $\sum_{i=1}^m\mathcal{G}_i(\rvw_i)$ via gradient flow, where each $\mathcal{G}_i(\rvw_i)$ is a $L$-homogeneous functions in $\rvw_i$ with $L>2$. Let $(\rvw_1(t),\cdots,\rvw_m(t))$ be the corresponding gradient flow trajectories:
	\begin{equation}
	\dot{\rvw}_i = \nabla_{\rvw_i}\left(\sum_{i=1}^m\mathcal{G}_i(\rvw_i)\right) =  \nabla_{\rvw_i}\mathcal{G}_i(\rvw_i), \rvw_i(0) = \rvw_{i0},
	\label{gf_ind_4hm}
	\end{equation}
	where $(\rvw_{10},\cdots,\rvw_{m0})$ is the initialization such that $\mathcal{G}_i(\rvw_{i0})>0$. For all $i\in [m]$, let $\rvw_{i0}\in \mathcal{S}(\rvw_i^*;\mathcal{G}_i(\rvw_i))$, that is, there exists a $T_i$ such that
	\begin{equation*}
	\lim_{t\rightarrow T_i} {\|\rvw_i(t)\|_2} = \infty \text{ and }\lim_{t\rightarrow T_i} \frac{\rvw_i(t)}{\|\rvw_i(t)\|_2} = \rvw_i^*,
	\end{equation*}
	where $\rvw_i^*$ is a positive KKT point of 
	\begin{equation}
	\widetilde{\mathcal{G}}_i(\rvw_i) \coloneqq \max_{\|\rvw_i\|_2^2 = 1} \mathcal{G}_i(\rvw_i).
	\label{g_NCF_4hm}
	\end{equation} 
	Then, 
	\begin{equation}
	T_i\in \left[\frac{1}{L(L-2)\mathcal{G}_i(\rvw_i^*)},\frac{1}{L(L-2)\mathcal{G}_i(\rvw_{i0})}\right], \text{ for all }i\in [m].
	\label{tm_bd}
	\end{equation}
	Define $T^* \coloneqq \min_{i\in [m]}T_i$, then
	\begin{equation}
	\lim_{t\rightarrow T^*}\sum_{i=1}^m\|\rvw_i(t)\|_2^2 = \infty \text{ and }\lim_{t\rightarrow T^*} \frac{\|\rvw_i(t)\|_2}{\sqrt{\sum_{j=1}^m\|\rvw_j(t)\|_2^2}} \left\{\begin{matrix}
	>0, \text{ if } T_i= T^*\\ = 0, \text{ if } T_i >T^*
	\end{matrix}\right..
	\label{lim_dir_4hm}
	\end{equation}
	\label{4hm_gf}
\end{lemma} 
When the degree of homogeneity is greater than two, the gradient flow solution blows up in finite time. Hence, at the limiting direction of all the variables together, the set of variables that blow up first will remain non-zero, while others will become zero. From \cref{tm_bd}, we observe that the time required for blow up depends on the value of $\mathcal{G}_i(\rvw_i^*)$ and $\mathcal{G}_i(\rvw_{i0})$. Specifically, suppose $\rvw_1^*$ is the most dominant KKT point, and $\rvw_{10}$ is sufficiently close to $\rvw_1^*$. Then, $T_1$ will be the smallest and $\rvw_1(t)$ will blow up first. Therefore, if at least one of the initial weights is near the most dominant KKT point, then at the limiting direction of all the variables together, the set of variables corresponding to the most dominant KKT points will remain nonzero while the rest will become zero. The proofs of Lemmata \ref{2hm_gf} and \ref{4hm_gf} are in \Cref{max_sum_homogeneous}.

Now, consider maximizing ${\mathcal{N}}_{{\rvp},\mathcal{H}_1}(\rvw_z)$ via gradient flow. If it is two-homogeneous, assume the conditions in Lemma \ref{2hm_gf} are satisfied. If  it is $L$-homogeneous with $L > 2$, assume the conditions in Lemma \ref{4hm_gf} hold and that at least one of the initial weights is sufficiently near the most dominant KKT point. Then, at the limiting direction of all the variables together, only those set of variables corresponding to the most dominant KKT points will remain nonzero while the rest will become zero. This behavior highlights a key benefit of over-parameterization. Recall that the decomposition of ${\mathcal{N}}_{{\rvp},\mathcal{H}_1}(\rvw_z)$ only contains $(L-1)$ distinct functions, each having multiple copies. If we over-parameterize by adding more neurons in $\mathcal{G}_z$, we do not introduce new functions in ${\mathcal{N}}_{{\rvp},\mathcal{H}_1}(\rvw_z)$, but rather more copies of those $(L-1)$ functions. This will imply that each of those $(L-1)$ functions are being maximized with more number of independent initializations, which in turn will boost the odds of selecting a better dominant KKT point.

The conclusion that only the set of variables corresponding to the most dominant KKT points will remain nonzero is not always true and depends on conditions such as those stated in Lemma \ref{2hm_gf} and Lemma \ref{4hm_gf}. We make some remarks about those conditions. In Lemma \ref{2hm_gf}, we assume that each KKT point $\rvw_i^*$ is a second-order KKT point. This assumption is used to get a bound on $\|\rvw_i(t)\|$. This condition is perhaps not too strict, since for many problems gradient descent avoids first-order stationary points \citep{lee1_esc_saddle,lee2_esc_saddle}. Next, along with conditions in Lemma \ref{4hm_gf}, we require at least one of the initial weights to be sufficiently near the most dominant KKT point. Suppose the objective is a sum of few distinct homogeneous functions and their multiple copies, similar to ${\mathcal{N}}_{{\rvp},\mathcal{H}_1}(\rvw_z)$. Increasing the number of copies will imply that  the same set of functions are being maximized with more number of independent initializations. This will increase the chances of at least one of the initial weights being sufficiently near the most dominant KKT point. Thus, over-parameterization can help in satisfying this additional condition. In summary, while the selection of variables based on dominant KKT points is not an absolute certainty, it appears to be quite likely as the conditions required for it does not seem to be overly restrictive in over-parameterized settings. 
\subsection{Numerical Experiments}\label{num_exp}
We next conduct numerical experiments to validate \Cref{thm_dir_convg}, with results presented in \Cref{fig:3l_sq_near_saddle} and \Cref{fig:3l_relu_near_saddle}. In both cases, we train a three layer neural network with ten neurons in each layer\textendash one with square activation function $(\sigma(x) = x^2)$ and one with ReLU activation function  $(\max(x,0))$. The weights are initialized near a saddle point where the incoming and outgoing weights of the last nine neurons of each layer is zero; these weights form $\rvw_z$ and the remaining form $\rvw_n$, and the last nine neurons of each layer form $\mathcal{G}_z$.\\
\\
\noindent\textbf{Square activation.} \Cref{fig:loss_evol_sq} depicts the evolution of the training loss and the distance of the weights from the saddle point. As expected, the loss does not change much, and the weights remain close to the saddle point, which can also be verified by comparing \Cref{fig:all_weights_init_sq} and \Cref{fig:all_weights_it_sq}. Thus, the weights in $\rvw_z$ remain small in norm. Next, a unit norm vector is a KKT point of the constrained NCF if the vector and the gradient of the NCF at that vector are parallel. From \Cref{fig:ncf_evol_sq}, we observe that $\rvw_z$ approximately converges in direction to a KKT point of the constrained NCF. \Cref{fig:wz_init_sq} and \Cref{fig:wz_it_sq} provide a closer look at the weights belonging to $\rvw_z$ of every layer, at initialization and at iteration 180000, respectively. At initialization, they are small and random, and at iteration 180000 their norm increases but remains small overall, however, they are structured and exhibit sparsity. Except for the incoming and outgoing weight of the second neuron in the first layer, all the other weights belonging to $\rvw_z $ have relatively small magnitude. Also, in accordance with Lemma \ref{lemma_bal_wt}, if the incoming weights of a hidden neuron belonging to $\mathcal{G}_z$ have small magnitude, then the outgoing weights are also small, and vice-versa.\\
\\
\textbf{ReLU activation.} Although our theoretical results do not hold for ReLU activation,  in \Cref{fig:3l_relu_near_saddle} we explore ReLU activation. \Cref{fig:loss_evol_relu} shows that the loss does not change much and the weights remain close to the saddle point, implying  $\rvw_z$ remains small. \Cref{fig:ncf_evol_relu} shows that $\rvw_z$ converges in direction to a KKT point of the constrained NCF. \Cref{fig:wz_init_relu} and \Cref{fig:wz_it_relu} depict the weights in $\rvw_z$ for all layers, at initialization and at iteration 6000, respectively. At initialization, they are small and random. At iteration 6000, their norm increases but remains small overall. However, as before, they are structured and exhibit sparsity, where the incoming and outgoing weights of several neurons in the first layer have relatively higher magnitude than other weights in $\rvw_z$. Yet, they are consistent with Lemma \ref{lemma_bal_wt}\textemdash  if the incoming weights of a hidden neuron in $\mathcal{G}_z$ are small, then the outgoing weights are small too, and vice-versa. For instance, $5$th and $8$th row of $\rmW_1$ is small and $5$th and $8$th column of $\rmW_2$ is small as well. 

Although multiple neurons belonging to $\mathcal{G}_z$ in the first layer activate due to this directional convergence, their corresponding weights exhibit a distinct rank-one structure. In \Cref{fig:wz_it_relu}, the rows of $\rmW_1$ with high norms appear as scalar multiples of one another; in fact, they are positive scalar multiples. Specifically, let $\rmW_{1z}$ denote the matrix containing all rows of $\rmW_1$ except the first. At iteration 6000, the ratio $\|\rmW_{1z}\|_F/\|\rmW_{1z}\|_2 = 1.000425$, quantitatively confirming that $\rmW_{1z}$ is approximately rank-one. Consequently, even though multiple neurons in $\mathcal{G}_z$ activate, they can be represented using a single neuron.\\

\noindent Overall, the empirical behavior is consistent with \Cref{thm_dir_convg}: in the initial stages, the weights belonging to $\rvw_z$ remain small in norm but converge in direction to a KKT of the constrained NCF\textemdash even in the ReLU case, where the theorem does not strictly apply. Moreover, in accordance with Lemma \ref{lemma_bal_wt}, the norm of incoming and outgoing weights of hidden neurons belonging to $\mathcal{G}_z$ is proportional. 
\begin{figure}[htbp]
	\centering
	\begin{minipage}[t]{0.45\textwidth}
		\centering
		\begin{subfigure}[t]{\linewidth}
			\centering
			\includegraphics[width=\linewidth]{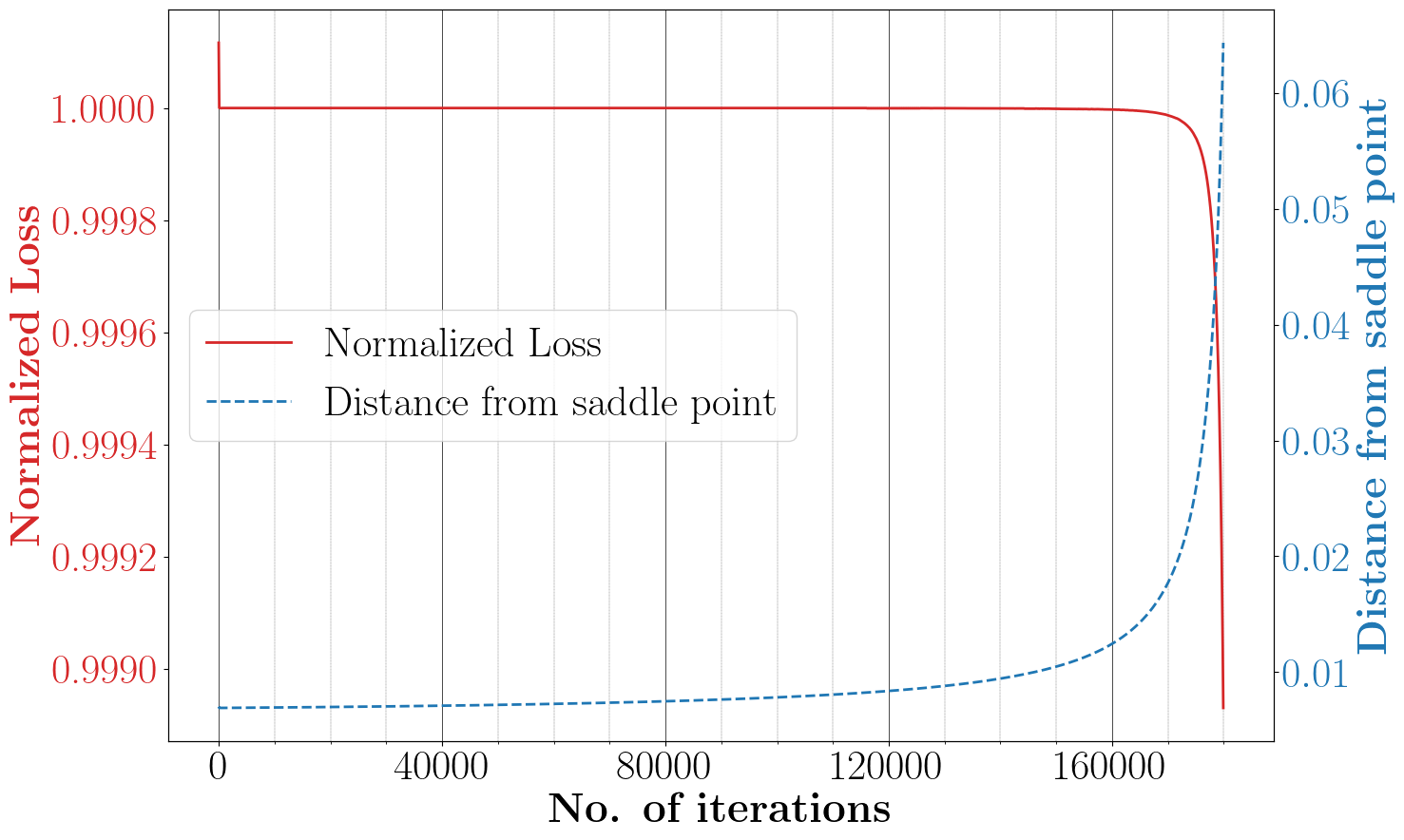}
			\caption{Evolution of training loss and the distance of weights from saddle point}
			\label{fig:loss_evol_sq}
		\end{subfigure}
		\vspace{0.3cm}
		
		\begin{subfigure}[t]{\linewidth}
			\centering
			\includegraphics[width=\linewidth]{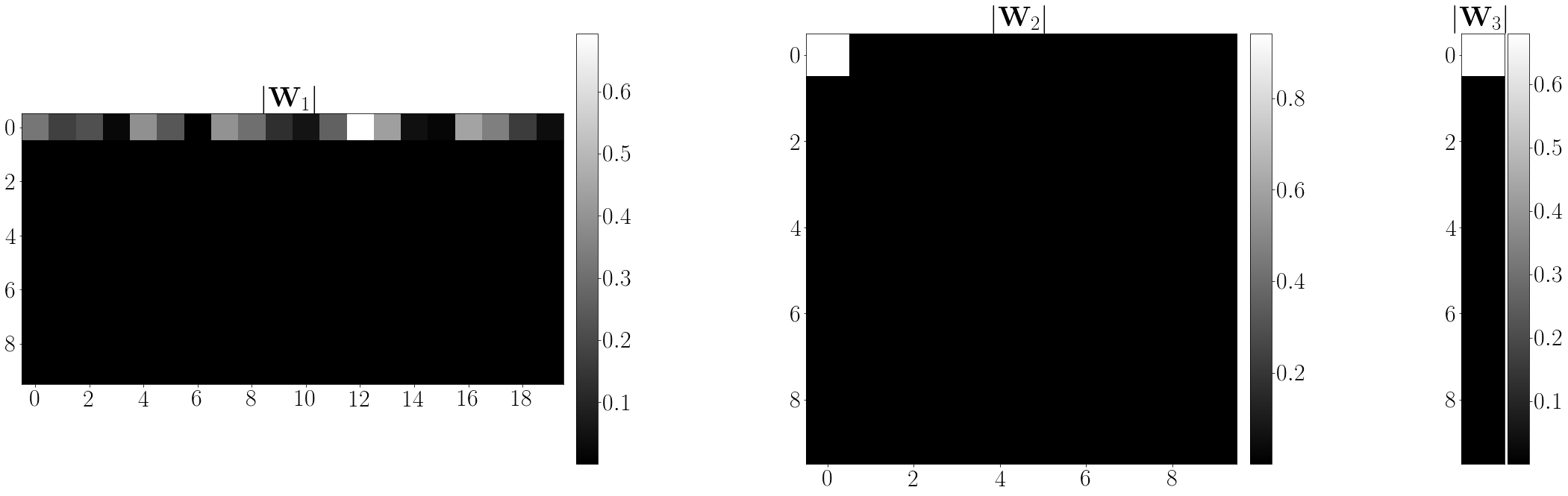}
			\caption{All weights at initialization}
			\label{fig:all_weights_init_sq}
		\end{subfigure}
		\vspace{0.3cm}
		
		\begin{subfigure}[t]{\linewidth}
			\centering
			\includegraphics[width=\linewidth]{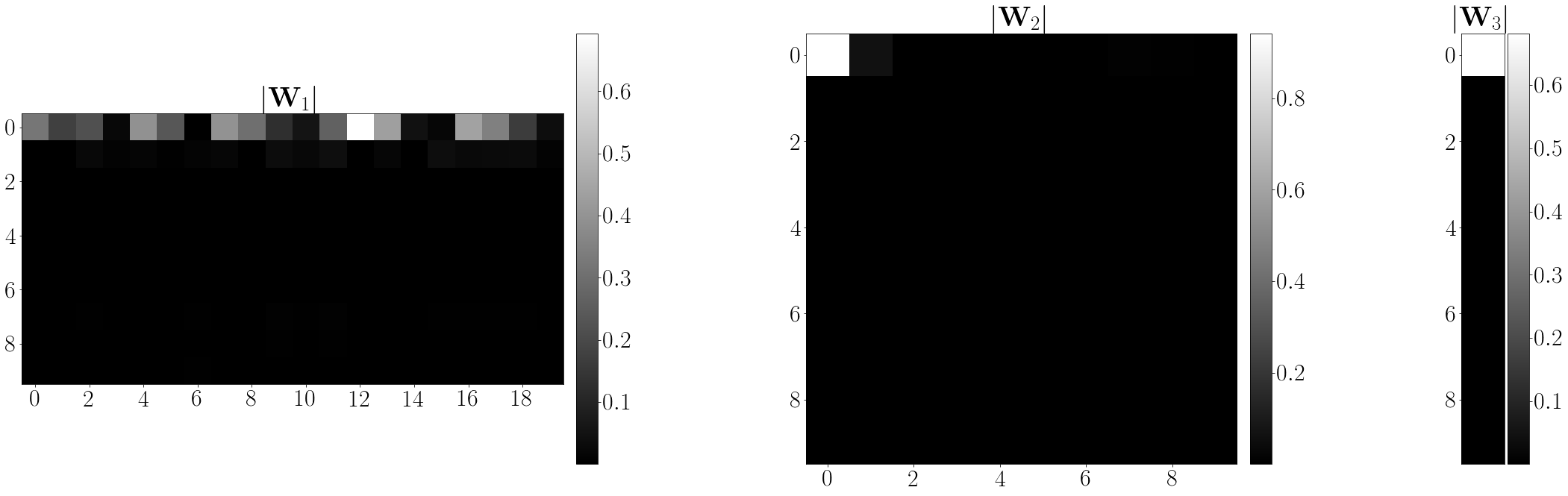}
			\caption{All weights at iteration 180000}
			\label{fig:all_weights_it_sq}
		\end{subfigure}
	\end{minipage}
	\hfill
	\begin{minipage}[t]{0.45\textwidth}
		\centering
		\begin{subfigure}[t]{\linewidth}
			\centering
			\includegraphics[width=\linewidth]{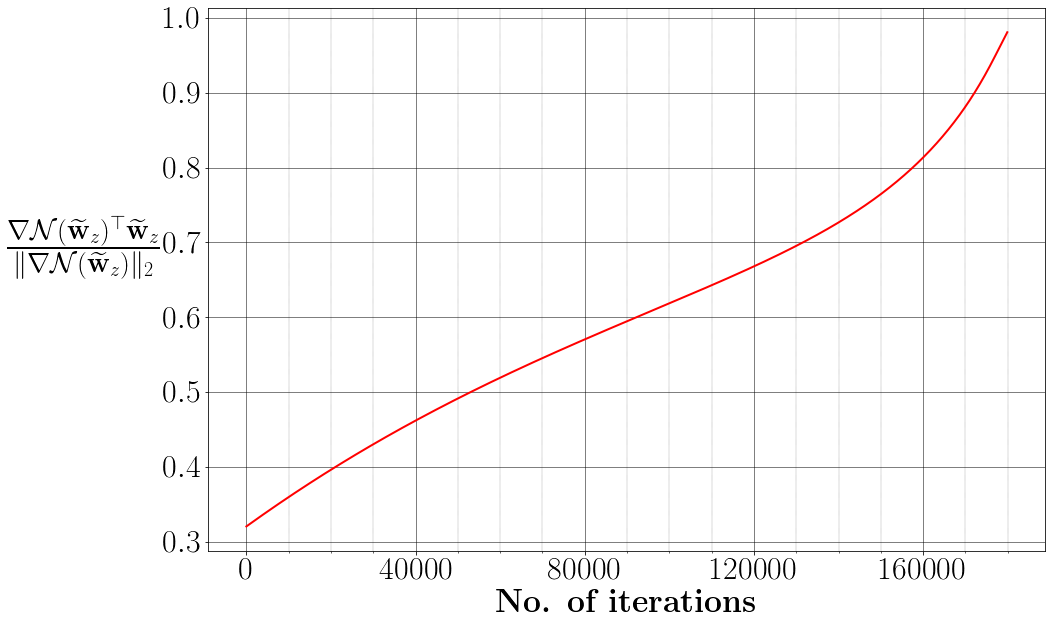}
			\caption{Evolution of inner product between gradient of the NCF and the weights}
			\label{fig:ncf_evol_sq}
		\end{subfigure}
		\vspace{0.3cm}
		
		\begin{subfigure}[t]{\linewidth}
			\centering
			\includegraphics[width=\linewidth]{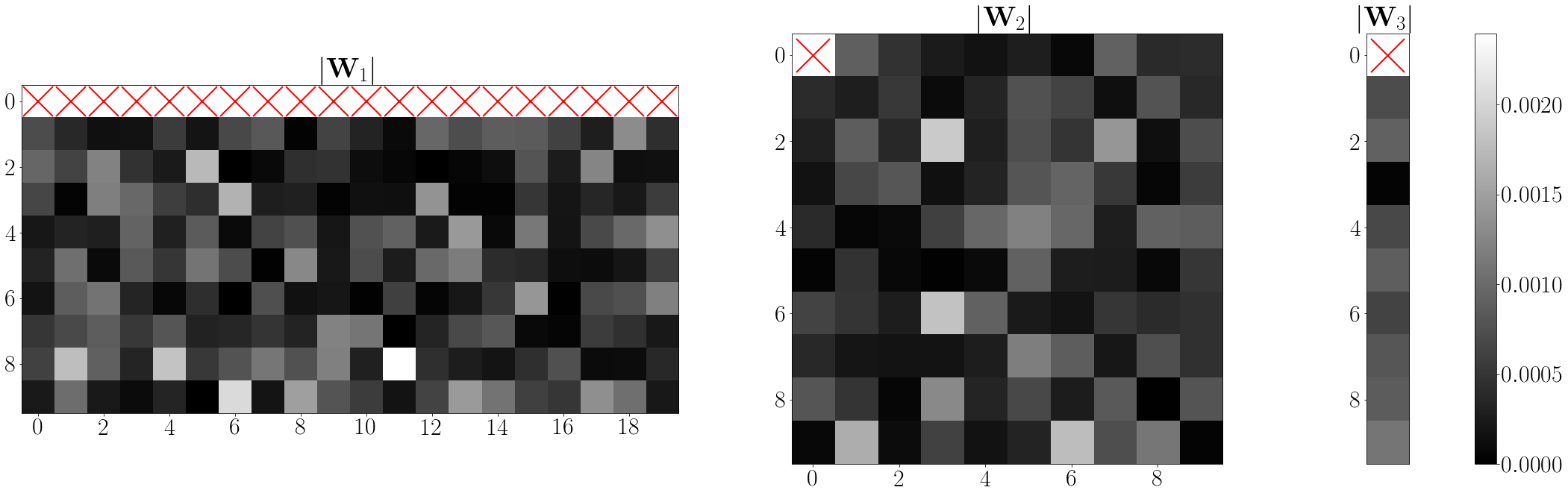}
			\caption{Weights in $\rvw_z$ at initialization}
			\label{fig:wz_init_sq}
		\end{subfigure}
		\vspace{0.3cm}
		
		\begin{subfigure}[t]{\linewidth}
			\centering
			\includegraphics[width=\linewidth]{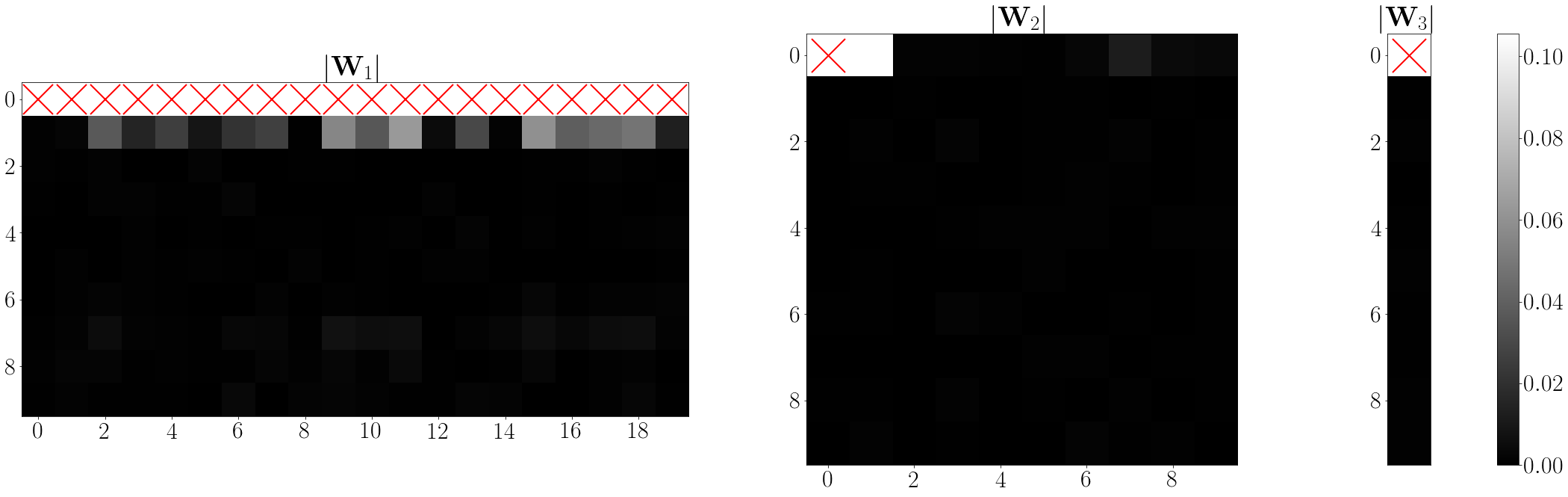}
			\caption{Weights in $\rvw_z$ at iteration 180000}
			\label{fig:wz_it_sq}
		\end{subfigure}
	\end{minipage}
	\caption{(\textbf{Gradient descent dynamics near saddle point}) We train a three-layer neural network via gradient descent whose output is $\rmW_3\sigma(\rmW_2\sigma(\rmW_1\rvx)) ,$ where $\sigma(x) = x^2$ (\textbf{square activation}), and $\rmW_3\in \mathbb{R}^{1\times 20},\rmW_2  \in \mathbb{R}^{10 \times 10}$, $\rmW_1\in\mathbb{R}^{10 \times 20}$. The weights are initialized near a saddle point where the incoming and outgoing weights of the last nine neurons of each layer is zero. These neurons form $\mathcal{G}_z$ and their incoming and outgoing weights would form $\rvw_z$; that is, $\rvw_z$ contains the last nine rows of $\rmW_1$, the last nine rows and columns of  $\rmW_2$ and the last nine entries of $\rmW_3$. The remaining weights form $\rvw_n$. Panel (a) shows the evolution of the training loss (normalized by the loss at saddle point) and the distance of the weights from the saddle point (normalized by the norm of the weights at saddle point).  Panels (b) and (c) depict the weights at initialization and at iteration 180000, respectively. The training loss does not change much and the weights remain near the saddle point. Also, weights in $\rvw_z$ remain small in magnitude. Panel (d) shows the evolution of ${\nabla\mathcal{N}}_{{\overline{\rvy}},\overline{\mathcal{H}}_1}(\widetilde{\rvw}_z)^\top\widetilde{\rvw}_z/\|{\nabla\mathcal{N}}_{{\overline{\rvy}},\overline{\mathcal{H}}_1}(\widetilde{\rvw}_z)\|_2 $, where  $\widetilde{\rvw}_z \coloneqq \rvw_z/\|\rvw_z\|_2$, which measures how close $\widetilde{\rvw}_z$ is to being a KKT point of the constrained NCF. It confirms that $\rvw_z$ have approximately converged in direction to a KKT point. Panels (e) and (f) depicts the weights in $\rvw_z$ of every layer, at initialization and at iteration 180000, respectively. Specifically, the weights belonging to $\rvw_n$ are crossed out and we only plot the weights in $\rvw_z$. At initialization, weights in $\rvw_z$ are small and random. At iteration 180000, they are sparse and structured, while the norm still stays small. In fact, only incoming and outgoing weights of the second neuron in the first layer have high magnitude, the remaining weights in $\rvw_z$ are relatively much smaller. This behavior is consistent with the result of  Lemma \ref{lemma_bal_wt}.  }
	\label{fig:3l_sq_near_saddle}
\end{figure}

\begin{figure}[htbp]
	\centering
	\begin{minipage}[t]{0.45\textwidth}
		\centering
		\begin{subfigure}[t]{\linewidth}
			\centering
			\includegraphics[width=\linewidth]{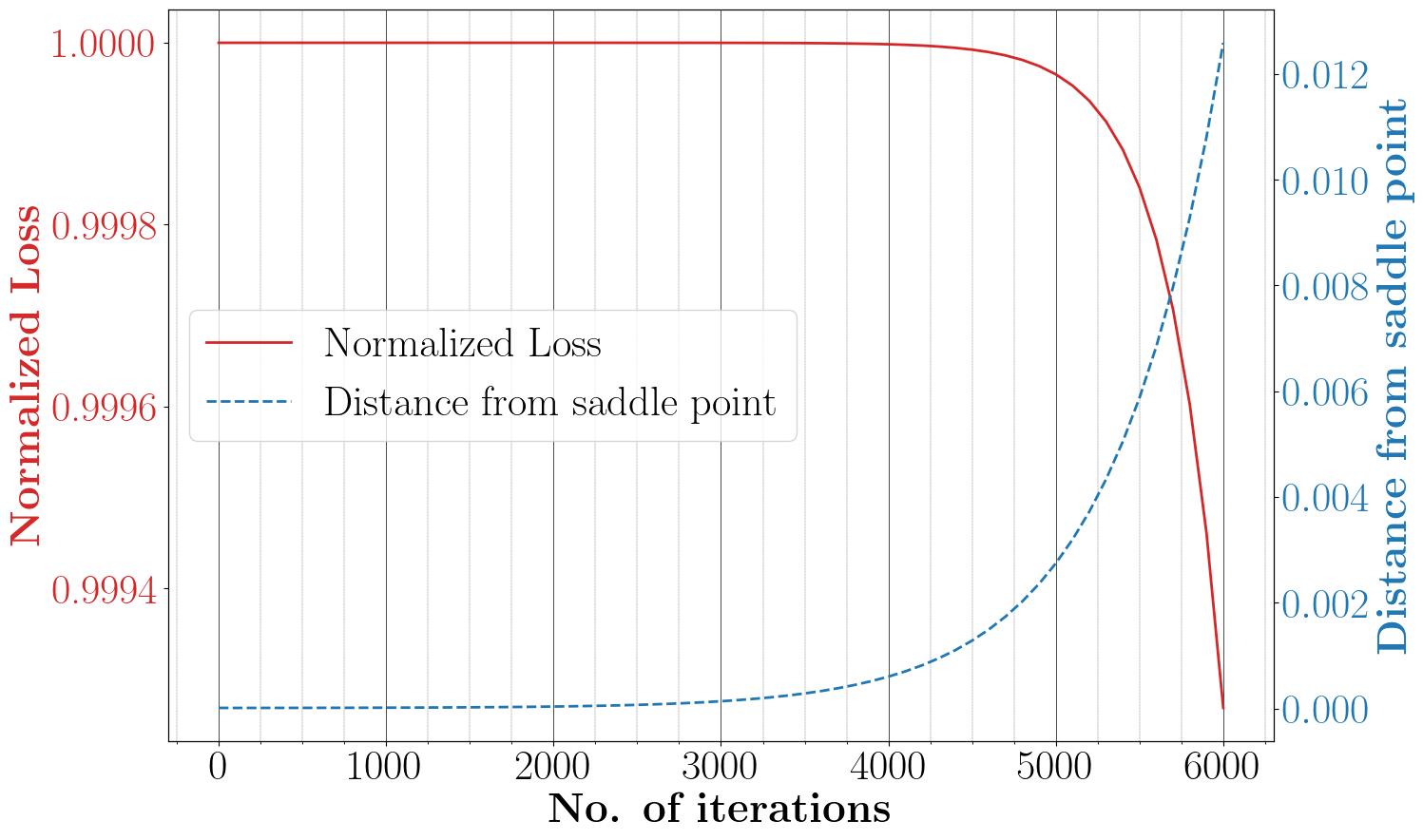}
			\caption{Evolution of training loss and distance of weights from saddle point }
			\label{fig:loss_evol_relu}
		\end{subfigure}
		\vspace{0.3cm}
		
		\begin{subfigure}[t]{\linewidth}
			\centering
			\includegraphics[width=\linewidth]{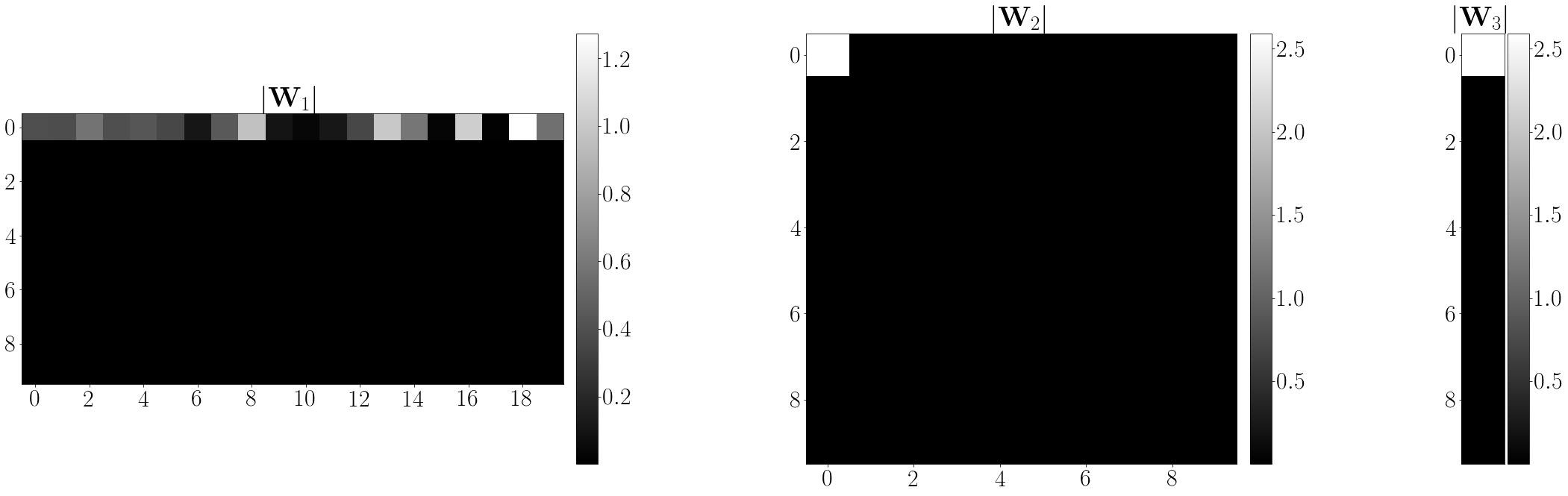}
			\caption{All weights at initialization}
			\label{fig:all_weights_init_relu}
		\end{subfigure}
		\vspace{0.3cm}
		
		\begin{subfigure}[t]{\linewidth}
			\centering
			\includegraphics[width=\linewidth]{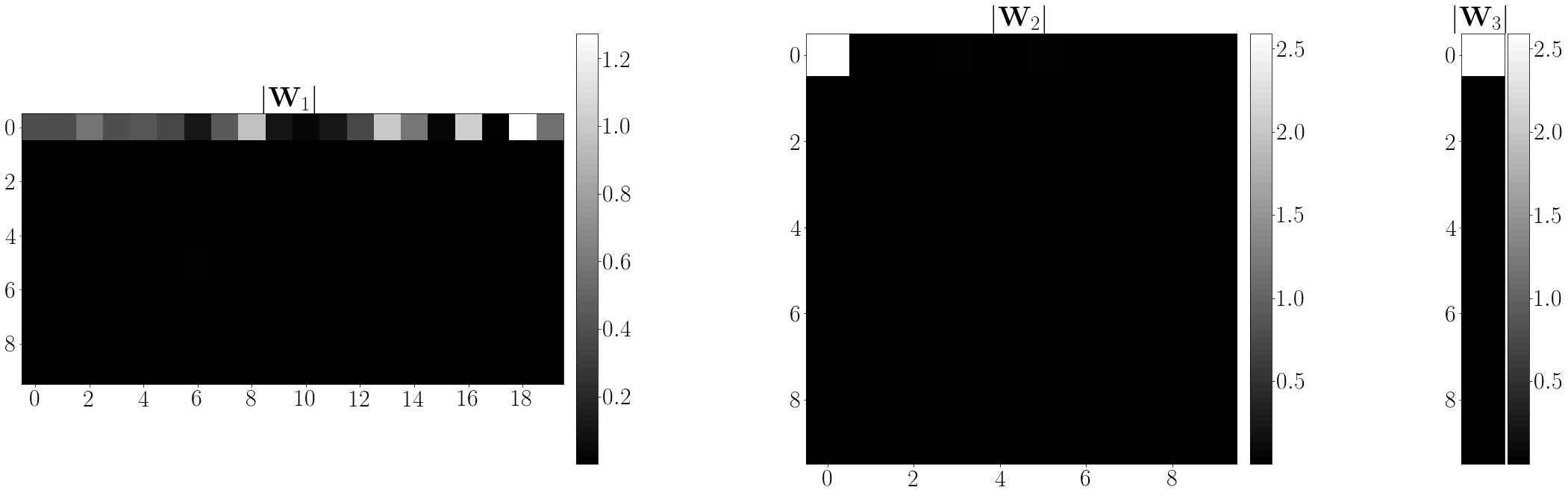}
			\caption{All weights at iteration 6000}
			\label{fig:all_weights_it_relu}
		\end{subfigure}
	\end{minipage}
	\hfill
	\begin{minipage}[t]{0.45\textwidth}
		\centering
		\begin{subfigure}[t]{\linewidth}
			\centering
			\includegraphics[width=\linewidth]{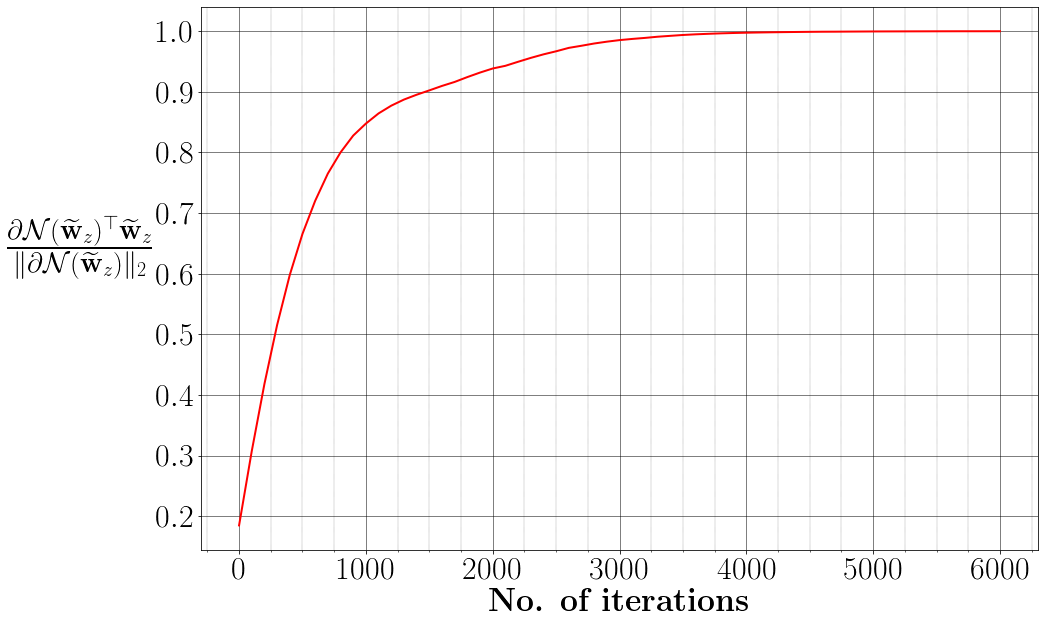}
			\caption{Evolution of inner product between gradient of the NCF and the weights}
			\label{fig:ncf_evol_relu}
		\end{subfigure}
		\vspace{0.3cm}
		
		\begin{subfigure}[t]{\linewidth}
			\centering
			\includegraphics[width=\linewidth]{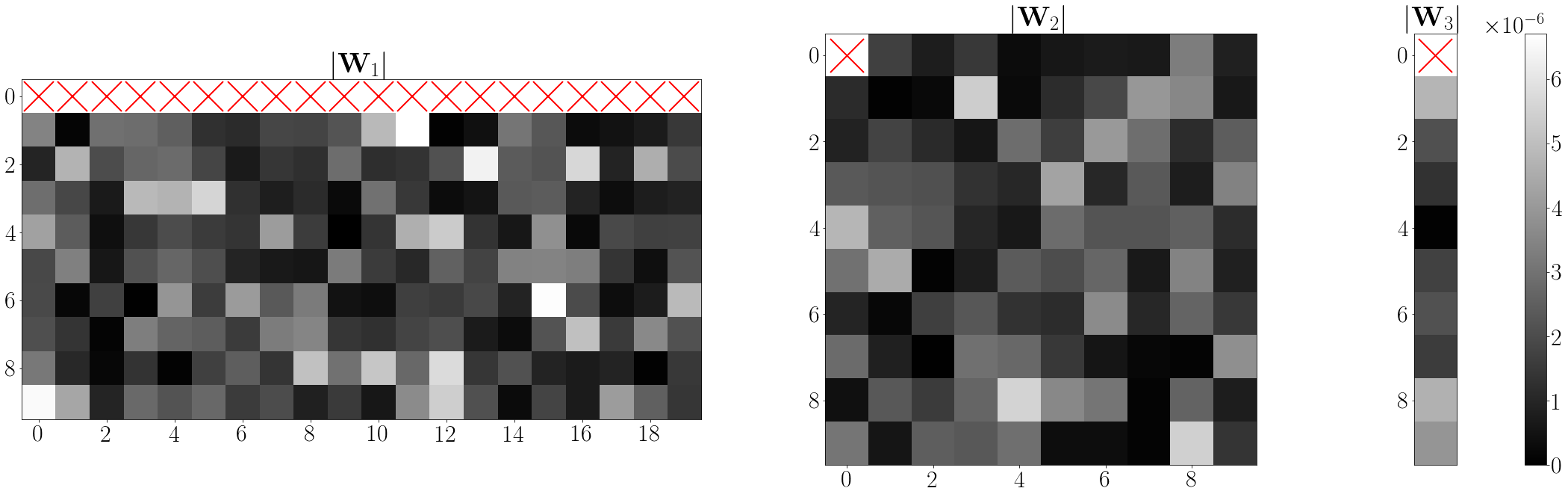}
			\caption{Weights in $\rvw_z$ at initialization}
			\label{fig:wz_init_relu}
		\end{subfigure}
		\vspace{0.3cm}
		
		\begin{subfigure}[t]{\linewidth}
			\centering
			\includegraphics[width=\linewidth]{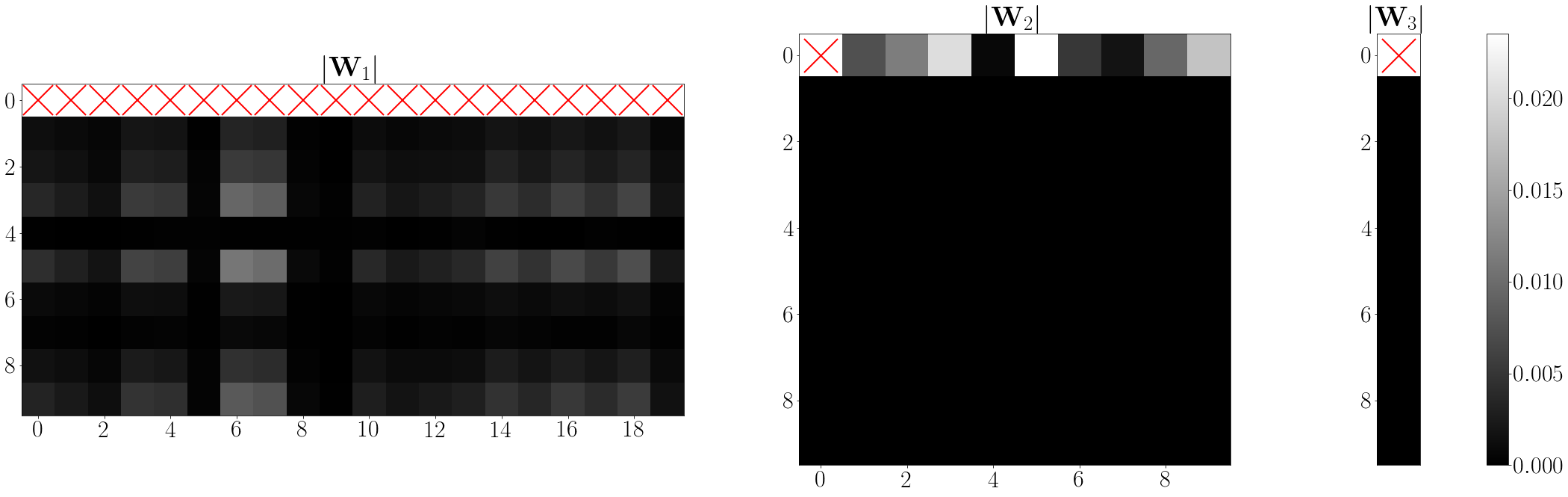}
			\caption{Weights in $\rvw_z$ at iteration 6000}
			\label{fig:wz_it_relu}
		\end{subfigure}
	\end{minipage}
	\caption{(\textbf{Gradient descent dynamics near saddle point}) We train a three-layer neural network via gradient descent whose output is $\rmW_3\sigma(\rmW_2\sigma(\rmW_1\rvx)) ,$ where $\sigma(x) = \max(x,0)$ (\textbf{ReLU activation}), and $\rmW_3\in \mathbb{R}^{1\times 20},\rmW_2  \in \mathbb{R}^{10 \times 10}$, $\rmW_1\in\mathbb{R}^{10 \times 20}$. The weights are initialized near a saddle point where the incoming and outgoing weights of the last nine neurons of each layer is zero; these neurons form $\mathcal{G}_z$ and their incoming and outgoing weights would form $\rvw_z$, and the remaining weights form $\rvw_n$. Panel (a) depicts the evolution of the training loss (normalized by the loss at saddle point) and the distance of the weights from the saddle point (normalized by the norm of the weights at saddle point). Panels (b) and (c) show the weights at initialization and at iteration 6000, respectively. We observe that the training loss does not change much, the weights remain near the saddle point and $\rvw_z$ remains small. Panel (d) shows the evolution of ${\partial\mathcal{N}}_{{\overline{\rvy}},\overline{\mathcal{H}}_1}(\widetilde{\rvw}_z)^\top\widetilde{\rvw}_z/\|{\partial\mathcal{N}}_{{\overline{\rvy}},\overline{\mathcal{H}}_1}(\widetilde{\rvw}_z)\|_2 $, where  $\widetilde{\rvw}_z \coloneqq \rvw_z/\|\rvw_z\|_2$, which confirms that $\rvw_z$ has converged in direction to a KKT point of the constrained NCF. Panels (e) and (f) depict the weights in $\rvw_z$ of every layer, at initialization and at iteration 6000, respectively, where in the latter the weights belonging to $\rvw_n$ are crossed out. At initialization, weights in $\rvw_z$ are small and random. At iteration 6000, they become structured and exhibit sparsity, while their norm increases but remains small. Compared to the experiment in \Cref{fig:3l_sq_near_saddle}, the incoming and outgoing weights of multiple neurons in the first layer have high magnitude. However, they are still consistent with Lemma \ref{lemma_bal_wt}, that is,  if the incoming weights of a hidden neuron belonging to $\mathcal{G}_z$ are small, then the outgoing weights is small too, and vice-versa.}
	\label{fig:3l_relu_near_saddle}
\end{figure}

Finally, we want to emphasize on the striking similarity in the gradient flow dynamics near the origin and near the saddle point. Near the origin, as discussed in \Cref{sec:early_dir}, the weights of a feed-forward neural network remains small in norm but converge in direction, where the norm of incoming and outgoing weights of each hidden neuron is proportional. This in turn leads to emergence of a sparsity structure among the weights, as the incoming and outgoing weights of many neurons are small, rendering those neurons approximately inactive. An identical phenomenon occurs near the saddle point: the weights that have small magnitude, denoted by $\rvw_z$, remain small in norm but converge in direction, where the norm of incoming and outgoing weights of hidden neurons in $\mathcal{G}_z$ is proportional. Here as well, a sparsity structure emerges among the weights which activates some neurons in $\mathcal{G}_z$, while the remaining neurons stay approximately inactive, as illustrated in \Cref{fig:3l_sq_near_saddle} and \Cref{fig:3l_relu_near_saddle}. Also, for networks with square activations, this directional convergence seems to activate a single neuron. For ReLU networks, while multiple neurons may activate, their weights exhibit a rank-one structure, making them functionally equivalent to a single activated neuron. A similar distinction between square and ReLU networks was also observed during directional convergence near the origin \citep{early_dc}, as mentioned in \Cref{sec:early_dir}.
\section{Gradient Flow Dynamics Beyond Saddle Points: Empirical Observations}
\label{sec:bey_saddle}
The preceding section described the evolution of weights while gradient flow remained near saddle points. In this section, we present important empirical observations about  the dynamics \emph{after} escaping these saddle points. Because exact gradient flow is a continuous-time mathematical object, it cannot be directly simulated. Therefore, we investigate these dynamics empirically by analyzing the trajectory of gradient descent, which serves as its discrete-time approximation. To do this, we extend the training horizon of the experiments illustrated in \Cref{fig:3l_sq_near_saddle} and \Cref{fig:3l_relu_near_saddle}.

\Cref{fig:3_layer_sq_sd2} depicts the result of running the experiments described in \Cref{fig:3l_sq_near_saddle} for more iterations. From \Cref{sd2_sq_loss_evol} we observe that after escaping from the saddle point, the loss rapidly decreases and then eventually stagnates, indicating that the weights have reached another saddle point. In  \Cref{sd2_sq_wt}, we plot the weights at this new saddle point. Notably, apart from the incoming and outgoing weights of the second neuron in the first layer, all the remaining weights in $\rvw_z$ are small. Now, recall from \Cref{fig:3l_sq_near_saddle} that near the saddle point, a sparsity structure emerged among the weights belonging to $\rvw_z$, where only the incoming and outgoing weights of the second neuron in the first layer was non-zero and the remaining weights remained small. This experiment suggests that the sparsity structure, which emerges among $\rvw_z$ near the saddle point, is preserved even after gradient descent escapes from the saddle point and reaches a new saddle point. 

\Cref{fig:3_layer_relu_sd2} depicts the result of running the experiments described in \Cref{fig:3l_relu_near_saddle} for more iterations. Overall, we observe a similar behavior to the previous experiment. \Cref{sd2_relu_loss_evol} shows that after escaping from the saddle point, the loss rapidly decreases and then eventually stagnates, indicating that weights have reached a new saddle point. Comparing \Cref{fig:wz_it_relu} and \Cref{sd2_sq_wt} confirms that the sparsity structure which emerged among $\rvw_z$ near the initial saddle point is preserved even after gradient descent escapes from that saddle point and reaches a new saddle point. For instance, $5$th and $8$th rows of $\rmW_1$, and the corresponding $5$th and $8$th columns of $\rmW_2$  were small before gradient descent escapes from the initial saddle point, and they remain small after gradient descent reaches the new saddle point. Similarly, all rows of $\rmW_2$ except for the first, and all entries of $\rmW_3$ except for the first, were small before gradient descent escapes from the initial saddle point, and they remain small after gradient descent reaches the new saddle point. Also, from \Cref{sd2_sq_wt}, it appears that the rank-one structure seems to preserved during this phase. At iteration $i_2$, the ratio $\|\rmW_{1z}\|_F/\|\rmW_{1z}\|_2 = 1.000435$, quantitatively confirming that $\rmW_{1z}$ is approximately rank-one, where recall $\rmW_{1z}$ denotes the matrix containing all rows of $\rmW_1$ except the first. 

Overall, these observations suggest a strong similarity between the gradient flow dynamics after escaping the origin and after escaping from saddle points. As discussed in \Cref{sec:bey_origin}, gradient flow escapes the origin such that the sparsity structure, which emerged due to directional convergence in the early stages of training, is preserved until reaching the next saddle point. In other words, the hidden neurons deactivated during the early stages of training, remain inactive even after gradient flow escapes the origin and until it reaches the next saddle point. The above experiments suggest that gradient flow escapes from saddle points in a similar manner. The sparsity structure, which emerged among $\rvw_z$ due to directional convergence near the saddle point, is preserved after gradient flow escapes from the saddle point and until it reaches the next saddle point. In other words, the hidden neurons in $\mathcal{G}_z$ which were inactive after directional convergence, they remain inactive even after gradient flow escapes the saddle point and until it reaches the next saddle point. Consequently, as the weights escape the saddle point and move towards the next saddle, the training loss is minimized using only the currently active neurons.

While we are unable to rigorously prove these empirical observations on post-escape training dynamics, assuming their general validity allows us to present a comprehensive, mechanistic picture of training feed-forward neural networks in the small initialization regime. We propose that the optimization trajectory moves from one saddle point to another during training; at each saddle point, a new set of neurons is activated via directional convergence, thereby augmenting the previously active neurons and increasing the representational capacity of the network. We articulate this more precisely in the following bullet points.
\begin{itemize}
	\item Since initialization is small, all hidden neurons can be considered approximately inactive at the initialization. Now, as discussed in \Cref{sec:early_dir}, in the early stages of training weights converge in direction while their norm remains small. Due to this directional convergence, incoming and outgoing weights of certain neurons have relatively smaller norms, rendering these neurons approximately inactive, while the complementary set of neurons are considered active.
	\item Subsequently, as discussed in \Cref{sec:bey_origin}, the gradient flow escapes from the origin and reaches the neighborhood of a saddle point (say, $\mathcal{S}_0$) of the training loss. Crucially, the set of active neurons is preserved during this phase: neurons that became inactive during early directional convergence remain inactive until the weights reach $\mathcal{S}_0$.
	\item The saddle point  $\mathcal{S}_0$ inherently exhibits sparsity, as the norm of incoming and outgoing weights of inactive neurons will be small. By invoking \Cref{thm_dir_convg} and Lemma \ref{lemma_bal_wt}, we can establish that near this saddle point, the set of weights associated with inactive neurons undergoes directional convergence. This  directional convergence activates a subset of the previously inactive neurons, augmenting the active set of neurons, while the remaining neurons stay inactive.
	\item Thereafter, guided by the empirical observations discussed in this section, weights escape from $\mathcal{S}_0$  and reach the neighborhood of another saddle point (say, $\mathcal{S}_1$). Once again, the set of active neurons is preserved during this phase\textemdash the neurons that remained inactive after the directional convergence near $\mathcal{S}_0$ remained inactive until weights reach $\mathcal{S}_1$.
	\item This new saddle point $\mathcal{S}_1$ also exhibits the identical sparsity structure since the norm of incoming and outgoing weights of inactive neurons will be small. We can recursively apply \Cref{thm_dir_convg} and Lemma \ref{lemma_bal_wt} to argue that another subset of inactive neurons is activated near $\mathcal{S}_1$. The weights then escape from $\mathcal{S}_1$ and and reach the neighborhood of another saddle point, again preserving the inactive set during this phases. This process of saddle-to-saddle traversal and incremental activation continues until the weights ultimately converge.
\end{itemize}
Therefore, the entire optimization trajectory can be viewed as a series of saddle-to-saddle transitions. At each saddle point, a new subset of neurons become activated, systematically augmenting the previously active neurons. Moreover, the identity of these newly activated neurons, and their optimal initial weights, is determined by the KKT point of the constrained NCF at the corresponding saddle point. This mechanistic view of the training dynamics is also consistent with the saddle-to-saddle dynamics hypothesis \citep{jacot_sd, lyu_resolving}, which argues that gradient descent descent passes through a sequence of saddle points and the effective complexity of the network gradually increases during training.

The observation that neurons seem to activate in this sequential, incremental fashion as training progresses naturally motivates the design of an algorithm that mimic this behavior by incrementally adding neurons to the network. The goal of the next section is to leverage all these insights to develop a greedy algorithm for training deep neural networks, wherein neurons are gradually added to the network and then the training loss is minimized using this augmented network.
\begin{figure}[htbp]
	\centering
	\begin{subfigure}{0.3\textwidth}
		\centering
		\includegraphics[width=\linewidth]{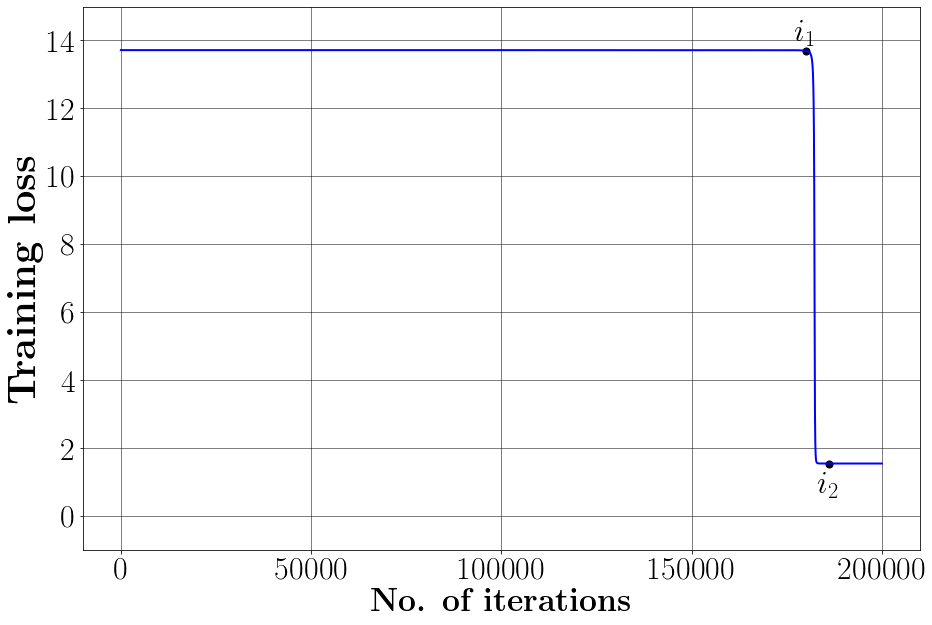}
		\caption{Evolution of training loss}
		\label{sd2_sq_loss_evol}
	\end{subfigure}
	\hfill
	\begin{subfigure}{0.6\textwidth}
		\centering
		\includegraphics[width=\linewidth]{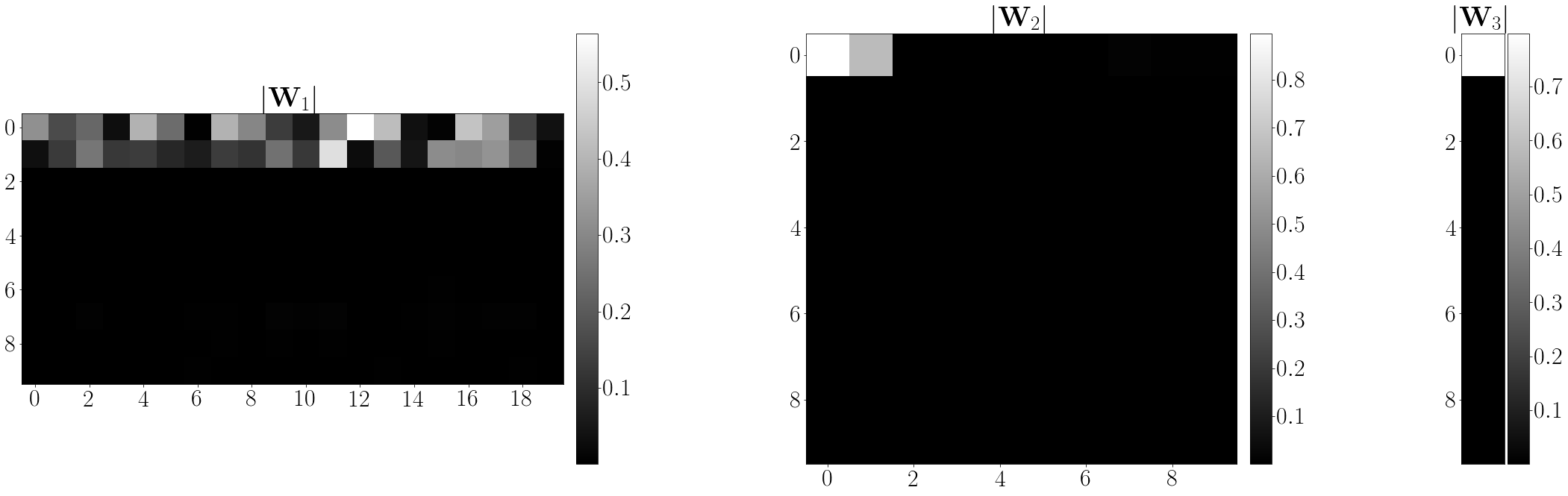}
		\caption{Weights at iteration $i_2$}
		\label{sd2_sq_wt}
	\end{subfigure}
	\caption{Extended optimization trajectory for the experiment illustrated in \Cref{fig:3l_sq_near_saddle}. Panel (a) depicts the evolution of training loss. After escaping from the saddle point, the loss rapidly decreases before plateauing at a new saddle point. Panel (b) depicts the absolute value of weights at iteration $i_2$ (marked in Panel (a)), shortly after reaching the new saddle point. Also $i_1 = 180000$, just before gradient descent escapes from the saddle point. Comparing Panel (b) with \Cref{fig:wz_it_sq} suggests that the sparsity structure which emerged among  $\rvw_z$ (as defined in \Cref{fig:3l_sq_near_saddle}) before escaping the saddle point is preserved until reaching the new saddle point.}
	\label{fig:3_layer_sq_sd2}
\end{figure}

\begin{figure}[htbp]
	\centering
	\begin{subfigure}{0.3\textwidth}
		\centering
		\includegraphics[width=\linewidth]{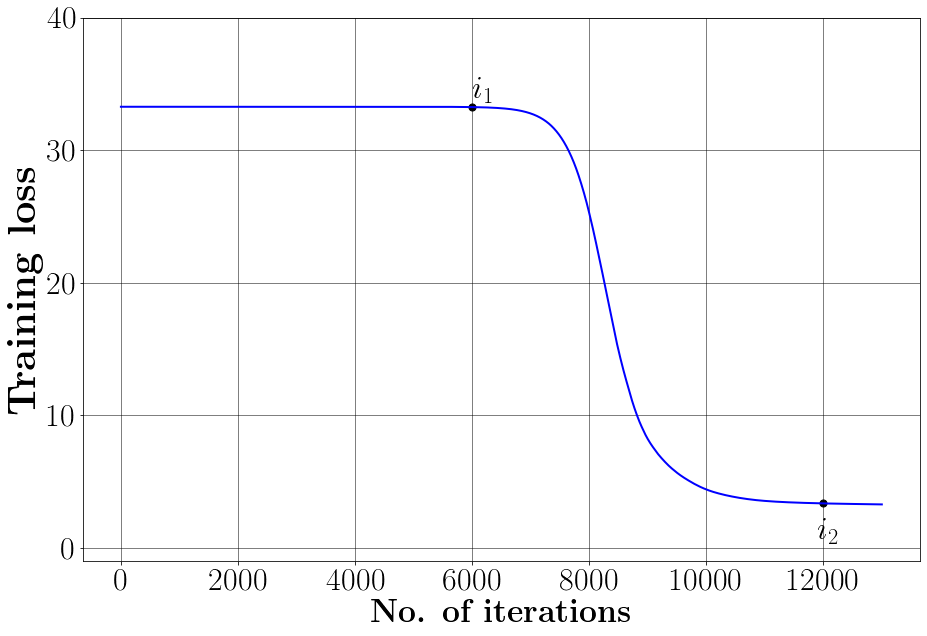}
		\caption{Evolution of training loss}
		\label{sd2_relu_loss_evol}
	\end{subfigure}
	\hfill
	\begin{subfigure}{0.6\textwidth}
		\centering
		\includegraphics[width=\linewidth]{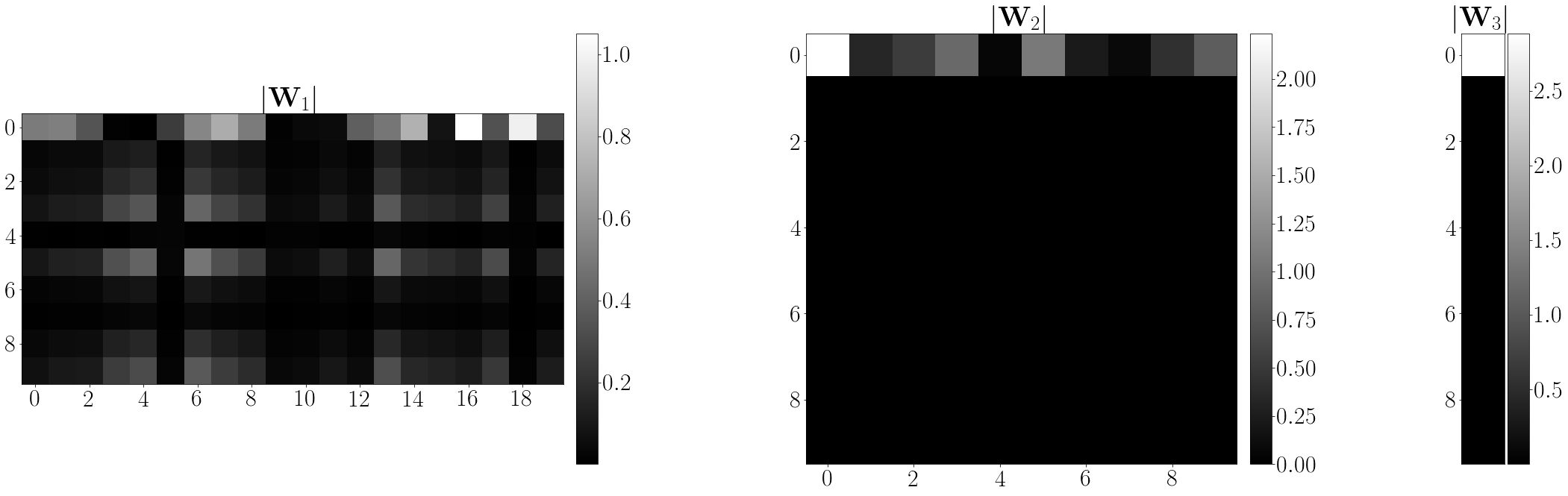}
		\caption{Weights at iteration $i_2$}
		\label{sd2_relu_wt}
	\end{subfigure}
	\caption{Extended optimization trajectory for the experiment illustrated in \Cref{fig:3l_relu_near_saddle}. Panel (a) depicts the evolution of training loss. After escaping from the saddle point, the loss rapidly decreases before plateauing at a new saddle point. Panel (b) depicts the absolute value of weights at iteration $i_2$ (marked in Panel (a)), shortly after reaching the new saddle point. Also $i_1 = 6000$, just before gradient descent escapes from the saddle point. Comparing Panel (b) with \Cref{fig:wz_it_relu} demonstrates that the sparsity structure which emerged among $\rvw_z$ (as defined in \Cref{fig:3l_relu_near_saddle}) before escaping the saddle point is preserved until reaching the new saddle point.}
	\label{fig:3_layer_relu_sd2}
\end{figure}
\section{Neuron Pursuit}\label{sec:np}
In this section we present our algorithm, \emph{Neuron Pursuit} (NP), to train homogeneous feed-forward neural networks. At a high level, the NP algorithm is inspired by saddle-to-saddle dynamics, and it also leverages insights gained into the emergence of sparsity structure near the saddle points and after escaping them, as discussed above and in previous works. More concretely, NP iteratively augments the network by adding neuron(s) selected via the maximization of an appropriate constrained NCF, and then minimizes the training loss via gradient descent using this augmented network. The algorithm is described in \Cref{np_algo}, and we next discuss the steps involved in the algorithm and their motivation. 
\begin{algorithm}[htbp]
	\caption{Neuron Pursuit}
	\begin{algorithmic}[1]
		\Require Training data $\mathcal{D} = \{(\rmX, \rvy)\} \in \sR^{d\times N}\times \sR^N$, activation function $\sigma(\cdot)$, small scalar $\delta$, number of iterations $E$, depth $L$
		\State Define $\mathcal{H}(\rvx;\rmW_1,\cdots,\rmW_L) = \rmW_L\sigma(\cdots\sigma(\rmW_1\rvx))$, where $\rmW_1\in \sR^{1 \times d} \text{ and } \rmW_l \in \sR $, for all $2\leq l\leq L$, and $\mathcal{L}(\rmW_1,\cdots,\rmW_L) = \|\rvy - \mathcal{H}(\rmX;\rmW_1,\cdots,\rmW_L)\|_2^2$.
		\State Initialize $\mathcal{N}_{max }\gets -\infty$
		\For{ $h = 1$ to $H$} \Comment{Maximizing the NCF.}
		\State $(\rmW_1^*,\cdots,\rmW_L^*)\gets \argmax \rvy^\top\mathcal{H}(\rmX;\rmW_1,\cdots,\rmW_L), \text{ s.t. } \sum_{l=1}^L\|\rmW_l\|_F^2=1.$
		\State $\mathcal{N}^* \gets \rvy^\top\mathcal{H}(\rmX;\rmW_1^*,\cdots,\rmW_L^*).$
		\If {$\mathcal{N}^* > \mathcal{N}_{\max}$} \Comment{Store the KKT point with highest value of the NCF.}
		\State $(\widehat{\rmW}_1,\cdots,\widehat{\rmW}_L) \gets (\rmW_1^*,\cdots,\rmW_L^*)$. 
		\State $\mathcal{N}_{\max} \gets \mathcal{N}^*$
		\EndIf
		\EndFor
		\State  $({\rmW}_1(0),\cdots,{\rmW}_L(0)) \gets (\delta\widehat{\rmW}_1,\cdots,\delta\widehat{\rmW}_L)$ \hfill
		\parbox[t]{0.5\linewidth}{\raggedright $\triangleright$ Small initial weights aligned along the most dominant KKT point.}
		\For{ $k = 1$ to $\infty$} \Comment{Minimizing training loss using gradient descent.}
		\State ${\rmW}_l(k+1) \gets {\rmW}_l(k)  - \eta \nabla_{\rmW_l}\mathcal{L}(\rmW_1(k),\cdots,\rmW_L(k)),$ for all $1\leq l\leq L$.
		\EndFor
		\State $(\overline{\rmW}_1,\cdots,\overline{\rmW}_L) \gets ({\rmW}_1(\infty),\cdots,{\rmW}_L(\infty))$.
		\For{iteration $= 1$ to $E$}
		\State $\overline{\rvy} \gets \rvy - \mathcal{H}(\rmX;\overline{\rmW}_1,\cdots,\overline{\rmW}_L)$ \Comment{Residual error}
		\State $\mathcal{N}_{max }= -\infty$
		\For{ $h = 1$ to $H$} \Comment{Maximizing the NCF.}
		\For { $l= 1$ to $L-1$} 
		\State $(\rva_l^*,\rvb_{l+1}^*) \gets \argmax \overline{\rvy}^\top\mathcal{H}_{1,l}(\rmX;\rvb_{l+1},\rva_l), \text{ s.t. } \|\rva_l\|_2^2+\|\rvb_{l+1}\|_2^2=1.$
		\State $\mathcal{N}^* \gets \overline{\rvy}^\top\mathcal{H}_{1,l}(\rmX;\rvb_{l+1}^*,\rva_l^*).$
		\If {$\mathcal{N}^* > \mathcal{N}_{\max}$} \Comment{Store the KKT point with highest value of the NCF.}
		\State $(\widehat{\rva},\widehat{\rvb})\gets(\rva_l^*,\rvb_{l+1}^*).$ 
		\State ${l}_* \gets l$, $\mathcal{N}_{\max} \gets \mathcal{N}^*.$
		\EndIf
		\EndFor
		\EndFor
		\State $\widehat{\rmW}_{{l}_*} \gets \begin{bmatrix}
		\overline{\rmW}_{{l}_*} \\\delta \widehat{\rva}^\top 
		\end{bmatrix}$ \hfill
		\parbox[t]{0.5\linewidth}{\raggedright $\triangleright$ The added weights are small and aligned along the most dominant KKT point.}
		\State $\widehat{\rmW}_{{l}_*+1} \gets \begin{bmatrix}
		\overline{\rmW}_{{l}_*+1} & \delta \widehat{\rvb}
		\end{bmatrix}$
		\State $\widehat{\rmW}_{l} \gets \overline{\rmW}_{{l}}$, for all $l\notin\{ l_*,l_*+1\}$
		\State  $({\rmW}_1(0),\cdots,{\rmW}_L(0)) \gets (\widehat{\rmW}_1, \cdots,\widehat{\rmW}_L)$
		\For{ $k = 1$ to $\infty$} \Comment{Minimizing training loss using gradient descent.}
		\State ${\rmW}_l(k+1) \gets {\rmW}_l(k)  - \eta \nabla_{\rmW_l}\mathcal{L}(\rmW_1(k),\cdots,\rmW_L(k)),$ for all $1\leq  l\leq L$.
		\EndFor
		\State $(\overline{\rmW}_1,\cdots,\overline{\rmW}_L) \gets ({\rmW}_1(\infty),\cdots,{\rmW}_L(\infty))$.
		\EndFor
	\end{algorithmic}
	\label{np_algo}
\end{algorithm}

The input to the algorithm is the training data, the activation function $\sigma(x)$, which is assumed to be of the form $\max(x,\alpha x)^p$, the depth of the network $L$ and the number of iterations $E$. It also requires a scalar $\delta$, which is assumed to be small.

The algorithm begins by initializing the neural network $\mathcal{H}$ with depth $L$ and one neuron in every layer. Then, inspired from the work of \citet{kumar_dc,kumar_escape,early_dc}, also discussed in \Cref{sec:background}, we minimize the training loss with initial weights small in magnitude and aligned along a KKT point of the constrained NCF. More specifically, in lines $3-10$, the constrained NCF defined with respect to $\mathcal{H}$ and $\rvy$ is maximized. Although not specified, we use projected gradient ascent with $H$ different random initializations. The most dominant KKT point, the one which leads to the largest value of  the NCF, is stored. In lines $12-14$, the training loss is minimized using gradient descent, where the initial weights have small norm and aligned along the most dominant KKT point. The limiting point of this minimization procedure is stored in $(\overline{\rmW}_1, \cdots,\overline{\rmW}_L) $, which will be a stationary point of the training loss.

Next, $(\overline{\rmW}_1,\cdots,\overline{\rmW}_L) $ can be viewed as a saddle point of the training loss, where only one neuron is active in every layer. In this view, inspired by our discussion in \Cref{sec:near_saddle} and \Cref{sec:bey_saddle}, we ``escape'' this saddle point by ``activating'' certain neurons. More specifically, in lines $19-28$, the constrained NCF with respect to the residual error $\overline{\rvy}$ and $\mathcal{H}_{1,l}$ is maximized (using projected gradient ascent with $H$ different random initializations), where 
\begin{align*}
&\mathcal{H}_{1,l}(\rvx;\rvb_{l+1}, \rva_l) = \nabla_\rvs f_{L,l+2}^\top\left(\overline{\rmW}_{l+1}\rvg_{l}(\rvx)\right)\rvb_{l+1}\sigma\left(\rva_l^\top\rvg_{l-1}(\rvx)\right), \text{ for all }1\leq l\leq L-2,\\
&\mathcal{H}_{1,L-1}(\rvx;\rvb_L, \rva_{L-1}) = \rvb_L\sigma\left(\rva_{L-1}^\top\rvg_{L-2}(\rvx)\right),
\end{align*}
and, for $1\leq l\leq L-1$,
\begin{equation*}
g_0(\rvx) = \rvx, g_l(\rvx) = \sigma(\overline{\rmW}_lg_{l-1}(\rvx)), \text{ and } f_{L,l+1}(\rvs) =  {\overline{\rmW}_{L}}\sigma\left(\overline{\rmW}_{L-1}\sigma\left(\cdots\sigma(\overline{\rmW}_{l+1}\sigma(\rvs))\right)\right).
\end{equation*}
The above choice of $\mathcal{H}_{1,l}$ is same as stated in \Cref{sec:prop_kkt}, where only the last neuron in each layer is assumed to be in $\mathcal{G}_z$,  and $\rva_l$'s and $\rvb_{l+1}$'s can be viewed as the incoming and outgoing weights of those neurons, respectively. Also, maximizing with $H$ different random initializations is inspired from our discussion at the end of  \Cref{sec:prop_kkt}: maximizing the constrained NCF is equivalent to maximizing its homogeneous components in parallel with multiple independent initializations, and the weights corresponding to the most dominant KKT point remains non-zero, while others become zero.

In lines $29-30$, the weights corresponding to the most dominant KKT point are added to the network, where the magnitude of the added weights are small. Note that this procedure is equivalent to adding a neuron with incoming and outgoing weights aligned along the KKT point. Then, in lines $33-36$, the training loss is minimized using gradient descent and the limiting point is stored in $(\overline{\rmW}_1, \overline{\rmW}_2,\cdots,\overline{\rmW}_L) $, which will be a stationary point of the training loss. Now, $(\overline{\rmW}_1, \overline{\rmW}_2,\cdots,\overline{\rmW}_L) $ can again be viewed as a saddle point of the training loss such that certain neurons are active in every layer. Hence, we again ``escape'' this saddle point by ``activating'' certain neurons in a similar way. This process is continued for $E$ number of iterations.

\begin{remark}	
	The NP algorithm and the analysis in \Cref{sec:near_saddle} are presented under the assumption that the output of the neural network is a scalar. However, both naturally extend to the vector-valued output setting. For example, in the NP algorithm, the label vector $\rvy$ can be replaced with a label matrix $\rmY$, where each column corresponds to a multi-dimensional label. Similarly, the NCF can be defined as the matrix inner product between the network's output and the label matrix. Analogous adjustments can be applied throughout the algorithm and the analysis to accommodate vector-valued outputs.
\end{remark}
\textbf{Example.} For a better understanding of the NP algorithm, we now illustrate its training mechanism with a concrete example. Specifically, we aim to learn the target function $f(\rvx) = \sigma(\sigma(2x_1+x_3) - \sigma(x_3-x_2))$, defined over the binary hypercube $\{\pm 1\}^{20}$, where $\sigma(x) = \max(x,0)$ is the ReLU activation and $x_i$ denotes the $i$th coordinate of $\rvx \in \sR^{20}$. To this end, we use a three-layer neural network with ReLU activation function, and train it on a dataset of 500 samples drawn uniformly at random from the hypercube. The NP algorithm is run for two iterations, successfully recovering the target function $f(\rvx)$. We depict the entire training process in \Cref{fig:np_ill}. 
\begin{figure}[htbp]
	\centering
	\begin{subfigure}[b]{\textwidth}
		\centering
		\begin{subfigure}[b]{0.3\textwidth}
			\centering
			\includegraphics[width=\linewidth]{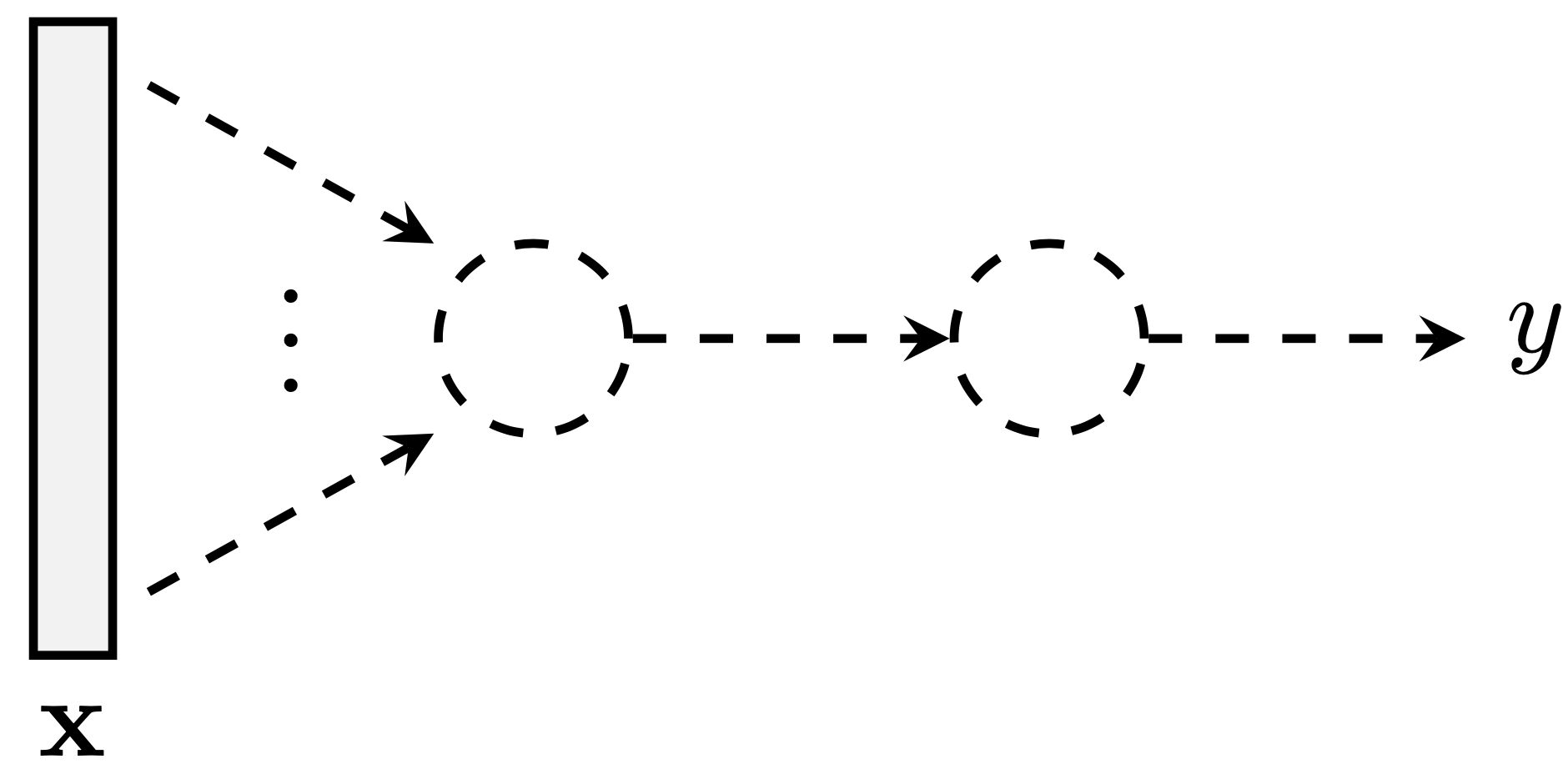}
			\vspace{0.5em}
			\caption*{(i) Network's architecture \vspace{1em}}
		\end{subfigure}
		\hspace{4em}
		\begin{subfigure}[b]{0.3\textwidth}
			\centering
			\includegraphics[width=\linewidth]{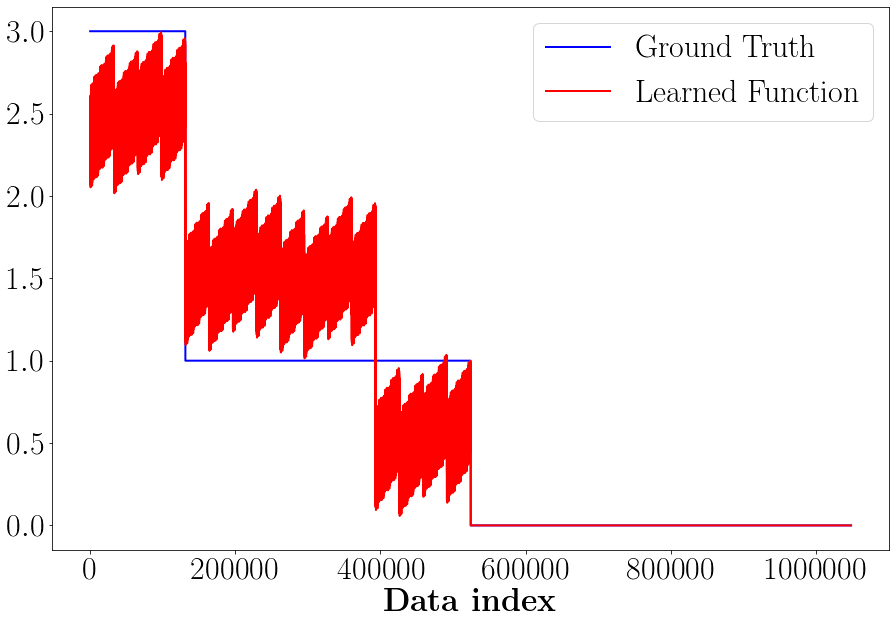}
			\caption*{ {(iv)} Learned function compared to ground truth}
		\end{subfigure}
		\par\medskip
		
		\begin{subfigure}[b]{0.3\textwidth}
			\centering
			\includegraphics[width=\linewidth]{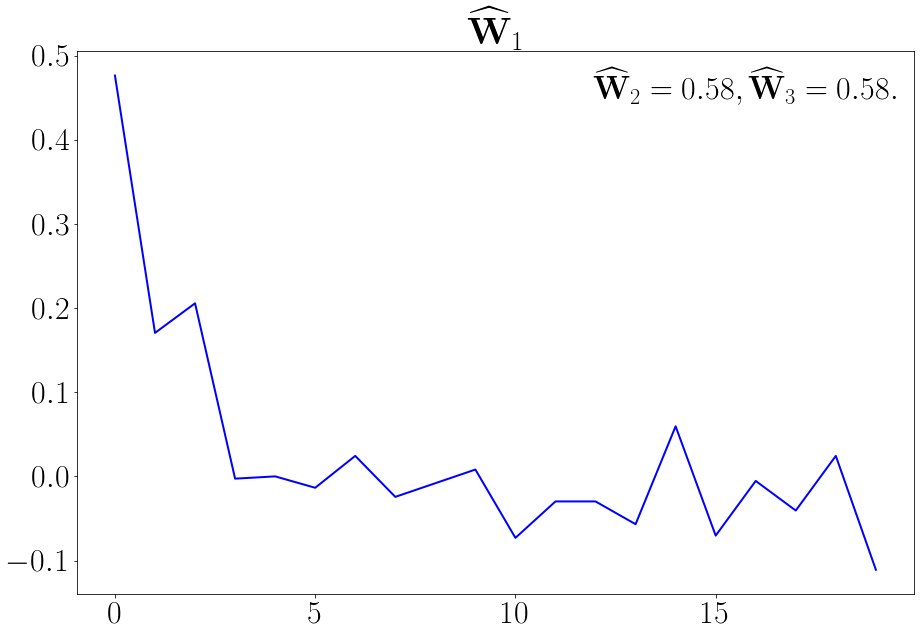}
			\caption*{(ii) KKT point of the NCF}
		\end{subfigure}
		\hspace{4em}
		\begin{subfigure}[b]{0.3\textwidth}
			\centering
			\includegraphics[width=\linewidth]{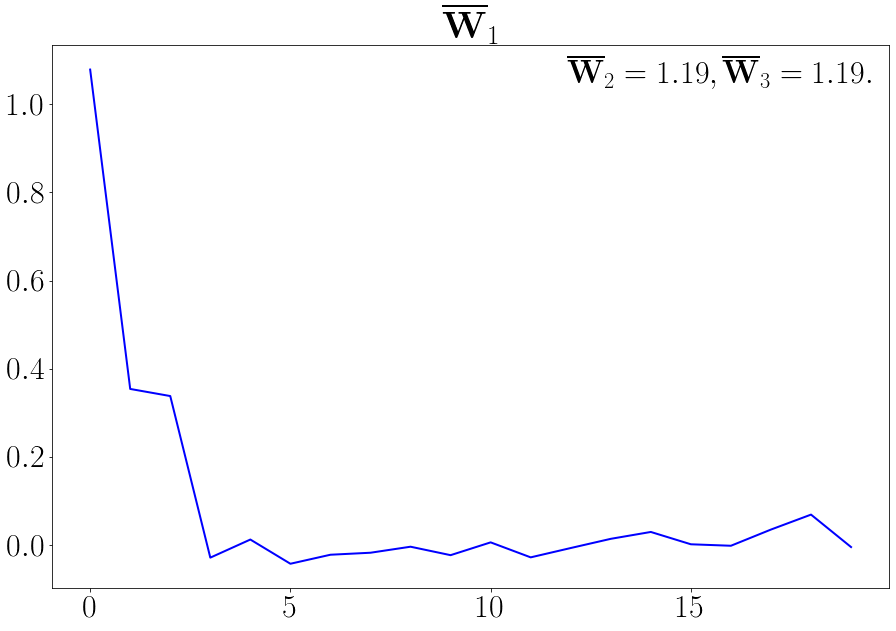}
			\caption*{(iii) Weights of the network}
		\end{subfigure}
		\caption{\textbf{Iteration 1}}
		\label{it1_wt}
	\end{subfigure}
	
	\vspace{1em}
	
	\begin{subfigure}[b]{\textwidth}
		\centering
		\begin{subfigure}[b]{0.3\textwidth}
			\centering
			\includegraphics[width=\linewidth]{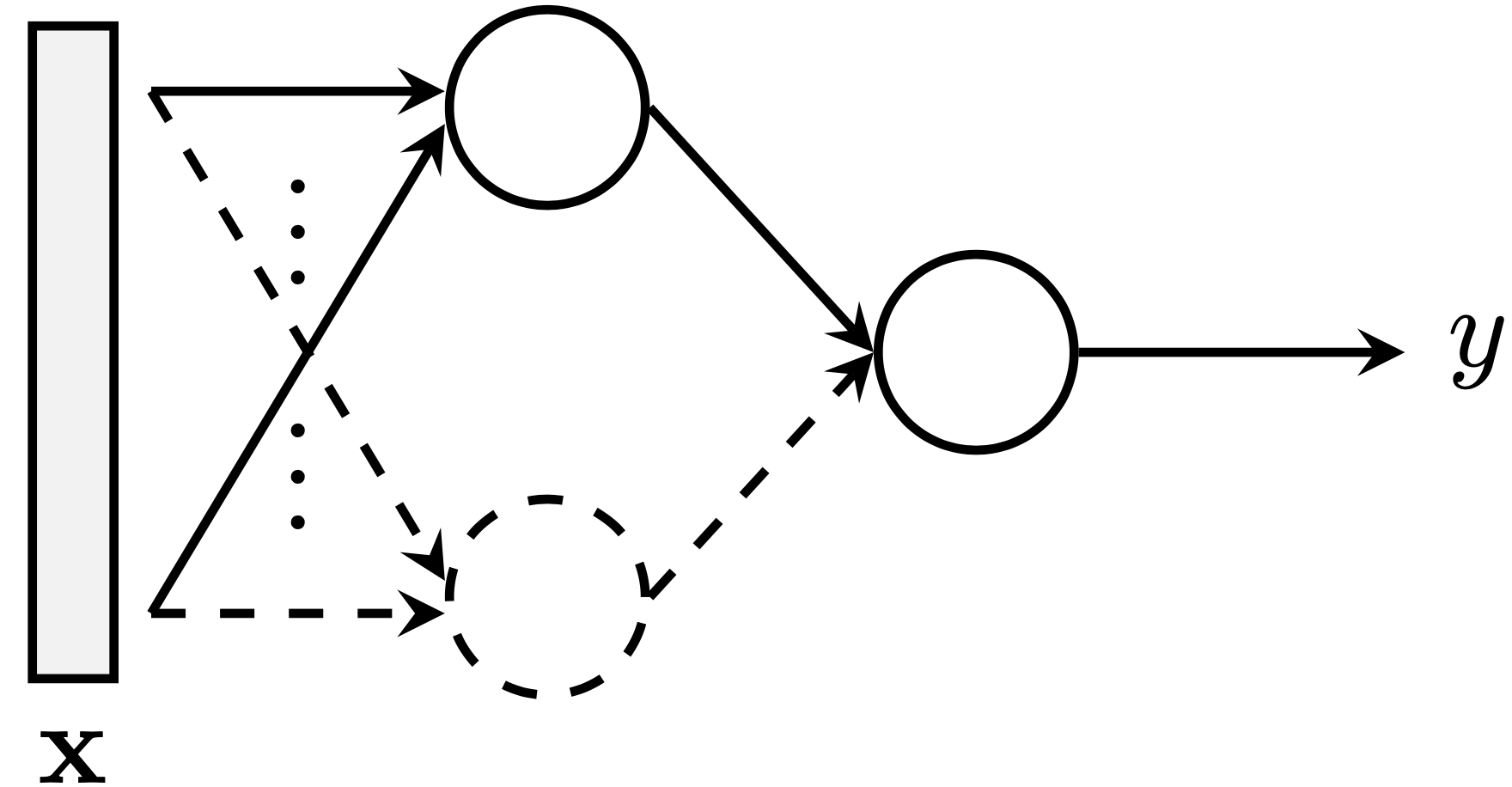}
			\vspace{0.5em}
			\caption*{(i) Network's architecture\vspace{1em}}
		\end{subfigure}
		\hspace{4em}
		\begin{subfigure}[b]{0.3\textwidth}
			\centering
			\includegraphics[width=\linewidth]{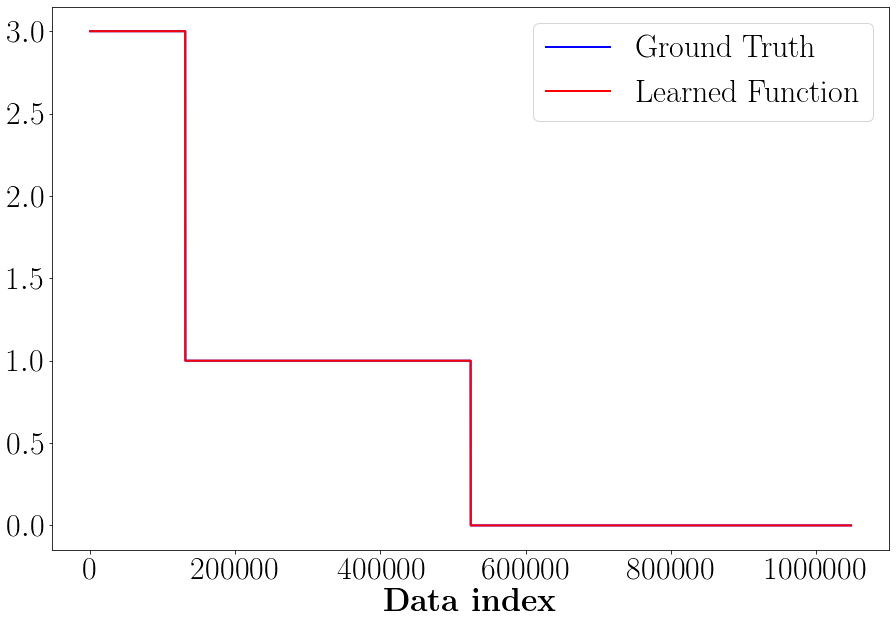}
			\caption*{(iv) Learned function compared to ground truth}
		\end{subfigure}
		\par\medskip
		\begin{subfigure}[b]{0.3\textwidth}
			\centering
			\includegraphics[width=\linewidth]{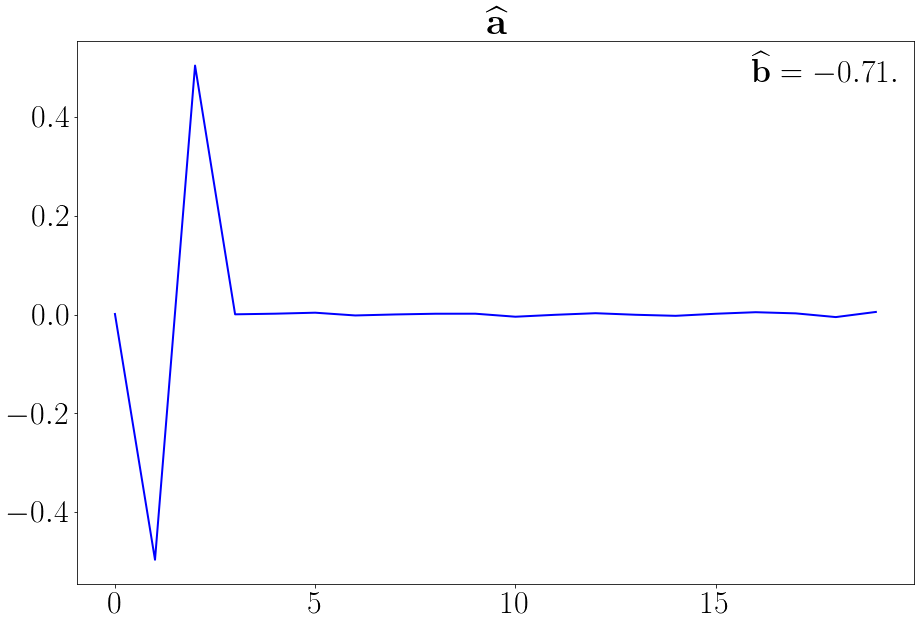}
			\caption*{(ii) KKT point of the NCF}
		\end{subfigure}
		\hspace{4em}
		\begin{subfigure}[b]{0.3\textwidth}
			\centering
			\includegraphics[width=\linewidth]{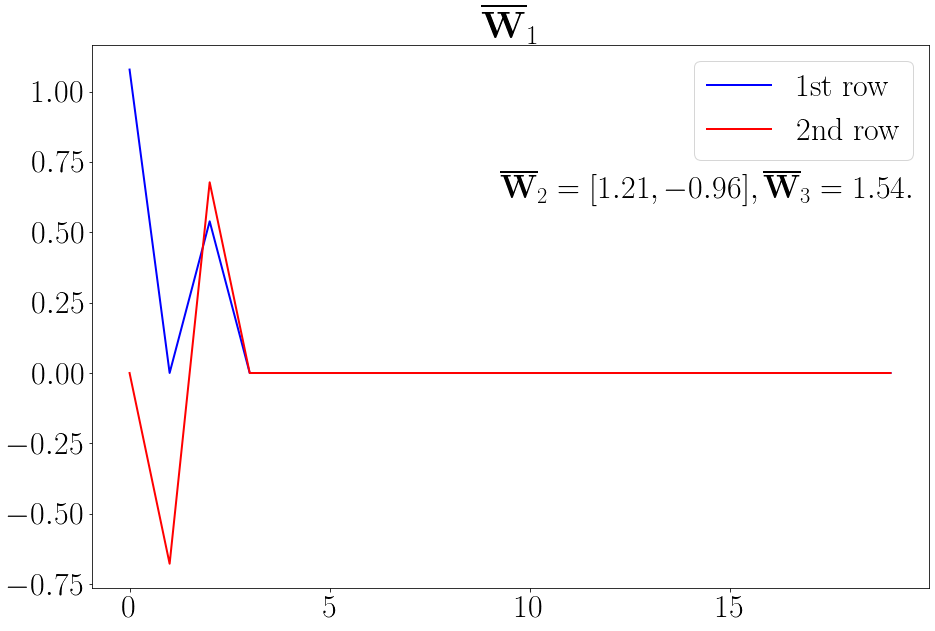}
			\caption*{(iii) Weights of the network}
		\end{subfigure}
		\caption{\textbf{Iteration 2}}
		\label{it2_wt}
	\end{subfigure}
	
	\caption{
		We train a three-layer neural network with ReLU activation using the NP algorithm over the binary hypercube $\{\pm 1\}^{20}$. For each iteration, we depict: $(i)$ the network architecture during that iteration after adding appropriate neuron(s), $(ii)$ the incoming and outgoing weights of the newly added neuron obtained by maximizing the NCF, $(iii)$ the full network weights after minimizing the training loss, and $(iv)$ the learned function at the end of that iteration compared to the ground truth $f(\rvx)$. Dashed circles and arrows indicate neurons and weights introduced in the current iteration, while solid elements represent those from previous iterations.}
	\label{fig:np_ill}
\end{figure}

At \textbf{Iteration 1}, we begin with a single neuron in every layer. We then maximize the constrained NCF, as described in line 4, using $5$ different initializations. The most dominant KKT point, denoted by $(\widehat{\rmW}_1,\widehat{\rmW}_3,\widehat{\rmW}_3)$, is depicted in \Cref{it1_wt}(ii). In $\widehat{\rmW}_1$, the first three entries have larger magnitude than the others, which are small but not negligible. We then minimize the training loss using gradient descent, where the initial weights of the network are small in magnitude and aligned along the most dominant KKT point. The resulting weights are shown in \Cref{it1_wt}(iii). Notably, the first three entries remain dominant, while the remaining weights shrink further. The function learned at the end of this iteration is plotted in \Cref{it1_wt}(iv), alongside the ground truth. To do that, we plot the value of the functions at the each vertex of the binary hypercube, ordered lexicographically from $(1,1,\cdots,1,1), (1,1,\cdots,1,-1)$ to $(-1,-1,\cdots -1,-1)$. So, inflection points in the ground truth occur only if $x_1$, $x_2$, or $x_3$ changes. The oscillations in the learned function are due to the small but non-zero weights on the coordinates other than the first three.

At \textbf{Iteration 2}, we begin by maximizing the constrained NCF with respect to the residual error $\overline{\rvy}$ and $\mathcal{H}_{1,l}$, as described in lines $19-28$, using $5$ different initializations to determine where to add a neuron. The most dominant KKT point $(\widehat{\rva},\widehat{\rvb})$ is depicted in \Cref{it2_wt}(ii), and it corresponds to adding a neuron in the first layer. Accordingly, a neuron is added in the first layer, and its incoming and outgoing weights are small in magnitude but aligned along $\widehat{\rva}$ and $\widehat{\rvb}$, respectively, while the remaining weights are same as they were at the end of iteration 1. We next minimize the training loss using gradient descent, and the resulting weights are depicted in \Cref{it2_wt}(iii). In the first layer, the weights corresponding to coordinates outside ${x_1, x_2, x_3}$ shrink to zero, indicating that the network has correctly identified the active components in $f(\rvx)$. Finally, in \Cref{it2_wt}(iv), we observe that the learned function perfectly matches the ground truth, demonstrating that the NP algorithm has successfully recovered $f(\rvx)$.

\subsection{The Descent Property}
\label{sec:descent_prop}
The NP algorithm attempts to iteratively minimize the training loss. A natural question, then, is whether it can actually reach a local or even a global minimum. Before addressing this, a more fundamental question pertains to whether the training loss decreases after each iteration. This is also known as the descent property of an algorithm, and is often the first step towards establishing convergence guarantees. In what follows, we discuss scenarios under which the descent property will be satisfied.

Each iteration of the NP algorithm consists of two stages: neurons are added to the network with small weights aligned along the most dominant KKT point, and then the training loss is minimized using this augmented network via gradient descent. We first study the impact of adding neurons on the training loss, beginning with the first iteration of the NP algorithm.
\begin{lemma}
	At the end of line 10 of \Cref{np_algo}, suppose $\mathcal{N}_{max} = \rvy^\top\mathcal{H}(\rmX;\widehat{\rmW}_1,\cdots,\widehat{\rmW}_L)>0$; that is, the most dominant KKT point obtained in the first iteration is positive. Then, for all sufficiently small $\delta>0$,
	\begin{align*}
	\mathcal{L}(\mathbf{0},\cdots,\mathbf{0}) > \mathcal{L}(\delta\widehat{\rmW}_1,\cdots,\delta\widehat{\rmW}_L).
	\end{align*}
\end{lemma} 
\begin{proof}
	Let $\zeta\coloneqq\rvy^\top\mathcal{H}(\rmX;\widehat{\rmW}_1,\cdots,\widehat{\rmW}_L)>0$, then, for all sufficiently small $\delta>0$, we have
	\begin{align*}
	\mathcal{L}(\delta\widehat{\rmW}_1,\cdots,\delta\widehat{\rmW}_L) &= \|\rvy - \mathcal{H}(\rmX;\delta\widehat{\rmW}_1,\cdots,\delta\widehat{\rmW}_L)\|_2^2\\
	&= \|\rvy - \delta^L\mathcal{H}(\rmX;\widehat{\rmW}_1,\cdots,\widehat{\rmW}_L)\|_2^2\\
	&=\|\rvy\|_2^2 - 2\delta^L\zeta + \delta^{2L}\|\mathcal{H}(\rmX;\widehat{\rmW}_1,\cdots,\widehat{\rmW}_L)\|_2^2\\
	&\leq \|\rvy\|_2^2 - \delta^L\zeta < \|\rvy\|_2^2 = \mathcal{L}(\mathbf{0},\cdots,\mathbf{0}),
	\end{align*}
	the second equality follows from homogeneity of the neural network, while the first inequality holds if $\delta^L<\zeta/\|\mathcal{H}(\rmX;\widehat{\rmW}_1,\cdots,\widehat{\rmW}_L)\|_2^2$.
\end{proof}
The above lemma shows that if the most dominant KKT point is positive, then initializing the network weights along the most dominant KKT point with sufficiently small norm leads to smaller training loss than at the origin. Importantly, no smoothness assumptions on the activation function are required -- it applies to $\sigma(x) = \max(x,\alpha x)^p$, where $p\in \sN$ and $p\geq 1$.

We now turn to addition of neurons in later iterations.
\begin{lemma}
	At the end of line 28 of \Cref{np_algo}, suppose $\mathcal{N}_{max} = \overline{\rvy}^\top\mathcal{H}_{1l_*}(\rmX;\widehat{\rva},\widehat{\rvb})>0$; that is, the most dominant KKT point is positive. Suppose $\sigma(x) = \max(x,\alpha x)^p$, where if $\alpha= 1$, then $p\geq 1$, otherwise, $p\geq 4$. Then, for all sufficiently small $\delta>0$,
	\begin{align*}
	\mathcal{L}(\overline{\rmW}_1,\cdots,\overline{\rmW}_L) > \mathcal{L}(\widehat{\rmW}_1,\cdots,\widehat{\rmW}_L),
	\end{align*}
	where $(\widehat{\rmW}_1,\cdots,\widehat{\rmW}_L)$ is obtained by adding $(\widehat{\rva},\widehat{\rvb})$ to $(\overline{\rmW}_1,\cdots,\overline{\rmW}_L)$ as in line 29-31.
\end{lemma} 
\begin{proof}
	Since
	\begin{align*}
	\widehat{\rmW}_{{l}_*} = \begin{bmatrix}
	\overline{\rmW}_{{l}_*} \\\delta \widehat{\rva}^\top 
	\end{bmatrix},  \widehat{\rmW}_{{l}_*+1} = \begin{bmatrix}
	\overline{\rmW}_{{l}_*+1} & \delta \widehat{\rvb}
	\end{bmatrix},  \widehat{\rmW}_{l} = \overline{\rmW}_{{l}} \text{ for all } l\notin \{ l_*, l_*+1 \},
	\end{align*}
	using Lemma \ref{h1_proof_poly}, for all sufficiently small $\delta>0$, we get
	\begin{align*}
	\mathcal{H}(\rvx;\widehat{\rmW}_1,\cdots,\widehat{\rmW}_L) &= \mathcal{H}(\rvx;\overline{\rmW}_1,\cdots,\overline{\rmW}_L) + \mathcal{H}_{1,l_*}(\rvx;\delta\widehat{\rvb},\delta\widehat{\rva}) + O(\delta^K)\\
	&=\mathcal{H}(\rvx;\overline{\rmW}_1,\cdots,\overline{\rmW}_L) + \delta^{p+1}\mathcal{H}_{1,l_*}(\rvx;\widehat{\rvb},\widehat{\rva}) + O(\delta^K),
	\end{align*}
	where $K>p+1$. The second equality follows from $(p+1)$-homogeneity of $\mathcal{H}_{1,l_*}$. Define $\zeta\coloneqq\overline{\rvy}^\top\mathcal{H}_{1l_*}(\rmX;\widehat{\rvb},\widehat{\rva})>0$. Using the above equality, we get 
	\begin{align*}
	\mathcal{L}(\widehat{\rmW}_1,\cdots,\widehat{\rmW}_L) &= \|\rvy - \mathcal{H}(\rmX;\widehat{\rmW}_1,\cdots,\widehat{\rmW}_L)\|_2^2\\
	&= \|\rvy - \mathcal{H}(\rmX;\overline{\rmW}_1,\cdots,\overline{\rmW}_L) - \delta^{p+1}\mathcal{H}_{1,l_*}(\rvx;\widehat{\rvb},\widehat{\rva}) - O(\delta^K)\|_2^2\\
	&= \|\overline{\rvy} - \delta^{p+1}\mathcal{H}_{1,l_*}(\rvx;\widehat{\rvb},\widehat{\rva}) - O(\delta^K)\|_2^2\\
	&=\|\overline{\rvy}\|_2^2 - 2\delta^{p+1}\zeta + \delta^{2(p+1)}\|\mathcal{H}_{1l_*}(\rvx;\widehat{\rvb},\widehat{\rva})\|_2^2 + O(\delta^K).
	\end{align*}
	Since $K>p+1$ and $\zeta>0$, for all sufficiently small $\delta>0$, we get
	\begin{equation*}
	\mathcal{L}(\widehat{\rmW}_1,\cdots,\widehat{\rmW}_L) \leq \|\overline{\rvy}\|_2^2 - \delta^{p+1}\zeta < \|\overline{\rvy}\|_2^2  = \mathcal{L}(\overline{\rmW}_1,\cdots,\overline{\rmW}_L),
	\end{equation*}
	which completes the proof
\end{proof}
Here as well, if the most dominant KKT point is positive, then adding neurons with sufficiently small incoming and outgoing weights and aligned with the most dominant KKT point leads to a strict decrease in training loss.  The extra condition $p\geq 4$ when $\alpha\neq 1$ is required to use Lemma \ref{h1_proof_poly}.

The above discussion implies that if the most dominant KKT point is positive, then the training loss decreases in the first stage. Next, consider the second stage of each iteration, where the training loss is minimized via gradient descent. It is well known that for sufficiently small step sizes, gradient descent reduces the loss when the objective has a Lipschitz-continuous gradient. However, loss functions for neural networks do not have globally Lipshitz gradient, and analyzing gradient descent in this setting remains an active area of research. Nevertheless, if the gradient descent updates do in fact reduce the training loss, then combining this with the above discussion of  the first stage above shows that the NP algorithm satisfies the descent property: the training loss decreases after each iteration.

\subsection{Numerical Experiments}
\label{sec:exp_NP}
We conduct experiments to evaluate the Neuron Pursuit (NP) algorithm and compare its performance with neural networks trained by gradient descent with small initialization.
\subsubsection{Non-linear Sparse Functions}
\label{sp_func_exp}
We attempt to learn non-linear sparse functions, functions which can be represented using a finite number of neurons, over different input distributions. We train deep neural networks using the NP algorithm and gradient descent (GD) with small initialization. For GD, we use 50 neurons per layer and train up to a maximum of $3\times 10^6$ iterations. For the NP algorithm, number of iterations is capped at 31. The performance is evaluated in terms of relative training and test error:
\begin{center}
	\begin{tabular}{ll}
		$\bullet$ Training error $\coloneqq \frac{\|f(\rmX_{\text{train}})-\hat f(\rmX_{\text{train}})\|_2}{\|f(\rmX_{\text{train}})\|_2}$,
		& $\bullet$ Test error $\coloneqq \frac{\|f(\rmX_{\text{test}})-\hat f(\rmX_{\text{test}})\|_2}{\|f(\rmX_{\text{test}})\|_2}$,
	\end{tabular}
\end{center}
where $\rmX_{\text{train}}$ and $\rmX_{\text{test}}$ denotes the training and test data, respectively, $f(\cdot)$ denotes the ground truth function, and $\hat{f}(\cdot)$ denotes the learned function. These normalized errors ensure fair comparison across varying training sample sizes. All algorithms are run until the training error is less than $0.001$ or the maximum number of iterations is exceeded. The number of training samples varies, while the test set size is fixed at $10^5$.  The specific hyperparameters used during training for each algorithm and input distribution is stated in \Cref{sp_exp} of the Appendix. The results reported below are averaged over 20 independent runs.\\
\noindent\textbf{Hypersphere.}  We train three-layer neural networks with activation function $\sigma(x) = \max(0, x)$ to learn two target functions: 
\begin{align*}
&(i) \ f_1(\rvx) =\sigma(\sigma(2x_1+x_2) - \sigma(x_3-x_4)), \text{ and }\\
&(ii)\ f_2(\rvx) = \sigma(\sigma(2x_1+x_2) - \sigma(x_3-x_4)) + 5\sum_{i=1}^2i\sigma(\sigma(2x_{4i+1}+x_{4i+2}) - \sigma(x_{4i+3}-x_{4i+4})),
\end{align*}
where $x_i$ denotes the $i$th coordinate of $\rvx \in \sR^{50}$. The training and test samples are drawn uniformly at random from $\sqrt{d}\sS^{d-1}$, where $d=50$, that is the hypersphere of radius $\sqrt{50}$. Here, $f_1(\rvx)$ is a three-layer neural network with two neurons in the first layer and one neuron in the second layer, and $f_2(\rvx)$ contains six neurons in the first layer and three neurons in the second layer. The above choice of functions is inspired from the notion of \emph{hierarchical} functions studied in \citet{poggio_hier,dandi_hier}.

\Cref{fig:sp_sphere} depicts the test and training errors (top row), and the number of iterations required by the NP algorithm for convergence (bottom row). For $f_1(\rvx)$, the NP algorithm achieves near-zero training error across all sample sizes. As the number of training samples increases, the test error decreases, and for sufficiently large samples sizes, it becomes nearly zero\textemdash indicating that the NP algorithm successfully learns the target function $f_1(\rvx)$. The number of iterations required by the NP algorithm for convergence displays a non-monotonic behavior: it increases as sample size grows, reaches a peak, and then decreases and remains stable. This suggests that in the intermediate sample regime, the NP algorithm finds it comparatively harder to fit the training data than in low-sample or high-sample regimes.

The overall behavior for $f_2(\rvx)$ is qualitatively similar: the NP algorithm achieves near-zero test error when the number of training samples is large, and the number of iterations behaves non-monotonically. However, learning $f_2(\rvx)$ requires more training samples, which is expected, as $f_2(\rvx)$ has more neurons and thus more complex. When compared to GD with small initialization, the NP algorithm performs better on $f_1(\rvx)$, as it requires fewer samples to learn. For $f_2(\rvx)$, GD performs better than NP. The training error for GD is small for all number of training samples, and it is not plotted to avoid clutter.\\
\begin{figure}
	\centering
	
	\begin{subfigure}[b]{0.4\textwidth}
		\centering
		\includegraphics[width=\linewidth]{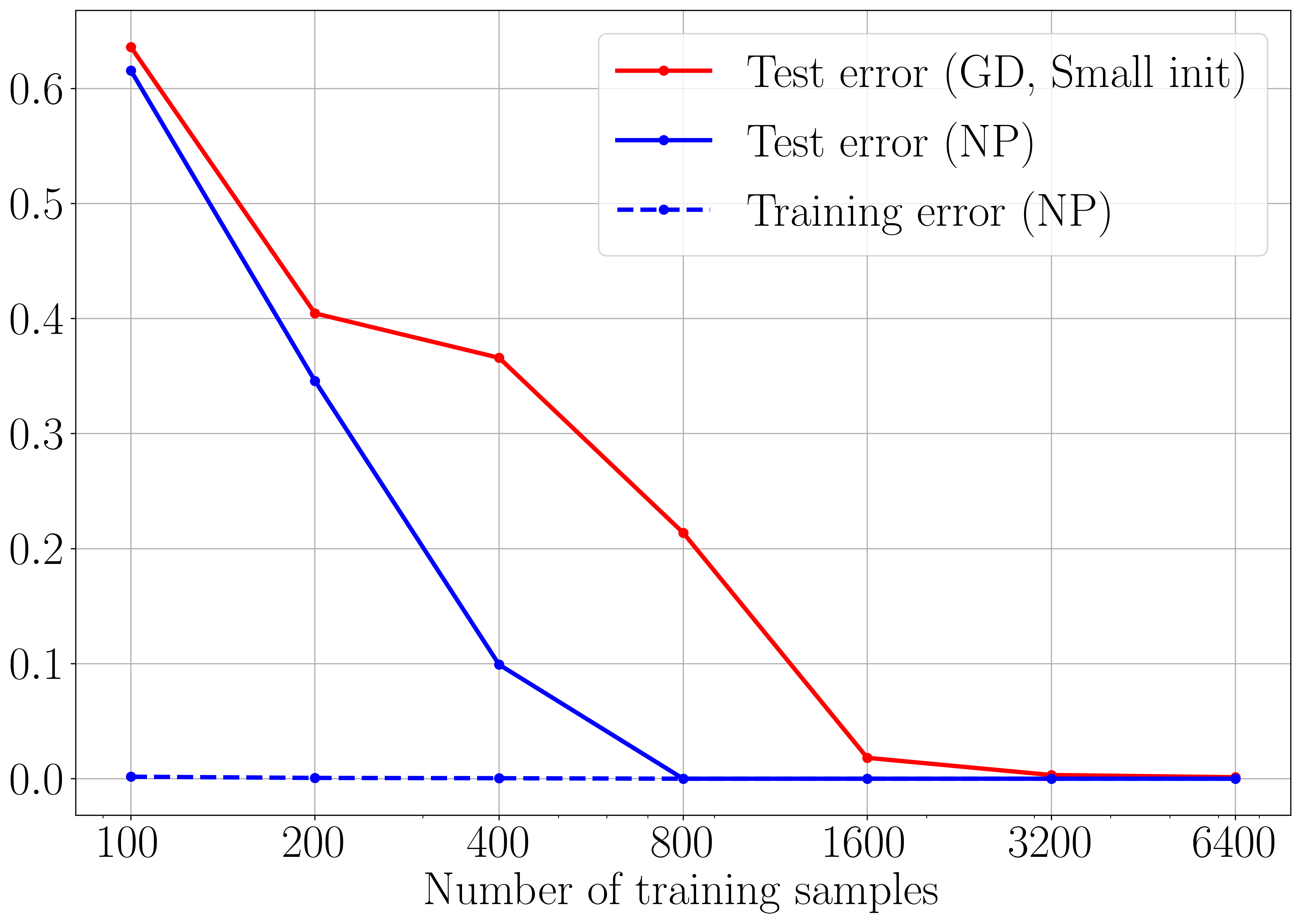}
	\end{subfigure}
	\hfill
	\begin{subfigure}[b]{0.4\textwidth}
		\centering
		\includegraphics[width=\linewidth]{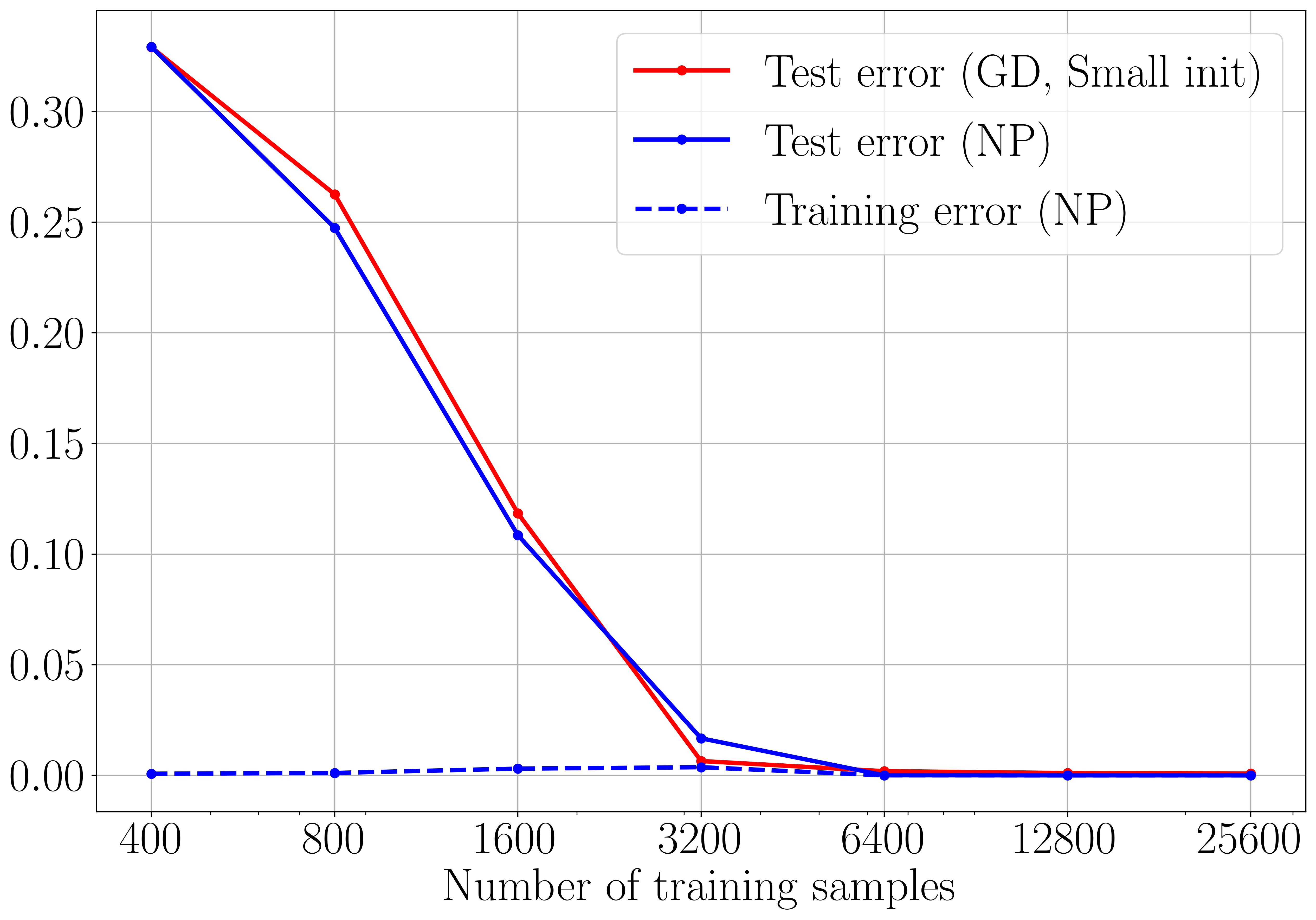}
	\end{subfigure}
	
	\vspace{1em}
	
	\begin{subfigure}[b]{0.4\textwidth}
		\centering
		\includegraphics[width=\linewidth]{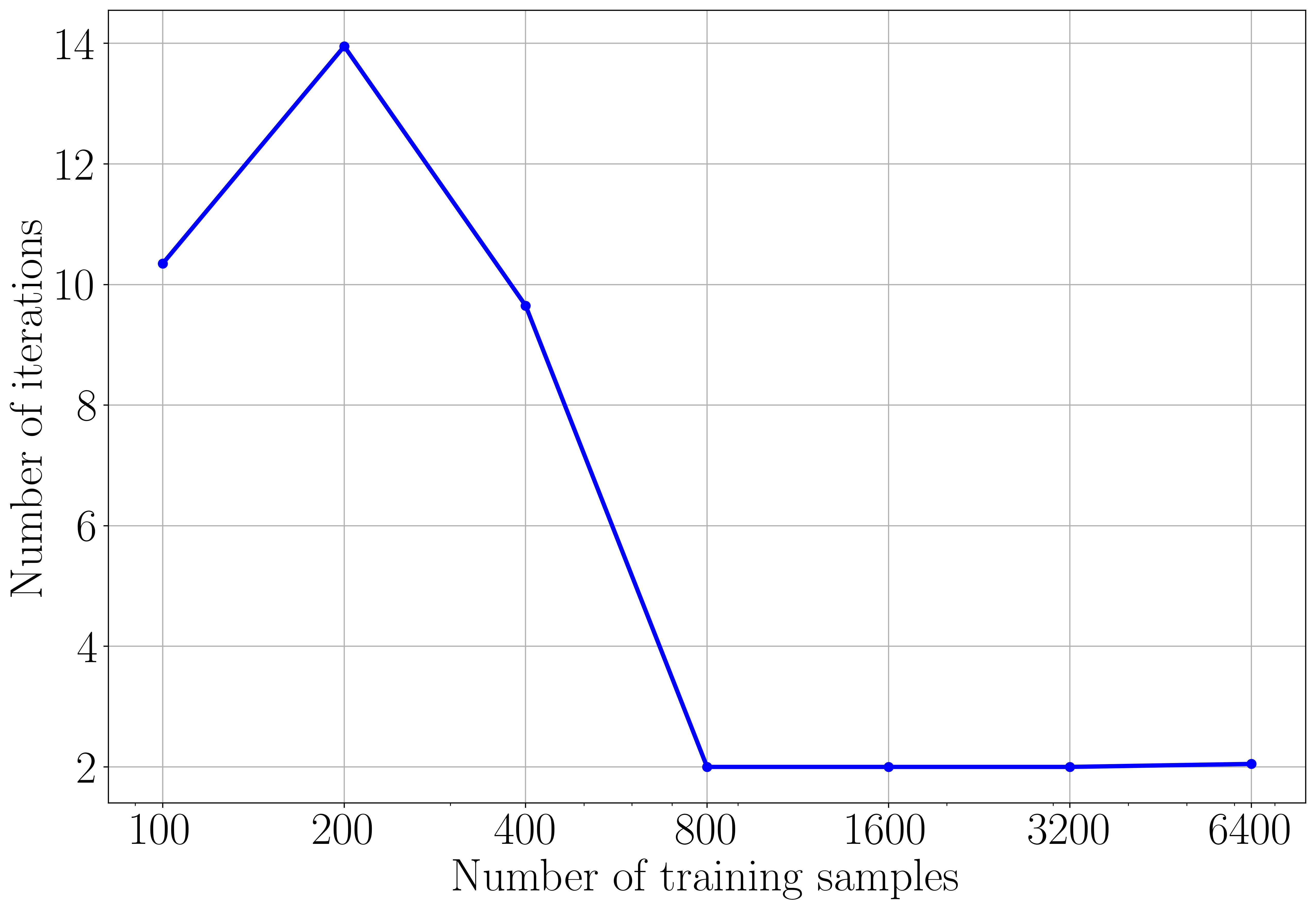}
		\caption{$f_1(\rvx)$}
	\end{subfigure}
	\hfill
	\begin{subfigure}[b]{0.4\textwidth}
		\centering
		\includegraphics[width=\linewidth]{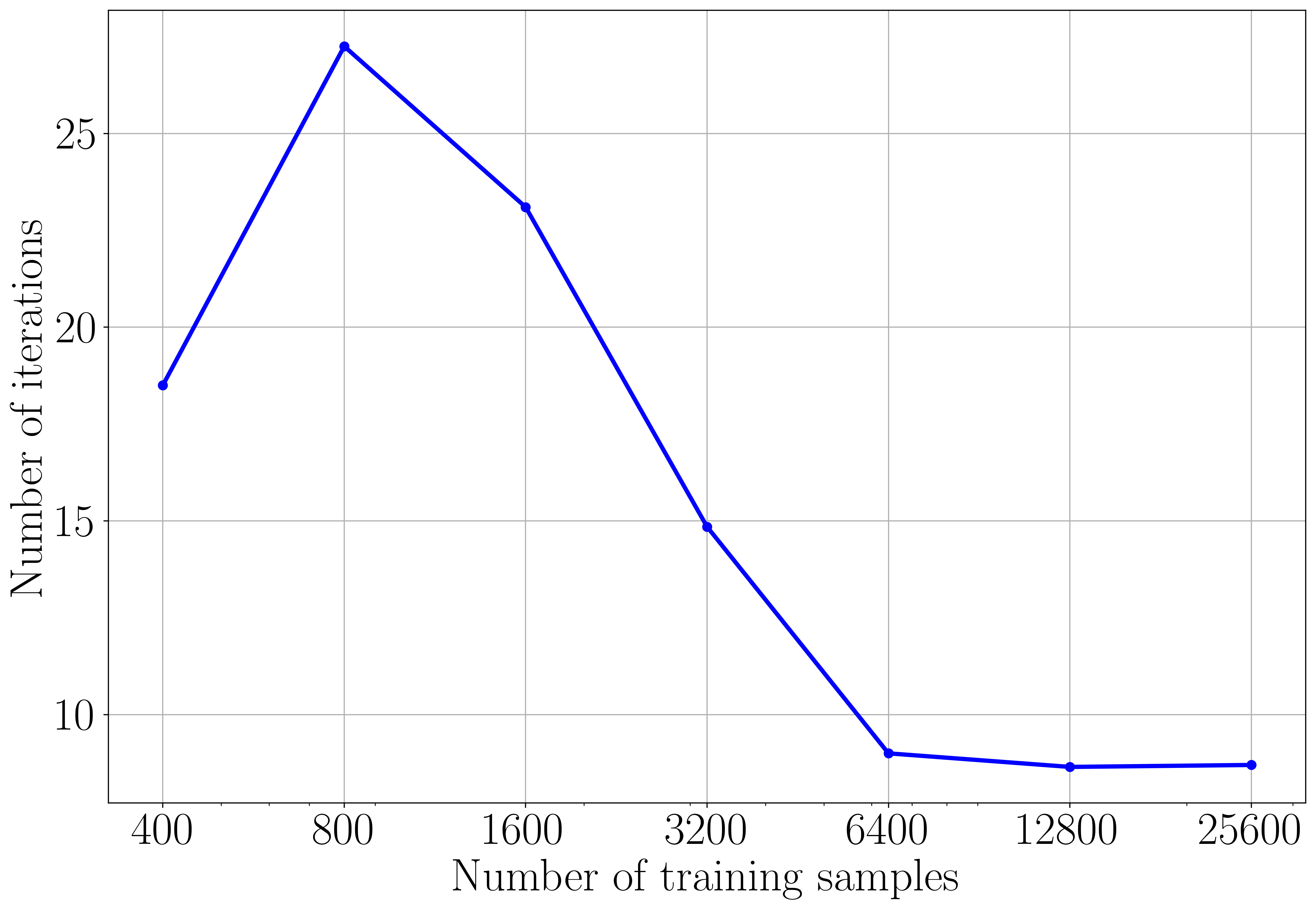}
		\caption{$f_2(\rvx)$}
	\end{subfigure}
	
	\caption{
		(Hypersphere) We train three-layer neural networks with activation function $\sigma(x) = \max(x, 0)$ to learn $f_1(\rvx)$ (left) and $f_2(\rvx)$ (right). The network is trained using the NP algorithm for a maximum of 31 iterations, and using gradient descent (GD) with 50 neurons per layer for up to $3 \times 10^6$ iterations. For GD, the initialization is  small (the weights are sampled uniformly at random from hypersphere of radius $0.01$). The top row shows the training and test errors of various algorithms as the number of training samples increases. The bottom row shows the number of iterations required by the NP algorithm to successfully fit the training data. As expected, performance improves for both methods with more training samples. For $f_1(\rvx)$, the NP algorithm achieves small test error with fewer samples. For $f_2(\rvx)$, GD performs better than the NP algorithm.
	}
	\label{fig:sp_sphere}
\end{figure}
\textbf{Hypercube.} We train three-layer neural networks with activation $\sigma(x) = \max(0.5x, x)$\footnote{For ReLU activation, the constrained NCF becomes zero at a certain iteration, when learning $g_1(\rvx)$. As discussed at the end of \Cref{sec:main_results}, for such cases our theoretical results and the NP algorithm are not applicable. We do not face this issue with Leaky ReLU activation.}  to learn two target functions: 
\begin{equation*}
(i) \ g_1(\rvx) = x_1x_2-x_1x_2x_3x_4, \text{ and } (ii)\ g_2(\rvx) = x_1x_2x_3x_4,
\end{equation*}
where the training and test samples are drawn uniformly at random from the binary hypercube $\{\pm1\}^{50}$. The  learning dynamics of neural networks for such functions has been extensively studied in recent works \citep{abbe_sgd,abbe_msp,suzuki_par}.
\begin{figure}[t]
	\centering
	
	\begin{subfigure}[b]{0.4\textwidth}
		\centering
		\includegraphics[width=\linewidth]{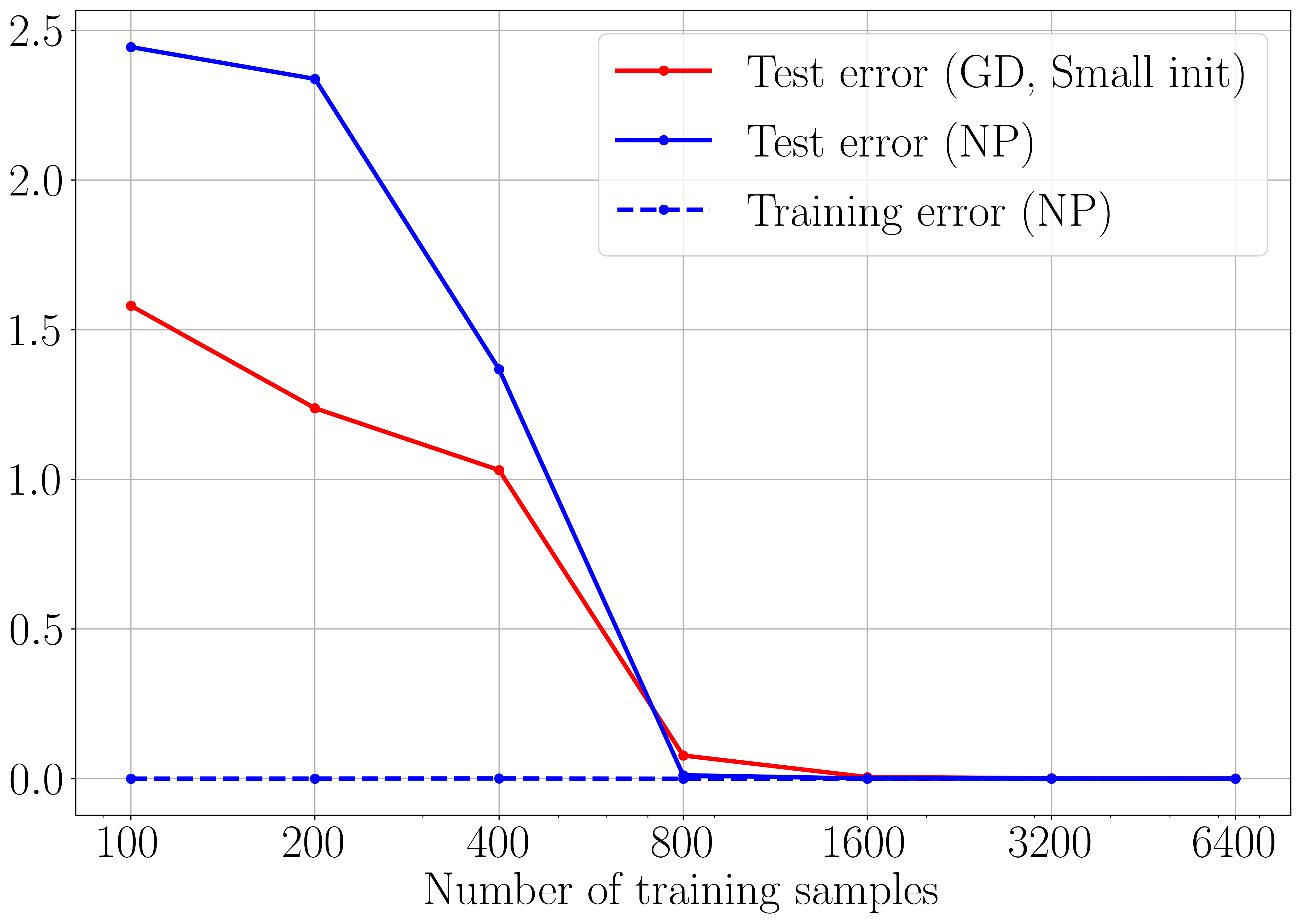}
	\end{subfigure}
	\hfill
	\begin{subfigure}[b]{0.4\textwidth}
		\centering
		\includegraphics[width=\linewidth]{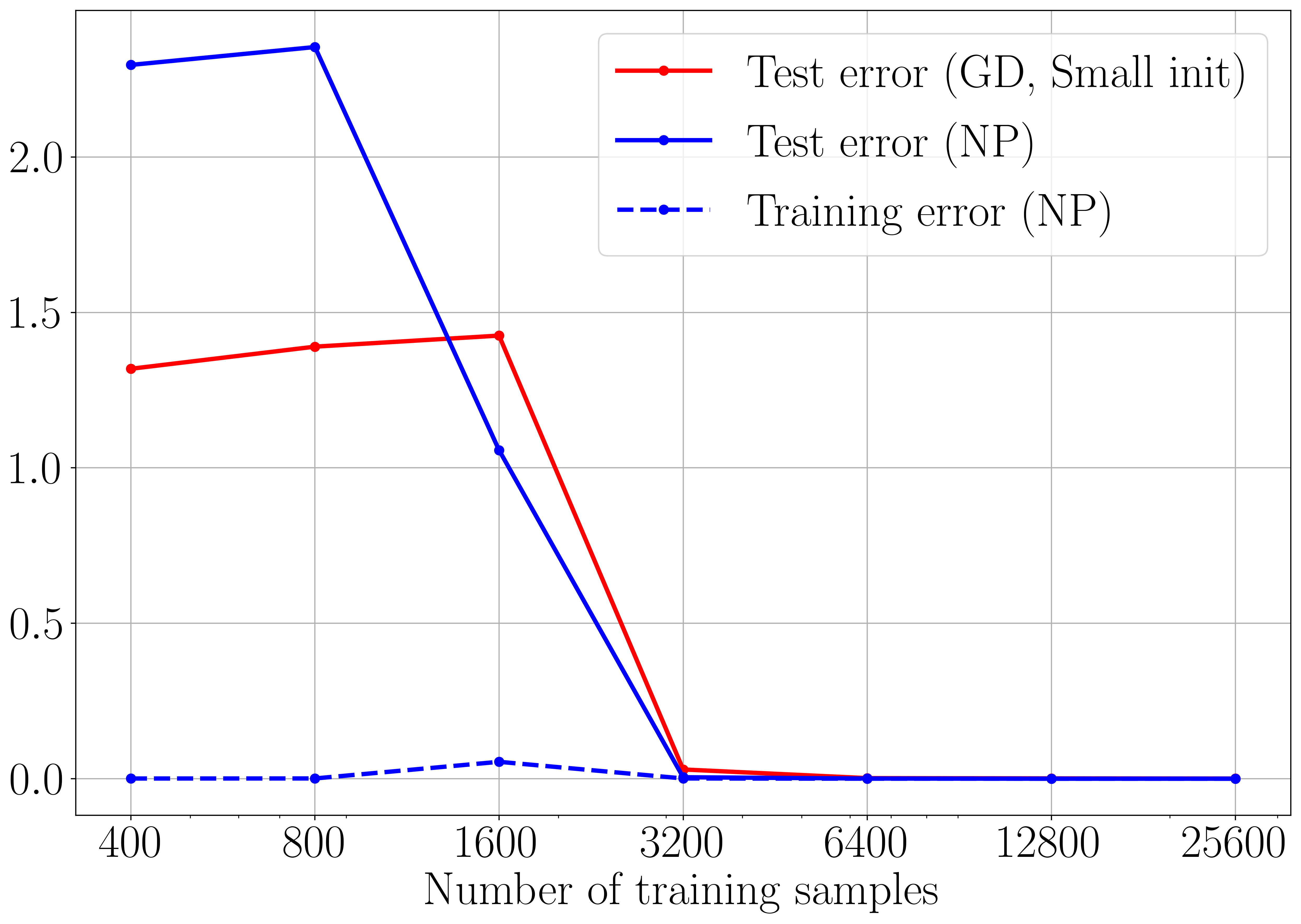}
	\end{subfigure}
	
	\vspace{1em}
	
	\begin{subfigure}[b]{0.4\textwidth}
		\centering
		\includegraphics[width=\linewidth]{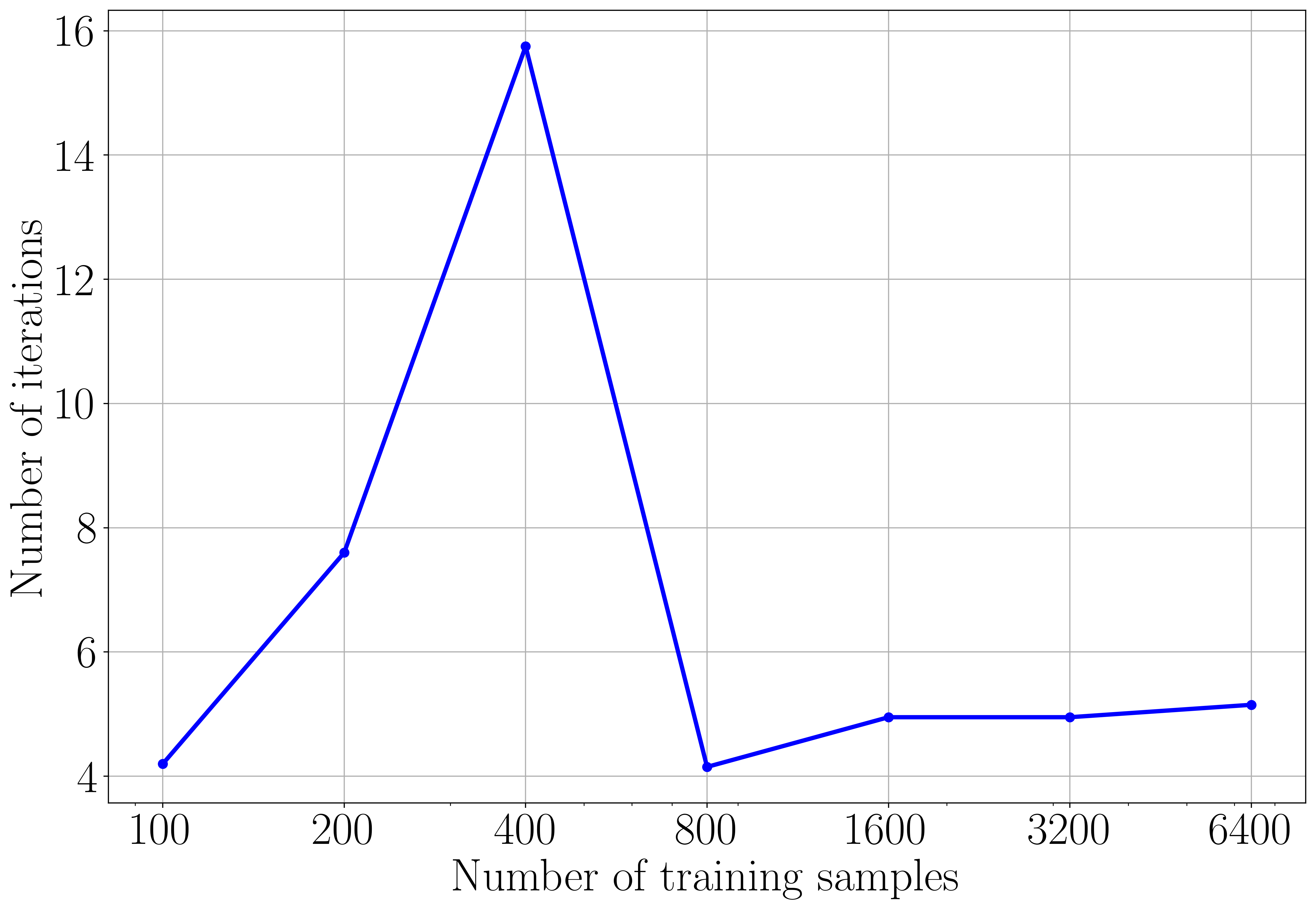}
		\caption{$g_1(\rvx)$}
	\end{subfigure}
	\hfill
	\begin{subfigure}[b]{0.4\textwidth}
		\centering
		\includegraphics[width=\linewidth]{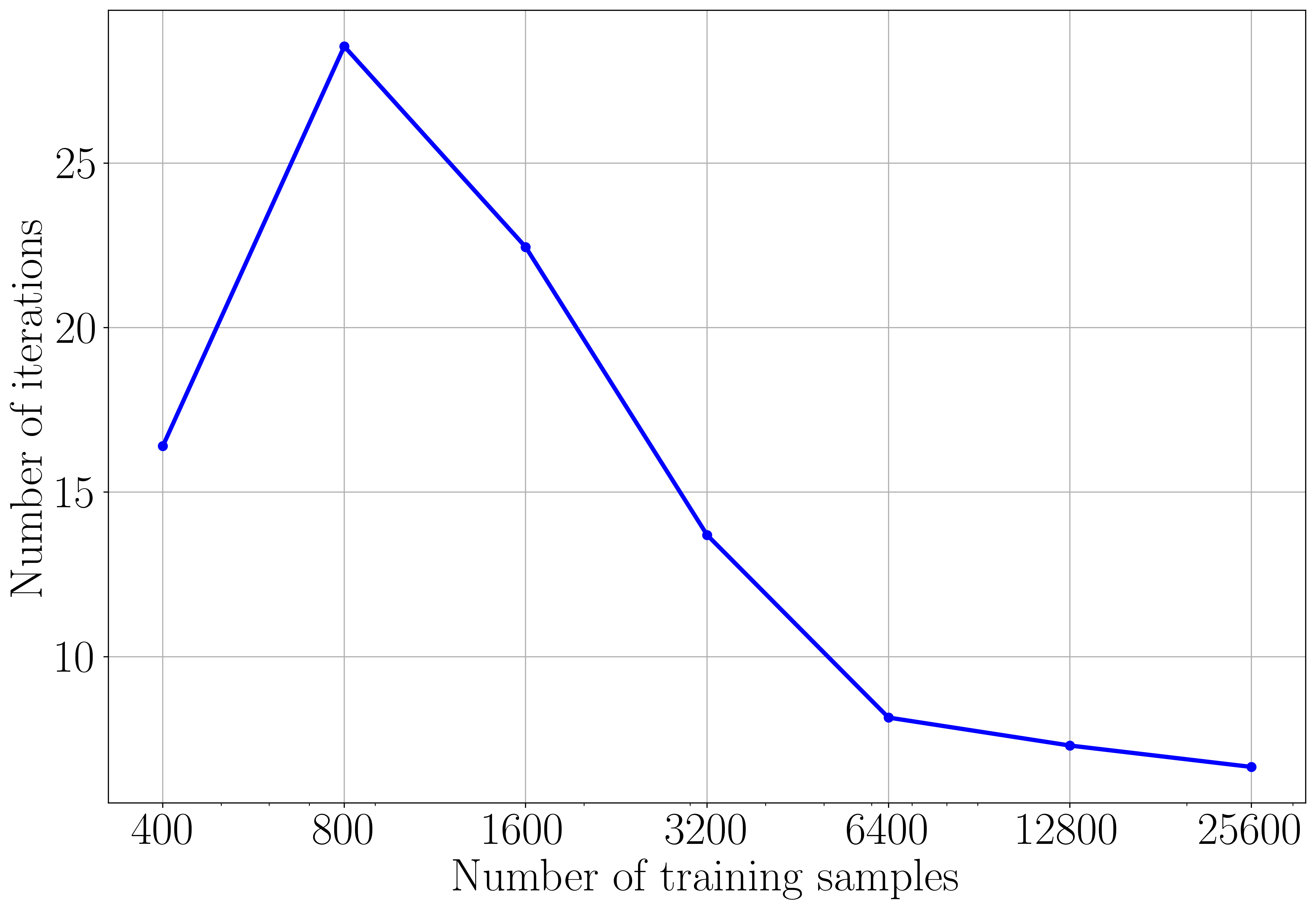}
		\caption{$g_2(\rvx)$}
	\end{subfigure}
	
	\caption{
		(Hypercube) We train three-layer neural networks with activation function $\sigma(x) = \max(x, 0.5x)$ to learn $g_1(\rvx)$ (left) and $g_2(\rvx)$ (right). The network is trained using the NP algorithm for a maximum of 31 iterations, and using gradient descent (GD) with 50 neurons per layer for up to $3 \times 10^6$ iterations. For GD, the initialization is  small (the weights are sampled uniformly at random from hypersphere of radius $0.01$). The top row shows the training and test errors of various algorithms as the number of training samples increases. The bottom row shows the number of iterations required by the NP algorithm to successfully fit the training data.  For $g_2(\rvx)$, both algorithms achieve small test error with a similar number of samples. For $g_1(\rvx)$, the NP algorithm performs slightly better than GD.
	}
	\label{fig:sp_cube}
\end{figure}

\Cref{fig:sp_cube} depicts the test and training errors (top row), and the number of iterations required by the NP algorithm to fit the training data (bottom row). Similar to the above hypersphere experiments, the NP algorithm achieves near-zero test error for sufficiently large number of samples, and number of iterations required by NP exhibits the same non-monotonic behavior. Notably, for $g_2(\rvx)$ with 1600 training samples, the NP algorithm fails to achieve small training error in many instances within 31 iterations, which is a sufficient number of iterations when learning with more or fewer samples. This again demonstrates that fitting the training data using the NP algorithm is comparatively harder in the intermediate sample regime than low-sample or high-sample regimes.

Compared to GD, the NP algorithm performs similarly on $g_2(\rvx)$ but has slightly better  performance on $g_1(\rvx)$. It is also worth noting that learning $g_1(\rvx)$ requires fewer samples than $g_2(\rvx)$, for both the algorithms, where $g_1(\rvx)$ is a sum of second- and fourth-order polynomial and $g_2(\rvx)$ is a fourth-order polynomial. For neural networks trained via gradient descent, this behavior has been attributed to the higher \emph{leap exponent} of $g_2(\rvx)$ \citep{abbe_msp}. It seems neural networks trained via the NP algorithm also exhibit similar behavior.  \\
\textbf{Gaussian.} We train four-layer neural networks with activation function $\sigma(x) = \max(0.5x, x)$ to learn two target functions: 
\begin{align*}
&(i) \ h_1(\rvx) =\max(2x_1,x_2),\text{ and } \\
&(ii) \ h_2(\rvx) = \max(x_1,2x_2) + \max(x_3,2x_4) - \max(x_1+x_3,-x_2,-x_4),
\end{align*}
where the training and test samples are sampled from $\mathcal{N}(\mathbf{0},\mathbf{I}_{50})$. Since maximum of linear functions is piecewise linear, these functions can be represented using finitely many neurons. 

\Cref{fig:sp_gauss} depicts the test and training errors (top row), and the number of iterations required by the NP algorithm to fit the training data (bottom row). The NP algorithm achieves small test error for large numbers of samples, and the number of iterations required evolves non-monotonically. For $h_1(\rvx)$, GD performs slightly worse than the NP algorithm, whereas GD performs slightly better for $h_2(\rvx)$. Also, for $h_2(\rvx)$ with 1600 training samples, the NP algorithm fails to achieve small training error in many instances within 31 iterations, which is a sufficient number of iterations when learning with more or fewer samples.
\subsubsection{Algorithmic Tasks}
We next test efficacy of the proposed method on some algorithmic tasks. Our goal here is to demonstrate that the NP algorithm is able to successfully perform this task.\\
\textbf{Modular Addition: } We consider the task of learning modular addition using a two-layer neural network with square activation function, $\sigma(x) = x^2$. The goal is to map inputs $a, b \in \{0, 1, \dots, p-1\}$ to the output $(a + b) \bmod p$, where we set $p = 59$. Following prior works \citep{gromov_mod,nanda_mod,zhong_mod}, we formulate this as a $p$-class classification problem, representing both inputs and outputs as one-hot vectors. For a given input pair $(a, b)$, the network input is constructed as $[\mathbf{1}_a, \mathbf{1}_b, 1] \in \mathbb{R}^{2p+1}$, where $\mathbf{1}_z$ denotes the one-hot encoding of $z \in \{0, 1, \dots, p-1\}$. The final coordinate is fixed to $1$, serving as an explicit bias term for the first layer. The target label is represented as $\mathbf{1}_c$, where $c=(a+b)\text{ mod }p$. 

We train the model using the squared loss. The training set consists of $1600$ input pairs, sampled uniformly at random from the $59\cdot 59$ possible pairs, with the remainder forming the test set. During prediction, we select the class corresponding to the maximum output coordinate. The performance is measured using classification error on training and test set, which is defined as fraction of samples for which the network’s prediction does not match the target label.
\begin{figure}[t]
	\centering
	
	\begin{subfigure}[b]{0.4\textwidth}
		\centering
		\includegraphics[width=\linewidth]{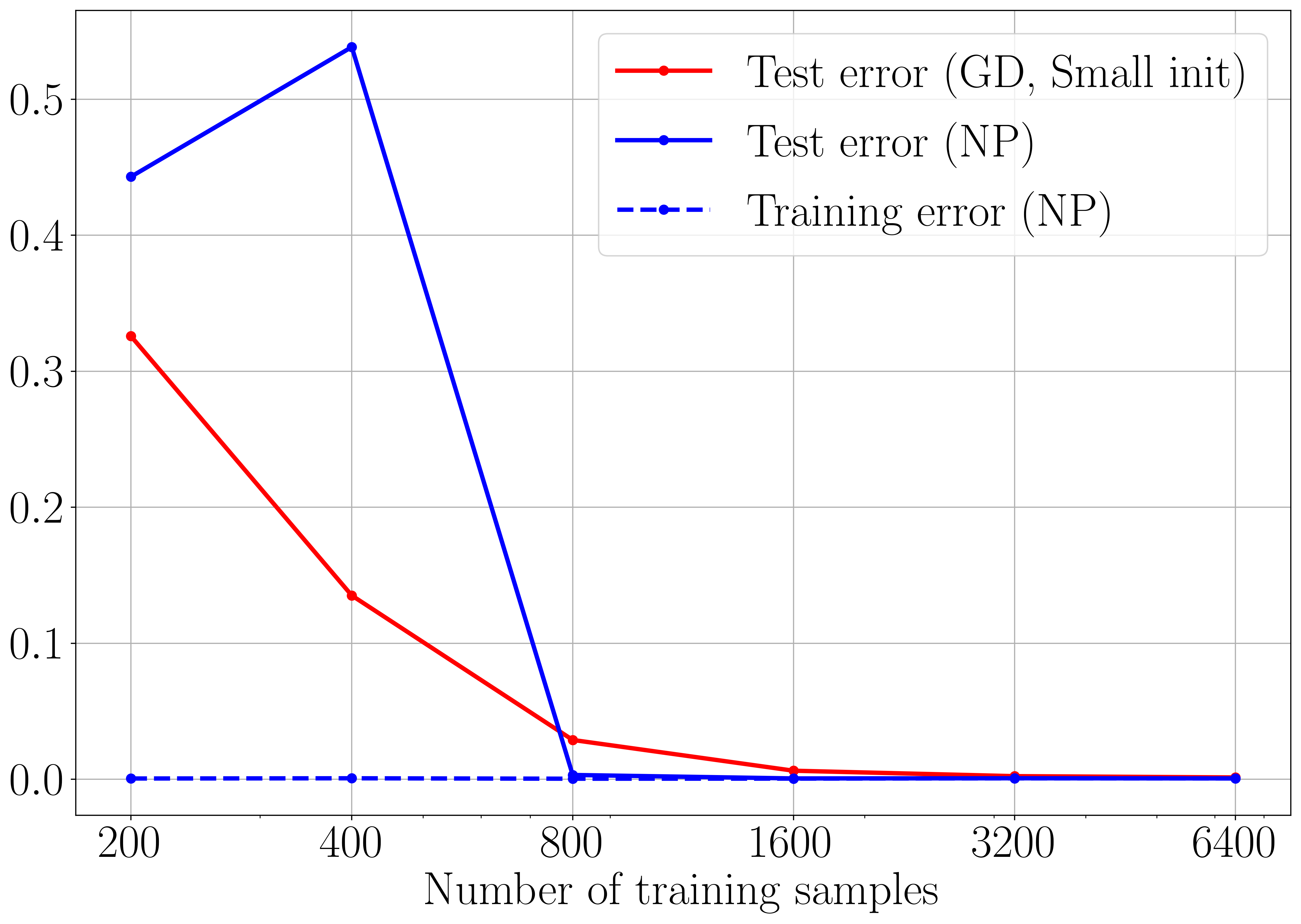}
	\end{subfigure}
	\hfill
	\begin{subfigure}[b]{0.4\textwidth}
		\centering
		\includegraphics[width=\linewidth]{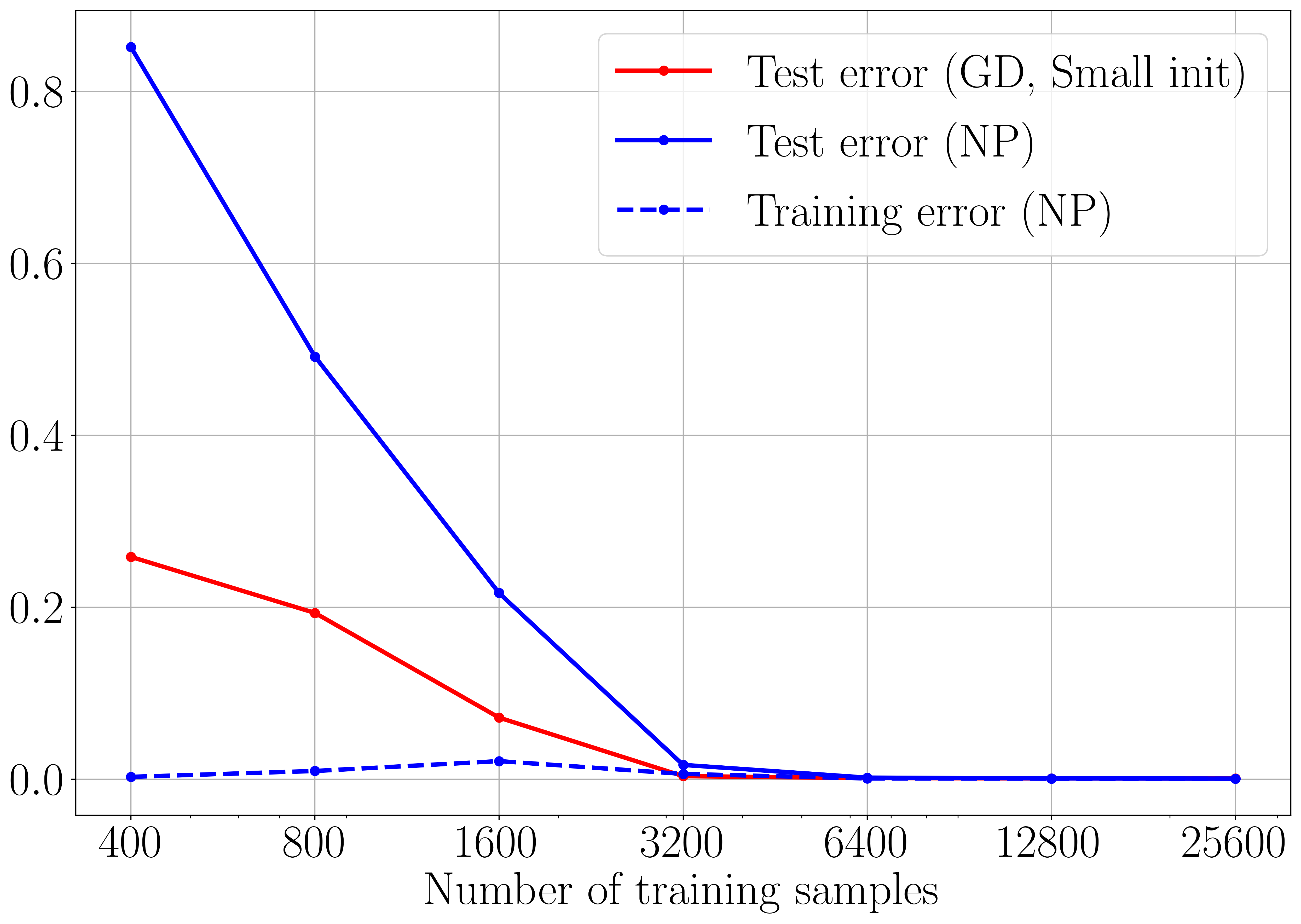}
	\end{subfigure}
	
	\vspace{1em}
	
	\begin{subfigure}[b]{0.4\textwidth}
		\centering
		\includegraphics[width=\linewidth]{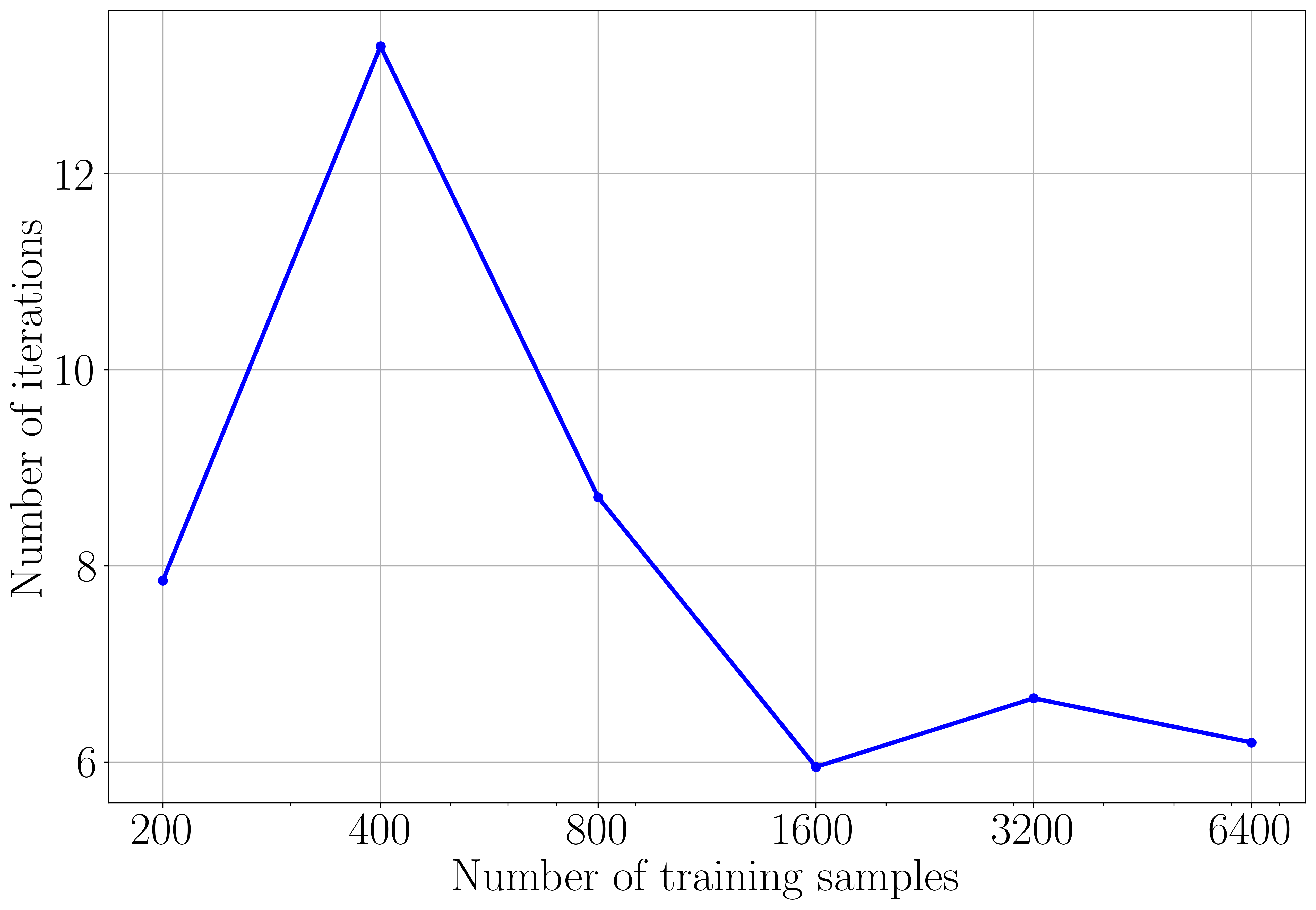}
		\caption{$h_1(\rvx)$}
	\end{subfigure}
	\hfill
	\begin{subfigure}[b]{0.4\textwidth}
		\centering
		\includegraphics[width=\linewidth]{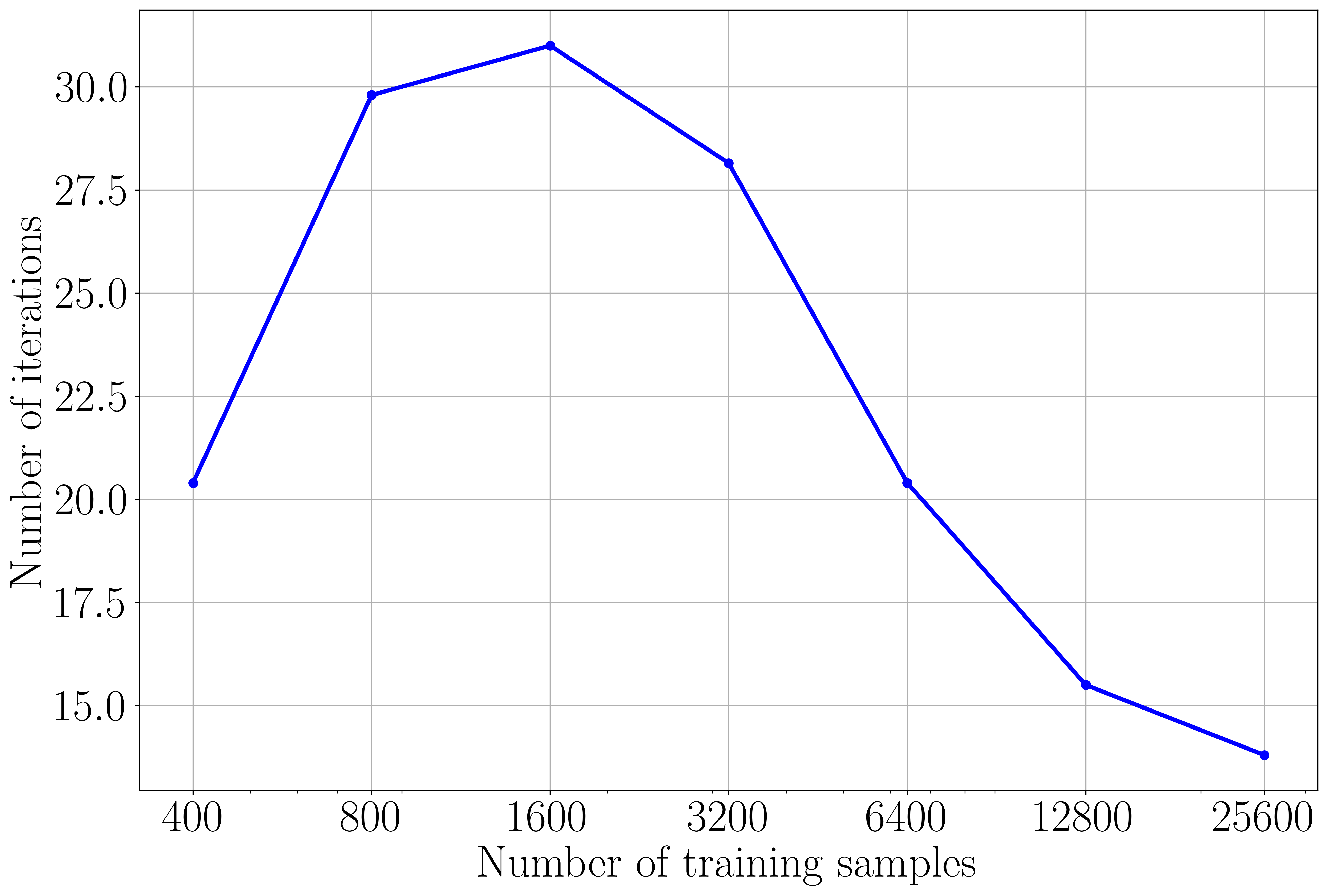}
		\caption{$h_2(\rvx)$}
	\end{subfigure}
	
	\caption{
		(Gaussian) We train four-layer neural networks with activation function $\sigma(x) = \max(x, 0.5x)$ to learn $h_1(\rvx)$ (left) and $h_2(\rvx)$ (right). The network is trained using the NP algorithm for a maximum of 31 iterations, and using gradient descent (GD) with 50 neurons per layer for up to $3 \times 10^6$ iterations. For GD, the initialization is  small (the weights are sampled uniformly at random from hypersphere of radius $0.1$). The top row shows the training and test errors of various algorithms as the number of training samples increases. The bottom row shows the number of iterations required by the NP algorithm to successfully fit the training data.  For $h_1(\rvx)$, performance of GD is slightly worse. For $h_2(\rvx)$, both algorithms require same number of  samples to achieve small test error.
	}
	\label{fig:sp_gauss}
\end{figure}

Training is performed using the NP algorithm until the training classification error reaches zero. The experiment is repeated over three independent runs; here, we present the results from one representative run, with the remaining results given in \Cref{mod_add_exp} of the Appendix along with other specifications of the algorithm. \Cref{mod_add_loss_evol} depicts the evolution of training and test classification errors with respect to the iterations. The algorithm converges in $35$ iterations, with the test error being small, indicating that the network successfully learns the modular addition function.
\begin{figure}[htbp]
	\centering
	\begin{subfigure}{0.4\textwidth}
		\centering
		\includegraphics[width=\linewidth]{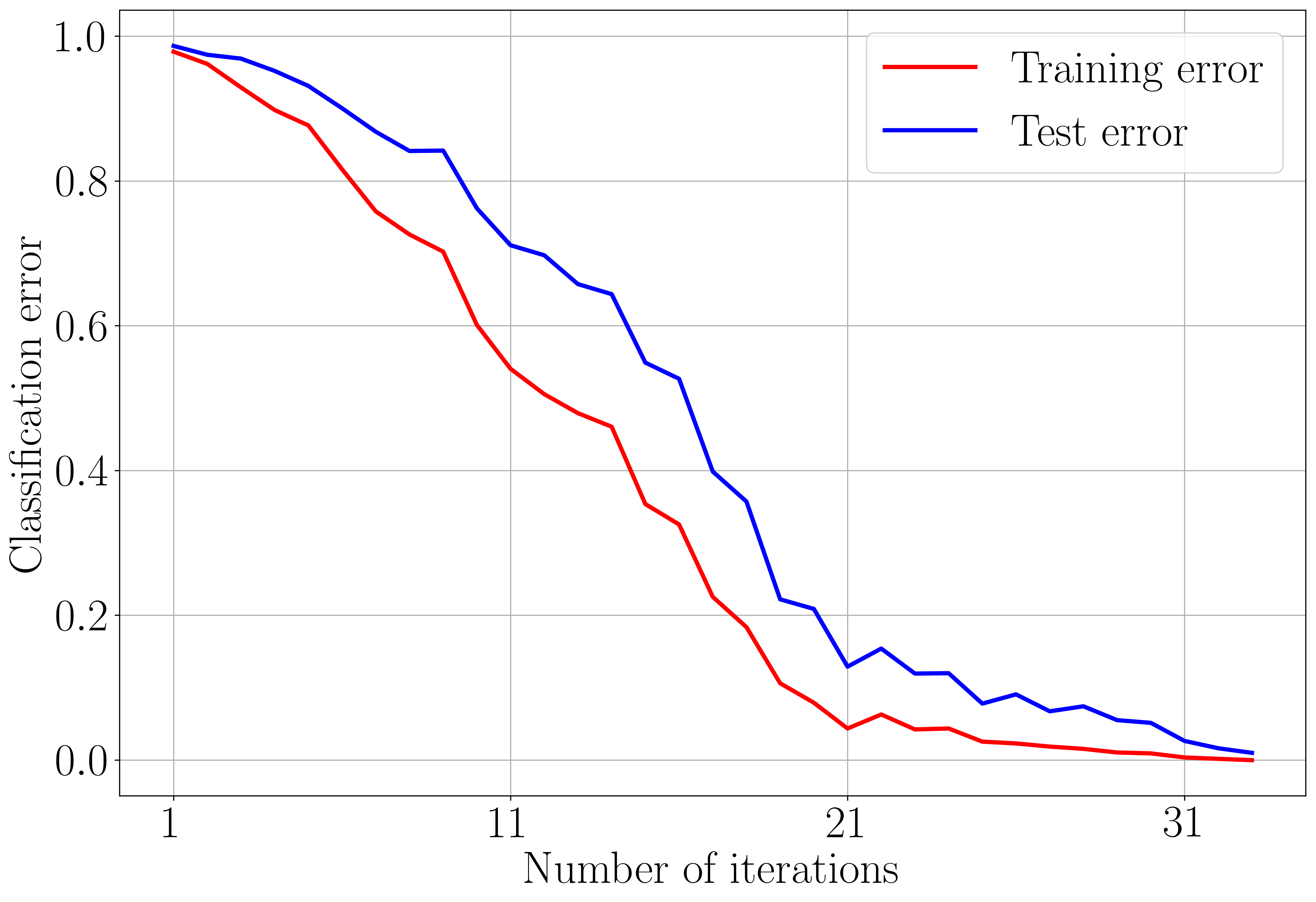}
		\caption{Evolution of training and test error}
		\label{mod_add_loss_evol}
	\end{subfigure}
	\hfill
	\begin{subfigure}{0.55\textwidth}
		\centering
		\includegraphics[width=\linewidth]{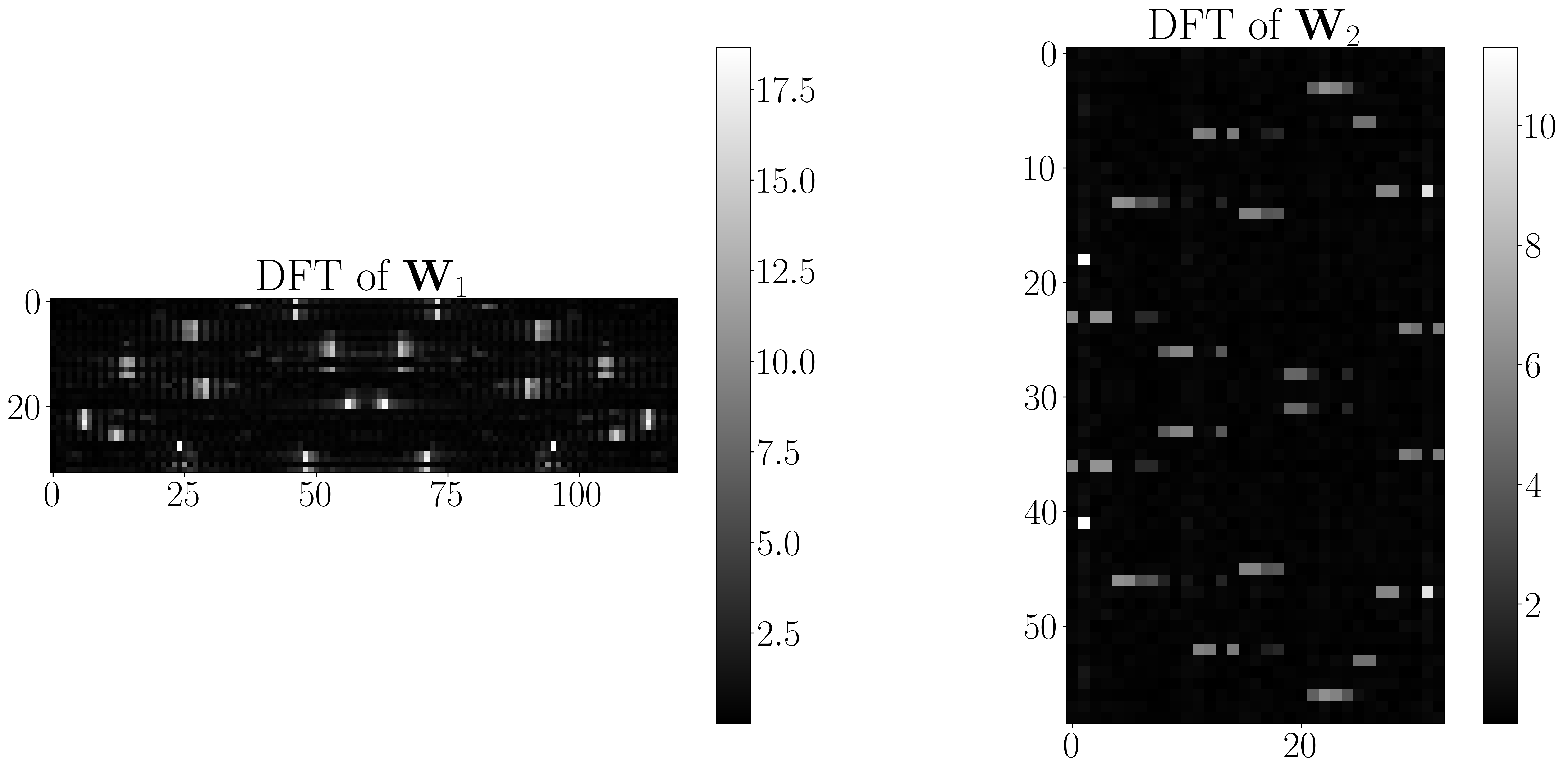}
		\caption{Absolute value of 2D DFT of the learned weights}
		\label{mod_add_wt}
	\end{subfigure}
	\caption{(Modular Addition) We train a two-layer neural network with square activation function via the NP algorithm to learn modular addition. Panel (a) depicts the evolution of training and test classification errors with respect to the iterations. It shows that the learned network is able to achieve small training  and test error. Panel (b) depicts the absolute value of 2D DFT of the learned weights. The DFT of each row of the first layer and each column of the second layer is concentrated around a certain frequency, indicating a clear sinusoidal structure among the learned weights.}
	\label{fig:2_layer_mod_add}
\end{figure}

For this task, prior works have shown that the learned weights of trained networks exhibit a sinusoidal structure \citep{gromov_mod,nanda_mod,zhong_mod}. In particular, \citet{gromov_mod} constructed a two-layer network with square activation in which the rows of the first-layer weights and the columns of the second-layer weights are sinusoidal. \Cref{mod_add_wt} shows the absolute value of the 2D discrete Fourier transform (DFT) of the weights learned by the NP algorithm. We observe that the DFT of each row of the first layer and each column of the second layer concentrates around a certain frequency. Thus, similar to networks trained with gradient descent, the weights learned by the NP algorithm also have a sinusoidal structure.\\
\textbf{Pointer Value Retrieval: } We next consider the task of Pointer Value Retrieval (PVR), introduced by \citet{zhang_pvr}. In this task, part of the input acts as a \emph{pointer} that selects a specific location in the input, whose value and its neighbors values determine the output.  Formally, for $\rvx\in \sR^n$ and a pointer $p\in \{1,2,\cdots n\}$, the output is
\begin{equation*}
f(p,\rvx) = \phi(x_p,x_{p+1},\cdots,x_{p+k}),
\end{equation*}	
where $x_i$ denotes the $i$th coordinate of $\rvx$ and $\phi(\cdot)$ is a scalar-valued function. The complexity of the task increases with number of neighbors and on the choice of $\phi(\cdot)$. In this setup, the network must learn to use the pointer to selectively attend to the relevant coordinates. 

In our experiments, we consider the case where $\rvx \in \{\pm1\}^{16}$, the pointer $p \in \{1,2,\ldots,15\}$, and the output is
\begin{equation*}
f(p,\rvx) = x_{p}x_{p+1}.
\end{equation*}
We train a three-layer ReLU neural network using the NP algorithm. The input to the network is $[\rvp,1,\rvx]\in \sR^{21}$, where $\rvp\in \sR^{4}$ encodes the pointer in symmetric binary form\textemdash for instance, $p=1$ is encoded as $\rvp=(-1,-1,-1,-1)$, $p=2$ as $\rvp=(-1,-1,-1,1)$, and so on. The fifth coordinate of the input is fixed at $1$, serving as a bias term for the first layer.

Training and test sets each consist of $100{,}000$ samples, where every entry of $\rvp$ and $\rvx$ is independently set to $1$ or $-1$ with equal probability. The performance is measured using training and test error, as defined in \Cref{sp_func_exp}. We use square loss and train until the training error drops below $0.01$. The experiment is repeated three times; we report a representative run here, with additional results in \Cref{pvr_exp} of the Appendix. As shown in \Cref{pvr_loss_evol}, the algorithm converges within $43$ iterations, achieving low test error and demonstrating that the network successfully learns the task.

The final network contains $31$ neurons in the first layer and $13$ in the second. \Cref{pvr_wt} depicts the absolute value of the learned weights. While a complete interpretation of the weights is beyond scope of this work, two notable observations can be made. First, the weights of the first layer associated with $\rvx$ (the last 16 columns) are sparse, with dominant entries localized within each row. This is consistent with the task’s dependence on specific coordinates and their neighbors. Second, the later rows and columns of the weights, which are associated with neurons added at the later stage, are relatively small. This is because the training error is already small by the later stages, and additional neurons and iterations were only needed to reduce the training error below the desired level.
\begin{figure}[htbp]
	\centering
	\begin{subfigure}[b]{0.33\textwidth}
		\centering
		\includegraphics[width=\linewidth]{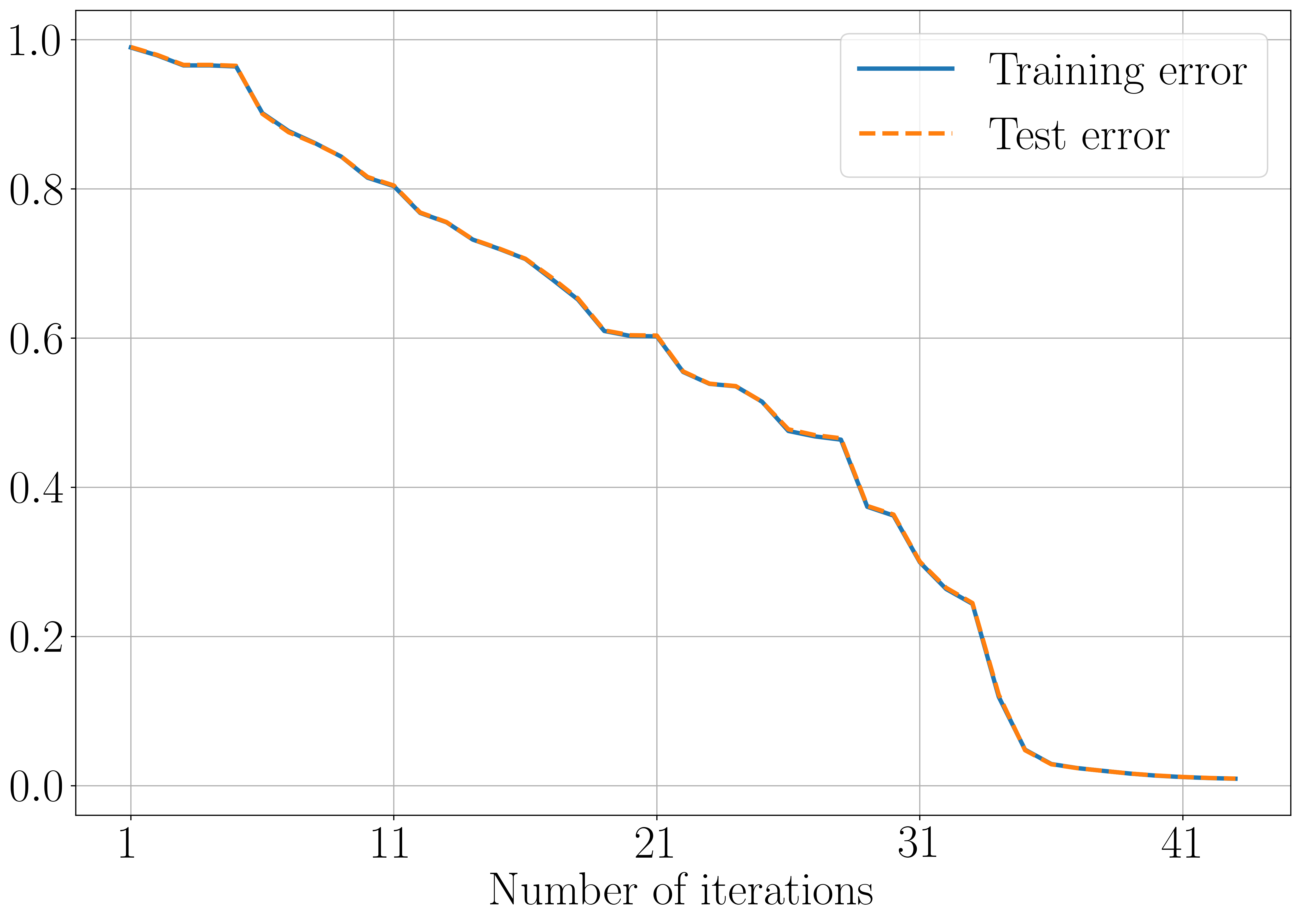}
		\caption{Training and test error}
		\label{pvr_loss_evol}
	\end{subfigure}
	\hfill
	\begin{subfigure}[b]{0.6\textwidth}
		\centering
		\includegraphics[width=\linewidth,height=3.8cm]{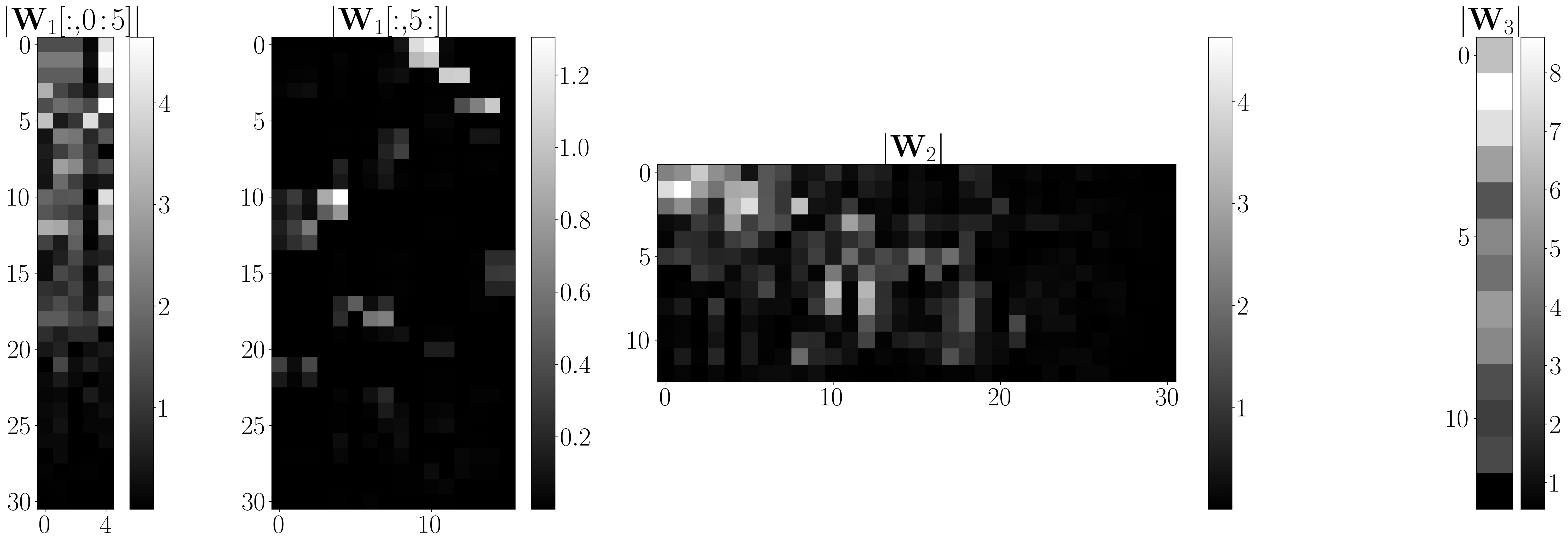}
		\caption{Absolute value of the learned weights}
		\label{pvr_wt}
	\end{subfigure}
	\caption{(Pointer Value Retrieval) We train a three-layer neural network with activation function $\sigma(x) = \max(0,x)$ via the NP algorithm to perform the PVR task. Panel (a) depicts the evolution of training and test errors with respect to the iterations, where both errors are almost overlapped. It shows that the learned network achieves small training  and test error. Panel (b) depicts the absolute values of the learned weights.  The weights of the first layer associated with $\rvx$ (the last 16 columns) are plotted separately to highlight them better. They are sparse and dominant entries are localized within each row. 
	}
	\label{fig:3_layer_pvr}
\end{figure}
\subsection{Discussion}
\label{fin_disc}
The above experiments demonstrate the learning capability of neural networks trained via the NP algorithm. However, in certain instances, we observed a gap between the performance of the NP algorithm and gradient descent with small initialization. This suggests that the NP algorithm is not necessarily equivalent to the gradient flow dynamics in the limit of initialization approaching zero. 

This is more evident if we consider two-layer diagonal linear networks. For such networks, it is known that in the limit of initialization approaching zero, the gradient flow converges to a minimum $\ell_1$-norm solution \citep{srebro_ib}. However, for diagonal linear networks, the NP algorithm is equivalent to the Orthogonal Matching Pursuit (OMP) algorithm \citep{omp_org, tropp_omp}, under certain assumptions; this is established in \Cref{omp_np_eq} of the Appendix, where we also empirically validate this claim. The OMP algorithm is a well-known technique to perform sparse recovery, but is different from minimizing the $\ell_1$-norm. This shows that even though the NP algorithm was motivated from the dynamics of gradient flow in the small initialization regime, it is not equivalent to gradient flow dynamics in the limit of initialization approaching zero.

We hypothesize that one of the reasons behind this gap is the way NP algorithm escapes from saddle points. As shown in \Cref{thm_dir_convg} and stated in Remark \ref{rem_dir_saddle}, near a saddle point, the initial direction of  the weights with small magnitude determine the KKT point to which those weights converge in direction. That, in turn, decides the way gradient flow escapes from that saddle point. Now, for homogeneous feed-forward neural networks, \Cref{lemma_small_wt} shows that after escaping from the origin, gradient flow reaches a saddle point  where a certain subset of weights have small magnitude; however, it does not characterize the direction of those subset of weights, which is critical for determining the KKT point. Lacking this information, the NP algorithm essentially uses random initial weights to obtain the KKT point. Perhaps, this leads to a different KKT point and a different escape direction from the saddle point. Further investigation along these lines to close the gap between gradient descent and the NP algorithm is an interesting future direction.  

Despite these differences, the NP algorithm offers a promising tool for studying neural networks. Unlike standard gradient-based methods, it provides a procedure in which neurons are incrementally added as training progresses, naturally favoring low-complexity or sparse solutions. This greedy construction offers a new lens for understanding how neural networks build representations. Developing a rigorous theoretical framework for the NP algorithm could therefore deepen our understanding of feature learning in neural networks. It can also serve as a foundation for future efforts to design similar greedy algorithms for more complex, non-homogeneous architectures.

Some related works have also explored greedy, stage-wise optimization of neural networks inspired by their training dynamics. For instance, \citet{lyu_resolving} proposed a greedy algorithm to train deep \emph{linear} neural networks. Their method proceeds by incrementally increasing the rank of the weight matrices and is heavily inspired by the saddle-to-saddle dynamics of gradient flow in the small initialization regime. However, the NP algorithm cannot be directly applied to deep linear networks; as discussed at the end of \Cref{sec:main_results} and in \Cref{zero_NCF}, the constrained NCF for deep linear networks becomes exactly zero at saddle points, rendering our specific trajectory analysis inapplicable in this setting. 

Another related work is \citet{kunin_agf}, which introduces \emph{Alternating Gradient Flows} (AGF). AGF is an algorithmic framework designed to characterize feature learning in \emph{two-layer} neural networks trained with small initialization. Similar to NP, AGF models the saddle-to-saddle dynamics as an alternating, stage-wise process: in the first phase, dormant neurons evaluate a utility function, and in the second phase the training loss is minimized over the active neurons. Conceptually, the instantaneous utility function in AGF is mathematically identical to our Neural Correlation Function (NCF)\textemdash both evaluate the correlation between the network's residual error and the parameters of the inactive neurons. However, a fundamental difference lies in how neurons are activated. AGF relies on \emph{accumulating} this utility during the first phase, activating dormant neurons whose accumulated utility is higher than a specific threshold. In contrast, NP adopts a greedy, instantaneous approach at the saddle point; it computes the dominant KKT point of the constrained NCF to decide which new neurons to activate. For certain architectures, AGF is equivalent to gradient flow in the limit of initialization approaching zero, however it remains open whether this equivalence always holds. Furthermore, while AGF is developed primarily for two-layer networks, NP algorithm is applicable for deep, multi-layer homogeneous architectures.

Finally, while the Neuron Pursuit algorithm was derived here strictly using the insights into the gradient flow dynamics, it appears that it can also be derived using simpler principles. In future work, we plan to formally explore these alternative foundations and also establish connections between NP and other greedy algorithms for training neural networks.
\section{Conclusion}
In this work, we began by analyzing the dynamics of gradient flow for homogeneous neural networks near saddle points with a specific sparsity structure. We showed that when initialized sufficiently close to such saddle points, gradient flow remains in their vicinity for a long time, during which weights with small magnitudes stay small and converge in direction. For feed-forward networks, this directional convergence enforces a proportionality between the norms of incoming and outgoing weights of hidden neurons, and activates a subset of previously inactive neurons. Our subsequent empirical analysis revealed that the set of active neurons is preserved after escaping from these saddle points and until reaching the next saddle point. By combining these findings with our previous analyses, we present a mechanistic picture of training feed-forward neural networks in the small initialization regime: weights move from one saddle point to another during training, and a new set of neurons is activated via directional convergence at each saddle. Motivated by this framework, we introduced Neuron Pursuit\textemdash a greedy algorithm for training neural networks.

While our framework provides a compelling lens to view the training dynamics, it relies on results that contain several critical gaps. Foremost is the need to rigorously formalize the empirical observations detailed in \Cref{sec:bey_saddle} regarding the gradient flow dynamics after escaping saddle points. Furthermore, verifying or even relaxing the \textit{Lojasiewicz}-type condition (\Cref{ass_loj}) utilized in \Cref{thm_dir_convg} is an important problem. Additionally, our current analysis assumes smooth activation functions, which excludes deep ReLU networks. Extending our theoretical analysis from the square loss to other loss functions, such as the logistic loss, is another significant challenge; in such settings, saddle points can diverge to infinity, considerably complicating the analysis of gradient flow near saddle points.

There are several promising avenues for future research regarding the Neuron Pursuit algorithm itself. Foremost is to theoretically close the gap between the NP algorithm and gradient flow with vanishing initialization, as discussed in \Cref{fin_disc}. While the experiments in \Cref{sec:exp_NP} demonstrate the strong empirical learning ability of NP, developing formal theoretical guarantees for its convergence will be a major step toward demystifying neural network optimization. Also, although NP was derived herein strictly through the lens of training dynamics, we anticipate that it can be alternatively formulated using simpler optimization principles\textemdash a connection we plan to formalize in future work. Finally, extending the NP algorithm to non-homogeneous architectures, such as Residual Networks and Transformers, is also a promising direction for further exploration.
\newpage
\appendix
\section{Key Lemmata}\label{appendix_key}
The following lemma, also known as Euler's theorem, states two important properties of homogeneous functions \citep[Theorem B.2]{Lyu_ib}.
\begin{lemma}\label{euler_thm}
	Let $f:\sR^d\rightarrow \sR$ be locally Lipschitz, differentiable and $L$-positively homogeneous for some $L>0$. Then,
	\begin{itemize}
		\item For any $\rvw\in\sR^d$ and $c\geq 0$, $\nabla f(c\rvw) = c^{L-1}f(\rvw)$.
		\item For any $\rvw\in\sR^d$ , $\rvw^\top\nabla f(\rvw)= Lf(\rvw).$
	\end{itemize}
\end{lemma}
This lemma also holds for non-differentiable homogeneous functions that satisfy additional requirements such as admitting a chain rule \citep{Lyu_ib} or definable under an o-minimal structure \citep{ji_matus_align}. This includes ReLU neural networks.  

The following lemma establishes a key property of functions homogeneous in a subset of variables.
\begin{lemma}\label{subset_homog}
	Let $f(\rvw_1,\rvw_2)$ be differentiable in $\rvw_1\in \sR^{m}$ and $L$-positively homogeneous in $\rvw_2\in \sR^{n}$, for some $L>0$ . Then, $\nabla_{\rvw_1}f(\rvw_1,\rvw_2)$ is also $L$-positively homogeneous in $\rvw_2$.
\end{lemma}
\begin{proof}
	For any $c\geq0$ and $\rvv\in \sR^m$, we have
	\begin{align*}
	\nabla_{\rvw_1}f(\rvw_1,c\rvw_2)^\top\rvv  &= \lim_{\epsilon\rightarrow 0}\frac{f(\rvw_1+\epsilon\rvv,c\rvw_2) - f(\rvw_1,c\rvw_2)}{\epsilon}\\
	&=\lim_{\epsilon\rightarrow 0}\frac{c^Lf(\rvw_1+\epsilon\rvv,\rvw_2) - c^Lf(\rvw_1,\rvw_2)}{\epsilon} = c^L\nabla_{\rvw_1}f(\rvw_1,\rvw_2)^\top\rvv.
	\end{align*}
	Since the above equality is true for any $\rvv$, we get $\nabla_{\rvw_1}f(\rvw_1,c\rvw_2) = c^L\nabla_{\rvw_1}f(\rvw_1,\rvw_2),$ which completes the proof.
\end{proof}
The next lemma states Gronwall's inequality.
\begin{lemma}
	Let $\alpha, \beta, u$ be real-valued functions defined on an interval $[a,b]$, where $\beta$ and $u$ are continuous and $\min(\alpha,0)$ is integrable on every closed and bounded sub-interval of $[a,b]$. If $\beta$ is non-negative, $\alpha$ is non-decreasing and if $u$ satisfies the integral inequality
	\begin{equation*}
	u(t) \leq \alpha(t) + \int_{a}^t \beta(s)u(s) ds, \forall t \in [a,b],
	\end{equation*}
	then
	\begin{equation*}
	u(t) \leq \alpha(t)\exp\left(\int_a^t \beta(s) ds\right), \forall t \in [a,b].
	\end{equation*}
	\label{gronwall}
\end{lemma}
\section{Proofs omitted from \Cref{sec:near_saddle}}
\subsection{Proof of \Cref{thm_dir_convg}}
\label{pf_main_thm}
We first state a key lemma used in the proof of \Cref{thm_dir_convg}. Assuming that the Łojasiewicz inequality holds, we establish that a reverse Łojasiewicz-type inequality also follows. The proof is adapted from a result in \citet{karimi_pl}, which shows that the Polyak-Łojasiewicz inequality implies a quadratic growth condition.   
\begin{lemma} 
	Suppose $\rvw_*$ is a local minima of $g(\rvw)$ such that Lojasiewicz's inequality is satisfied in a neighborhood of ${\rvw}_*$ for some $\alpha \in \left(0,1\right)$, that is, there exists $\mu,\gamma>0$ such that
	\begin{equation*}
	\|	\nabla g(\rvw)\|_2\geq \mu(g(\rvw) - g({\rvw}_*))^{\alpha}, \text{ if } \|\rvw- {\rvw}_*\|_2\leq \gamma,
	\end{equation*}
	where $\alpha \in \left(0,1\right)$. Then, there exists a ${\mu}_1>0$, ${\gamma}_1 \in (0,\gamma)$ such that
	\begin{equation*}
	\|	\nabla g(\rvw)\|_2\leq {\mu}_1(g(\rvw) - g({\rvw}_*))^{1-\alpha}, \text{ if } \|\rvw- {\rvw}_*\|_2\leq {\gamma}_1.
	\end{equation*}
	\label{loj_low_bd}
\end{lemma}
\begin{proof}
	We define the set $\rvw_*^\gamma$ as
	\begin{equation*}
	\rvw_*^\gamma \coloneqq \{\rvw:\|\rvw-\rvw_*\|_2\leq \gamma\}.
	\end{equation*}
	For $\rvw\in \rvw_*^\gamma$, define $f(\rvw) \coloneqq (g(\rvw) - g({\rvw}_*))^{1-\alpha}$, and let $\mathcal{W}_* \coloneqq \{\rvw\in\rvw_*^\gamma: g(\rvw) = g({\rvw}_*)\}$. Then, for any $\rvw\in\rvw_*^\gamma\symbol{92}\mathcal{W}_*$, we have
	\begin{equation*}
	\|\nabla f(\rvw)\|_2 = \left\|\frac{\nabla g(\rvw)}{(g(\rvw) - g({\rvw}_*))^{\alpha}}\right\|_2 \geq \mu.
	\end{equation*}
	Choose $\gamma_1\in (0,\gamma)$ small enough such that $\gamma_1\leq \gamma/4$ and $f(\rvw)/\mu \leq \gamma/4$, for all $\rvw\in \rvw_*^{{\gamma}_1}$.
	Next, suppose $\rvw_0\in  \rvw_*^{{\gamma}_1}$ and let $\rvw(t)$ be the solution of the following differential equation:
	\begin{equation*}
	\dot{\rvw} = -\nabla f(\rvw), \rvw(0) = \rvw_0.
	\end{equation*}
	Define
	\begin{equation*}
	T^* = \min_{t\geq 0}\{t:\rvw(t)\in \mathcal{W}_* \text{ or } \rvw(t) \notin \rvw_*^\gamma\}.
	\end{equation*}
	Thus, $T^*$ denotes the first time when either $\rvw(t)$ reaches the set $\mathcal{W}_*$ or $\rvw(t)$ escapes from $\rvw_*^\gamma$. Thus, for all $t\in [0,T^*)$, $\|\nabla f(\rvw(t))\|_2 \geq \mu$. \\
	Now, let $p(t) = \int_0^t \|\dot{\rvw}(s)\|_2 ds$. Then, for all $t\in [0,T^*)$, we have
	\begin{equation*}
	\frac{d f(\rvw)}{dt} = -\|\nabla f(\rvw(t))\|_2^2 \leq -\mu \|\nabla f(\rvw(t))\|_2 = -\mu\|\dot{\rvw}(t)\|_2 = -\mu\dot{p}(t). 
	\end{equation*}
	Integrating the above equation from $0$ to $t\in [0,T^*)$ we get
	\begin{equation}
	\mu p(t) \leq  f(\rvw_0) - f(\rvw(t)) \leq f(\rvw_0).
	\label{loj_pt_bd}
	\end{equation}
	Since $\|\rvw(t) - \rvw_0\|_2 \leq p(t)$, we get $\|\rvw(t) - \rvw_0\|_2\leq f(\rvw_0)/\mu,$ which implies, for all $t\in [0,T^*)$, 
	\begin{equation*}
	\|\rvw(t) - \rvw_*\|_2 \leq \|\rvw(t) - \rvw_0\|_2 +\|\rvw_0 - \rvw_*\|_2 \leq f(\rvw_0)/\mu + {\gamma}_1 \leq \gamma/2.
	\end{equation*}
	The above inequality implies $\rvw(t) \in \rvw_*^{\gamma/2}$,  for all $t\in [0,T^*)$ and hence, $\rvw(T^*) \in \rvw_*^\gamma$. \\
	Next, we derive an upper bound on $T^*$. For all $t\in [0,T^*)$, we have
	\begin{equation*}
	f(\rvw(t)) - f(\rvw_0) = -\int_0^t\|\nabla f(\rvw(s))\|_2^2 ds \leq -\mu^2t,
	\end{equation*}
	which implies
	\begin{equation*}
	\mu^2t\leq  f(\rvw_0) -	f(\rvw(t)) \leq  f(\rvw_0).
	\end{equation*}
	Taking $t\rightarrow T^*$, we get $T^* \leq f(\rvw_0)/\mu^2$. Thus, $T^*$ is finite and since  $\rvw(T^*) \in \rvw_*^\gamma$, we get $\rvw(T^*) \in  \mathcal{W}_*$. Now, from \cref{loj_pt_bd}, we know
	\begin{equation*}
	f(\rvw_0) \geq \mu\|\rvw(t) - \rvw_0\|_2, \text{ for all }t\in [0,T^*).
	\end{equation*}
	Taking $t\rightarrow T^*$ in the above equation gives us
	\begin{equation}
	f(\rvw_0) \geq \mu\|\rvw(T^*) - \rvw_0\|_2.
	\label{f_bd}
	\end{equation}
	Since $g(\rvw)$ has locally Lipschitz gradient, there exists a large enough $K>0$ such that, for any $\rvw_1, \rvw_2\in \rvw_*^\gamma$, we have
	\begin{equation*}
	\|\nabla g(\rvw_1) - \nabla g(\rvw_2)\|_2 \leq K\|\rvw_1-\rvw_2\|_2.
	\end{equation*}
	Since $\nabla g(\rvw(T^*)) = \mathbf{0}$, using the above equation in \cref{f_bd} gives us
	\begin{equation*}
	(g(\rvw_0) - g({\rvw}_*))^{1-\alpha} = f(\rvw_0) \geq \mu\|\rvw(T^*) - \rvw_0\|_2 \geq \frac{	\mu\|\nabla g(\rvw_0) \|_2 }{K},
	\end{equation*}
	which completes the proof since $\rvw_0\in \rvw_*^{\gamma_1}$ and it is arbitrary. 
\end{proof}
\textbf{Proof of \Cref{thm_dir_convg}:}
Since $(\overline{\rvw}_n,\mathbf{0})$ is a saddle point of the loss function in \cref{loss_fn_sep} such that \Cref{ass_loj} is satisfied, then there exists $\mu_1,\gamma>0$ such that
\begin{equation}
\|	\nabla\widetilde{\mathcal{L}}(\rvw_n)\|_2\geq \mu_1(\widetilde{\mathcal{L}}(\rvw_n) - \widetilde{\mathcal{L}}(\overline{\rvw}_n))^{\alpha}, \text{ if } \|\rvw_n - \overline{\rvw}_n\|_2\leq \gamma,
\label{loj_proof}
\end{equation}
where $\alpha \in \left(0,\frac{L}{2(L-1)}\right)$. \\
Without loss of generality, from Lemma \ref{loj_low_bd}, we assume that there exists  $\mu_2>0$ such that
\begin{equation}
\|	\nabla\widetilde{\mathcal{L}}(\rvw_n)\|_2\leq \mu_2(\widetilde{\mathcal{L}}(\rvw_n) - \widetilde{\mathcal{L}}(\overline{\rvw}_n))^{1-\alpha}, \text{ if } \|\rvw_n - \overline{\rvw}_n\|_2\leq \gamma.
\label{lips_proof}
\end{equation}	
Suppose $\rvu_1(t)$ is the solution of the following differential equation:
\begin{equation*}
\dot{\rvu} = \nabla_{\rvw_z}\mathcal{N}_{\overline{\rvy},\mathcal{H}_1} =  \mathcal{J}_1(\rmX;\overline{\rvw}_n,\rvu) ^\top\overline{\rvy}, \rvu(0) = \rvz,
\end{equation*}
If $\rvz\in \mathcal{S}(\rvz_*;\mathcal{N}_{\overline{\rvy},\mathcal{H}_1})$, where $\rvz_*$  is a non-negative KKT point of $\widetilde{\mathcal{N}}_{\overline{\rvy},\mathcal{H}_1}$, then for any arbitrarily small $\epsilon>0$ there exists a $T_\epsilon$ such that
\begin{equation*}
\frac{\rvz_*^\top\rvu_1(T_\epsilon)}{\|\rvu_1(T_\epsilon)\|_2} \geq 1-\epsilon. 
\end{equation*}
We can also assume that there exists a sufficiently large constant $B_\epsilon$ and small enough $\eta>0$ such that  $\eta \leq \|\rvu_1(t)\|_2\leq B_\epsilon$, for all $t\in [0,T_\epsilon]$. 

Else, from Lemma \ref{lemma:gf_ncf},  $\rvu_1(t)$ converges to the origin. In this case, we define $\eta = 1$, $B_\epsilon = \epsilon$, and for any arbitrarily small $\epsilon>0$ there exists a $T_\epsilon$ such that
\begin{equation*}
\|\rvu_1(T_\epsilon)\|_2 \leq B_\epsilon = \epsilon.
\end{equation*}
Next, for $\delta>0$, let $\rvu_\delta(t)$ be the solution of the following differential equation:
\begin{equation*}
\dot{\rvu} = \nabla_{\rvw_z}\mathcal{N}_{\overline{\rvy},\mathcal{H}_1} =  \mathcal{J}_1(\rmX;\overline{\rvw}_n,\rvu) ^\top\overline{\rvy}, \rvu(0) = \delta\rvz,
\end{equation*}
From Lemma \ref{traj_init_eq}, we know $\rvu_1(t) = \rvu_{\delta}(t/\delta^{L-2})/\delta$. Hence, if $\rvz\in \mathcal{S}(\rvz_*;\mathcal{N}_{\overline{\rvy},\mathcal{H}_1})$, then $\eta\delta\leq\|\rvu_\delta(t)\|_2\leq B_\epsilon\delta$, for all $t\in [0,T_\epsilon/\delta^{L-2}]$, and
\begin{equation*}
\frac{\rvz_*^\top\rvu_\delta(T_\epsilon/\delta^{L-2})}{\|\rvu_\delta(T_\epsilon/\delta^{L-2})\|_2} \geq 1-\epsilon.
\end{equation*}
Else, $\|\rvu_\delta(T_\epsilon/\delta^{L-2})\|_2\leq B_\epsilon\delta = \epsilon\delta$. For brevity, we define $\overline{T}_\epsilon \coloneqq T_\epsilon/\delta^{L-2}$.\\
We next define
\begin{equation*}
\mathbf{Z}(t) \coloneqq \max(\|\mathbf{w}_n(t) - \overline{\mathbf{w}}_n\|_2^2, \|\mathbf{w}_z(t)-\mathbf{u}_\delta(t)\|_2^2/\delta^2).
\end{equation*}
Let $\epsilon \in (0,\gamma)$ be arbitrarily small. We will show that $\mathbf{Z}(t)\leq \epsilon^2$, for all $t\in [0,\overline{T}_\epsilon]$ and for all sufficiently small $\delta>0$. Note that $\mathbf{Z}(0) = \delta^2 <  \epsilon^2 \leq\gamma^2$. Define 
\begin{equation*}
\overline{T}_\delta \coloneqq \min \left( \inf \{t \geq 0 : \mathbf{Z}(t) = \epsilon^2\}, \overline{T}_\epsilon \right)
\end{equation*}
Here, $\overline{T}_\delta$ denotes the first time $\mathbf{Z}(t)$ reaches $\epsilon^2$, or it is equal to $\overline{T}_\epsilon$ if $\mathbf{Z}(t)$ fails to reach $\epsilon^2$ within the interval $[0,\overline{T}_\epsilon]$. Thus, for all $t\in [0,\overline{T}_\delta],$  $\mathbf{Z}(t) \leq \epsilon^2 \leq \gamma^2$. Now, if we can show $\overline{T}_\delta = \overline{T}_\epsilon$, then it would imply $\mathbf{Z}(t)\leq \epsilon^2$, for all $t\in [0, \overline{T}_\epsilon]$. We aim to show this next.\\
Now, for all $t\in [0,\overline{T}_\delta]$,
\begin{align}
\frac{d \mathbf{w}_n}{dt}
&=  \mathcal{J}_n\left(\mathbf{X};\mathbf{w}_n(t),\mathbf{w}_z(t)\right)^\top\left(\mathbf{y} - \mathcal{H}\left(\mathbf{X};\mathbf{w}_n(t), \mathbf{w}_z(t)\right) \right)\nonumber\\
&= \mathcal{J}_n\left(\mathbf{X};\mathbf{w}_n(t),\mathbf{0}\right)^\top\left(\mathbf{y} - \mathcal{H}\left(\mathbf{X};\mathbf{w}_n(t), \mathbf{0}\right) \right) + \rvq(t),
\end{align}
where
\begin{align*}
\rvq(t) &= (\mathcal{J}_n\left(\mathbf{X};\mathbf{w}_n(t),\mathbf{w}_z(t)\right)-\mathcal{J}_n\left(\mathbf{X};\mathbf{w}_n(t),\mathbf{0}\right))^\top\left(\mathbf{y} - \mathcal{H}\left(\mathbf{X};\mathbf{w}_n(t), \mathbf{w}_z(t)\right) \right)\\
& + \mathcal{J}_n\left(\mathbf{X};\mathbf{w}_n(t),\mathbf{0}\right)^\top\left(\mathcal{H}\left(\mathbf{X};\mathbf{w}_n(t), \mathbf{0}\right) - \mathcal{H}\left(\mathbf{X};\mathbf{w}_n(t), \mathbf{w}_z(t)\right) \right).
\end{align*}
Using the definition of $\widetilde{\mathcal{L}}(\rvw_n)$, we get
\begin{equation}
\frac{d \mathbf{w}_n}{dt}
=  -\nabla_{\rvw_n}\widetilde{\mathcal{L}}(\rvw_n) + \rvq(t).
\label{eq_wn_evol}
\end{equation}
Now, recall that for  all $t\in [0,\overline{T}_\delta]$, 
\begin{equation*}
\|\rvw_z(t)\|_2 \leq \|\rvu_\delta(t)\|_2 + \delta\epsilon \leq \delta(B_\epsilon+\epsilon).
\end{equation*}
Therefore, using the third condition in \Cref{ass_str}, for all sufficiently small $\delta>0$ we get
\begin{align*}
&\hspace{-3em}\left\|(\mathcal{J}_n\left(\mathbf{X};\mathbf{w}_n(t),\mathbf{w}_z(t)\right)-\mathcal{J}_n\left(\mathbf{X};\mathbf{w}_n(t),\mathbf{0}\right))^\top\left(\mathbf{y} - \mathcal{H}\left(\mathbf{X};\mathbf{w}_n(t), \mathbf{w}_z(t)\right) \right)\right\|_2\\
 \leq & 		\left\|\mathcal{J}_n\left(\mathbf{X};\mathbf{w}_n(t),\mathbf{w}_z(t)\right)-\mathcal{J}_n\left(\mathbf{X};\mathbf{w}_n(t),\mathbf{0}\right)\right\|_2 \left\|\mathbf{y} - \mathcal{H}\left(\mathbf{X};\mathbf{w}_n(t), \mathbf{w}_z(t)\right)\right\|_2 \\
\leq & \left\|\mathcal{J}_{n1}\left(\mathbf{X};\mathbf{w}_n(t),\mathbf{w}_z(t)\right)+O(\delta^K)\right\|_2  \left\|\mathbf{y} - \mathcal{H}\left(\mathbf{X};\mathbf{w}_n(t), \mathbf{w}_z(t)\right)\right\|_2 = O(\delta^L),
\end{align*}
where $\mathcal{J}_{n1}\left(\mathbf{X};\mathbf{w}_n(t),\mathbf{w}_z(t)\right)$ denotes the Jacobian of $\mathcal{H}_{1}\left(\mathbf{X};\mathbf{w}_n(t),\mathbf{w}_z(t)\right)$  with respect to $\rvw_n$. The final equality holds since $\mathcal{J}_{n1}\left(\mathbf{X};\mathbf{w}_n(t),\mathbf{w}_z(t)\right)$ is $L$-homogeneous in $\rvw_z$, $\|\rvw_z\|_2 = O(\delta)$, and $K>L$. \\
Next, using the first condition in \Cref{ass_str}, for all sufficiently small $\delta>0$ we get
\begin{align*}
&\hspace{-3em}\|\mathcal{J}_n\left(\mathbf{X};\mathbf{w}_n(t),\mathbf{0}\right)^\top\left(\mathcal{H}\left(\mathbf{X};\mathbf{w}_n(t), \mathbf{0}\right) - \mathcal{H}\left(\mathbf{X};\mathbf{w}_n(t), \mathbf{w}_z(t)\right) \right)\|_2\\
\leq& 		\|\mathcal{J}_n\left(\mathbf{X};\mathbf{w}_n(t),\mathbf{0}\right)\|_2\|\mathcal{H}\left(\mathbf{X};\mathbf{w}_n(t), \mathbf{0}\right) - \mathcal{H}\left(\mathbf{X};\mathbf{w}_n(t), \mathbf{w}_z(t)\right)\|_2\\
\leq& \|\mathcal{J}_n\left(\mathbf{X};\mathbf{w}_n(t),\mathbf{0}\right)\|_2\|\mathcal{H}_1\left(\mathbf{X};\mathbf{w}_n(t), \mathbf{w}_z(t)\right) + O(\delta^K)\|_2 = O(\delta^L),
\end{align*}
where the final equality holds since $\mathcal{H}_{1}\left(\mathbf{X};\mathbf{w}_n(t),\mathbf{w}_z(t)\right)$ is $L$-homogeneous in $\rvw_z$ and $\|\rvw_z\|_2 = O(\delta)$, and $K>L$. Therefore, combining the above two inequalities, we get
\begin{equation*}
\|\rvq(t)\|_2 \leq C\delta^L, \text{ for all }t\in [0,\overline{T}_\delta]
\end{equation*}
where $C>0$ is a large enough constant. Now, let $p(t)$ denote the length of $\rvw_n(t)$, that is,
\begin{equation*}
p(t) = \int_{0}^t \|\dot{\rvw}_n(s)\|_2 ds.
\end{equation*}
Then, using \cref{eq_wn_evol}, we have
\begin{equation*}
\dot{p}(t) \leq  {\|\nabla\widetilde{\mathcal{L}}(\rvw_n(t))\|_2}+\|\rvq(t)\|_2 \leq  {\|\nabla\widetilde{\mathcal{L}}(\rvw_n)\|_2}+C\delta^L.
\end{equation*}
Now, if for some $t\in [0,\overline{T}_\delta]$, $\widetilde{\mathcal{L}}(\mathbf{w}_n(t))-\widetilde{\mathcal{L}}(\overline{\mathbf{w}}_n) \leq (C\delta^{L}/\mu_1)^\frac{1}{\alpha}$, then, from \cref{lips_proof},
\begin{equation*}
\|\nabla\widetilde{\mathcal{L}}(\rvw_n(t))\|_2 \leq\mu_2(\widetilde{\mathcal{L}}(\mathbf{w}_n(t))-\widetilde{\mathcal{L}}(\overline{\mathbf{w}}_n) )^{1-\alpha}=   O(\delta^{\frac{L(1-\alpha)}{\alpha}}).
\end{equation*} 
Define $\beta \coloneqq \min(L,L(1-\alpha)/\alpha)$, then the above two inequalities imply 
\begin{equation}
\dot{p}(t) \leq  {\|\nabla\widetilde{\mathcal{L}}(\rvw_n(t))\|_2} + O(\delta^L)  =   O(\delta^\beta),
\label{pt_bd}
\end{equation}
If $(\widetilde{\mathcal{L}}(\mathbf{w}_n(t))-\widetilde{\mathcal{L}}(\overline{\mathbf{w}}_n))\geq (C\delta^{L}/\mu_1)^\frac{1}{\alpha}$, \cref{loj_proof} implies $\|\nabla\widetilde{\mathcal{L}}(\rvw_n(t))\|_2\geq C\delta^L$. Hence,
\begin{align}
\frac{d(\widetilde{\mathcal{L}}(\mathbf{w}_n)-\widetilde{\mathcal{L}}(\overline{\mathbf{w}}_n))}{dt} =  \nabla\widetilde{\mathcal{L}}(\mathbf{w}_n)^\top  \dot{\mathbf{w}}_n
&=  -{\|\nabla\widetilde{\mathcal{L}}(\rvw_n)\|_2^2}+ {\nabla\widetilde{\mathcal{L}}(\mathbf{w}_n)^\top \rvq(t)} \nonumber\\
&\leq  -{\|\nabla\widetilde{\mathcal{L}}(\rvw_n)\|_2^2} +\|\nabla\widetilde{\mathcal{L}}(\mathbf{w}_n)\|_2 \|\rvq(t)\|_2\nonumber\\
&\leq  -{\|\nabla\widetilde{\mathcal{L}}(\rvw_n)\|_2^2} + C\|\nabla\widetilde{\mathcal{L}}(\mathbf{w}_n)\|_2 \delta^L \nonumber\\
&= \|\nabla\widetilde{\mathcal{L}}(\mathbf{w}_n)\|_2 \left( -{\|\nabla\widetilde{\mathcal{L}}(\rvw_n)\|_2} + C\delta^L\right)\leq 0.
\label{loss_change_ineq}
\end{align}
The above chain of inequalities imply that if for some $t_0\in [0,\overline{T}_\delta]$, $(\widetilde{\mathcal{L}}(\mathbf{w}_n(t_0))-\widetilde{\mathcal{L}}(\overline{\mathbf{w}}_n))= (C\delta^{L}/\mu_1)^\frac{1}{\alpha},$ then $(\widetilde{\mathcal{L}}(\mathbf{w}_n(t))-\widetilde{\mathcal{L}}(\overline{\mathbf{w}}_n))\leq (C\delta^{L}/\mu_1)^\frac{1}{\alpha},$ for all $t\in [t_0,\overline{T}_\delta]$. 

Now, suppose $\widetilde{\mathcal{L}}(\mathbf{w}_n(0))-\widetilde{\mathcal{L}}(\overline{\mathbf{w}}_n)>(C\delta^{L}/\mu_1)^\frac{1}{\alpha}$ and there exists $T_1\in(0,\overline{T}_\delta]$ such that $\widetilde{\mathcal{L}}(\mathbf{w}_n(T_1))-\widetilde{\mathcal{L}}(\overline{\mathbf{w}}_n)=(C\delta^{L}/\mu_1)^\frac{1}{\alpha}$ for the first time. Then, for all $t\in (0,T_1]$, $\widetilde{\mathcal{L}}(\mathbf{w}_n(t))-\widetilde{\mathcal{L}}(\overline{\mathbf{w}}_n)\geq(C\delta^{L}/\mu_1)^\frac{1}{\alpha}$, and for all $t\in [T_1,\overline{T}_\delta]$,  $\widetilde{\mathcal{L}}(\mathbf{w}_n(t))-\widetilde{\mathcal{L}}(\overline{\mathbf{w}}_n)\leq(C\delta^{L}/\mu_1)^\frac{1}{\alpha}$. Since, for $t\in [0,T_1]$, $\widetilde{\mathcal{L}}(\mathbf{w}_n(t))-\widetilde{\mathcal{L}}(\overline{\mathbf{w}}_n)\geq (C\delta^{L}/\mu_1)^\frac{1}{\alpha}$, we get
\begin{align*}
\frac{d(\widetilde{\mathcal{L}}(\mathbf{w}_n(t))-\widetilde{\mathcal{L}}(\overline{\mathbf{w}}_n))}{dt} &\leq \|\nabla\widetilde{\mathcal{L}}(\mathbf{w}_n)\|_2 \left( -{\|\nabla\widetilde{\mathcal{L}}(\rvw_n)\|_2} + C\delta^L\right)\\
&\leq {\mu_1(\widetilde{\mathcal{L}}(\mathbf{w}_n(t))-\widetilde{\mathcal{L}}(\overline{\mathbf{w}}_n))^\alpha}\left( -{\|\nabla\widetilde{\mathcal{L}}(\rvw_n)\|_2}+ C\delta^L\right)\\
&\leq {\mu_1(\widetilde{\mathcal{L}}(\mathbf{w}_n(t))-\widetilde{\mathcal{L}}(\overline{\mathbf{w}}_n))^\alpha}\left( -\dot{p}(t) + 2C\delta^L\right),\\
\end{align*}
where the first inequality follows from \cref{loss_change_ineq}. The second inequality follows from \cref{loj_proof} and since $\|\nabla\widetilde{\mathcal{L}}(\mathbf{w}_n(t))\|_2\geq C\delta^L$. The last inequality is true since  $\dot{p}(t) \leq  {\|\nabla\widetilde{\mathcal{L}}(\rvw_n(t))\|_2}+C\delta^L,$ for all $t\in [0,\overline{T}_\delta]$. Hence,
\begin{equation*}
\frac{d(\widetilde{\mathcal{L}}(\mathbf{w}_n(t))-\widetilde{\mathcal{L}}(\overline{\mathbf{w}}_n))^{1-\alpha}}{dt} \leq \mu_1(1-\alpha)(-\dot{p}(t) + 2C_1\delta^L).
\end{equation*}
Integrating both sides from $0$ to $t\in [0,T_1]$ we get
\begin{align*}
\mu_1(1-\alpha){p}(t) &\leq  (\widetilde{\mathcal{L}}(\mathbf{w}_n(0))-\widetilde{\mathcal{L}}(\overline{\mathbf{w}}_n))^{1-\alpha} - (\widetilde{\mathcal{L}}(\mathbf{w}_n(t))-\widetilde{\mathcal{L}}(\overline{\mathbf{w}}_n))^{1-\alpha}+2\mu_1(1-\alpha)C\delta^Lt \\
&\leq (\widetilde{\mathcal{L}}(\mathbf{w}_n(0))-\widetilde{\mathcal{L}}(\overline{\mathbf{w}}_n))^{1-\alpha} +2\mu_1(1-\alpha)C\delta^Lt\\
&\leq K_1\|\rvw_n(0) - \overline{\rvw}_n\|_2^{1-\alpha} + 2\mu_1(1-\alpha)C_1\delta^Lt = K_1\delta^{1-\alpha} + 2\mu_1(1-\alpha)C\delta^Lt ,
\end{align*}
where $K_1>0$ is a sufficiently large constant, and the last inequality follows from $\widetilde{\mathcal{L}}(\cdot)$ being locally Lipschitz. Hence, since $T_1\leq \overline{T}_\delta \leq \overline{T}_\epsilon = T_\epsilon/\delta^{L-2}$, for all $t\in [0,T_1]$, 
\begin{equation*}
\|\rvw_n(t)-\rvw_n(0)\|_2 \leq p(t) \leq \frac{K_1\delta^{1-\alpha}}{\mu_1(1-\alpha)} + 2C\delta^Lt \leq \frac{K_1\delta^{1-\alpha}}{\mu_1(1-\alpha)} + 2C\delta^2T_\epsilon = O(\delta^{1-\alpha}).
\end{equation*}
Next, for $t\in [T_1,\overline{T}_\delta]$, $(\widetilde{\mathcal{L}}(\mathbf{w}_n(t))-\widetilde{\mathcal{L}}(\overline{\mathbf{w}}_n))\leq (C\delta^{L}/\mu_1)^\frac{1}{\alpha}$. Therefore, from \cref{pt_bd}, we get
\begin{equation*}
\dot{p}(t) = O(\delta^\beta).
\end{equation*}     
For all $t\in [T_1,\overline{T}_\delta]$, since  $t-T_1\leq \overline{T}_\delta \leq \overline{T}_\epsilon = T_\epsilon/\delta^{L-2}$, we get
\begin{equation*}
\|\rvw_n(t)-\rvw_n(T_1)\|_2 \leq \int_{T_1}^t\|\dot{\rvw}_n(s)\|_2 ds = \int_{T_1}^t \dot{p} (s) ds \leq O(\delta^\beta)(t-T_1) = O(\delta^{\beta-L+2}).
\end{equation*}
Note that since $\alpha\in (0,L/(2L-2))$, we have $L(1-\alpha)/\alpha - L+2>0$, implying $\beta>L-2$.\\
Now, define $\beta_1 = \min(\beta-L+2,1-\alpha)$, then, for all $t\in [0,\overline{T}_\delta]$,
\begin{equation}
\|\rvw_n(t)-\rvw_n(0)\|_2 = O(\delta^{\beta_1}), \text{ and thus, } \|\rvw_n(t)-\overline{\rvw}_n\|_2 = O(\delta^{\beta_1}).
\label{wn_bd}
\end{equation}
Note that we have assumed $\widetilde{\mathcal{L}}(\mathbf{w}_n(0))-\widetilde{\mathcal{L}}(\overline{\mathbf{w}}_n)>(C\delta^{L}/\mu_1)^\frac{1}{\alpha}$. If this assumption is not true, we can then choose $T_1 = 0$ and still get the above bound.  \\
Next, for all $t\in [0,\overline{T}_\delta]$, we have
\begin{align}
\frac{d \mathbf{w}_z}{dt}
&= \mathcal{J}_{z}\left(\mathbf{X};\mathbf{w}_n(t),\mathbf{w}_z(t)\right)^\top\left(\mathbf{y} - \mathcal{H}\left(\mathbf{X};\mathbf{w}_n(t), \mathbf{w}_z(t)\right) \right)\nonumber\\
&= \mathcal{J}_{z1}\left(\mathbf{X};\overline{\mathbf{w}}_n,\mathbf{w}_z(t)\right)^\top\left(\mathbf{y} - \mathcal{H}\left(\mathbf{X};\overline{\mathbf{w}}_n, \mathbf{0}\right) \right) + \rvr(t),
\end{align}
where
\begin{align*}
\rvr(t) &= (\mathcal{J}_z\left(\mathbf{X};\mathbf{w}_n(t),\mathbf{w}_z(t)\right)-\mathcal{J}_{z1}\left(\mathbf{X};\overline{\mathbf{w}}_n,\mathbf{w}_z(t)\right))^\top\left(\mathbf{y} - \mathcal{H}\left(\mathbf{X};\mathbf{w}_n(t), \mathbf{w}_z(t)\right) \right) \\
&+ \mathcal{J}_{z 1}\left(\mathbf{X};\overline{\mathbf{w}}_n,\mathbf{w}_z(t)\right)^\top\left(\mathcal{H}\left(\mathbf{X};\overline{\mathbf{w}}_n, \mathbf{0}\right) - \mathcal{H}\left(\mathbf{X};{\mathbf{w}}_n(t), \mathbf{w}_z(t)\right) \right),
\end{align*}
and $\mathcal{J}_{z 1}\left(\mathbf{X};\overline{\mathbf{w}}_n,\mathbf{w}_z\right)$ denotes the Jacobian of $\mathcal{H}_{1}\left(\mathbf{X};\overline{\mathbf{w}}_n,\mathbf{w}_z\right) $ with respect to $\rvw_z$.
We next show that, for $t\in [0,\overline{T}_\delta]$, $\|\rvr(t)\|_2= O(\delta^{L-1+\beta_2})$, for some $\beta_2>0$. Note that
\begin{align}
&\hspace{-5em}\left\|\mathcal{J}_z(\mathbf{X};\mathbf{w}_n(t),\mathbf{w}_z(t))-\mathcal{J}_{z1}(\mathbf{X};\overline{\mathbf{w}}_n,\mathbf{w}_z(t)) \right\|_2\nonumber\\
\leq & \left\|\mathcal{J}_z(\mathbf{X};\mathbf{w}_n(t),\mathbf{w}_z(t))-\mathcal{J}_{z1}(\mathbf{X};{\mathbf{w}}_n(t),\mathbf{w}_z(t)) \right\|_2 \nonumber\\
+ & \left\|\mathcal{J}_{z1}(\mathbf{X};{\mathbf{w}}_n(t),\mathbf{w}_z(t))-\mathcal{J}_{z 1}(\mathbf{X};\overline{\mathbf{w}}_n,\mathbf{w}_z(t)) \right\|_2\nonumber\\
= & O(\delta^{K-1}) + \delta^{L-1}\left\|\mathcal{J}_{z1}\left(\mathbf{X};{\mathbf{w}}_n(t),\mathbf{w}_z(t)/\delta\right)-\mathcal{J}_{z 1}\left(\mathbf{X};\overline{\mathbf{w}}_n,\mathbf{w}_z(t)/\delta\right) \right\|_2\nonumber\\
\leq &  O(\delta^{K-1}) + K_2\delta^{L-1}\|{\mathbf{w}}_n(t) -\overline{\mathbf{w}}_n \|_2 = O(\delta^{K-1}) + O(\delta^{L-1+\beta_1}),
\label{rt_bd1}
\end{align}  
where $K_2$ is a large enough constant. The first equality follows from the second condition in \Cref{ass_str} (since $\|\rvw_z(t)\|_2 = O(\delta)$) and $(L-1)$-homogeneity of  $\mathcal{J}_{z1}\left(\mathbf{X};{\mathbf{w}}_n,\mathbf{w}_z\right)$ in $\rvw_z$.  The second inequality holds because $\mathcal{J}_{z1}\left(\mathbf{X};{\mathbf{w}}_n,\mathbf{w}_z\right)$  is locally Lipschitz in $\rvw_n$ and $\rvw_z(t)/\delta$ is bounded for all $t\in [0,\overline{T}_\delta]$ and small $\delta$. The final equality follows from \cref{wn_bd}. Next,
\begin{align}
&\left\|\mathcal{J}_{z 1}\left(\mathbf{X};\overline{\mathbf{w}}_n,\mathbf{w}_z(t)\right)^\top\left(\mathcal{H}\left(\mathbf{X};\overline{\mathbf{w}}_n, \mathbf{0}\right) - \mathcal{H}\left(\mathbf{X};{\mathbf{w}}_n(t), \mathbf{w}_z(t)\right) \right)\right\|_2\nonumber\\
&\leq\left\|\mathcal{J}_{z 1}\left(\mathbf{X};\overline{\mathbf{w}}_n,\mathbf{w}_z(t)\right)\right\|_2\left\|\mathcal{H}\left(\mathbf{X};\overline{\mathbf{w}}_n, \mathbf{0}\right) - \mathcal{H}\left(\mathbf{X};{\mathbf{w}}_n(t), \mathbf{w}_z(t)\right)\right\|_2\nonumber\\
&= \delta^{L-1}\left\|\mathcal{J}_{z 1}\left(\mathbf{X};\overline{\mathbf{w}}_n,\mathbf{w}_z(t)/\delta\right)\right\|_2\left\|\mathcal{H}\left(\mathbf{X};\overline{\mathbf{w}}_n, \mathbf{0}\right) - \mathcal{H}\left(\mathbf{X};{\mathbf{w}}_n(t), \mathbf{w}_z(t)\right)\right\|_2\nonumber\\
&\leq K_3\delta^{L-1}(\|\mathcal{H}(\mathbf{X};\overline{\mathbf{w}}_n, \mathbf{0}) -\mathcal{H}(\mathbf{X};\overline{\mathbf{w}}_n, \mathbf{w}_z(t))\|_2\hspace{-0.1cm}+\hspace{-0.1cm}\|\mathcal{H}(\mathbf{X};\overline{\mathbf{w}}_n, \mathbf{w}_z(t))- \mathcal{H}(\mathbf{X};{\mathbf{w}}_n(t), \mathbf{w}_z(t))\|_2)\nonumber\\
&\leq K_3\delta^{L-1}(K_4\|\rvw_z(t)\|_2 + K_5\|\overline{\mathbf{w}}_n-{\mathbf{w}}_n(t)\|_2) = O(\delta^{L}) + O(\delta^{L-1+\beta_1}),
\label{rt_bd2}
\end{align}
where $K_3,K_4,K_5$ are large enough positive constants. The first equality uses the $(L-1)$-homogeneity of $\mathcal{J}_{z1}\left(\mathbf{X};{\mathbf{w}}_n,\mathbf{w}_z\right)$ in $\rvw_z$. The second inequality is true since $\rvw_z(t)/\delta$ is bounded for all small $\delta>0$ and thus, $\left\|\mathcal{J}_{z 1}\left(\mathbf{X};\overline{\mathbf{w}}_n,\mathbf{w}_z(t)/\delta\right)\right\|_2$ is bounded as well, for all  $t\in [0,\overline{T}_\delta]$. The final inequality is true since $\mathcal{H}(\rmX;\rvw_n,\rvw_z)$ is locally Lipschitz in $\rvw_n$ and $\rvw_z$. The final equality follows since $\|\rvw_z(t)\|_2=O(\delta)$, for all  $t\in [0,\overline{T}_\delta]$, and from \cref{wn_bd}. Therefore, if we define $\beta_2 = \min(K-L,\beta_1,1)$, then from \cref{rt_bd1} and \cref{rt_bd2}, we get
\begin{equation}
\|\rvr(t)\|_2= O(\delta^{L-1+\beta_2}), \text{ for all } t\in  [0,\overline{T}_\delta].
\label{rt_bd}
\end{equation}
Next, note that
\begin{align*}
&\frac{1}{2}\frac{d\|\rvw_z(t)-\rvu_\delta(t)\|_2^2}{dt}\\
&= (\rvw_z(t)-\rvu_\delta(t))^\top(\mathcal{J}_{z 1}\left(\mathbf{X};\overline{\mathbf{w}}_n,\mathbf{w}_z(t)\right)-\mathcal{J}_{z 1}\left(\mathbf{X};\overline{\mathbf{w}}_n,\mathbf{u}_\delta(t)\right))^\top\overline{\rvy} + (\rvw_z(t)-\rvu_\delta(t))^\top\rvr(t)\\
&\leq \|\rvw_z(t)-\rvu_z(t)\|_2\left(\|\overline{\rvy}\|_2\|\mathcal{J}_{z 1}\left(\mathbf{X};\overline{\mathbf{w}}_n,\mathbf{w}_z(t)\right)-\mathcal{J}_{z 1}\left(\mathbf{X};\overline{\mathbf{w}}_n,\mathbf{u}_\delta(t)\right)\|_2 + \|\rvr(t)\|_2\right).
\end{align*}
The above equation can be simplified to get
\begin{align*}
\frac{d\|\rvw_z(t)-\rvu_\delta(t)\|_2}{dt} &\leq  \|\overline{\rvy}\|_2\|\mathcal{J}_{z 1}\left(\mathbf{X};\overline{\mathbf{w}}_n,\mathbf{w}_z(t)\right)-\mathcal{J}_{z 1}\left(\mathbf{X};\overline{\mathbf{w}}_n,\mathbf{u}_\delta(t)\right)\|_2 +   \|\rvr(t)\|_2\\
&=  \delta^{L-1}\|\overline{\rvy}\|_2\|\mathcal{J}_{z 1}\left(\mathbf{X};\overline{\mathbf{w}}_n,\mathbf{w}_z(t)/\delta\right)-\mathcal{J}_{z 1}\left(\mathbf{X};\overline{\mathbf{w}}_n,\mathbf{u}_\delta(t)/\delta\right)\|_2 +   \|\rvr(t)\|_2\\
&\leq  K_6\delta^{L-2}\|\overline{\rvy}\|_2\left\|\rvw_z(t)-\rvu_\delta(t)\right\|_2 + K_7\delta^{L-1+\beta_2},
\end{align*}
where $K_6,K_7$ are large enough positive constants. The first equality uses the $(L-1)$-homogeneity of $\mathcal{J}_{z1}\left(\mathbf{X};{\mathbf{w}}_n,\mathbf{w}_z\right)$ in $\rvw_z$. The second inequality holds because $\rvw_z(t)/\delta$ and $\rvu_\delta(t)/\delta$ are bounded for small $\delta>0$, $\mathcal{J}_{z1}\left(\mathbf{X};{\mathbf{w}}_n,\mathbf{w}_z\right)$  is locally Lipschitz in $\rvw_z$, and from  \cref{rt_bd}. Integrating the above equation from $0$ to $t$ and then using Lemma \ref{gronwall} gives us
\begin{equation*}
\|\rvw_z(t)-\rvu_\delta(t)\|_2 \leq K_7e^{K_6\delta^{L-2}\|\overline{\rvy}\|_2t}\delta^{L-1+\beta_2} t, \text{ for all } t\in [0,\overline{T}_\delta].
\end{equation*}
Since $\overline{T}_\delta\leq \overline{T}_\epsilon = T_\epsilon/\delta^{L-2}$, we get
\begin{equation}
\|\rvw_z(t)-\rvu_\delta(t)\|_2 \leq K_7T_\epsilon e^{K_6\|\overline{\rvy}\|_2T_\epsilon}\delta^{1+\beta_2} = O(\delta^{1+\beta_2}), \text{ for all } t\in [0,\overline{T}_\delta].
\label{wz_bd}
\end{equation}
Now, for the sake of contradiction, let $\overline{T}_\delta< T_\epsilon/\delta^{L-2}$. Then, from the definition of $\overline{T}_\delta$, at least one of the following equations must hold  
\begin{equation*}
\|\mathbf{w}_n(t) - \overline{\mathbf{w}}_n\|_2^2 = \epsilon^2 \text{ or } \|\mathbf{w}_z(t)-\mathbf{u}_\delta(t)\|_2^2/\delta^2 = \epsilon^2.
\end{equation*}
However, from \cref{wn_bd} and \cref{wz_bd}, we observe that neither of the above two equations can hold leading to a contradiction. Hence, $\overline{T}_\delta = T_\epsilon/\delta^{L-2}$.  Thus, from \cref{wn_bd}, \cref{wz_bd} and since $\|\rvu_{\delta}(t)\|_2 = O(\delta)$, we get 
\begin{equation*}
\|\rvw_n(t)-\overline{\rvw}_n\|_2 = O(\delta^{\beta_{1}}) \text{ and }	\|\rvw_z(t)\|_2 = O(\delta),  \text{ for all } t\in \left[0,\frac{T_\epsilon}{\delta^{L-2}}\right], \text{ where } \beta_1>0.
\end{equation*}
Now, define $\overline{T} = T_\epsilon/\delta^{L-2}$. From \cref{wz_bd}, we know
\begin{equation}
\left\|{\mathbf{w}_z(\overline{T})}- \mathbf{u}_\delta(\overline{T})\right\|_2 = O(\delta^{1+\beta_2}).
\end{equation}
Thus, we may write ${\mathbf{w}_z(\overline{T})} = \mathbf{u}_\delta(\overline{T}) + \zeta$, where $\|\zeta\|_2= O(\delta^{1+\beta_2}).$ If $\rvz \in \mathcal{S}(\rvz_*;\mathcal{N}_{\overline{\rvy},\mathcal{H}_1})$, then since $\|\mathbf{u}_\delta(\overline{T})\|_2\geq \eta\delta$, we have $\|\mathbf{u}_\delta(\overline{T}) + \zeta\|_2 \geq \eta\delta/2$, for all sufficiently small $\delta$. Hence,
\begin{equation*}
\frac{\mathbf{w}_z(\overline{T})}{\|\mathbf{w}_z(\overline{T})\|_2} = \frac{\mathbf{u}_\delta(\overline{T}) + \zeta}{\|\mathbf{u}_\delta(\overline{T}) + \zeta\|_2},
\end{equation*}
which implies
\begin{equation*}
\frac{\mathbf{w}_z(\overline{T})^\top\rvz_*}{\|\mathbf{w}_z(\overline{T})\|_2} = \frac{\mathbf{u}_\delta(\overline{T})^\top\rvz_* + \zeta^\top\rvz_*}{\|\mathbf{u}_\delta(\overline{T}) + \zeta\|_2} = \left(\frac{\mathbf{u}_\delta(\overline{T})^\top\rvz_*}{\|\mathbf{u}_\delta(\overline{T})\|_2}\right)\frac{\|\mathbf{u}_\delta(\overline{T})\|_2}{\|\mathbf{u}_\delta(\overline{T}) + \zeta\|_2} + \frac{\zeta^\top\rvz_*}{\|\mathbf{u}_\delta(\overline{T})+ \zeta\|_2}.
\end{equation*}
Now, note that
\begin{equation*}
\frac{\|\mathbf{u}_\delta(\overline{T})\|_2}{\|\mathbf{u}_\delta(\overline{T})+ \zeta\|_2} \geq \frac{\|\mathbf{u}_\delta(\overline{T})\|_2}{\|\mathbf{u}_\delta(\overline{T})\|_2 + \|\zeta\|_2} = \frac{1}{1+\frac{ \|\zeta\|_2}{\|\mathbf{u}_\delta(\overline{T})\|_2}}\geq 1 - \frac{ \|\zeta\|_2}{\|\mathbf{u}_\delta(\overline{T})\|_2} = 1 - O(\delta^{\beta_2}),
\end{equation*}
and
\begin{equation*}
\frac{\zeta^\top\rvz_*}{\|\mathbf{u}(\overline{T})+ \zeta\|_2} \geq -O(\delta^{\beta_2}).
\end{equation*}
Hence, for all sufficiently small $\delta>0$, we get
\begin{equation*}
\frac{\mathbf{w}_z(\overline{T})^\top\rvz_*}{\|\mathbf{w}_z(\overline{T})\|_2} \geq (1-\epsilon)\left(1 - O(\delta^{\beta_2})\right) - O(\delta^{\beta_2}) = 1 - O(\epsilon).
\end{equation*}
Else,  $\|\mathbf{u}_\delta(\overline{T}) \|_2 \leq \epsilon\delta$, which implies
\begin{equation*}
\|\mathbf{w}_z(\overline{T})\|_2 \leq \|\mathbf{u}_\delta(\overline{T}) \|_2 + \|\zeta\|_2  = \epsilon O(\delta) + O(\delta^{1+\beta_2}) = \epsilon O(\delta),
\end{equation*}
for all sufficiently small $\delta>0$. This completes the proof. \hfill
\subsection{Proof of Lemma \ref{lemma_poly_str}}
\label{pf_decomp_poly}
For all $l\in [L-1]$ and $1\leq k\leq l$, define
\begin{align*}
&\phi_0 = \rvx, \rvh_1 = \rmW_1\rvx, \phi_l = \sigma(\rvh_l), \rvh_{l+1} = \rmW_{l+1}\phi_l,\text{ and }\\
& f_{l,k}(\rvs) = \sigma\left(\rmN_l\sigma\left(\rmN_{l-1}\sigma(\cdots\sigma(\rmN_k\sigma(\rvs))\cdots)\right)\right).
\end{align*}
We first prove the second case, where $\alpha= 1$ and $p\geq 1$.\\
\textbf{Case 2 ($\alpha= 1$ and $p\geq 1$):} 
For all $2\leq l\leq L-1$, we claim that
\begin{equation}
\sigma(\rvh_{l}) = \begin{bmatrix}
\sigma({\rmN}_{l}\sigma(\rmN_{l-1}\sigma(\cdots\sigma(\rmN_1\rvx)))) + p_l(\rvx;\rvw_n,\rvw_z) + q_l(\rvx;\rvw_n,\rvw_z)\\s_l(\rvx;\rvw_n,\rvw_z) + r_l(\rvx;\rvw_n,\rvw_z)
\end{bmatrix},
\label{h_dec_poly}
\end{equation}
where $p_l(\rvx;\rvw_n,\rvw_z)$ and $s_l(\rvx;\rvw_n,\rvw_z)$ are a vector-valued homogeneous polynomial functions in $\rvw_z$ with degree of homogeneity  $p+1$ and $p$, respectively. More specifically,
\begin{align*}
p_l(\rvx;\rvw_n,\rvw_z) &= \sum_{k=1}^{l-2} \nabla_\rvs f_{l,k+2}(\rmN_{k+1}\rvg_{k}(\rvx))\rmB_{k+1}\sigma(\rmA_{k}\rvg_{k-1}(\rvx)) \\
&+  \diag\left(\sigma'(\rmN_{l}\rvg_{l-1}(\rvx))\right)\rmB_l\sigma(\rmA_{l-1}\rvg_{l-2}(\rvx)),\\
 s_l(\rvx;\rvw_n,\rvw_z)  &= \sigma(\rmA_l g_{l-1}(\rvx)).
\end{align*}
Furthermore, $q_l(\rvx;\rvw_n,\rvw_z)$ ($r_l(\rvx;\rvw_n,\rvw_z)$) is a vector-valued function that is a sum of homogeneous polynomial functions in $\rvw_z$, where the degree of homogeneity of each function is strictly greater than $p+1$ ($p$). \\
Now, if the claim stated in \cref{h_dec_poly} is true, then we get
\begin{align*}
&\mathcal{H}(\rvx;\rvw_n,\mathbf{w}_z)= \rmW_L\sigma(\rvh_{L-1})\\
&=\begin{bmatrix}
{{\rmN}_{L}}&\rmB_L
\end{bmatrix}\begin{bmatrix}
\sigma({\rmN}_{L-1}\sigma(\rmN_{L-2}\sigma(\cdots\sigma(\rmN_1\rvx)\cdots))) + p_{L-1}(\rvx;\rvw_n,\rvw_z) + q_{L-1}(\rvx;\rvw_n,\rvw_z)\\ {s}_{L-1}(\rvx;\rvw_n,\rvw_z) + r_{L-1}(\rvx;\rvw_n,\rvw_z)
\end{bmatrix}\\
&=  {{\rmN}_{L}}\sigma({\rmN}_{L-1}\sigma(\rmN_{L-2}\sigma(\cdots\sigma(\rmN_1\rvx)\cdots)))  +  {{\rmN}_{L}}{p}_{L-1}(\rvx;\rvw_n,\rvw_z) +  {{\rmN}_{L}}{q}_{L-1}(\rvx;\rvw_n,\rvw_z)\\
& + \rmB_L{s}_{L-1}(\rvx;\rvw_n,\rvw_z) + \rmB_L{r}_{L-1}(\rvx;\rvw_n,\rvw_z)\\
&= \mathcal{H}(\rvx;\rvw_n,\mathbf{0}) +  {{\rmN}_{L}}{p}_{L-1}(\rvx;\rvw_n,\rvw_z) +  {{\rmN}_{L}}{q}_{L-1}(\rvx;\rvw_n,\rvw_z) + \rmB_L{s}_{L-1}(\rvx;\rvw_n,\rvw_z)\\
& + \rmB_L{r}_{L-1}(\rvx;\rvw_n,\rvw_z),
\end{align*}
where the final equality uses $\mathcal{H}(\rvx;\rvw_n,\mathbf{0}) = {\rmN}_{L}\sigma({\rmN}_{L-1}\sigma(\rmN_{L-2}\sigma(\cdots\sigma(\rmN_1\rvx)\cdots)))$. Here, ${{\rmN}_{L}}{p}_{L-1}(\rvx;\rvw_n,\rvw_z)$ is $(p+1)$-homogeneous in $\rvw_z$, because $\rmN_L$ does not belong to $\rvw_z$ and ${p}_{L-1}(\rvx;\rvw_n,\rvw_z)$ is $(p+1)$-homogeneous in $\rvw_z$. Also, using the definition of ${q}_{L-1}(\rvx;\rvw_n,\rvw_z)$, we get that ${{\rmN}_{L}}{q}_{L-1}(\rvx;\rvw_n,\rvw_z)$  will be sum of homogeneous polynomial functions in $\rvw_z$, where the degree of homogeneity of  each function is strictly greater than $p+1$. 

Next, $\rmB_L{s}_{L-1}(\rvx;\rvw_n,\rvw_z) $ is $(p+1)$-homogeneous in $\rvw_z$, because $\rmB_L$ is $1$-homogeneous and ${s}_{L-1}(\rvx;\rvw_n,\rvw_z)$ is $p$-homogeneous in $\rvw_z$. Also, using the definition of ${r}_{L-1}(\rvx;\rvw_n,\rvw_z)$ and since $\rmB_L$ is $1$-homogeneous in $\rvw_z$, ${{\rmB}_{L}}{r}_{L-1}(\rvx;\rvw_n,\rvw_z)$  will be sum of homogeneous polynomial functions in $\rvw_z$, where the degree of homogeneity of  each function is strictly greater than $p+1$. Therefore, for some $m\geq 2$, 
\begin{equation*}
\mathcal{H}(\rvx;\rvw_n,\mathbf{w}_z) = \mathcal{H}(\rvx;\rvw_n,\mathbf{0}) + \sum_{i=1}^m \mathcal{H}_i(\rvx;\rvw_n,\mathbf{w}_z),
\end{equation*}
where $\mathcal{H}_1(\rvx;\rvw_n,\mathbf{w}_z) \coloneqq {{\rmN}_{L}}{p}_{L-1}(\rvx;\rvw_n,\rvw_z) + {{\rmB}_{L}}{s}_{L-1}(\rvx;\rvw_n,\rvw_z)$ and it is $(p+1)$-homogeneous in $\rvw_z$. Each of $\{\mathcal{H}_i(\rvx;\rvw_n,\mathbf{w}_z)\}_{i=2}^m$ have degree of homogeneity strictly greater than $p+1$.\\
Then, using the definition of ${p}_{L-1}(\rvx;\rvw_n,\rvw_z)$ and ${s}_{L-1}(\rvx;\rvw_n,\rvw_z)$, we get
\begin{align*}
\mathcal{H}_1(\rvx;\rvw_n,\mathbf{w}_z) &= \sum_{k=1}^{L-3} \rmN_L\nabla_\rvs f_{L-1,k+2}(\rmN_{k+1}\rvg_{k}(\rvx))\rmB_{k+1}\sigma(\rmA_{k}\rvg_{k-1}(\rvx))\\
& +  \rmN_L\diag\left(\sigma'(\rmN_{L-1}\rvg_{L-2}(\rvx))\right)\rmB_{L-1}\sigma(\rmA_{L-2}\rvg_{L-3}(\rvx))
+ \rmB_L \sigma(\rmA_{L-1}g_{L-2}(\rvx)).
\end{align*}
Note that $ f_{L,k}(\rvs)  = \rmN_L f_{L-1,k}(\rvs)$, for all $2\leq k\leq L-1$, and  $f_{L,L}(\rvs) = \rmN_L\sigma(\rvs)$. Hence,  $ \nabla_\rvs f_{L,k}^\top(\rvs)  = \rmN_L \nabla_\rvs f_{L-1,k}(\rvs)$, and  $ \nabla_\rvs f_{L,L}^\top(\rvs)  = \rmN_L \text{diag}(\sigma'(\rvs))$. Therefore,
\begin{equation*}
\mathcal{H}_1(\rvx;\rvw_n,\mathbf{w}_z)
=  \sum_{k=1}^{L-2}\nabla_\rvs f_{L,k+2}^\top(\rmN_{k+1}\rvg_{k}(\rvx))\rmB_{k+1}\sigma(\rmA_{k}\rvg_{k-1}(\rvx)) + \rmB_L \sigma(\rmA_{L-1}g_{L-2}(\rvx)),
\end{equation*}
which completes the proof. \\
We next prove the claim in \cref{h_dec_poly} via induction. For $l=2$, since $\sigma(x) = x^p$, from Binomial theorem we get
\begin{align*}
\sigma(\rvh_2) = & \sigma\left(\begin{bmatrix}
{{\rmN}_{2}}&\rmB_2  \\
\rmA_2& \rmC_2 \\
\end{bmatrix}\sigma\left( \begin{bmatrix}
{{\rmN}_{1}\rvx}\\\rmA_1\rvx
\end{bmatrix}\right)\right) = \begin{bmatrix}
\sigma({{\rmN}_{2}}\sigma( {{\rmN}_{1}\rvx}) + \rmB_2\sigma(\rmA_1\rvx)) \\ \sigma(\rmA_2\sigma( {{\rmN}_{1}\rvx})  + \rmC_2\sigma(\rmA_1\rvx))
\end{bmatrix} \\
= & \begin{bmatrix}
(\rmN_{2}\sigma(\rmN_{1}\rvx))^p + p({{\rmN}_{2}}\sigma( {{\rmN}_{1}\rvx}))^{p-1} \odot (\rmB_2\sigma(\rmA_1\rvx)) + q_2(\rvx;\rvw_n,\rvw_z) \\ (\rmA_2\sigma( {{\rmN}_{1}\rvx}))^p  + r_2(\rvx;\rvw_n,\rvw_z)
\end{bmatrix},
\end{align*}
where
\begin{align*}
&q_2(\rvx;\rvw_n,\rvw_z) \coloneqq  \sum_{k=2}^p {p \choose k}({{\rmN}_{2}}\sigma({{\rmN}_{1}\rvx}))^{p-k} \odot (\rmB_2\sigma(\rmA_1\rvx))^{k}, \text{ and } \\
&r_2(\rvx;\rvw_n,\rvw_z) \coloneqq \sum_{k=1}^p {p \choose k}({{\rmA}_{2}}\sigma({{\rmN}_{1}\rvx}))^{p-k} \odot (\rmC_2\sigma(\rmA_1\rvx))^{k}
\end{align*}
Define $p_2(\rvx;\rvw_n,\rvw_z) \coloneqq  p({{\rmN}_{2}}\sigma( {{\rmN}_{1}\rvx}))^{p-1} \odot (\rmB_2\sigma(\rmA_1\rvx))$ and $s_2(\rvx;\rvw_n,\rvw_z) \coloneqq (\rmA_2\sigma( {{\rmN}_{1}\rvx}))^p $, then they are $(p+1)$-homogeneous and $p$-homogeneous in $\rvw_z$, respectively. Observe that $q_2(\rvx;\rvw_n,\rvw_z)$  $(r_2(\rvx;\rvw_n,\rvw_z))$ is sum of homogeneous polynomial functions in $\rvw_z$, where the degree of homogeneity each function is strictly greater than $p+1$ $(p)$. Also, since $\sigma'(x) = px^{p-1}$, we have
\begin{equation*}
p_2(\rvx;\rvw_n,\rvw_z) =  p({{\rmN}_{2}}\sigma( {{\rmN}_{1}\rvx}))^{p-1} \odot (\rmB_2\sigma(\rmA_1\rvx)) = \text{diag}(\sigma'({{\rmN}_{2}}\sigma( {{\rmN}_{1}\rvx})))\rmB_2\sigma(\rmA_1\rvx).
\end{equation*}
Hence, the claim in \cref{h_dec_poly} is true for $l=2$.  Now, suppose the claim is true for some $2<l<L-1$. For brevity, let $\widetilde{p}_l(\rvx;\rvw_n,\rvw_z) \coloneqq p_l(\rvx;\rvw_n,\rvw_z) + q_l(\rvx;\rvw_n,\rvw_z)$ and $\widetilde{s}_l(\rvx;\rvw_n,\rvw_z) \coloneqq  s_l(\rvx;\rvw_n,\rvw_z) + r_l(\rvx;\rvw_n,\rvw_z)$. Then,
\begin{align*}
\sigma(\rvh_{l+1}) &= \sigma\left(\begin{bmatrix}
{{\rmN}_{l+1}}&\rmB_{l+1} \\
\rmA_{l+1}& \rmC_{l+1}\\
\end{bmatrix}\begin{bmatrix}
\sigma({\rmN}_{l}\sigma(\cdots\sigma(\rmN_1\rvx)\cdots)) + p_l(\rvx;\rvw_n,\rvw_z) + q_l(\rvx;\rvw_n,\rvw_z)\\s_l(\rvx;\rvw_n,\rvw_z) + r_l(\rvx;\rvw_n,\rvw_z)
\end{bmatrix}\right)\\
&= \sigma\left(\begin{bmatrix}
{{\rmN}_{l+1}}&\rmB_{l+1} \\
\rmA_{l+1}& \rmC_{l+1}\\
\end{bmatrix}\begin{bmatrix}
\sigma({\rmN}_{l}\sigma(\cdots\sigma(\rmN_1\rvx)\cdots)) + \widetilde{p}_l(\rvx;\rvw_n,\rvw_z) \\\widetilde{s}_l(\rvx;\rvw_n,\rvw_z)
\end{bmatrix}\right)\\
&= \sigma\left(\begin{bmatrix}
{{\rmN}_{l+1}}\sigma({\rmN}_{l}\sigma(\cdots\sigma(\rmN_1\rvx)\cdots)) + {{\rmN}_{l+1}}\widetilde{p}_l(\rvx;\rvw_n,\rvw_z) + {{\rmB}_{l+1}}\widetilde{s}_l(\rvx;\rvw_n,\rvw_z) \\\rmA_{l+1}\sigma({\rmN}_{l}\sigma(\cdots\sigma(\rmN_1\rvx)\cdots))  + \rmA_{l+1}\widetilde{p}_l(\rvx;\rvw_n,\rvw_z) +  \rmC_{l+1}\widetilde{s}_l(\rvx;\rvw_n,\rvw_z)
\end{bmatrix}\right)
\end{align*}
Let 
\begin{align*}
&\hat{p}_l(\rvx;\rvw_n,\rvw_z) \coloneqq {{\rmN}_{l+1}}\widetilde{p}_l(\rvx;\rvw_n,\rvw_z) + {{\rmB}_{l+1}}\widetilde{s}_l(\rvx;\rvw_n,\rvw_z),\\
& \hat{s}_l(\rvx;\rvw_n,\rvw_z)  \coloneqq \rmA_{l+1}\widetilde{p}_l(\rvx;\rvw_n,\rvw_z) +  \rmC_{l+1}\widetilde{s}_l(\rvx;\rvw_n,\rvw_z),  \text{ and } \rvx_l = \sigma({\rmN}_{l}\cdots\sigma(\rmN_1\rvx).
\end{align*}
 Then,
\begin{align*}
&\sigma(\rvh_{l+1}) = \begin{bmatrix}
\sigma\left({{\rmN}_{l+1}}\rvx_l ) + \hat{p}_l(\rvx;\rvw_n,\rvw_z)\right) \\\sigma\left({{\rmA}_{l+1}}\rvx_l )  + \hat{s}_l(\rvx;\rvw_n,\rvw_z)\right)
\end{bmatrix}\\
= & \begin{bmatrix}
({{\rmN}_{l+1}}\rvx_l )^p + p({{\rmN}_{l+1}}\rvx_l )^{p-1}\odot\hat{p}_l(\rvx;\rvw_n,\rvw_z) + \hat{q}_l(\rvx;\rvw_n,\rvw_z) \\\left(\rmA_{l+1}\rvx_l \right)^p  + {r}_{l+1}(\rvx;\rvw_n,\rvw_z)
\end{bmatrix},
\end{align*}
where 
\begin{align*}
&\hat{q}_l(\rvx;\rvw_n,\rvw_z) \coloneqq \sum_{k=2}^p {p \choose k} \left({{\rmN}_{l+1}}\rvx_l \right)^{p-k}\odot\hat{p}_l(\rvx;\rvw_n,\rvw_z)^k, \text{ and }\\
&{r}_{l+1}(\rvx;\rvw_n,\rvw_z) \coloneqq  \sum_{k=1}^p {p \choose k} \left(\rmA_{l+1}\rvx_l \right)^{p-k}\odot\hat{s}_l(\rvx;\rvw_n,\rvw_z)^{k}.
\end{align*}
Now, note that $\widetilde{p}_l(\rvx;\rvw_n,\rvw_z)$ and $\widetilde{s}_l(\rvx;\rvw_n,\rvw_z)$ are sums of homogeneous polynomial functions in $\rvw_z$ where degree of homogeneity of each function is greater than or equal to $p+1$ and $p$, respectively. Hence, same is true for $\hat{s}_l(\rvx;\rvw_n,\rvw_z)$, but the degree of homogeneity of each function is greater than or equal to $p$. This in turn implies ${r}_{l+1}(\rvx;\rvw_n,\rvw_z)$ is a sum of homogeneous polynomial functions in $\rvw_z$ with degree of homogeneity greater than $p$. Also,
\begin{equation*}
\rvs_{l+1}(\rvx;\rvw_n,\rvw_z) = \left(\rmA_{l+1}\rvx_l\right)^p = \sigma(\rmA_{l+1}\sigma({\rmN}_{l}\sigma(\hdots\sigma(\rmN_1\rvx)))),
\end{equation*}
and $\rvs_{l+1}(\rvx;\rvw_n,\rvw_z) $ is $p$-homogeneous in $\rvw_z$. We next define 
\begin{align*}
&p_{l+1}(\rvx;\rvw_n,\rvw_z) = p\left({{\rmN}_{l+1}}\rvx_l\right)^{p-1}\odot\left( {{\rmN}_{l+1}}{p}_l(\rvx;\rvw_n,\rvw_z) + {{\rmB}_{l+1}}{s}_l(\rvx;\rvw_n,\rvw_z)\right),\\
& q_{l+1}(\rvx;\rvw_n,\rvw_z) = p\left({{\rmN}_{l+1}}\rvx_l\right)^{p-1}\odot\left( {{\rmN}_{l+1}}{q}_l(\rvx;\rvw_n,\rvw_z) + {{\rmB}_{l+1}}{r}_l(\rvx;\rvw_n,\rvw_z)\right) + \hat{q}_l(\rvx;\rvw_n,\rvw_z).
\end{align*}
By expanding $\hat{p}_l(\rvx;\rvw_n,\rvw_z)$, we get
\begin{equation*}
\sigma(\rvh_{l+1}) = \begin{bmatrix}
\sigma\left({{\rmN}_{l+1}}\sigma({\rmN}_{l}\sigma(\hdots\sigma(\rmN_1\rvx))) \right) + {p}_{l+1}(\rvx;\rvw_n,\rvw_z) + {q}_{l+1}(\rvx;\rvw_n,\rvw_z) \\\sigma\left(\rmA_{l+1}\sigma({\rmN}_{l}\sigma(\hdots\sigma(\rmN_1\rvx)))  \right) + {s}_{l+1}(\rvx;\rvw_n,\rvw_z) + {r}_{l+1}(\rvx;\rvw_n,\rvw_z)
\end{bmatrix},
\end{equation*}
Note that $ {{\rmN}_{l+1}}{p}_l(\rvx;\rvw_n,\rvw_z) + {{\rmB}_{l+1}}{s}_l(\rvx;\rvw_n,\rvw_z)$ is $(p+1)$-homogeneous in $\rvw_z$, which implies $p_{l+1}(\rvx;\rvw_n,\rvw_z)$ is $(p+1)$-homogeneous in $\rvw_z$. Also, $\hat{p}_l(\rvx; \rvw_n,\rvw_z)$ is a sum of homogeneous polynomial functions in $\rvw_z$ where degree of homogeneity of each function is greater than or equal to $p +1$. Thus, the same is true for $\hat{q}_l(\rvx; \rvw_n,\rvw_z)$. Since  ${{\rmN}_{l+1}}{q}_l(\rvx;\rvw_n,\rvw_z) + {{\rmB}_{l+1}}{r}_l(\rvx;\rvw_n,\rvw_z) $ is sum of homogeneous functions where degree of homogeneity of each function is greater than $p+1$, therefore the same is true for  $q_{l+1}(\rvx;\rvw_n,\rvw_z)$. This proves the claim in \cref{h_dec_poly}. \\
Finally, we simplify  $p_{l+1}(\rvx;\rvw_n,\rvw_z)$. For all $l\in [L-1]$ and $k\leq l$, we have 
\begin{equation}
\rvx_l = g_{l}(\rvx) = \sigma({\rmN}_{l}\sigma({\rmN}_{l-1}\hdots\sigma(\rmN_1\rvx)) = f_{l,k}(\rmN_{k-1}g_{k-2}(\rvx)).
\label{gl_simp}
\end{equation}
Since $f_{l+1,k}(\rvs) = \sigma(\rmN_{l+1}f_{l,k}(\rvs))$, we get
\begin{equation}
\nabla_\rvs f_{l+1,k}(\rvs) = \diag(\sigma'(\rmN_{l+1}f_{l,k}(\rvs)))\rmN_{l+1}\nabla_\rvs f_{l,k}(\rvs).
\label{fl_grad_simp}
\end{equation}
Hence, using the definition of  ${s}_l(\rvx;\rvw_n,\rvw_z)$, we get
\begin{equation}
p\left({{\rmN}_{l+1}}\rvx_l\right)^{p-1}\odot\left( {{\rmB}_{l+1}}{s}_l(\rvx;\rvw_n,\rvw_z)\right) = \diag(\sigma'(\rmN_{l+1}g_{l}(\rvx)))\rmB_{l+1}\sigma(\rmA_lg_{l-1}(\rvx)).
\label{pl_half1}
\end{equation}
Also, using the definition of  ${p}_l(\rvx;\rvw_n,\rvw_z)$, we get
\begin{align}
&\hspace{-3em}p\left({{\rmN}_{l+1}}\rvx_l\right)^{p-1}\odot\left( {{\rmN}_{l+1}}{p}_l(\rvx;\rvw_n,\rvw_z)\right)\nonumber \\
= &\diag(\sigma'(\rmN_{l+1}g_{l}(\rvx))){{\rmN}_{l+1}} \sum_{k=1}^{l-2} \nabla_\rvs f_{l,k+2}(\rmN_{k+1}\rvg_{k}(\rvx))\rmB_{k+1}\sigma(\rmA_{k}\rvg_{k-1}(\rvx)) \nonumber\\
&+  \diag(\sigma'(\rmN_{l+1}g_{l}(\rvx))){{\rmN}_{l+1}}\diag\left(\sigma'(\rmN_{l}\rvg_{l-1}(\rvx))\right)\rmB_l\sigma(\rmA_{l-1}\rvg_{l-2}(\rvx))\nonumber\\
= &\sum_{k=1}^{l-2} \nabla_\rvs f_{l+1,k+2}(\rmN_{k+1}\rvg_{k}(\rvx))\rmB_{k+1}\sigma(\rmA_{k}\rvg_{k-1}(\rvx)) + \nabla_\rvs f_{l+1,l+1}(\rmN_{l}\rvg_{l-1}(\rvx)),
\label{pl_half2}
\end{align}
where the second equality uses \cref{gl_simp} and \cref{fl_grad_simp}. Combining \cref{pl_half1} and \cref{pl_half2} gives us the expression for ${p}_{l+1}(\rvx;\rvw_n,\rvw_z)$, which completes the proof for this case.\\

\noindent We next move towards proving the first case.\\
\textbf{Case 1 ($\alpha\neq 1$ and $p\geq 4$):} For all $2\leq l\leq L-1$, we claim that
\begin{equation}
\sigma(\rvh_{l}) = \begin{bmatrix}
\sigma({\rmN}_{l}\sigma(\rmN_{l-1}\sigma(\hdots\rmN_2\sigma(\rmN_1\rvx)))) + p_l(\rvx;\rvw_n,\rvw_z) + q_l(\rvx;\rvw_n,\rvw_z)\\s_l(\rvx;\rvw_n,\rvw_z) + r_l(\rvx;\rvw_n,\rvw_z)
\end{bmatrix},
\label{claim_relu}
\end{equation}
where $p_l(\rvx;\rvw_n,\rvw_z)$ and $s_l(\rvx;\rvw_n,\rvw_z)$ are a vector-valued homogeneous functions in $\rvw_z$ with degree of homogeneity  $p+1$ and $p$ respectively. More specifically,
\begin{align*}
p_l(\rvx;\rvw_n,\rvw_z) &= \sum_{k=1}^{l-2} \nabla_\rvs f_{l,k+2}(\rmN_{k+1}\rvg_{k}(\rvx))\rmB_{k+1}\sigma(\rmA_{k}\rvg_{k-1}(\rvx))\\
& +  \diag\left(\sigma'(\rmN_{l}\rvg_{l-1}(\rvx))\right)\rmB_l\sigma(\rmA_{l-1}\rvg_{l-2}(\rvx)), \\
s_l(\rvx;\rvw_n,\rvw_z)  &= \sigma(\rmA_l g_{l-1}(\rvx)).
\end{align*}
Furthermore, if $\|\rvw_z\|_2 = O(\delta)$, then 
\begin{align*}
&\|q_l(\rvx;\rvw_n,\rvw_z)\|_2  = O(\delta^{2p+1}) = \|\nabla_{\rvw_n}q_l(\rvx;\rvw_n,\rvw_z)\|_2 , \|\nabla_{\rvw_z}q_l(\rvx;\rvw_n,\rvw_z)\|_2 = O(\delta^{2p}), \text{ and }\\
&\|r_l(\rvx;\rvw_n,\rvw_z)\|_2  = O(\delta^{2p}) = \|\nabla_{\rvw_n}r_l(\rvx;\rvw_n,\rvw_z)\|_2 , \|\nabla_{\rvw_z}r_l(\rvx;\rvw_n,\rvw_z)\|_2 = O(\delta^{2p-1}).
\end{align*}
Now, if the claim stated in \cref{claim_relu} is true, then we have
\begin{align*}
&\mathcal{H}(\rvx;\rvw_n,\mathbf{w}_z)= \rmW_L\sigma(\rvh_{L-1})\\
&=\begin{bmatrix}
{{\rmN}_{L}}&\rmB_L
\end{bmatrix}\begin{bmatrix}
\sigma({\rmN}_{L-1}\sigma(\rmN_{L-2}\sigma(\cdots\sigma(\rmN_1\rvx)\cdots))) + p_{L-1}(\rvx;\rvw_n,\rvw_z) + q_{L-1}(\rvx;\rvw_n,\rvw_z)\\ {s}_{L-1}(\rvx;\rvw_n,\rvw_z) + r_{L-1}(\rvx;\rvw_n,\rvw_z)
\end{bmatrix}\\
&=  {{\rmN}_{L}}\sigma({\rmN}_{L-1}\sigma(\rmN_{L-2}\sigma(\cdots\sigma(\rmN_1\rvx)\cdots)))  +  {{\rmN}_{L}}{p}_{L-1}(\rvx;\rvw_n,\rvw_z) +  {{\rmN}_{L}}{q}_{L-1}(\rvx;\rvw_n,\rvw_z)\\
& + \rmB_L{s}_{L-1}(\rvx;\rvw_n,\rvw_z) + \rmB_L{r}_{L-1}(\rvx;\rvw_n,\rvw_z)\\
&= \mathcal{H}(\rvx;\rvw_n,\mathbf{0}) +  {{\rmN}_{L}}{p}_{L-1}(\rvx;\rvw_n,\rvw_z) +  {{\rmN}_{L}}{q}_{L-1}(\rvx;\rvw_n,\rvw_z) + \rmB_L{s}_{L-1}(\rvx;\rvw_n,\rvw_z)\\
& + \rmB_L{r}_{L-1}(\rvx;\rvw_n,\rvw_z),
\end{align*}
where the final equality follows since $\mathcal{H}(\rvx;\rvw_n,\mathbf{0}) = {{\rmN}_{L}}\sigma({\rmN}_{L-1}\sigma(\rmN_{L-2}\sigma(\cdots\sigma(\rmN_1\rvx)\cdots)))$. Now, define $\mathcal{H}_1(\rvx;\rvw_n,\rvw_z) \coloneqq {{\rmN}_{L}}{p}_{L-1}(\rvx;\rvw_n,\rvw_z) + \rmB_L{s}_{L-1}(\rvx;\rvw_n,\rvw_z)$, then it follows analogously to Case 2 that $\mathcal{H}_1(\rvx;\rvw_n,\rvw_z)$ is $(p+1)$-homogeneous in $\rvw_z$.

We next derive bounds on the remainder terms and gradient. Because the dependencies on $(\rvx;\rvw_n,\rvw_z)$ remain constant throughout this derivation, we omit them from the notation of $q_{L-1}$ and $r_{L-1}$ to prevent clutter. We will use this shorthand for other terms as well throughout the remainder of the proof. Assume $\|\rvw_z\|_2 = O(\delta)$, then
\begin{align*}
\| {{\rmN}_{L}}{q}_{L-1}(\rvx;\rvw_n,\rvw_z)+\rmB_Lr_{L-1}(\rvx;\rvw_n,\rvw_z)\|_2 &\leq  \|{{\rmN}_{L}}{q}_{L-1}\|_2  +  \|\rmB_Lr_{L-1}\|_2 \\
&\leq \|{{\rmN}_{L}}\|_2\|{q}_{L-1}\|_2  +  \|\rmB_L\|_2\|r_{L-1}\|_2\\
&= O(\delta^{2p+1}) + O(\delta)O(\delta^{2p}) = O(\delta^{2p+1}),
\end{align*}
where the penultimate equality follows since $\|q_{L-1}\|_2  = O(\delta^{2p+1})$, $\|r_{L-1}\|_2  = O(\delta^{2p})$, and $\|\rmB_L\|_2 = O(\delta)$. Next,
\begin{align*}
&\hspace{-3em}\|\nabla_{\rvw_z} {{\rmN}_{L}}{q}_{L-1}(\rvx;\rvw_n,\rvw_z)+\nabla_{\rvw_z} \rmB_Lr_{L-1}(\rvx;\rvw_n,\rvw_z)\|_2\\
&\leq \| {{\rmN}_{L}}\nabla_{\rvw_z}{q}_{L-1}\|_2+\|\nabla_{\rvw_z} \rmB_L\|_2\|r_{L-1}\|_2+ \|\rmB_L\|_2\|\nabla_{\rvw_z}r_{L-1}\|_2\\
&=  O(\delta^{2p}) + O(1)O(\delta^{2p}) + O(\delta)O(\delta^{2p-1}) = O(\delta^{2p}),
\end{align*}
where the penultimate equality follows since $\|\nabla_{\rvw_z}q_{L-1}\|_2  = O(\delta^{2p})$, $\|\nabla_{\rvw_z}r_{L-1}\|_2  = O(\delta^{2p-1})$, $\|r_{L-1}\|_2  = O(\delta^{2p})$, and $\|\rmB_L\|_2 = O(\delta)$. Also, since $\rmB_L$ is $1$-homogeneous in $\rvw_z$, we have $\|\nabla_{\rvw_z}\rmB_L\|_2 = O(1)$. Next,
\begin{align*}
&\hspace{-5em}\|\nabla_{\rvw_n} {{\rmN}_{L}}{q}_{L-1}(\rvx;\rvw_n,\rvw_z)+\nabla_{\rvw_n} \rmB_Lr_{L-1}(\rvx;\rvw_n,\rvw_z)\|_2\\
&\leq  \| \nabla_{\rvw_n}{{\rmN}_{L}}\|_2\|{q}_{L-1}\|_2+\|{{\rmN}_{L}}\|_2\| \nabla_{\rvw_n}{q}_{L-1}\|_2 +  \|\rmB_L\|_2\|\nabla_{\rvw_n}r_{L-1}\|_2\\
&=  O(\delta^{2p+1}) + O(\delta^{2p+1}) + O(\delta) O(\delta^{2p})   = O(\delta^{2p+1}),
\end{align*}
where the penultimate equality follows since $\|\nabla_{\rvw_n}q_{L-1}\|_2  = O(\delta^{2p+1})$, $\|\nabla_{\rvw_n}r_{L-1}\|_2  = O(\delta^{2p})$, $\|q_{L-1}\|_2  = O(\delta^{2p+1})$, and $\|\rmB_L\|_2 = O(\delta)$.\\
Finally, using the definition of ${p}_{L-1}(\rvx;\rvw_n,\rvw_z)$ and ${s}_{L-1}(\rvx;\rvw_n,\rvw_z)$, we get
\begin{align*}
\mathcal{H}_1(\rvx;\rvw_n,\mathbf{w}_z)
&=  \sum_{k=1}^{L-3} \rmN_L\nabla_\rvs f_{L-1,k+2}(\rmN_{k+1}\rvg_{k}(\rvx))\rmB_{k+1}\sigma(\rmA_{k}\rvg_{k-1}(\rvx)) \\
&+  \rmN_L\diag\left(\sigma'(\rmN_{L-1}\rvg_{L-2}(\rvx))\right)\rmB_{L-1}\sigma(\rmA_{L-2}\rvg_{L-3}(\rvx)) + \rmB_L \sigma(\rmA_{L-1}g_{L-2}(\rvx)).
\end{align*}
Then, analogous to Case 2, we get
\begin{equation*}
\mathcal{H}_1(\rvx;\rvw_n,\mathbf{w}_z)
=  \sum_{k=1}^{L-2}\nabla_\rvs f_{L,k+2}^\top(\rmN_{k+1}\rvg_{k}(\rvx))\rmB_{k+1}\sigma(\rmA_{k}\rvg_{k-1}(\rvx)) + \rmB_L \sigma(\rmA_{L-1}g_{L-2}(\rvx)),
\end{equation*}
which completes the proof.\\
We next prove the claim in \cref{claim_relu} via induction. From Taylor's theorem, we know
\begin{equation*}
\sigma(x+h) = \sigma(x) +\int_0^h{\sigma'(x+t)} dt, \ \ \sigma(x+h) = \sigma(x) + \sigma'(x)h +\int_0^h{\sigma''(x+t)} (h-t) dt.
\end{equation*} 
Hence, for $l=2$, we have
\begin{align*}
\sigma(\rvh_2) &= \sigma\left(\begin{bmatrix}
{{\rmN}_{2}}&\rmB_2  \\
\rmA_2& \rmC_2 \\
\end{bmatrix}\sigma\left( \begin{bmatrix}
{{\rmN}_{1}\rvx}\\\rmA_1\rvx
\end{bmatrix}\right)\right) = \sigma\left(\begin{bmatrix}
{{\rmN}_{2}}\sigma( {{\rmN}_{1}\rvx}) + \rmB_2\sigma(\rmA_1\rvx) \\ \rmA_2\sigma( {{\rmN}_{1}\rvx})  + \rmC_2\sigma(\rmA_1\rvx)
\end{bmatrix}\right)\\
&= \begin{bmatrix}
\sigma( {{\rmN}_{2}}\sigma( {{\rmN}_{1}\rvx})) + \sigma'( {{\rmN}_{2}}\sigma( {{\rmN}_{1}\rvx}))\odot  (\rmB_2\sigma(\rmA_1\rvx)) +  q_2(\rvx;\rvw_n,\rvw_z)  \\ \sigma(\rmA_2\sigma( {{\rmN}_{1}\rvx}))  +  r_2(\rvx;\rvw_n,\rvw_z)	\end{bmatrix},
\end{align*}
where 
\begin{align*}
&q_2(\rvx;\rvw_n,\rvw_z) = \int_\mathbf{0}^{\rmB_2\sigma(\rmA_1\rvx)} {\sigma''(	 {{\rmN}_{2}}\sigma( {{\rmN}_{1}\rvx}) + \rvt)} \odot (\rmB_2\sigma(\rmA_1\rvx) - \rvt) d\rvt, \text{ and }\\
&r_2(\rvx;\rvw_n,\rvw_z) =  \int_\mathbf{0}^{\rmC_2\sigma(\rmA_1\rvx)}{\sigma'(	\rmA_2\sigma( {{\rmN}_{1}\rvx}) + \rvt)}  d\rvt.
\end{align*}
Define $p_2(\rvx;\rvw_n,\rvw_z) \coloneqq \sigma'( {{\rmN}_{2}}\sigma( {{\rmN}_{1}\rvx}))\odot  (\rmB_2\sigma(\rmA_1\rvx))$ and $s_2(\rvx;\rvw_n,\rvw_z) \coloneqq \sigma(\rmA_2\sigma( {{\rmN}_{1}\rvx}))$. Then, $p_2$ and $s_2$ are $(p+1)$-homogeneous and $p$-homogeneous in $\rvw_z$, respectively.

We next derive bounds on $ q_2(\rvx;\rvw_n,\rvw_z)$, $ r_2(\rvx;\rvw_n,\rvw_z)$ and their gradients. Since  $\|\rvw_z\|_2 = O(\delta)$, then $\|\rmB_2\sigma(\rmA_1\rvx)\|_2 = O(\delta^{p+1})$,   $\|\nabla_{\rvw_n}\rmB_2\sigma(\rmA_1\rvx)\|_2 = 0$, and $\|\nabla_{\rvw_z}\rmB_2\sigma(\rmA_1\rvx)\|_2 =  O(\delta^{p})$, where the last equality is true since $\rmB_2\sigma(\rmA_1\rvx)$ is $(p+1)$-homogeneous in $\rvw_z$. Hence, from Lemma \ref{int_bd}, we get
\begin{align*}
\|q_2\|_2 = O(\delta^{2p+1}), \|\nabla_{\rvw_z}q_2\|_2 = O(\delta^{2p}), \text{ and }  \|\nabla_{\rvw_n}q_2\|_2 = O(\delta^{2p+1}).
\end{align*}	
We also have $\|\rmC_2\sigma(\rmA_1\rvx)\|_2 = O(\delta^{p+1})$,   $\|\nabla_{\rvw_n}\rmC_2\sigma(\rmA_1\rvx)\|_2 = 0$, and $\|\nabla_{\rvw_z}\rmC_2\sigma(\rmA_1\rvx)\|_2 =  O(\delta^{p})$. Hence, from Lemma \ref{int_bd}, we get
\begin{align*}
\|\rvr_2\|_2 = O(\delta^{2p}), \|\nabla_{\rvw_z}\rvr_2\|_2 = O(\delta^{2p-1}), \text{ and }  \|\nabla_{\rvw_n}\rvr_2\|_2 = O(\delta^{2p}).
\end{align*}	
Therefore, the claim is true for $l=2$. Now, suppose the claim is true for some $2<l<L-1$.  For brevity, let $\widetilde{p}_l(\rvx;\rvw_n,\rvw_z) \coloneqq p_l(\rvx;\rvw_n,\rvw_z) + q_l(\rvx;\rvw_n,\rvw_z)$ and $\widetilde{s}_l(\rvx;\rvw_n,\rvw_z) \coloneqq  s_l(\rvx;\rvw_n,\rvw_z) + r_l(\rvx;\rvw_n,\rvw_z)$. Then,
\begin{align*}
\sigma(\rvh_{l+1}) &= \sigma\left(\begin{bmatrix}
{{\rmN}_{l+1}}&\rmB_{l+1} \\
\rmA_{l+1}& \rmC_{l+1}\\
\end{bmatrix}\begin{bmatrix}
\sigma({\rmN}_{l}\sigma(\cdots\sigma(\rmN_1\rvx)\cdots)) + p_l(\rvx;\rvw_n,\rvw_z) + q_l(\rvx;\rvw_n,\rvw_z)\\s_l(\rvx;\rvw_n,\rvw_z) + r_l(\rvx;\rvw_n,\rvw_z)
\end{bmatrix}\right)\\
&= \sigma\left(\begin{bmatrix}
{{\rmN}_{l+1}}&\rmB_{l+1} \\
\rmA_{l+1}& \rmC_{l+1}\\
\end{bmatrix}\begin{bmatrix}
\sigma({\rmN}_{l}\sigma(\cdots\sigma(\rmN_1\rvx)\cdots)) + \widetilde{p}_l(\rvx;\rvw_n,\rvw_z) \\\widetilde{s}_l(\rvx;\rvw_n,\rvw_z)
\end{bmatrix}\right)\\
&= \sigma\left(\begin{bmatrix}
{{\rmN}_{l+1}}\sigma({\rmN}_{l}\sigma(\cdots\sigma(\rmN_1\rvx)\cdots)) + {{\rmN}_{l+1}}\widetilde{p}_l(\rvx;\rvw_n,\rvw_z) + {{\rmB}_{l+1}}\widetilde{s}_l(\rvx;\rvw_n,\rvw_z) \\\rmA_{l+1}\sigma({\rmN}_{l}\sigma(\cdots\sigma(\rmN_1\rvx)\cdots))  + \rmA_{l+1}\widetilde{p}_l(\rvx;\rvw_n,\rvw_z) +  \rmC_{l+1}\widetilde{s}_l(\rvx;\rvw_n,\rvw_z)
\end{bmatrix}\right)
\end{align*}
Let 
\begin{align*}
&\hat{p}_l(\rvx;\rvw_n,\rvw_z) \coloneqq {{\rmN}_{l+1}}\widetilde{p}_l(\rvx;\rvw_n,\rvw_z) + {{\rmB}_{l+1}}\widetilde{s}_l(\rvx;\rvw_n,\rvw_z),\\
& \hat{s}_l(\rvx;\rvw_n,\rvw_z)  \coloneqq \rmA_{l+1}\widetilde{p}_l(\rvx;\rvw_n,\rvw_z) +  \rmC_{l+1}\widetilde{s}_l(\rvx;\rvw_n,\rvw_z),  \text{ and } \rvx_l = \sigma({\rmN}_{l}\cdots\sigma(\rmN_1\rvx).
\end{align*}
Then, using Taylor's theorem,
\begin{align*}
\sigma(\rvh_{l+1}) &= \begin{bmatrix}
\sigma\left({{\rmN}_{l+1}}\rvx_l + \hat{p}_l(\rvx;\rvw_n,\rvw_z)\right) \\\sigma\left(\rmA_{l+1}\rvx_l  + \hat{s}_l(\rvx;\rvw_n,\rvw_z)\right)
\end{bmatrix}\\
&= \begin{bmatrix}
\sigma\left({{\rmN}_{l+1}}\rvx_l\right) + \sigma'\left({{\rmN}_{l+1}}\rvx_l\right)\odot\hat{p}_l(\rvx;\rvw_n,\rvw_z) + \hat{q}_l(\rvx;\rvw_n,\rvw_z) \\\sigma\left(\rmA_{l+1}\rvx_l\right)  + {r}_{l+1}(\rvx;\rvw_n,\rvw_z)
\end{bmatrix},
\end{align*}
where 
\begin{align*}
&\hat{q}_l(\rvx;\rvw_n,\rvw_z) = \int_\mathbf{0}^{\hat{p}_l(\rvx;\rvw_n,\rvw_z) } \sigma''({\rmN}_{l+1}\rvx_l+\rvt) \odot (\hat{p}_l(\rvx;\rvw_n,\rvw_z)  - \rvt) d\rvt, \text{ and }\\
&{r}_{l+1}(\rvx;\rvw_n,\rvw_z) =  \int_\mathbf{0}^{\hat{s}_l(\rvx;\rvw_n,\rvw_z)}{\sigma'({\rmA}_{l+1}\rvx_l+\rvt)} d\rvt.
\end{align*}
We next define
\begin{align*}
&p_{l+1}(\rvx;\rvw_n,\rvw_z) = \sigma'({\rmN}_{l+1}\rvx_l) \odot\left( {\rmN}_{l+1}{p}_l(\rvx;\rvw_n,\rvw_z) + 	{\rmB}_{l+1}{s}_l(\rvx;\rvw_n,\rvw_z) \right),\\
&q_{l+1}(\rvx;\rvw_n,\rvw_z) = \sigma'({\rmN}_{l+1}\rvx_l) \odot\left( {\rmN}_{l+1}{q}_l(\rvx;\rvw_n,\rvw_z) + 	{\rmB}_{l+1}{r}_l(\rvx;\rvw_n,\rvw_z) \right) + \hat{q}_l(\rvx;\rvw_n,\rvw_z),\\
&s_{l+1}(\rvx;\rvw_n,\rvw_z) =\sigma({\rmA}_{l+1}\rvx_l).
\end{align*}
Then, by expanding  $\hat{p}_l(\rvx;\rvw_n,\rvw_z)$, we get
\begin{equation*}
\sigma(\rvh_{l+1}) = \begin{bmatrix}
\sigma\left({{\rmN}_{l+1}}\sigma({\rmN}_{l}\sigma(\hdots\sigma(\rmN_1\rvx))) \right) + {p}_{l+1}(\rvx;\rvw_n,\rvw_z) + {q}_{l+1}(\rvx;\rvw_n,\rvw_z) \\\sigma\left(\rmA_{l+1}\sigma({\rmN}_{l}\sigma(\hdots\sigma(\rmN_1\rvx)))  \right) + {s}_{l+1}(\rvx;\rvw_n,\rvw_z) + {r}_{l+1}(\rvx;\rvw_n,\rvw_z)
\end{bmatrix},
\end{equation*}
The $(p+1)$-homogeneity of $ {{\rmN}_{l+1}}{p}_l(\rvx;\rvw_n,\rvw_z) + {{\rmB}_{l+1}}{s}_l(\rvx;\rvw_n,\rvw_z)$ in $\rvw_z$ implies $(p+1)$-homogeneity of  $p_{l+1}(\rvx;\rvw_n,\rvw_z)$ in $\rvw_z$. Also, $s_{l+1}(\rvx;\rvw_n,\rvw_z)$ is $p$-homogeneous in $\rvw_z$.  \\
We next derive bounds on $q_{l+1}(\rvx;\rvw_n,\rvw_z)$, $r_{l+1}(\rvx;\rvw_n,\rvw_z)$ and their gradients. To that end, we first derive relevant bounds for $\widetilde{p}_l, \widetilde{s}_l, \hat{p}_l, \hat{s}_l,$ and $\hat{q}_l$.  Since $p_{l}$ is $(p+1)$-homogeneous in $\rvw_z$, $\|q_l\|_2  = O(\delta^{2p+1}), \|\nabla_{\rvw_z}q_l\|_2 = O(\delta^{2p}), $ and $\|\nabla_{\rvw_n}q_l\|_2 = O(\delta^{2p+1})$, we have
\begin{align*}
&\|\widetilde{p}_l(\rvx;\rvw_n,\rvw_z)\|_2\leq \|{p}_l\|_2 + \|{q}_l\|_2   = O(\delta^{p+1}),\\ 
&\|\nabla_{\rvw_z}\widetilde{p}_l(\rvx;\rvw_n,\rvw_z)\|_2 \leq \|\nabla_{\rvw_z}{p}_l\|_2 + \|\nabla_{\rvw_z}{q}_l\|_2 = O(\delta^{p}),\\ &\|\nabla_{\rvw_n}\widetilde{p}_l(\rvx;\rvw_n,\rvw_z)\|_2 \leq \|\nabla_{\rvw_n}{p}_l\|_2 + \|\nabla_{\rvw_n}{q}_l\|_2 = O(\delta^{p+1}),
\end{align*}
where the final equality is true since, from Lemma \ref{subset_homog}, $\nabla_{\rvw_n}{p}_l$ is $(p+1)$-homogeneous in $\rvw_z$. Next, since $s_{l}$ is $p$-homogeneous in $\rvw_z$, $\|r_l\|_2  = O(\delta^{2p}), \|\nabla_{\rvw_z}r_l\|_2 = O(\delta^{2p-1}),$ and $ \|\nabla_{\rvw_n}r_l\|_2 = O(\delta^{2p})$, we have
\begin{align*}
&\|\widetilde{s}_l(\rvx;\rvw_n,\rvw_z)\|_2\leq \|{s}_l\|_2+\|{r}_l\|_2  = O(\delta^{p}),\\ 
&\|\nabla_{\rvw_z}\widetilde{s}_l(\rvx;\rvw_n,\rvw_z)\|_2 \leq \|\nabla_{\rvw_z}{s}_l\|_2+\|\nabla_{\rvw_z}{r}_l\|_2 = O(\delta^{p-1}),\\ 
&\|\nabla_{\rvw_n}\widetilde{s}_l(\rvx;\rvw_n,\rvw_z)\|_2 \leq \|\nabla_{\rvw_n}{s}_l\|_2+\|\nabla_{\rvw_n}{r}_l\|_2 = O(\delta^{p}),
\end{align*}
where the final equality is true since, from \Cref{subset_homog}, $\nabla_{\rvw_n}{s}_l$ is $p$-homogeneous in $\rvw_z$. Using the above two sets of inequalities, we get
\begin{align*}
\|\hat{p}_{l}(\rvx;\rvw_n,\rvw_z)\|_2 &\leq \|\rmN_{l+1}\|_2\|\widetilde{p}_l\|_2  + \|\rmB_{l+1}\|_2\|\widetilde{s}_l\|_2 \\
&=O(\delta^{p+1}) + O(\delta)O(\delta^{p}) = O(\delta^{p+1}),\\
\|\nabla_{\rvw_z}\hat{p}_{l}(\rvx;\rvw_n,\rvw_z)\|_2 &\leq\|\rmN_{l+1}\|_2\|\nabla_{\rvw_z}\widetilde{p}_l\|_2  +\|\rmB_{l+1}\|_2\|\nabla_{\rvw_z}\widetilde{s}_l\|_2 + \|\nabla_{\rvw_z}\rmB_{l+1}\|_2\|\widetilde{s}_l\|_2\\
&= O(\delta^p) +  O(\delta)O(\delta^{p-1})+O(1)O(\delta^p) = O(\delta^p),\\
\|\nabla_{\rvw_n}\hat{p}_{l}(\rvx;\rvw_n,\rvw_z)\|_2&\leq\|\rmN_{l+1}\|_2\|\nabla_{\rvw_n}\widetilde{p}_l\|_2 +\|\nabla_{\rvw_n}\rmN_{l+1}\|_2\|\widetilde{p}_l\|_2 +\|\rmB_{l+1}\|_2\|\nabla_{\rvw_n}\widetilde{s}_l\|_2\\
&=O(\delta^{p+1}) + O(\delta^{p+1})  + O(\delta)O(\delta^{p}) = O(\delta^{p+1}),
\end{align*}
and
\begin{align*}
\|\hat{s}_{l}(\rvx;\rvw_n,\rvw_z)\|_2 &\leq \|\rmA_{l+1}\|_2\|\widetilde{p}_l\|_2  + \|\rmC_{l+1}\|_2\|\widetilde{s}_l\|_2 \\
&=O(\delta)O(\delta^{p+1}) + O(\delta)O(\delta^{p}) = O(\delta^{p+1}),\\
\|\nabla_{\rvw_z}\hat{s}_{l}(\rvx;\rvw_n,\rvw_z)\|_2 &\leq\|\rmA_{l+1}\|_2\|\nabla_{\rvw_z}\widetilde{p}_l\|_2  + \|\nabla_{\rvw_z}\rmA_{l+1}\|_2\|\widetilde{p}_l\|_2 \\
&+\|\rmC_{l+1}\|_2\|\nabla_{\rvw_z}\widetilde{s}_l\|_2 + \|\nabla_{\rvw_z}\rmC_{l+1}\|_2\|\widetilde{s}_l\|_2\\
&= O(\delta)O(\delta^p) + O(1)O(\delta^{p+1}) + O(\delta)O(\delta^{p-1})+O(1)O(\delta^p) = O(\delta^p),\\
\|\nabla_{\rvw_n}\hat{s}_{l}(\rvx;\rvw_n,\rvw_z)\|_2&\leq\|\rmA_{l+1}\|_2\|\nabla_{\rvw_n}\widetilde{p}_l\|_2 +\|\rmC_{l+1}\|_2\|\nabla_{\rvw_n}\widetilde{s}_l\|_2\\
&=O(\delta)O(\delta^{p+1})  + O(\delta)O(\delta^{p}) = O(\delta^{p+1}).
\end{align*}
Since $\|\hat{p}_l\|_2  = O(\delta^{p+1}), \|\nabla_{\rvw_z}\hat{p}_l\|_2 = O(\delta^{p}), \|\nabla_{\rvw_n}\hat{p}_l\|_2 = O(\delta^{p+1})$, from Lemma \ref{int_bd}, we get
\begin{equation*}
\|\hat{q}_l\|_2  = O(\delta^{2p+2}), \|\nabla_{\rvw_z}\hat{q}_l\|_2 = O(\delta^{2p+1}), \|\nabla_{\rvw_n}\hat{q}_l\|_2 = O(\delta^{2p+2}).
\end{equation*}
For brevity, define $\zeta\coloneqq \sigma'({\rmN}_{l+1}\rvx_l)$. Then,
\begin{align*}
\|q_{l+1}(\rvx;\rvw_n,\rvw_z)\|_2 &\leq \|\zeta\|_2\left(\|\rmN_{l+1}\|_2\|q_l\|_2 + \|\rmB_{l+1}\|_2\|r_l\|_2\right) + \|\hat{q}_l\|_2\\
&= O(\delta^{2p+1}) + O(\delta)O(\delta^{2p}) + O(\delta^{2p+2}) = O(\delta^{2p+1}),\\
\|\nabla_{\rvw_z}q_{l+1}(\rvx;\rvw_n,\rvw_z)\|_2 &\leq\|\zeta\|_2\|\rmN_{l+1}\|_2\|\nabla_{\rvw_z}q_l\|_2 +  \|\nabla_{\rvw_z}\hat{q}_l\|_2\\
&+ \|\zeta\|_2\left(\|\nabla_{\rvw_z}\rmB_{l+1}\|_2\|r_l\|_2 +\|\rmB_{l+1}\|_2\|\nabla_{\rvw_z}r_l\|_2 \right) \\
&= O(\delta^{2p}) +   O(\delta^{2p+1})+O(1)O(\delta^{2p}) + O(\delta)O(\delta^{2p-1}) = O(\delta^{2p}),\\
\|\nabla_{\rvw_n}q_{l+1}(\rvx;\rvw_n,\rvw_z)\|_2 &\leq\|\nabla_{\rvw_n}\zeta\|_2\left(\|\rmN_{l+1}\|_2\|q_l\|_2 + \|\rmB_{l+1}\|_2\|r_l\|_2\right)\\
&+ \|\zeta\|_2\left(\|\nabla_{\rvw_n}\rmN_{l+1}\|_2\|q_l\|_2 +\|\rmN_{l+1}\|_2\|\nabla_{\rvw_n}q_l\|_2\right)\\
&+ \|\zeta\|_2 \|\rmB_{l+1}\|_2\|\nabla_{\rvw_n}r_l\|_2 + \|\nabla_{\rvw_n}\hat{q}_l\|_2\\
&= O(\delta^{2p+1}) + O(\delta)O(\delta^{2p}) + O(\delta^{2p+1}) + O(\delta^{2p+1}) \\
&+ O(\delta)O(\delta^{2p}) + O(\delta^{2p+2}) = O(\delta^{2p+1}).
\end{align*} 
Since  $\|\hat{s}_l\|_2  = O(\delta^{p+1}), \|\nabla_{\rvw_z}\hat{s}_l\|_2 = O(\delta^{p}), \|\nabla_{\rvw_n}\hat{s}_l\|_2 = O(\delta^{p+1})$, from Lemma \ref{int_bd}, we get
\begin{equation*}
\|r_{l+1}\|_2  = O(\delta^{2p}), \|\nabla_{\rvw_z}r_{l+1}\|_2 = O(\delta^{2p-1}), \|\nabla_{\rvw_n}r_{l+1}\|_2 = O(\delta^{2p}),
\end{equation*}
which proves the claim in \cref{claim_relu}. \\
Finally, we simplify  $p_{l+1}(\rvx;\rvw_n,\rvw_z)$. Here as well, for all $l\in [L-1]$ and $k\leq l$, we have 
\begin{equation}
\rvx_l = g_{l}(\rvx) = \sigma({\rmN}_{l}\sigma({\rmN}_{l-1}\hdots\sigma(\rmN_1\rvx)) = f_{l,k}(\rmN_{k-1}g_{k-2}(\rvx)).
\label{gl_simp1}
\end{equation}
Since $f_{l+1,k}(\rvs) = \sigma(\rmN_{l+1}f_{l,k}(\rvs))$, we get
\begin{equation}
\nabla_\rvs f_{l+1,k}(\rvs) = \diag(\sigma'(\rmN_{l+1}f_{l,k}(\rvs)))\rmN_{l+1}\nabla_\rvs f_{l,k}(\rvs).
\label{fl_grad_simp1}
\end{equation}
Hence, analogous to Case 2, we can show that
\begin{align*}
p_{l+1}(\rvx;\rvw_n,\rvw_z) &= \sum_{k=1}^{l-1} \nabla_\rvs f_{l+1,k+2}(\rmN_{k+1}\rvg_{k}(\rvx))\rmB_{k+1}\sigma(\rmA_{k}\rvg_{k-1}(\rvx)) \\
&+ \diag(\sigma'(\rmN_{l+1}g_{l}(\rvx)))\rmB_{l+1}\sigma(\rmA_lg_{l-1}(\rvx)),
\end{align*}
which completes the proof.\hfill \ensuremath{\blacksquare}

\begin{lemma}
	Consider the setting of Lemma \ref{lemma_poly_str}, and suppose $\rvw_n$ is fixed and $\|\rvw_z\|_2 = O(\delta)$. For any $2\leq l\leq L-1$, let
	\begin{align*}
	&{a}(\rvx;\rvw_n,\rvw_z) = \int_\mathbf{0}^{{c}(\rvx;\rvw_n,\rvw_z) } \sigma''({{\rmN}_{l+1}}\sigma({\rmN}_{l}\sigma(\cdots\sigma(\rmN_1\rvx)\cdots))+\rvt) \odot ({c}(\rvx;\rvw_n,\rvw_z)  - \rvt) d\rvt,\\
	&{b}(\rvx;\rvw_n,\rvw_z) =  \int_\mathbf{0}^{{c}(\rvx;\rvw_n,\rvw_z)}{\sigma'({{\rmA}_{l+1}}\sigma({\rmN}_{l}\sigma(\cdots\sigma(\rmN_1\rvx)\cdots))+\rvt)} d\rvt,
	\end{align*}
	where $\|c(\rvx;\rvw_n,\rvw_z) \|_2 = O(\delta^{p+1})$, $\|\nabla_{\rvw_z}c(\rvx;\rvw_n,\rvw_z) \|_2 = O(\delta^{p})$ and $\|\nabla_{\rvw_n}c(\rvx;\rvw_n,\rvw_z) \|_2 = O(\delta^{p+1})$. Then, 
	\begin{align*}
	& \|b(\rvx;\rvw_n,\rvw_z) \|_2 = O(\delta^{2p}) = \|\nabla_{\rvw_n}b(\rvx;\rvw_n,\rvw_z) \|_2 , \|\nabla_{\rvw_z}b(\rvx;\rvw_n,\rvw_z) \|_2 = O(\delta^{2p-1}), \text{ and }\\
	&\|a(\rvx;\rvw_n,\rvw_z) \|_2 = O(\delta^{2p+2}) = \|\nabla_{\rvw_n}a(\rvx;\rvw_n,\rvw_z) \|_2, \|\nabla_{\rvw_z}a(\rvx;\rvw_n,\rvw_z) \|_2 = O(\delta^{2p+1})
	\end{align*}
	\label{int_bd}
\end{lemma}
\begin{proof}
	Since $\rvw_n$ is fixed and  $\|\rvw_z\|_2 = O(\delta)$, there exists a constant $C>0$ such that
	\begin{align*}
	\|{a}(\rvx;\rvw_n,\rvw_z)\|_2 &\leq \left\| \int_\mathbf{0}^{{c}(\rvx;\rvw_n,\rvw_z) } C({c}(\rvx;\rvw_n,\rvw_z)  - \rvt) d\rvt\right\|_2\\
	& = \frac{C}{2}\left\| {{c}(\rvx;\rvw_n,\rvw_z) } \odot{c}(\rvx;\rvw_n,\rvw_z)\right\|_2 = O(\delta^{2p+2}).
	\end{align*}
	Next, since $\|\rmA_{l+1}\|_2 = O(\delta)$,  $\|c(\rvx;\rvw_n,\rvw_z) \|_2 = O(\delta^{p+1})$ and $\sigma'(x)$ is $(p-1)$-homogeneous,  there exists a large enough constant $C>0$ such that
	\begin{align*}
	\|{b}(\rvx;\rvw_n,\rvw_z)\|_2 &= \left\|\int_\mathbf{0}^{{c}(\rvx;\rvw_n,\rvw_z)}{\sigma'({\rmA}_{l+1}\sigma({\rmN}_{l}\sigma(\cdots\sigma(\rmN_1\rvx)\cdots))+\rvt)} d\rvt\right\|_2\\
	&\leq C\delta^{p-1} \left\|\int_\mathbf{0}^{{c}(\rvx;\rvw_n,\rvw_z)} d\rvt\right\|_2\\
	& = C\delta^{p-1}\|c(\rvx;\rvw_n,\rvw_z) \|_2 = O(\delta^{2p}). 
	\end{align*} 
	Using Leibniz's rule, for any $\mathcal{Q}:\sR^d \times \sR \rightarrow \sR$ and $m:\sR^d \rightarrow \sR$, where $\mathcal{Q}(\rvs,t)$, $\nabla_\rvs\mathcal{Q}(\rvs,t)$, $m(\rvs)$ and $\nabla_\rvs m(\rvs)$ are continuous in $\rvs$, we have 
	\begin{equation*}
	\nabla_\rvs \left(\int_0^{m(\rvs)} \mathcal{Q}(\rvs,t)  dt\right)= \mathcal{Q}(\rvs, m(\rvs))\nabla_\rvs m(\rvs) + \int_0^{m(\rvs)} \nabla_\rvs \mathcal{Q}(\rvs,t)  dt.
	\end{equation*}
	Let $a_i(\rvx;\rvw_n,\rvw_z)$ and $c_i(\rvx;\rvw_n,\rvw_z)$ denote the $i$th entry of $a(\rvx;\rvw_n,\rvw_z)$ and $c(\rvx;\rvw_n,\rvw_z)$, respectively. Thus,
	\begin{align*}
	a_i(\rvx;\rvw_n,\rvw_z) = \left(\int_{0}^{{c}_i(\rvx;\rvw_n,\rvw_z) } \sigma''({\rmN}_{l+1}[i,:]\sigma({\rmN}_{l}\sigma(\cdots\sigma(\rmN_1\rvx)\cdots))+t)  ({c}_i(\rvx;\rvw_n,\rvw_z)  - t) dt\right),
	\end{align*}
	  since $p\geq 4$, $\sigma''(x)$ has continuous derivatives. Using Leibniz's rule, we get
	\begin{align*}
	&\hspace{-5em}\nabla_{\rvw_z}a_i(\rvx;\rvw_n,\rvw_z)\\
	&= \sigma''({\rmN}_{l+1}[i,:]\sigma({\rmN}_{l}\sigma(\cdots\sigma(\rmN_1\rvx)\cdots))+{c}_i)  ({c}_i- {c}_i) \nabla_{\rvw_z}{c}_i\\
	&+\int_{0}^{{c}_i(\rvx;\rvw_n,\rvw_z) } \sigma''({\rmN}_{l+1}[i,:]\sigma({\rmN}_{l}\sigma(\cdots\sigma(\rmN_1\rvx)\cdots))+t)  \nabla_{\rvw_z}{c}_idt\\
	&= \int_{0}^{{c}_i(\rvx;\rvw_n,\rvw_z) } \sigma''({\rmN}_{l+1}[i,:]\sigma({\rmN}_{l}\sigma(\cdots\sigma(\rmN_1\rvx)\cdots))+t)  \nabla_{\rvw_z}{c}_idt,
	\end{align*}
	where we have omitted $(\rvx;\rvw_n,\rvw_z)$ from the notation of $c_i$ to avoid clutter. Since $\rvw_n$ is fixed and  $\|\rvw_z\|_2 = O(\delta)$, there exists a large enough constant $C>0$ such that
	\begin{align*}
	&\left\|\nabla_{\rvw_z}a_i(\rvx;\rvw_n,\rvw_z)\right\|_2 \leq C|{c}_i(\rvx;\rvw_n,\rvw_z)| \|\nabla_{\rvw_z}{c}_i(\rvx;\rvw_n,\rvw_z)\|_2
	\end{align*}
implies
\begin{align*}
	&\left\|\nabla_{\rvw_z}a(\rvx;\rvw_n,\rvw_z)\right\|_2 =  O(\delta^{p+1}\delta^{p}) =  O(\delta^{2p+1}).
	\end{align*}
	Let $b_i(\rvx;\rvw_n,\rvw_z)$ denote the $i$th entry of $b(\rvx;\rvw_n,\rvw_z)$. Then,
	\begin{align*}
	\nabla_{\rvw_z}b_i(\rvx;\rvw_n,\rvw_z) &= \nabla_{\rvw_z}\left(\int_{0}^{{c}_i(\rvx;\rvw_n,\rvw_z)}{\sigma'({\rmA}_{l+1}[i,:]\sigma({\rmN}_{l}\sigma(\cdots\sigma(\rmN_1\rvx)\cdots))+t)} dt\right)\\
	&= \sigma'({\rmA}_{l+1}[i,:]\sigma({\rmN}_{l}\sigma(\cdots\sigma(\rmN_1\rvx)\cdots))+{c}_i)\nabla_{\rvw_z}{c}_i\\
	&+\int_{0}^{{c}_i(\rvx;\rvw_n,\rvw_z) } \nabla_{\rvw_z}\sigma'({\rmA}_{l+1}[i,:]\sigma({\rmN}_{l}\sigma(\cdots\sigma(\rmN_1\rvx)\cdots))+t)dt.
	\end{align*}
	Since $\rvw_n$ is fixed and  $\|\rvw_z\|_2 = O(\delta)$, there exists a large enough constant $C>0$ such that
	\begin{align*}
	&\left\|\nabla_{\rvw_z}b_i(\rvx;\rvw_n,\rvw_z)\right\|_2 \leq C\|\rmA_{l+1}[i,:]\|_2^{p-1}\|\nabla_{\rvw_z}{c}_i\|_2  + C\|\rmA_{l+1}[i,:]\|_2^{p-2}|{c}_i|
	\end{align*}
	implies
	\begin{align*}
	&\left\|\nabla_{\rvw_z}b(\rvx;\rvw_n,\rvw_z)\right\|_2 =  O(\delta^{p-1}\delta^{p}) + O(\delta^{p-2}\delta^{p+1})  =  O(\delta^{2p-1}).
	\end{align*}
	Next,
	\begin{align*}
	&\hspace{-3em}\nabla_{\rvw_n}a_i(\rvx;\rvw_n,\rvw_z)\\
	&= \sigma''({\rmN}_{l+1}[i,:]\sigma({\rmN}_{l}\sigma(\cdots\sigma(\rmN_1\rvx)\cdots))+c_i)  ({c}_i  - {c}_i) \nabla_{\rvw_n}{c}_i\\
	&+\int_{0}^{{c}_i(\rvx;\rvw_n,\rvw_z) } \sigma''({\rmN}_{l+1}[i,:]\sigma({\rmN}_{l}\sigma(\cdots\sigma(\rmN_1\rvx)\cdots))+t)  \nabla_{\rvw_n}{c}_idt\\
	&+\int_{0}^{{c}_i(\rvx;\rvw_n,\rvw_z) } \nabla_{\rvw_n}\sigma''({\rmN}_{l+1}[i,:]\sigma({\rmN}_{l}\sigma(\cdots\sigma(\rmN_1\rvx)\cdots))+t) {c}_idt\\
	&= \int_{0}^{{c}_i(\rvx;\rvw_n,\rvw_z) } \sigma''({\rmN}_{l+1}[i,:]\sigma({\rmN}_{l}\sigma({\rmN}_{l}\sigma(\cdots\sigma(\rmN_1\rvx)\cdots))+t)  \nabla_{\rvw_n}{c}_idt\\
	&+\int_{0}^{{c}_i(\rvx;\rvw_n,\rvw_z) } \nabla_{\rvw_n}\sigma''({\rmN}_{l+1}[i,:]\sigma({\rmN}_{l}\sigma({\rmN}_{l}\sigma(\cdots\sigma(\rmN_1\rvx)\cdots))+t) {c}_idt.
	\end{align*}
	Since $\rvw_n$ is fixed and  $\|\rvw_z\|_2 = O(\delta)$, there exists a large enough constant $C>0$ such that
	\begin{align*}
	&\left\|\nabla_{\rvw_n}a_i(\rvx;\rvw_n,\rvw_z)\right\|_2 \leq C|{c}_i(\rvx;\rvw_n,\rvw_z)| \|\nabla_{\rvw_n}{c}_i(\rvx;\rvw_n,\rvw_z)\|_2 + C|{c}_i(\rvx;\rvw_n,\rvw_z)|^2
	\end{align*}
	implies
	\begin{align*}
	&\left\|\nabla_{\rvw_n}a(\rvx;\rvw_n,\rvw_z)\right\|_2 =  O(\delta^{p+1}\delta^{p+1}) + O(\delta^{2p+2}) =   O(\delta^{2p+2}).
	\end{align*}
	Next,
	\begin{align*}
	\nabla_{\rvw_n}b_i(\rvx;\rvw_n,\rvw_z) &= \nabla_{\rvw_n}\left(\int_{0}^{{c}_i(\rvx;\rvw_n,\rvw_z)}{\sigma'({\rmA}_{l+1}[i,:]\sigma({\rmN}_{l}\sigma({\rmN}_{l}\sigma(\cdots\sigma(\rmN_1\rvx)\cdots))+t)} dt\right)\\
	&= \sigma'({\rmA}_{l+1}[i,:]\sigma({\rmN}_{l}\sigma({\rmN}_{l}\sigma(\cdots\sigma(\rmN_1\rvx)\cdots))+{c}_i)\nabla_{\rvw_n}{c}_i\\
	&+\int_{0}^{{c}_i(\rvx;\rvw_n,\rvw_z) } \nabla_{\rvw_n}\sigma'({\rmA}_{l+1}[i,:]\sigma({\rmN}_{l}\sigma({\rmN}_{l}\sigma(\cdots\sigma(\rmN_1\rvx)\cdots))+t)dt.
	\end{align*}
	Since $\rvw_n$ is fixed and  $\|\rvw_z\|_2 = O(\delta)$, there exists a large enough constant $C>0$ such that
	\begin{align*}
	&\left\|\nabla_{\rvw_n}b_i(\rvx;\rvw_n,\rvw_z)\right\|_2 \leq C\|\rmA_{l+1}[i,:]\|_2^{p-1}\|\nabla_{\rvw_n}{c}_i\|_2  + C\|\rmA_{l+1}[i,:]\|_2^{p-1}|{c}_i|
	\end{align*}
	implies
	\begin{align*}
	&\left\|\nabla_{\rvw_n}b(\rvx;\rvw_n,\rvw_z)\right\|_2 =  O(\delta^{p-1}\delta^{p+1}) + O(\delta^{p-1}\delta^{p+1})  =  O(\delta^{2p}).
	\end{align*}
\end{proof}
\subsection{Proof of Lemma \ref{lemma_bal_wt}}\label{pf_bal_lemma}
We will make repeated use of the following lemma to prove Lemma \ref{lemma_bal_wt}.
\begin{lemma}
	\label{lemma:kkt_ncf}
	Let $h(\rva,\rvb) = \sum_{i=1}^n\rvq_i^\top\rvb\sigma(\rva^\top\rvr_i)$, where $\rva\in \sR^{d_1}$, $\rvb\in \sR^{d_2}$ and $\sigma(x) = \max(x,\alpha x)^p$, for some $p\in \sN, p\geq 1$ and $\alpha\in\sR$. Suppose there exists $\lambda>0$ and $(\rva_*,\rvb_*)$ such that
	\begin{equation*}
	\nabla_\rva h(\rva_*,\rvb_*) = \lambda \rva_*, \nabla_\rvb h(\rva_*,\rvb_*) = \lambda \rvb_*.
	\end{equation*}  
	Then, $p\|\rvb_*\|_2^2 = \|\rva_*\|_2^2$.
\end{lemma}
\begin{proof}
	Since $h(\rva,\rvb) $ is $1$-homogeneous in $\rvb$ and $p$-homogeneous in $\rva$, we get
	\begin{align*}
	&\lambda \rva_*^\top\rva_* = \rva_*^\top\nabla_\rva h(\rva_*,\rvb_*) = ph(\rva_*,\rvb_*),\\
	&\lambda \rvb_*^\top\rvb_* = \rvb_*^\top\nabla_\rvb h(\rva_*,\rvb_*) = h(\rva_*,\rvb_*).
	\end{align*}
	From the above equation, we get $p\|\rvb_*\|_2^2 = \|\rva_*\|_2^2$.
\end{proof}
\textbf{Proof of Lemma \ref{lemma_bal_wt}:} Since $\overline{\rvw}_z$ is a KKT point of 
\begin{equation*}
\max_{\|\rvw_z\|_2^2=1} \rvp^\top\mathcal{H}_1(\rmX;\rvw_n,\rvw_z),
\end{equation*}
there exist a scalar $\lambda$ such that
\begin{equation}
\nabla_{\rvw_z} \left(\rvp^\top\mathcal{H}_1(\rmX;\rvw_n,\overline{\rvw}_z)\right) = \lambda\overline{\rvw}_z.
\label{kkt_wz}
\end{equation}
Since $\mathcal{H}_1(\rvx;\rvw_n,\rvw_z) $ is $(p+1)$-homogeneous in $\rvw_z$, we get
\begin{equation*}
\lambda = \lambda\overline{\rvw}_z^\top\overline{\rvw}_z = \lambda\overline{\rvw}_z^\top\nabla_{\rvw_z} \left(\rvp^\top\mathcal{H}_1(\rmX;\rvw_n,\overline{\rvw}_z)\right) =(p+1) \rvp^\top\mathcal{H}_1(\rmX;\rvw_n,\overline{\rvw}_z).
\end{equation*} 
Since $\overline{\rvw}_z$ is a positive KKT point, we have $\rvp^\top\mathcal{H}_1(\rmX;\rvw_n,\overline{\rvw}_z)>0$, implying $\lambda > 0$. Next, as $\{\rmC_l\}_{l=2}^{L-1}$ does not appear in the expression of $\mathcal{H}_1(\rvx;\rvw_n,\rvw_z) $, from \cref{kkt_wz}, we get
\begin{equation*}
\mathbf{0} = \lambda\overline{\rmC}_l, \text{ for all } 2\leq  l\leq L-1,
\end{equation*}
which implies $\overline{\rmC}_l = \mathbf{0}$. Now, from the discussion in \Cref{sec:prop_kkt}, we know
\begin{align*}
\rvp^\top\mathcal{H}_1(\rmX;\rvw_n,\rvw_z) &= \sum_{l=1}^{L-2}\sum_{j=1}^{\Delta_l}\sum_{i=1}^n p_i\nabla_\rvs f_{L,l+2}^\top\left(\rmN_{l+1}\rvg_{l}(\rvx_i)\right)\rmB_{l+1}[:,j]\sigma\left(\rmA_{l}[j,:]\rvg_{l-1}(\rvx_i)\right) \\
&+  \sum_{j=1}^{\Delta_{L-1}}\sum_{i=1}^n p_i\rmB_L[:,j]\sigma\left(\rmA_{L-1}[j,:]g_{L-2}(\rvx_i)\right),
\end{align*}
where $\Delta_l = k_l-p_l$, for all $l\in [L-1]$. For any $1\leq l\leq L-2$ and $1\leq j\leq \Delta_l$, $(\rmB_{l+1}[:,j],\rmA_{l}[j,:])$ only appears in the following term:
\begin{equation*}
\sum_{i=1}^n p_i\nabla_\rvs f_{L,l+2}^\top\left(\rmN_{l+1}\rvg_{l}(\rvx_i)\right)\rmB_{l+1}[:,j]\sigma\left(\rmA_{l}[j,:]\rvg_{l-1}(\rvx_i)\right).
\end{equation*}
Hence, using the above fact, \cref{kkt_wz} and Lemma \ref{lemma:kkt_ncf}, we have
\begin{equation*}
p\|\overline{\rmB}_{l+1}[:,j]\|_2^2  = \|\overline{\rmA}_{l}[j,:]\|_2^2.
\end{equation*}
Similarly, for any $1\leq j\leq \Delta_{L-1}$, $(\rmB_{L}[:,j],\rmA_{L-1}[j,:])$ only appears in the following term:
\begin{equation*}
\sum_{i=1}^n p_i\rmB_L[:,j]\sigma\left(\rmA_{L-1}[j,:]g_{L-2}(\rvx_i)\right).
\end{equation*}
Hence, again using the above fact, \cref{kkt_wz} and Lemma \ref{lemma:kkt_ncf}, we have
\begin{equation*}
p\|\overline{\rmB}_{L}[:,j]\|_2^2  = \|\overline{\rmA}_{L-1}[j,:]\|_2^2.
\end{equation*}
This completes the proof. \hfill \ensuremath{\blacksquare}
\section{Experimental Details}
This section outlines the implementation and hyperparameter details for the experiments in this paper, all conducted using PyTorch \citep{pytorch_ref}.
\subsection{Non-linear Sparse Functions}
\label{sp_exp}
\subsubsection{Hypersphere}
\label{sphere_exp}
We train three-layer fully connected neural networks for both the NP algorithm and gradient descent. The specific hyperparameters for each method are as follows:\\
\textbf{Gradient Descent.} The initial weights were sampled from the hypersphere of radius 0.01.  Training continued until the training error dropped below 0.001 or until a maximum of $3 \times 10^6$ iterations. We evaluated different learning rates: $\{0.2,0.1\}$ for $f_1(\rvx)$, and $\{0.02,0.01\}$ for $f_2(\rvx)$. We report the results corresponding to the lowest test error.\\ 
\textbf{Neuron Pursuit.} We run the NP algorithm until the training error drops below 0.001 or until 31 iterations were completed. In each iteration, the constrained NCF is maximized using projected gradient ascent with step-size 0.2 for 2500 iterations, where the initial weights are sampled uniformly from the unit-norm sphere. We use $H=10$ different random initializations to identify the most dominant KKT point.

The scalar $\delta$ in the first iteration was set to 0.05. For subsequent iterations, we have $\delta = 0.01\|\rvw\|_2$, where $\|\rvw\|_2$ denotes the norm of the weights at the end of the previous iteration. After each neuron addition, the network is trained via full-batch gradient descent for 70,000 iterations. We used two step-sizes: $\{0.1,0.01\}$ for $f_1(\rvx)$ and $\{0.01,0.001\}$ for $f_2(\rvx)$. To report the final training and test errors for NP, we applied the following procedure:
\begin{itemize}
	\item \textbf{If multiple step-sizes yield test error $< 0.002$}, report the result corresponding to the fewest number of iterations.
	\item \textbf{Otherwise, if multiple step-sizes achieve training error $< 0.001$}, report the one with the lowest test error.
	\item \textbf{Otherwise}, report the one with lowest training error.
\end{itemize}  
The first criterion favors efficient solutions when test error is small. The next two step safeguards against misleading conclusions when the smallest test error does not coincide with smallest training error. This particularly happens in low-sample regimes where the NP algorithm is not able to achieve small test error. Without this safeguard, one might wrongly infer that NP fails to fit the training data, even when a good fit is possible with a different learning rate.
\subsubsection{Hypercube}
We train three-layer fully connected neural networks to for both the NP algorithm and gradient descent. The specific hyperparameters for each method are as follows:\\
\textbf{Gradient Descent.} The initial weights were sampled from the hypersphere of radius 0.01.  Training continued until the training error dropped below 0.001 or until a maximum of $3 \times 10^6$ iterations. We evaluated different learning rates: $\{0.01,0.002\}$ for $g_1(\rvx)$, and $\{0.07,0.05\}$ for $g_2(\rvx)$. We report the results corresponding to the lowest test error. \\
\textbf{Neuron Pursuit.} We run the NP algorithm until the training error drops below 0.001 or until 31 iterations were completed. In each iteration, the constrained NCF is maximized using projected gradient ascent with step-size 0.2 for 2500 iterations, where the initial weights are sampled uniformly from the unit-norm sphere. We use $H=10$ different random initializations to identify the most dominant KKT point.

The scalar $\delta$ is chosen in the same way as in \Cref{sphere_exp}. After each neuron addition, we trained the network using full-batch gradient descent for 100,000 iterations. We use learning rates $\{0.01,0.002\}$ for $g_1(\rvx)$ and$\{0.005,0.002,0.001\}$ for $g_2(\rvx)$. To report the final training and test errors for NP, we applied the same procedure as in \Cref{sphere_exp}.
\subsubsection{Gaussian}
We train four-layer fully connected neural networks to for both the NP algorithm and gradient descent. The specific hyperparameters for each method are as follows:\\
\textbf{Gradient Descent.} The initial weights were sampled from the hypersphere of radius 0.1.  Training continued until the training error dropped below 0.001 or until a maximum of $3 \times 10^6$ iterations. We evaluated different learning rates for both $h_1(\rvx)$ and $h_2(\rvx)$: $\{0.1,0.01\}$. We report the results corresponding to the lowest test error. \\
\textbf{Neuron Pursuit.} We run the NP algorithm until the training error drops below 0.001 or until 31 iterations were completed. In each iteration, the constrained NCF is maximized using projected gradient ascent with step-size 0.2 for 2500 iterations, where the initial weights are sampled uniformly from the unit-norm sphere. We use $H=10$ different random initializations to identify the most dominant KKT point.

The scalar $\delta$ in the first iteration was set to 0.1. For subsequent iterations, we have $\delta = 0.01\|\rvw\|_2$, where $\|\rvw\|_2$ denotes the norm of the weights at the end of the previous iteration.  After each neuron addition, we trained the network using full-batch gradient descent for 100,000 iterations. However, if the training error gets less than $0.05$, the number of iterations is increased to 300,000 for $h_1(\rvx)$, and for $h_2(\rvx)$, it is increased to 500,000. This is done because the optimization gets slower. We use learning rates $\{0.01,0.002\}$ for $h_1(\rvx)$ and $\{0.02, 0.01\}$ for $h_2(\rvx)$. To report the final training and test errors for NP, we applied the same procedure as in \Cref{sphere_exp}.
\begin{figure}[htbp]
	\centering
	
	\begin{subfigure}[b]{0.4\textwidth}
		\centering
		\includegraphics[width=\linewidth]{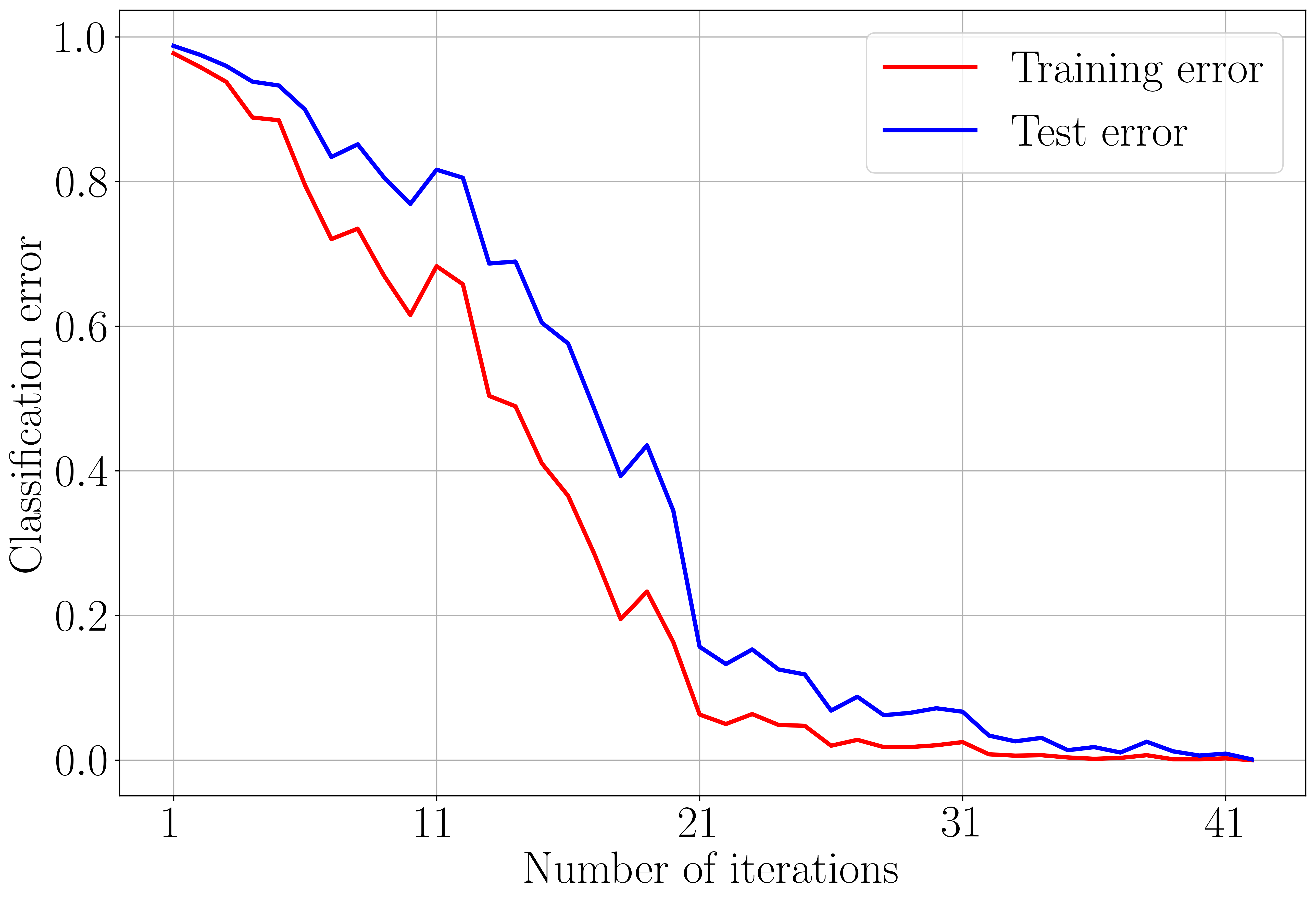}
	\end{subfigure}
	\hfill
	\begin{subfigure}[b]{0.55\textwidth}
		\centering
		\includegraphics[width=\linewidth]{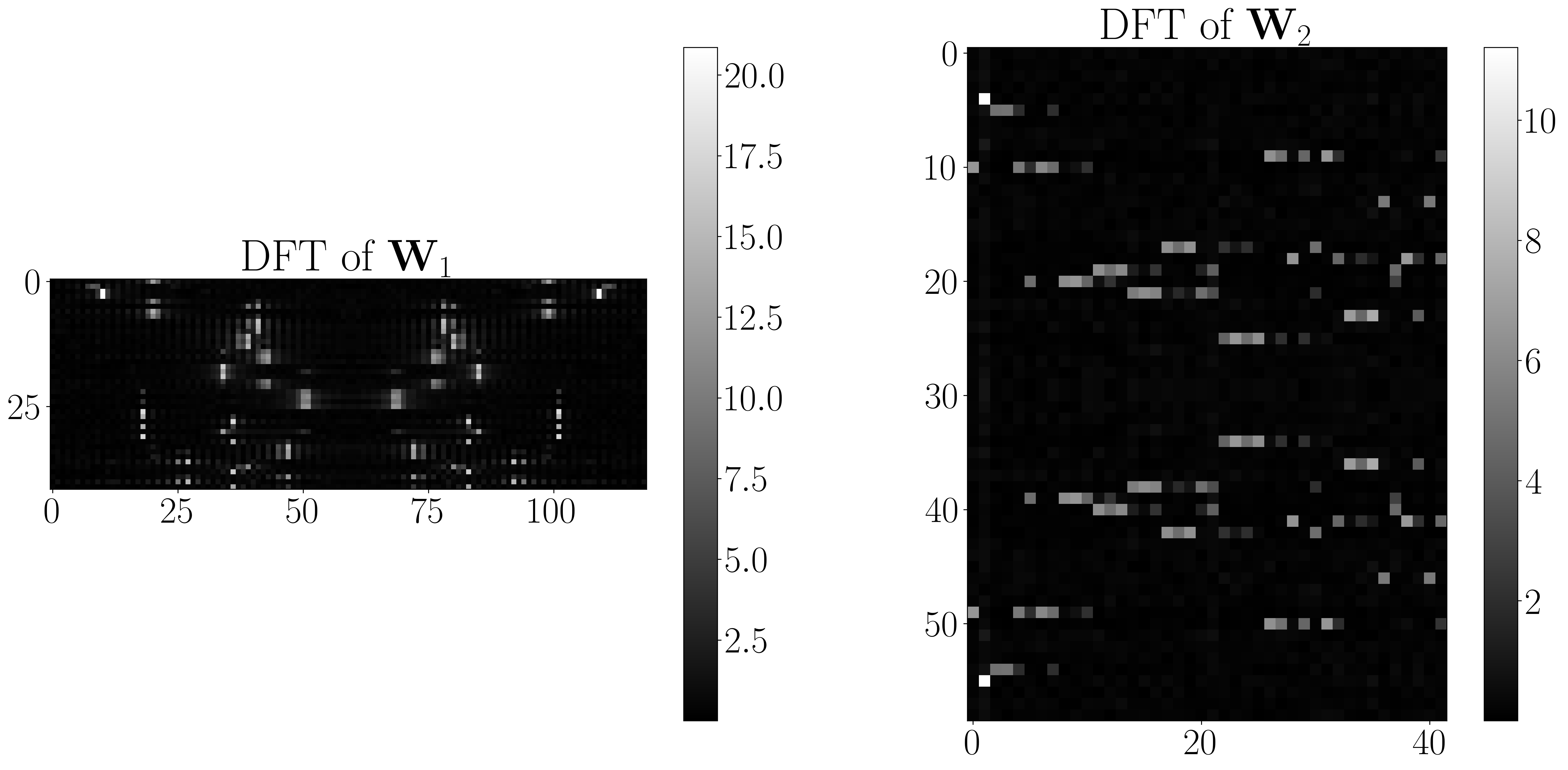}
	\end{subfigure}
	
	\vspace{1em}
	
	\begin{subfigure}[b]{0.4\textwidth}
		\centering
		\includegraphics[width=\linewidth]{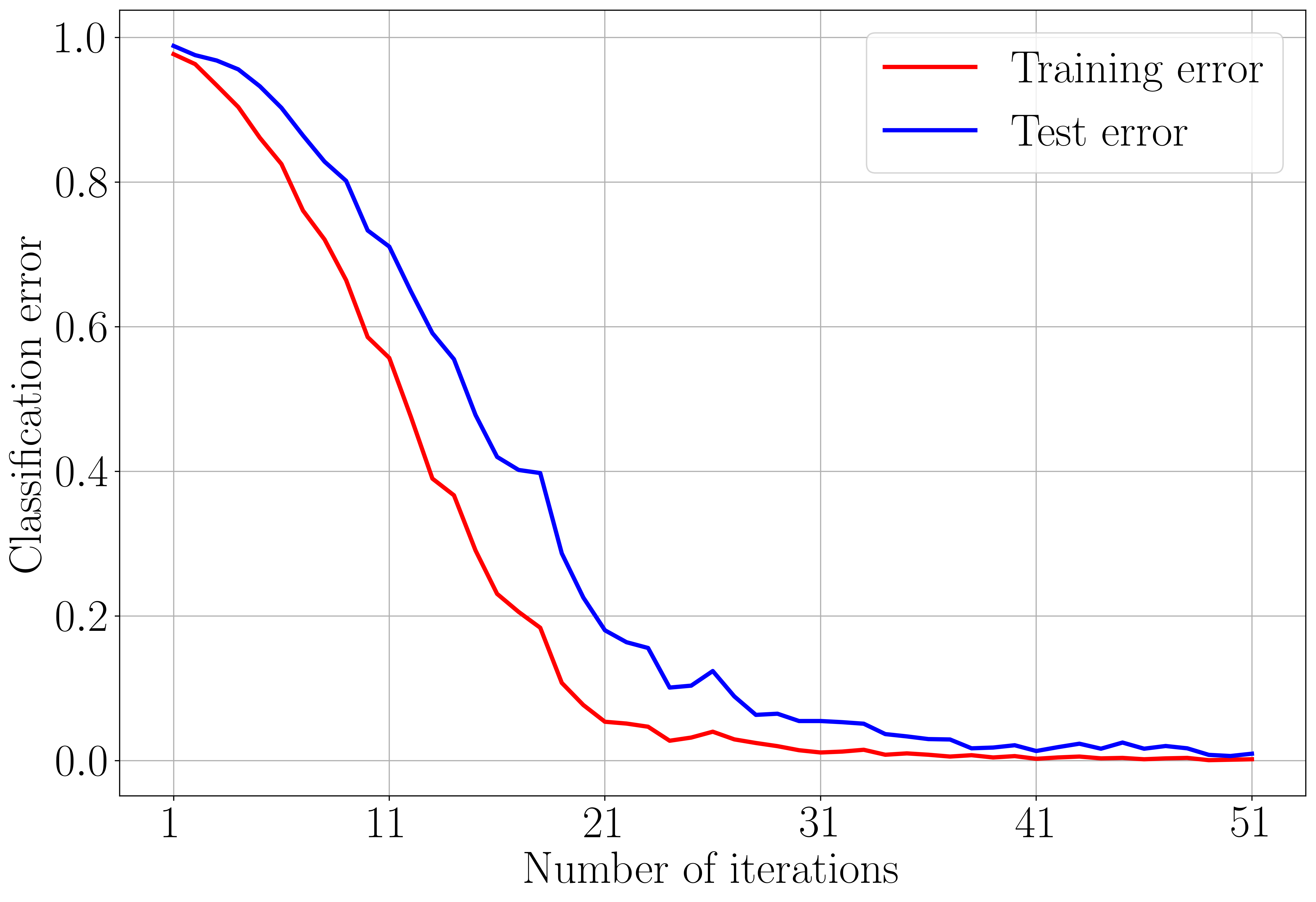}
		\caption{Evolution of training and test error}
	\end{subfigure}
	\hfill
	\begin{subfigure}[b]{0.55\textwidth}
		\centering
		\includegraphics[width=\linewidth]{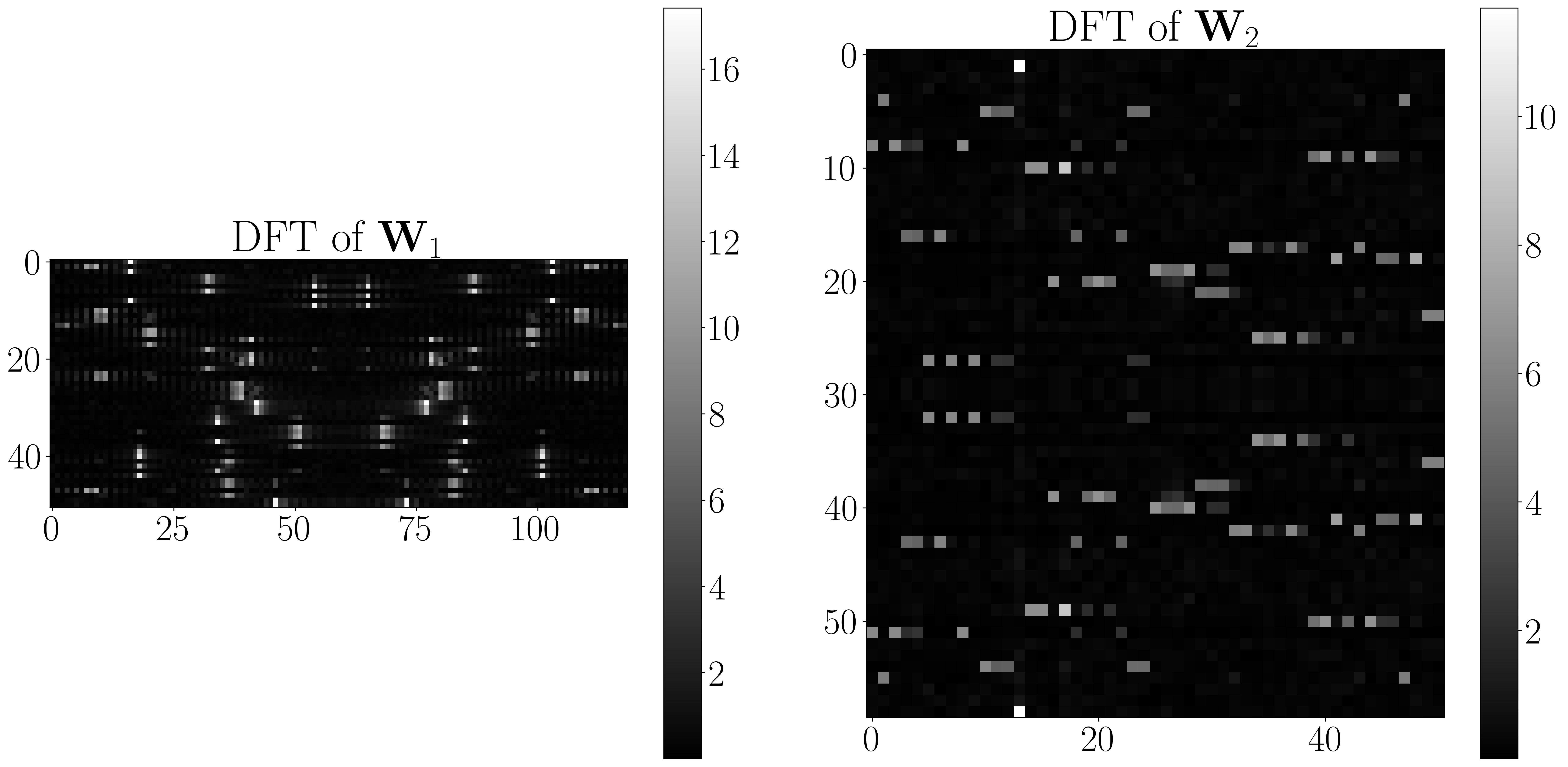}
		\caption{Absolute value of 2D DFT of the learned weights}
	\end{subfigure}
	
	\caption{
		The results from two additional independent runs of learning modular addition using a two-layer neural network with a square activation function, trained via the NP algorithm. Both runs achieve low training and test errors, with the DFT of each row in the first layer and each column in the second layer concentrated around a specific frequency.
	}
	\label{fig:mod_add_ext}
\end{figure}
\subsection{Algorithmic Tasks}
\subsubsection{Modular addition}
\label{mod_add_exp}
We run the NP algorithm until the classification error on the training data becomes $0$ or until 51 iterations are completed. In each iteration, the constrained NCF is maximized using projected gradient ascent with step-size 2 for 5000 iterations, where the initial weights are sampled uniformly from the unit-norm sphere. We use $H=10$ different random initializations to identify the most dominant KKT point.

The scalar $\delta$ in the first iteration was set to 0.125. For subsequent iterations, we have $\delta = 0.025\|\rvw\|_2$, where $\|\rvw\|_2$ denotes the norm of the weights at the end of the previous iteration. After each neuron addition, we trained the network using full-batch gradient descent for 100,000 iterations with step-size 5. However, if the  misclassification rate on the training data is below $20\%$, then the step-size becomes 2.5 and also $\delta = 0.0125\|\rvw\|_2$. \Cref{fig:mod_add_ext} depicts the evolution of training and test error, along with the absolute value of the 2D DFT of the learned weights, for two additional independent runs.
\subsubsection{Pointer Value Retrieval}
\label{pvr_exp}
We run the NP algorithm until the training error drops below 0.01 or until 61 iterations are completed. In the first iteration, the constrained NCF is maximized using projected gradient ascent with step-size 0.25 for 6000 iterations. In subsequent iteration, we use 5000 iterations and different step-sizes $\{0.4, 0.6\}$.  The initial weights are sampled uniformly from the unit-norm sphere, and we use $H=10$ different random initializations to identify the most dominant KKT point.

The scalar $\delta$ in the first iteration was set to 0.5. For subsequent iterations, we first choose $\delta = 0.01\|\rvw\|_2$, where $\|\rvw\|_2$ denotes the norm of the weights at the end of the previous iteration. For this $\delta$, if the training loss exceeds the loss before adding the neuron, then $\delta$ is reduced multiplicatively by a factor of 0.8. We repeat this procedure up to 10 times. After each neuron addition, we trained the network using full-batch gradient descent for 500,000 iterations with different step-sizes: $\{0.007, 0.005\}$. Also, we compute the training loss every 50,000 iterations, and if the loss increases, step-size is halved and the weights are reset to their previous state. This halving occurs at most four times. If the training error is below 0.4, then the number of iterations is increased to $2\times 10^6$, since after this point the step-size has typically shrunk substantially, slowing optimization. \Cref{fig:pvr_ext} depicts the evolution of training and test error, along with the absolute value of the learned weights, for two additional independent runs. 

We introduce one modification to the NP algorithm for the PVR task. After every iteration, we balance the weights so that the incoming and outgoing weights of each hidden neuron have same norm, without changing the network output. To do this, we use the method of \citet{saul_bal}, specifically Algorithm 1 with $p=q=2$ . This extra step is motivated by the fact that, under gradient flow with small initialization and ReLU activation, the norm of incoming and outgoing weights of each hidden neuron remains nearly balanced \citet{du_bal}. However, when training with gradient descent, especially with large step sizes or large number of iterations, the weights can become unbalanced. We observe this unbalance in NP when applied to the PVR task. Importantly, such unbalance can alter which neurons the NP algorithm chooses to add, as discussed in \Cref{scale_wt}. To prevent this behavior, the weights are rebalanced.
\begin{figure}[H]
	\centering
	
	\begin{subfigure}[b]{0.33\textwidth}
		\centering
		\includegraphics[width=\linewidth]{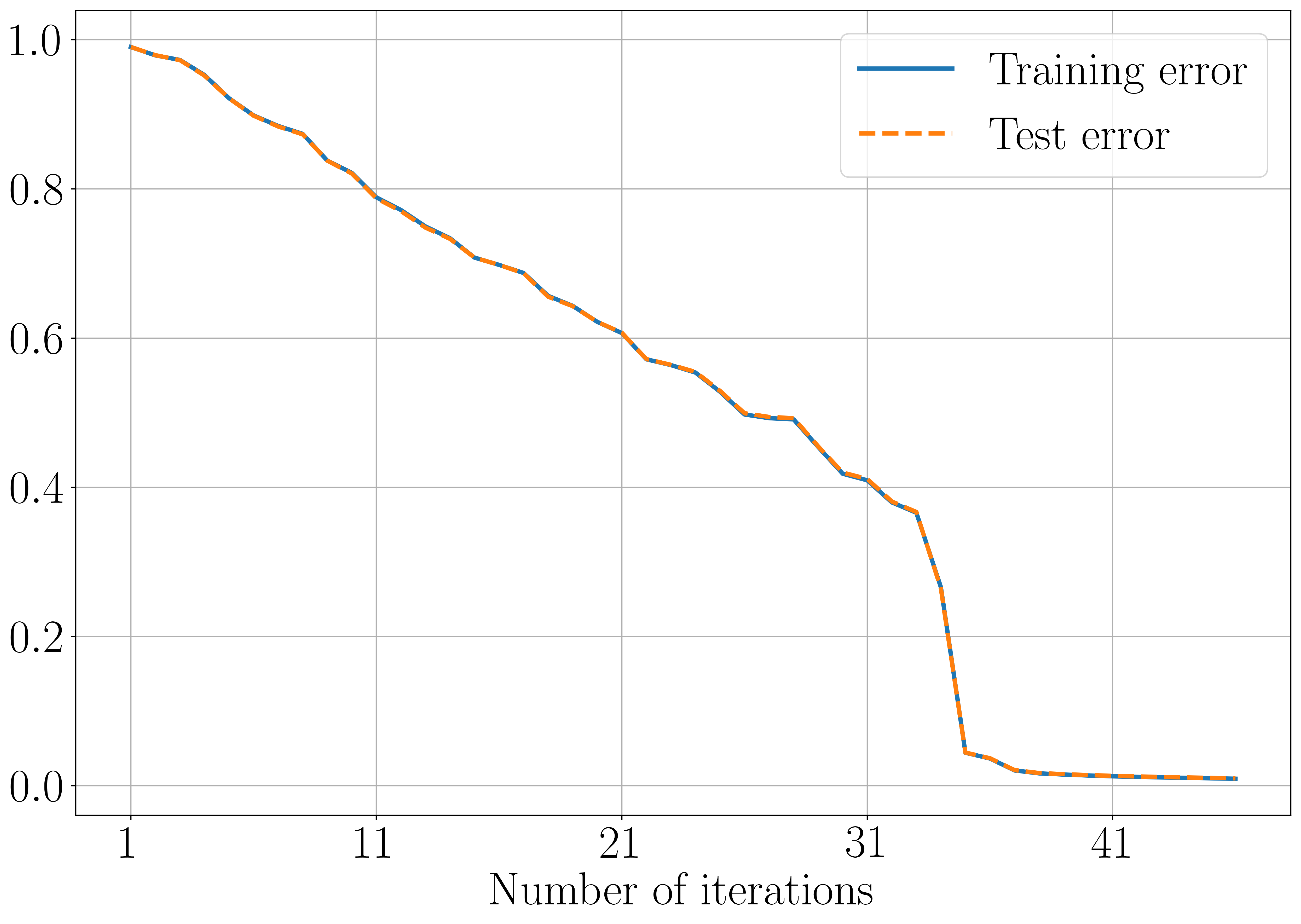}
	\end{subfigure}
	\hfill
	\begin{subfigure}[b]{0.6\textwidth}
		\centering
		\includegraphics[width=\linewidth,height=3.8cm]{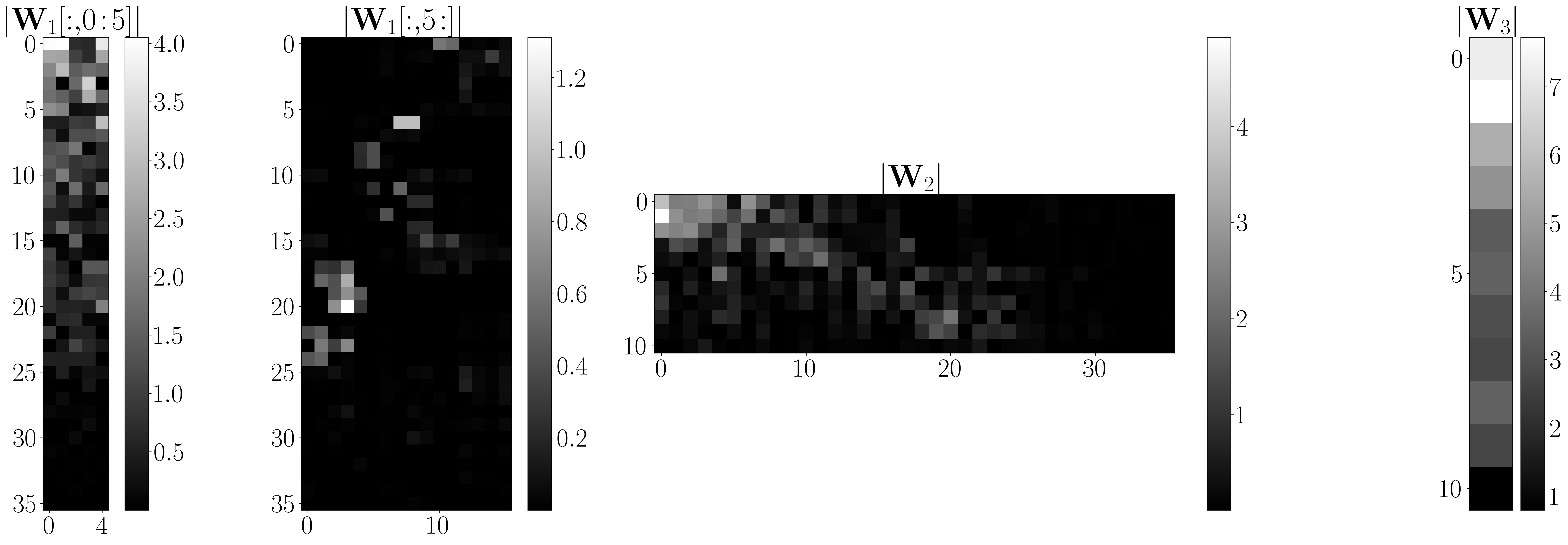}
	\end{subfigure}
	
	\vspace{1em}
	
	\begin{subfigure}[b]{0.33\textwidth}
		\centering
		\includegraphics[width=\linewidth]{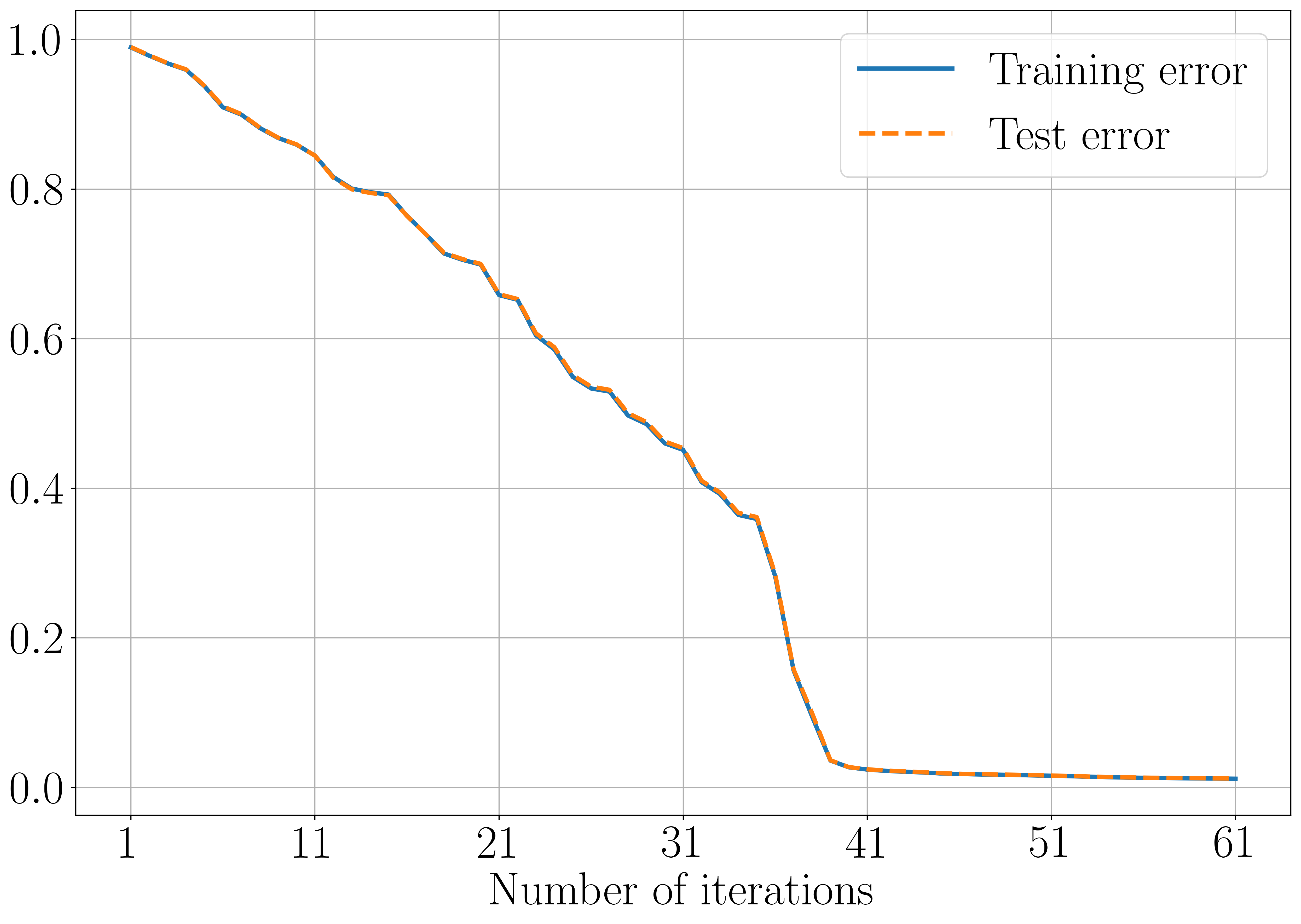}
		\caption{Training and test error}
	\end{subfigure}
	\hfill
	\begin{subfigure}[b]{0.6\textwidth}
		\centering
		\includegraphics[width=\linewidth,height=3.8cm]{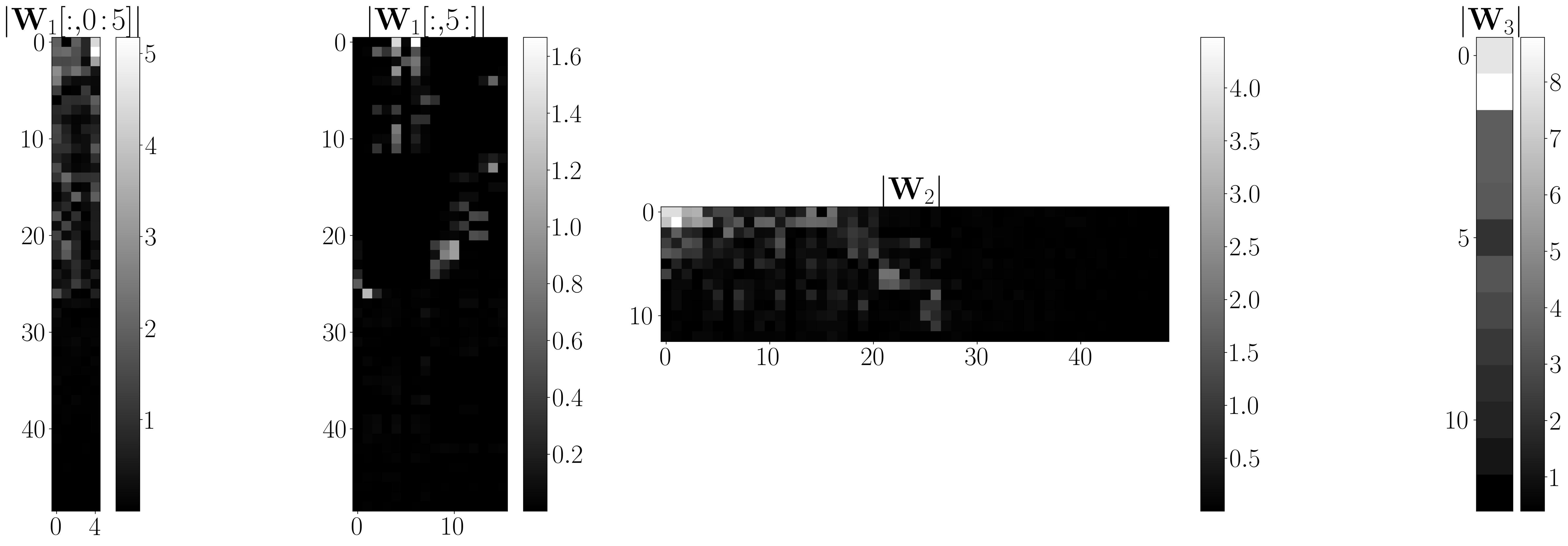}
		\caption{Absolute value of the learned weights}
	\end{subfigure}
	
	\caption{
		The results from two additional independent runs of learning the PVR task using a three-layer neural network activation function $\sigma(x) = \max(0,x)$, trained via the NP algorithm. Both runs achieve low training and test errors, and the weights of the first layer associated with $\rvx$ are sparse, with dominant entries localized within each row
	}
	\label{fig:pvr_ext}
\end{figure}
\section{Additional Discussion and Results}
\subsection{Proof of Lemma \ref{traj_init_eq}}\label{pf_traj_eq}
We will show that $\rvz(t)\coloneqq \frac{1}{\delta}\rvs_{\delta}\left(\frac{t}{\delta^{L-2}}\right)$ is the solution of
\begin{equation*}
\dot{\rvs} = \nabla _\rvs g(\rvs), \rvs(0) = \rvs_0.
\end{equation*}
Note that, $\rvz(0) = \rvs_{\delta}(0)/\delta = \rvs_0$. Next,
\begin{equation*}
\dot{\rvz} = \frac{1}{\delta^{L-1}}\dot{\rvs}_{\delta}\left(\frac{t}{\delta^{L-2}}\right) = \frac{1}{\delta^{L-1}}\nabla _\rvs g\left(\rvs_{\delta}\left(\frac{t}{\delta^{L-2}}\right)\right) = \nabla _\rvs g\left(\frac{1}{\delta}\rvs_{\delta}\left(\frac{t}{\delta^{L-2}}\right)\right) = \nabla _\rvs g(\rvz),
\end{equation*}
where the second equality follows from the definition of $\rvs_{\delta}(t)$, and the third equality holds since $ \nabla _\rvs g(\cdot)$ is $(L-1)$-homogeneous. This completes the proof. 
\subsection{Lojasieweicz's Inequality:  An Example}
\label{loj_ineq_ex}
To derive \Cref{thm_dir_convg}, we assumed that  $\overline{\rvw}_n$ is a local minimum of $\widetilde{\mathcal{L}}(\rvw_n)$ such that Lojasieweicz's inequality is satisfied in the neighborhood of $\overline{\rvw}_n$ with $\alpha\in \left(0,\frac{L}{2(L-1)}\right)$. We present a simple example where this assumption is true with $\alpha = \frac{1}{2}$. 

Suppose $(\rmX,\rvy) \in \sR^{d\times 2d} \times \sR^{2d} $ is the training dataset such that $y_i = (\rvw_*^\top\rvx_i)^p$, for some $p\geq 2$ and $i\in [2d]$, where $\rvw_*\in \sR^d$. We assume that $\min_{i\in [2d]}|\rvw_*^\top\rvx_i | = \eta>0$, $\rvw_*^\top\rvx_i > 0$, for all $1\leq i\leq d$, and $\rvw_*^\top\rvx_i < 0$ for all $d+1\leq i\leq 2d$. Let $\rmX_d\in \sR^{d\times d}$ denote the submatrix formed by first $d$ column of $\rmX$ and $\rho>0$ denotes the minimum singular value of $\rmX_d^\top\rmX_d$. Consider a two-layer neural network with $H\geq 2$ neurons:
\begin{equation*}
\mathcal{H}(\rvx;\{v_i,\rvu_i\}_{i=1}^H)  = \sum_{i=1}^Hv_i\sigma(\rvu_i^\top\rvx),
\end{equation*}
where $\sigma(x) = \max(0,x)^p$. Thus, the training loss is
\begin{equation*}
\mathcal{L}(\{v_i,\rvu_i\}_{i=1}^H) = \frac{1}{2}\left\|\sum_{i=1}^Hv_i\sigma(\rmX^\top\rvu_i) - \rvy\right\|_2^2
\end{equation*}
\begin{lemma}
	Consider the following optimization problem:
	\begin{equation*}
	\widetilde{\mathcal{L}}(v_1,\rvu_1) = \frac{1}{2}\left\|v_1\sigma(\rmX^\top\rvu_1) - \rvy\right\|_2^2,
	\end{equation*}
	Let $(\overline{v}_1,\overline{\mathbf{u}}_1) = (\alpha,\beta\rvw_*)$, where $\alpha\beta^p = 1$ and $\beta>0$. Then, $(\overline{v}_1,\overline{\mathbf{u}}_1)$ is a local minimum of $\widetilde{\mathcal{L}}(v_1,\rvu_1)$ such that  Lojasieweicz's inequality is satisfied in the neighborhood of $(\overline{v}_1,\overline{\mathbf{u}}_1)$ with $\alpha=\frac{1}{2}$:
	there exists $\mu_1,\gamma>0$ such that
	\begin{equation}
	\|\nabla\widetilde{\mathcal{L}}(v_1,\rvu_1)\|_2\geq \mu_1\left(\widetilde{\mathcal{L}}(v_1,\rvu_1) - \widetilde{\mathcal{L}}(\overline{v}_1,\overline{\mathbf{u}}_1)\right)^{\frac{1}{2}}, \text{ if } (v_1 - \overline{v}_1)^2 + \|\rvu_1 - \overline{\rvu}_1\|_2^2\leq \gamma^2.
	\label{loj_ex}
	\end{equation} 
\end{lemma}
\begin{proof}
	Define $\rve = v_1\sigma(\rmX^\top\rvu_1) - \rvy$, then
	\begin{equation*}
	\nabla_{v_1}\widetilde{\mathcal{L}}(v_1,\rvu_1) = \sigma(\rmX^\top\rvu_1)^\top\rve,  \nabla_{\rvu_1}\widetilde{\mathcal{L}}(v_1,\rvu_1) = v_1\rmX\text{diag}(\sigma'(\rmX^\top\rvu_1))\rve.
	\end{equation*}
	We first show that $\nabla_{v_1}\widetilde{\mathcal{L}}(\overline{v}_1,\overline{\mathbf{u}}_1) = 0$ and $\nabla_{\rvu_1}\widetilde{\mathcal{L}}(\overline{v}_1,\overline{\mathbf{u}}_1) = \mathbf{0}$. For $1\leq i \leq d$, since $\overline{\mathbf{u}}_1^\top\rvx_i>0$, we have $\sigma'(\overline{\mathbf{u}}_1^\top\rvx_i)>0$ and
	\begin{equation*}
	e_i = \overline{v}_1\sigma(\overline{\mathbf{u}}_1^\top\rvx_i) - y_i= \alpha\max(0,\beta\overline{\mathbf{u}}_1^\top\rvx_i)^p - (\overline{\mathbf{u}}_1^\top\rvx_i)^p = \alpha\beta^p(\overline{\mathbf{u}}_1^\top\rvx_i)^p-(\overline{\mathbf{u}}_1^\top\rvx_i)^p = 0.
	\end{equation*}
	For $d+1\leq i \leq 2d$, since $\overline{\mathbf{u}}_1^\top\rvx_i<0$, we have $\sigma'(\overline{\mathbf{u}}_1^\top\rvx_i)=0 = \sigma(\overline{\mathbf{u}}_1^\top\rvx_i)$ and
	\begin{equation*}
	e_i = \overline{v}_1\sigma(\overline{\mathbf{u}}_1^\top\rvx_i) - y_i= \alpha\max(0,\beta\overline{\mathbf{u}}_1^\top\rvx_i)^p - (\overline{\mathbf{u}}_1^\top\rvx_i)^p =-(\overline{\mathbf{u}}_1^\top\rvx_i)^p.
	\end{equation*}
	Hence,
	\begin{align*}
	&\nabla_{v_1}\widetilde{\mathcal{L}}(v_1,\rvu_1) = \sum_{i=1}^{d}\sigma(\overline{\mathbf{u}}_1^\top\rvx_i)e_i + \sum_{i=d+1}^{2d}\sigma(\overline{\mathbf{u}}_1^\top\rvx_i)e_i= 0,\\  &\nabla_{\rvu_1}\widetilde{\mathcal{L}}(v_1,\rvu_1) = \overline{v}_1\sum_{i=1}^{d}\rvx_i\sigma'(\overline{\mathbf{u}}_1^\top\rvx_i)e_i +\overline{v}_1\sum_{i=d+1}^{d}\rvx_i\sigma'(\overline{\mathbf{u}}_1^\top\rvx_i)e_i= \mathbf{0}.
	\end{align*}
	Let $\rvb$ be any unit-norm vector orthogonal to $\rvw_*$. We next aim to show that there exists $\gamma_1>0$ such that if $\epsilon_1^2+\epsilon_2^2+\epsilon_3^2\leq \gamma_1$, then
	\begin{equation*}
	\widetilde{\mathcal{L}}(\overline{v}_1+\epsilon_1,(1+\epsilon_2)\overline{\mathbf{u}}_1 + \epsilon_3\rvb)\leq \widetilde{\mathcal{L}}(\overline{v}_1,\overline{\mathbf{u}}_1), 
	\end{equation*} 
	which would imply $(\overline{v}_1,\overline{\mathbf{u}}_1)$ is a local minimum of $\widetilde{\mathcal{L}}({v}_1,{\mathbf{u}}_1)$. For any $1\leq i\leq d$, we have
	\begin{align}
	&\hspace{-3em}\left((\overline{v}_1+\epsilon_1)\sigma\left(\rvx_i^\top((1+\epsilon_2)\overline{\mathbf{u}}_1 + \epsilon_3\rvb)\right) - y_i\right)^2 - \left(\overline{v}_1\sigma\left(\rvx_i^\top\overline{\mathbf{u}}_1 \right) - y_i\right)^2 \nonumber\\
	= &\left((\alpha+\epsilon_1)\left(\beta(1+\epsilon_2)\rvx_i^\top\rvw_* + \epsilon_3\rvx_i^\top\rvb\right)^p - y_{i}\right)^2 \nonumber\\
	= &\left(\alpha\beta^p(\rvx_i^\top\rvw_*)^p + r_i - y_{i}\right)^2 = r_i^2\geq0,
	\label{pos_err}
	\end{align}
	where $r_i \coloneqq (1+\epsilon_1)\left(\beta(1+\epsilon_2)\rvx_i^\top\rvw_* + \epsilon_3\rvx_i^\top\rvb\right)^p -\alpha\beta^p(\rvx_i^\top\rvw_*)^p$.  The first equality follows because $\beta\rvx_i^\top\rvw_*>0$, hence if $|\epsilon_2|,|\epsilon_3|$ is sufficiently small, then $\beta(1+\epsilon_2)\rvx_i^\top\rvw_* + \epsilon_3\rvx_i^\top\rvb> 0$. Next, for $d+1\leq i\leq 2d$, we have
	\begin{align}
	&\hspace{-3em}\left((\overline{v}_1+\epsilon_1)\sigma\left(\rvx_i^\top((1+\epsilon_2)\overline{\mathbf{u}}_1 + \epsilon_3\rvb)\right) - y_i\right)^2 - \left(\overline{v}_1\sigma\left(\rvx_i^\top\overline{\mathbf{u}}_1 \right) - y_i\right)^2 \nonumber\\
	= &\left((\alpha+\epsilon_1)\sigma\left(\beta(1+\epsilon_2)\rvx_i^\top\rvw_* + \epsilon_3\rvx_i^\top\rvb\right) - y_{i}\right)^2 - \left( y_{i}\right)^2 \nonumber\\
	=& \left( y_{i}\right)^2 - \left( y_{i}\right)^2 = 0.
	\label{neg_err}
	\end{align}
	The second equality holds since $\beta\rvx_i^\top\rvw_*<0$, hence if $|\epsilon_2|,|\epsilon_3|$ is sufficiently small, then $\beta(1+\epsilon_2)\rvx_i^\top\rvw_* + \epsilon_3\rvx_i^\top\rvb< 0$. Combining \cref{neg_err} and \cref{pos_err} implies that $(\overline{v}_1,\overline{\mathbf{u}}_1)$ is a local minimum of $\widetilde{\mathcal{L}}({v}_1,{\mathbf{u}}_1)$. We next prove \cref{loj_ex}. Note that
	\begin{align*}
	\|\nabla_{v_1}\widetilde{\mathcal{L}}(v_1,\rvu_1)\|_2^2 + \|\nabla_{\rvu_1}\widetilde{\mathcal{L}}(v_1,\rvu_1)\|_2^2 &\geq \|\nabla_{\rvu_1}\widetilde{\mathcal{L}}(v_1,\rvu_1)\|_2^2\\
	&  =  v_1^2\rve^\top\text{diag}(\sigma'(\rmX^\top\rvu_1))\rmX^\top\rmX\text{diag}(\sigma'(\rmX^\top\rvu_1))\rve .
	\end{align*} 
	We use $(\widetilde{v}_1,\widetilde{\rvu}_1)\coloneqq(\overline{v}_1+\epsilon_1,(1+\epsilon_2)\overline{\mathbf{u}}_1 + \epsilon_3\rvb)$ to denote a vector in the neighborhood of $(\overline{v}_1,\overline{\mathbf{u}}_1)$. If  $|\epsilon_2|,|\epsilon_3|$ is sufficiently small, then, for $d+1\leq i\leq 2d$,  
	\begin{equation*}
	\rvx_i^\top\widetilde{\rvu}_1 = \beta(1+\epsilon_2)\rvx_i^\top\rvw_* + \epsilon_3\rvx_i^\top\rvb \leq 0,
	\end{equation*}
	which implies $\sigma(\rvx_i^\top\widetilde{\rvu}_1) = 0 = \sigma'(\rvx_i^\top\widetilde{\rvu}_1)$. Hence,
	\begin{equation*}
	\text{diag}(\sigma'(\rmX^\top\widetilde{\rvu}_1))\rmX^\top\rmX\text{diag}(\sigma'(\rmX^\top\widetilde{\rvu}_1)) =  \begin{bmatrix}
	\text{diag}(\sigma'(\rmX_d^\top\widetilde{\rvu}_1))\rmX_d^\top\rmX_d\text{diag}(\sigma'(\rmX_d^\top\widetilde{\rvu}_1)) &\mathbf{0}  \\
	\mathbf{0} & \mathbf{0}  \\
	\end{bmatrix}.
	\end{equation*}
	Therefore,
	\begin{align*}
	\|\nabla_{v_1}\widetilde{\mathcal{L}}(\widetilde{v}_1,\widetilde{\rvu}_1\|_2^2 + \|\nabla_{\rvu_1}\widetilde{\mathcal{L}}(\widetilde{v}_1,\widetilde{\rvu}_1)\|_2^2&\geq (\overline{v}_1+\epsilon_1)^2\rho\sum_{i=1}^{d}\left(\sigma'(\rvx_i^\top\widetilde{\rvu}_1)\left((\overline{v}_1+\epsilon_1)\sigma\left(\rvx_i^\top\widetilde{\rvu}_1 \right) - y_i\right)\right)^2\\
	&= (\alpha+\epsilon_1)^2\rho\sum_{i=1}^{d}(\sigma'(\rvx_i^\top\widetilde{\rvu}_1)r_i)^2 ,
	\end{align*} 
	where $r_i$ is same as in \cref{pos_err}. The first inequality is true since the minimum singular value of $\rmX_d^\top\rmX_d$ is $\rho$. The last equality holds true if $|\epsilon_2|,|\epsilon_3|$ is sufficiently small, as shown in \cref{pos_err}. If $|\epsilon_2|,|\epsilon_3|$ is sufficiently small, then
	\begin{equation*}
	\rvx_i^\top\widetilde{\rvu}_1 = \beta(1+\epsilon_2)\rvx_i^\top\rvw_* + \epsilon_3\rvx_i^\top\rvb \geq \beta\rvx_i^\top\rvw_*/2, \text{ for all }1\leq i\leq d.
	\end{equation*}
	Since $\rvx_i^\top\rvw_*\geq \eta$, for all $1\leq i\leq d$, we get
	\begin{equation}
	\|\nabla_{v_1}\widetilde{\mathcal{L}}(\widetilde{v}_1,\widetilde{\rvu}_1)\|_2^2 + \|\nabla_{\rvu_1}\widetilde{\mathcal{L}}(\widetilde{v}_1,\widetilde{\rvu}_1)\|_2^2\geq (\alpha+\epsilon_1)^2p\rho\left(\frac{\beta\eta}{2}\right)^{p-1}\sum_{i=1}^{d}r_i^2.
	\label{gd_bd_ex}
	\end{equation}
	If $|\epsilon_2|,|\epsilon_3|$ is sufficiently small, from \cref{pos_err} and \cref{neg_err}, we know
	\begin{equation*}
	\widetilde{\mathcal{L}}(\widetilde{v}_1,\widetilde{\rvu}_1)-\widetilde{\mathcal{L}}(\overline{v}_1\overline{\mathbf{u}}_1) = \sum_{i=1}^dr_i^2.
	\end{equation*}
	Hence, from \cref{gd_bd_ex} and the above equation, we get
	\begin{align*}
	\|\nabla_{v_1}\widetilde{\mathcal{L}}(\widetilde{v}_1,\widetilde{\rvu}_1)\|_2^2 + \|\nabla_{\rvu_1}\widetilde{\mathcal{L}}(\widetilde{v}_1,\widetilde{\rvu}_1)\|_2^2  &\geq (\alpha+\epsilon_1)^2p\rho\left(\frac{\beta\eta}{2}\right)^{p-1}\left(\widetilde{\mathcal{L}}(\widetilde{v}_1,\widetilde{\rvu}_1)-\widetilde{\mathcal{L}}(\overline{v}_1\overline{\mathbf{u}}_1) \right)\\
	&\geq p\rho\frac{\alpha^2}{4}\left(\frac{\beta\eta}{2}\right)^{p-1}\left(\widetilde{\mathcal{L}}(\widetilde{v}_1,\widetilde{\rvu}_1)-\widetilde{\mathcal{L}}(\overline{v}_1\overline{\mathbf{u}}_1)\right), 
	\end{align*}
	where the last inequality holds if $|\epsilon_1|<\alpha/2$. Hence, \cref{loj_ex} holds true in a sufficiently small neighborhood of $(\overline{v}_1,\overline{\rvu}_1)$. 
\end{proof}

\subsection{What happens if ${\mathcal{N}}_{\overline{\rvy},\overline{\mathcal{H}}_1}(\rvw_z) = 0$?}
\label{zero_NCF}
In \Cref{thm_dir_convg}, we showed that, near the saddle point $(\overline{\rvw}_n,\mathbf{0})$, weights in $\rvw_z$ remain small in magnitude but converge in direction to the constrained NCF corresponding to ${\mathcal{N}}_{\overline{\rvy},\overline{\mathcal{H}}_1}(\rvw_z)$. The proof of \Cref{thm_dir_convg} proceeds by first showing that the output of the neural network $\mathcal{H}(\rvx;\rvw_n,\rvw_z)$ can be decomposed into a leading term that is independent of $\rvw_z$, another term $\mathcal{H}_1(\rvx;\rvw_n,\rvw_z)$  that is homogeneous in $\rvw_z$ and other residual terms that are small. Then, it was shown that the evolution of $\rvw_z$ near the saddle point is close to the gradient flow of ${\mathcal{N}}_{\overline{\rvy},\overline{\mathcal{H}}_1}(\rvw_z)$, where $\overline{\mathcal{H}}_1(\rvx;\rvw_z) =  \mathcal{H}_1(\rvx;\overline{\rvw}_n,\rvw_z). $    

Here, we consider the scenario when ${\mathcal{N}}_{\overline{\rvy},\overline{\mathcal{H}}_1}(\rvw_z) = 0$, for all $\rvw_z$. Technically, we can still apply \Cref{thm_dir_convg}. Since the constrained NCF is zero, all unit-norm vectors are its KKT point. Therefore, directional convergence holds trivially at initialization. However, this argument, while technically correct, does not capture the behavior observed in our experiments. Empirically, we find that even in this setting, the weights in $\rvw_z$ converge in direction, however, the residual terms play an important  role in determining where they converge. To investigate this phenomenon, we next study the problem of matrix decomposition using three-layer linear neural network.\\

\noindent \textbf{Deep linear network. }Suppose $\rmS\in \sR^{d\times d}$ is a rank $d$ matrix with singular values $\{s_1,s_2,\cdots, s_d\}$
and singular vectors $\{\rvu_i,\rvv_i\}_{i=1}^d$. Consider using a three-layer linear neural network to learn $\rmS$, then the training loss becomes
\begin{equation*}
\mathcal{L}(\rmW_1,\rmW_2,\rmW_3) = \frac{1}{2}\left\|\rmW_3\rmW_2\rmW_1 - \rmS \right\|_F^2,
\end{equation*}
where each of the weight matrices is in $\sR^{d\times d}$. Here, the input can be assumed to be the identity matrix, the output of the neural network is $\mathcal{H}(\rmW_1,\rmW_2,\rmW_3) = \rmW_3\rmW_2\rmW_1$, and $\rmS$ is the label. The gradient of the training loss with respect to the weights is as follows:
\begin{align*}
&\nabla_{\rmW_1}\mathcal{L}(\rmW_1,\rmW_2,\rmW_3)  = \rmW_2^\top\rmW_3^\top(\rmW_3\rmW_2\rmW_1 - \rmS),\\
&\nabla_{\rmW_2}\mathcal{L}(\rmW_1,\rmW_2,\rmW_3)  = \rmW_3^\top(\rmW_3\rmW_2\rmW_1 - \rmS)\rmW_1^\top,\\
&\nabla_{\rmW_3}\mathcal{L}(\rmW_1,\rmW_2,\rmW_3)  =(\rmW_3\rmW_2\rmW_1 - \rmS)\rmW_1^\top \rmW_2^\top.
\end{align*}  
Define
\begin{equation}
\overline{\rmW}_1 = \begin{bmatrix}
{s_1^{1/3}\rvv_{1}^\top}\\\mathbf{0} 
\end{bmatrix}, 
\overline{\rmW}_2 = \begin{bmatrix}
s_1^{1/3}&\mathbf{0}  \\
\mathbf{0} & \mathbf{0}  \\
\end{bmatrix}, \overline{\rmW}_3 = \begin{bmatrix}
s_1^{1/3}\rvu_{1}&\mathbf{0} 
\end{bmatrix},
\label{sad_linear}
\end{equation}
then $(\overline{\rmW}_1,\overline{\rmW}_2,\overline{\rmW}_3)$ is a saddle point of the training loss. This is true since 
\begin{align*}
&\nabla_{\rmW_1}\mathcal{L}(\overline{\rmW}_1,\overline{\rmW}_2,\overline{\rmW}_3) =   -\begin{bmatrix}
s_1^{1/3}&\mathbf{0}  \\
\mathbf{0} & \mathbf{0}  \\
\end{bmatrix}\begin{bmatrix}
{s_1^{1/3}\rvu_{1}^\top}\\\mathbf{0} 
\end{bmatrix}\sum_{i=2}^d s_i\rvu_i\rvv_i^\top = \mathbf{0},\\
&\nabla_{\rmW_2}\mathcal{L}(\overline{\rmW}_1,\overline{\rmW}_2,\overline{\rmW}_3) =   -\begin{bmatrix}
{s_1^{1/3}\rvu_{1}^\top}\\\mathbf{0} 
\end{bmatrix}\sum_{i=2}^d s_i\rvu_i\rvv_i^\top\begin{bmatrix}
s_1^{1/3}\rvv_{1}&\mathbf{0} 
\end{bmatrix} = \mathbf{0},\\
&\nabla_{\rmW_3}\mathcal{L}(\overline{\rmW}_1,\overline{\rmW}_2,\overline{\rmW}_3) =   
\sum_{i=2}^d s_i\rvu_i\rvv_i^\top\begin{bmatrix}
s_1^{1/3}\rvv_{1}&\mathbf{0} 
\end{bmatrix}\begin{bmatrix}
s_1^{1/3}&\mathbf{0}  \\
\mathbf{0} & \mathbf{0}  \\
\end{bmatrix} = \mathbf{0},\\
\end{align*}
where we used mutual orthogonality of singular vectors and
$ \rmS - \overline{\rmW}_3\overline{\rmW}_2\overline{\rmW}_1 = \sum_{i=2}^d s_i\rvu_i\rvv_i^\top$. Let us look at the network's output near the saddle point $(\overline{\rmW}_1,\overline{\rmW}_2,\overline{\rmW}_3)$. Let
\begin{equation*}
{\rmW}_1 = \begin{bmatrix}
{s_1^{1/3}\rvv_{1}^\top}\\\mathbf{A}_1 
\end{bmatrix}, 
{\rmW}_2 = \begin{bmatrix}
s_1^{1/3}&\rvb_2^\top  \\
\rva_2 & \mathbf{C}_2  \\
\end{bmatrix}, {\rmW}_3 = \begin{bmatrix}
s_1^{1/3}\rvu_{1}&\mathbf{B}_3 
\end{bmatrix},
\end{equation*}
where $\rva_2,\rvb_2$ are vectors. Here, $\rmA_1, \rva_2,\rvb_2,\rmC_2, \rmB_3$ belong to $\rvw_z$ and other weights belong to $\rvw_n$. Next, we can write 
\begin{align*}
\mathcal{H}(\overline{\rvw}_n,\rvw_z)  = \mathcal{H}(\rmW_1,\rmW_2,\rmW_3) &= \begin{bmatrix}
s_1^{1/3}\rvu_{1}&\mathbf{B}_3 
\end{bmatrix}\begin{bmatrix}
s_1^{1/3}&\rvb_2^\top  \\
\rva_2 & \mathbf{C}_2  \\
\end{bmatrix}\begin{bmatrix}
{s_1^{1/3}\rvv_{1}^\top}\\\mathbf{A}_1 
\end{bmatrix}\\
&= \begin{bmatrix}
s_1^{1/3}\rvu_{1}&\mathbf{B}_3 
\end{bmatrix}\begin{bmatrix}
{s_1^{2/3}\rvv_{1}^\top + \rvb_2^\top\rmA_1}\\s^{1/3}\rva_2\rvv_1^\top+\rmC_2\mathbf{A}_1 
\end{bmatrix}\\
&= s_1\rvu_{1}\rvv_{1}^\top + s_1^{1/3}\rvu_{1} \rvb_2^\top\rmA_1 +s_1^{1/3}\mathbf{B}_3\rva_2\rvv_1^\top +  \mathbf{B}_3\rmC_2\mathbf{A}_1.
\end{align*}
Hence,
\begin{equation*}
\mathcal{H}(\overline{\rvw}_n,\mathbf{0}) =  s_1\rvu_{1}\rvv_{1}^\top , \mathcal{H}_1(\overline{\rvw}_n,\rvw_z) =  s_1^{1/3}\rvu_{1} \rvb_2^\top\rmA_1 +s_1^{1/3}\mathbf{B}_3\rva_2\rvv_1^\top, \mathcal{H}_2(\overline{\rvw}_n,\rvw_z) = \mathbf{B}_3\rmC_2\mathbf{A}_1.
\end{equation*}
The residual error is defined as $ \overline{\rmS}\coloneqq\rmS - \overline{\rmW}_3\overline{\rmW}_2\overline{\rmW}_1 = \sum_{i=2}^d s_i\rvu_i\rvv_i^\top.$ From mutual orthogonality of singular vectors, we get
\begin{equation*}
\text{trace}\left( \overline{\rmS}^\top\mathcal{H}_1(\overline{\rvw}_n,\rvw_z)\right) = 0.
\end{equation*}
Thus, the inner product between the residual error and $\mathcal{H}_1(\overline{\rvw}_n,\rvw_z)$ is $0$. Now, when initialized near $(\overline{\rmW}_1,\overline{\rmW}_2,\overline{\rmW}_3)$, the experiment in \Cref{fig:zero_ncf_3l_lin} shows that weights in $\rvw_z$ remain small in norm but converge in direction towards a KKT point of constrained NCF defined with respect to the residual error and $\mathcal{H}_2(\overline{\rvw}_n,\rvw_z)$. Therefore, even when ${\mathcal{N}}_{\overline{\rvy},\overline{\mathcal{H}}_1}(\rvw_z) = 0$, for all $\rvw_z$, the weights in $\rvw_z$ remain small in magnitude and converge in direction during the initial stages of training. Moreover, the direction of convergence is determined by the constrained NCF defined with respect to the residual error and $\mathcal{H}_2(\overline{\rvw}_n,\rvw_z)$.
\begin{figure}[htbp]
	\centering
	\begin{minipage}[t]{0.45\textwidth}
		\centering
		\begin{subfigure}[t]{\linewidth}
			\centering
			\includegraphics[width=\linewidth]{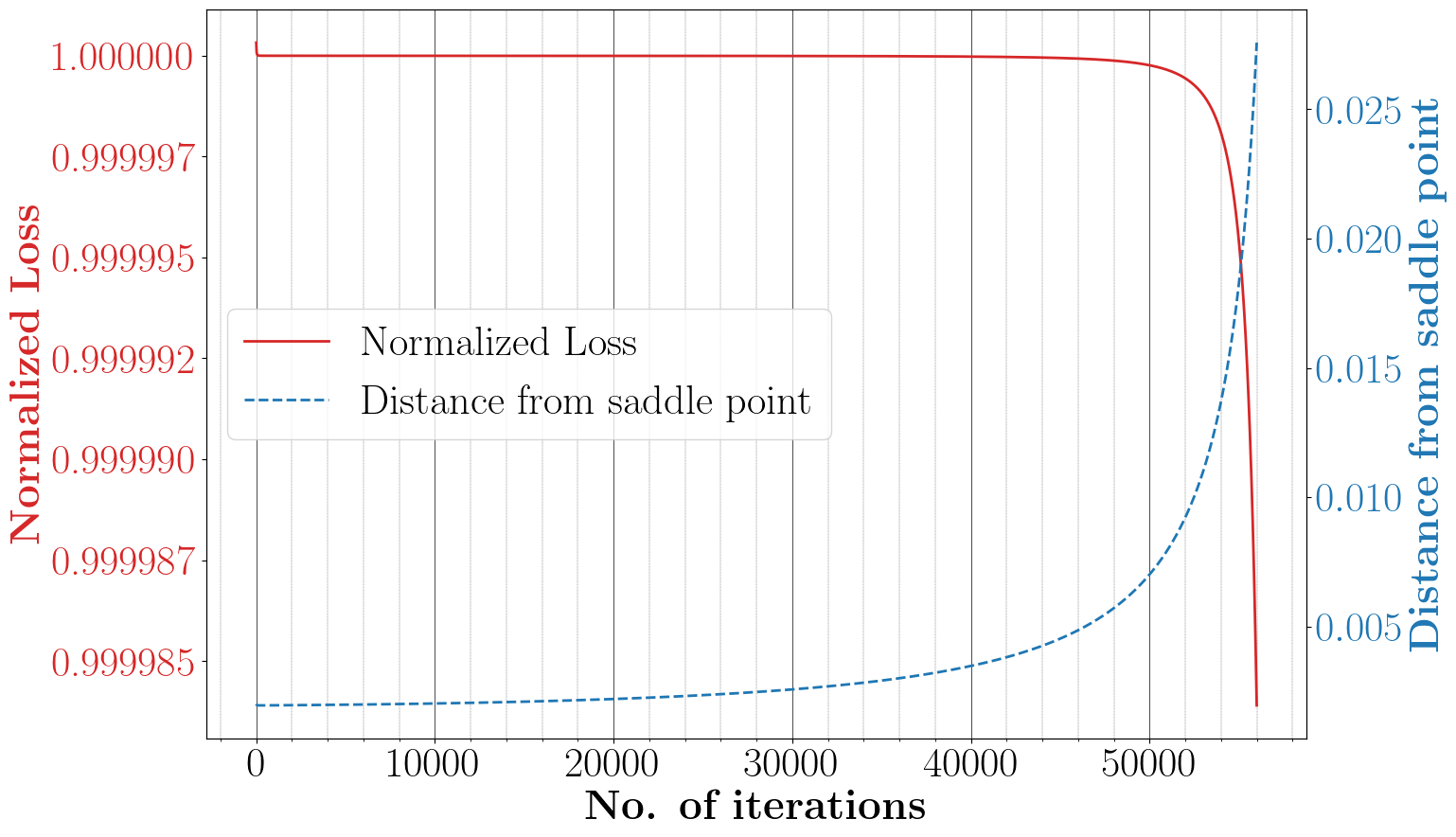}
			\caption{Evolution of training loss and distance of weights from saddle point with iterations }
			\label{fig:loss_evol_lin}
		\end{subfigure}
		\vspace{0.3cm}
		
		\begin{subfigure}[t]{\linewidth}
			\centering
			\includegraphics[width=\linewidth]{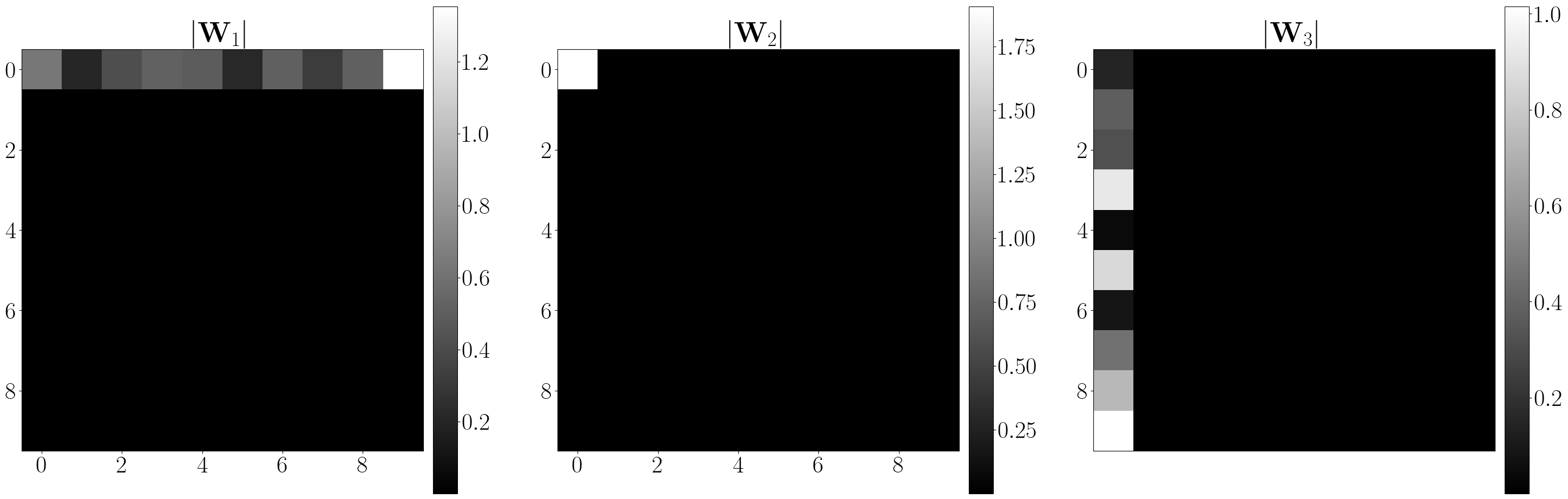}
			\caption{Weights at initialization}
			\label{fig:all_weights_init_lin}
		\end{subfigure}
		\vspace{0.3cm}
		
		\begin{subfigure}[t]{\linewidth}
			\centering
			\includegraphics[width=\linewidth]{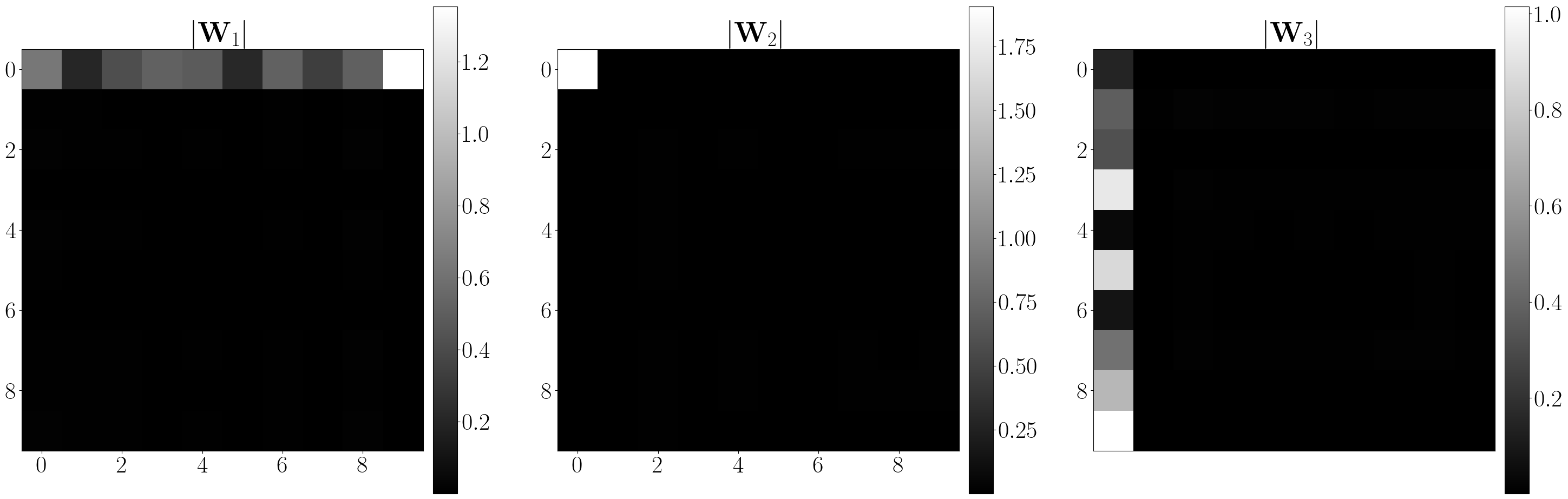}
			\caption{Weights at iteration 56000}
			\label{fig:all_weights_it_lin}
		\end{subfigure}
	\end{minipage}
	\hfill
	\begin{minipage}[t]{0.45\textwidth}
		\centering
		\begin{subfigure}[t]{\linewidth}
			\centering
			\includegraphics[width=\linewidth]{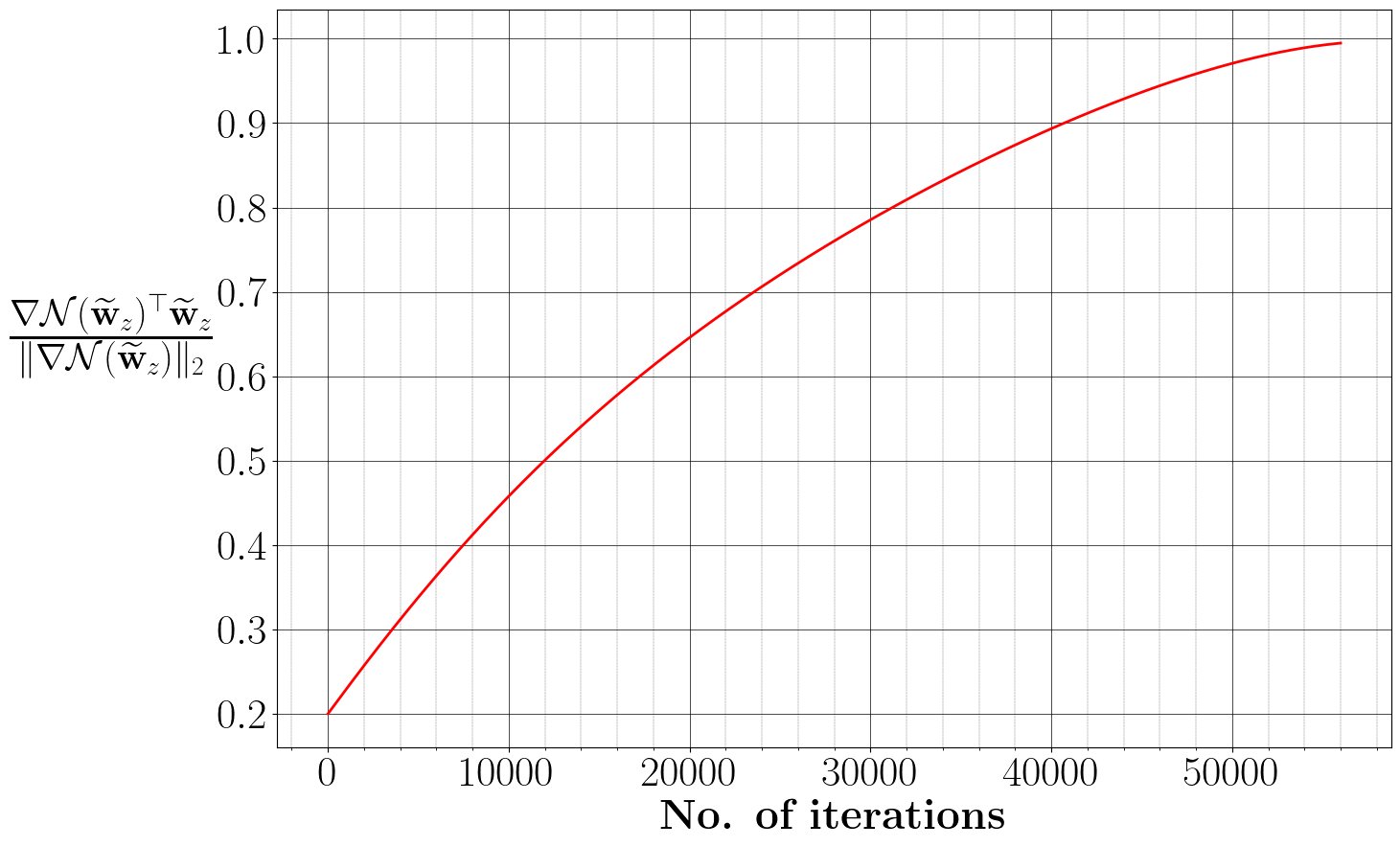}
			\caption{Evolution of inner product between gradient of the NCF and the weights}
			\label{fig:ncf_evol_lin}
		\end{subfigure}
		\vspace{0.3cm}
		
		\begin{subfigure}[t]{\linewidth}
			\centering
			\includegraphics[width=\linewidth]{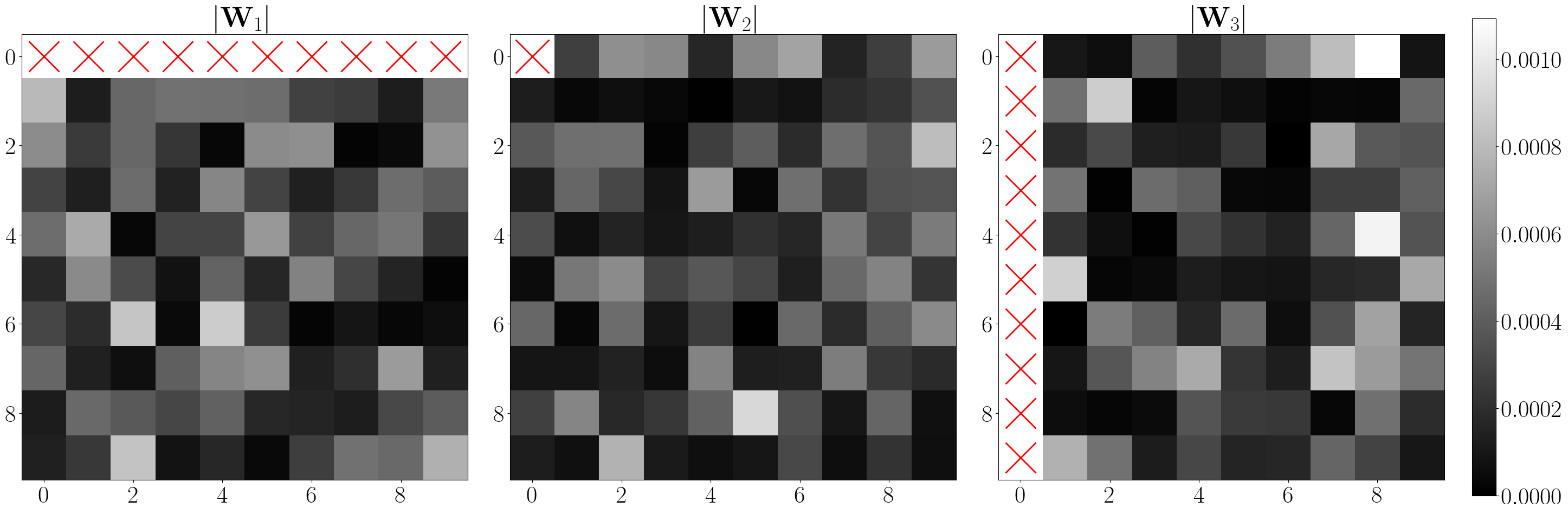}
			\caption{Weights belonging to $\rvw_z$ at initialization}
			\label{fig:wz_init_lin}
		\end{subfigure}
		\vspace{0.3cm}
		
		\begin{subfigure}[t]{\linewidth}
			\centering
			\includegraphics[width=\linewidth]{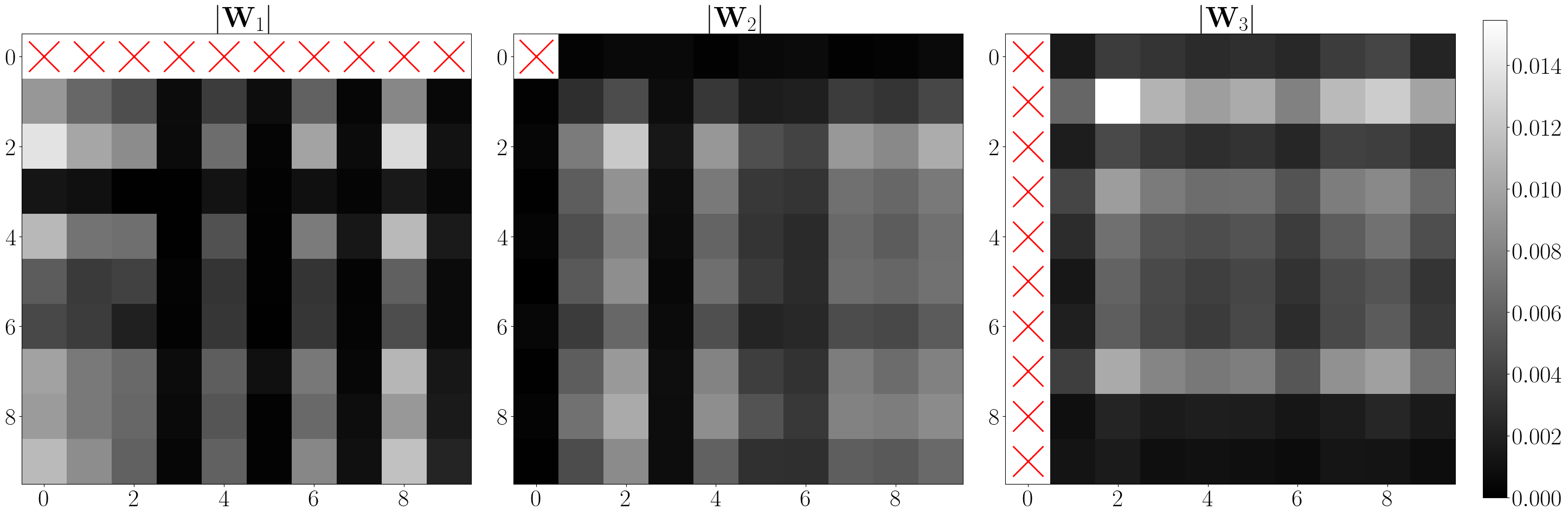}
			\caption{Weights belonging to $\rvw_z$ at iteration 56000}
			\label{fig:wz_it_lin}
		\end{subfigure}
	\end{minipage}
	\caption{(\textbf{Gradient descent dynamics near saddle point; deep linear network}) We train a three-layer linear neural network using gradient descent whose output is $\rmW_3\rmW_2\rmW_1,$ where $\rmW_3\in \mathbb{R}^{10\times 10},\rmW_2  \in \mathbb{R}^{10 \times 10}$, $\rmW_1\in\mathbb{R}^{10 \times 10}$ are the trainable weights. The entries of the label matrix $\rmS\in \mathbb{R}^{10 \times 10}$ are drawn from the standard normal distribution. The weights are initialized near a saddle point where the incoming and outgoing weights of the last nine neurons of each layer is zero, just as the saddle point defined in \cref{sad_linear}; these weights form $\rvw_z$ and remaining form $\rvw_n$. Panel (a) depicts the evolution of the training loss (normalized by the loss at the saddle point) and the distance of the weights from the saddle point (normalized by the norm of the weights at the saddle point). Panels (b) and (c) show the weights at initialization and at iteration 56000, respectively. We observe that the training loss does not change much, the weights remain near the saddle point and $\rvw_z$ remains small. Panel (d) shows the evolution of ${\nabla\mathcal{N}}_{{\overline{\rmS}},\overline{\mathcal{H}}_2}(\widetilde{\rvw}_z)^\top\widetilde{\rvw}_z/\|{\nabla\mathcal{N}}_{{\overline{\rmS}},\overline{\mathcal{H}}_2}(\widetilde{\rvw}_z)\|_2 $, where  $\widetilde{\rvw}_z \coloneqq \rvw_z/\|\rvw_z\|_2$, which confirms that $\rvw_z$ has converged in direction to a KKT point of the constrained NCF defined with respect to the residual error $\overline{\rmS}$ and $\overline{\mathcal{H}}_2({\rvw}_z)\coloneqq\mathcal{H}_2(\overline{\rvw}_n,\rvw_z)$. Panels (e) and (f) depict the weights belonging to $\rvw_z$ of every layer, at initialization and at iteration 56000 respectively, where the weights belonging to $\rvw_n$ are crossed out.}
	\label{fig:zero_ncf_3l_lin}
\end{figure}
\subsection{Maximizing Sum of Homogeneous Functions via Gradient Flow}
\label{max_sum_homogeneous}
This section contains the proof of Lemma \ref{2hm_gf} and Lemma \ref{4hm_gf}.\\
\textbf{Proof of Lemma \ref{2hm_gf}}:  We will first show that there exists a time $T$ such that
\begin{equation}
\|\rvw_i(t)\|_2 = \alpha_i(t)e^{2t\mathcal{G}_i(\rvw_i^*)}, \text{for all }t\geq T\text{ and for all }i\in [m],
\label{2hm_nm_wi}
\end{equation}
where $T$ is sufficiently large. Also, $\alpha_i(t) \in [\kappa_1,\kappa_2]$, for all $t\geq T$ and $i\in [m]$, where $\kappa_2\geq\kappa_1>0$. If the above equation is true, then the proof can be finished in the following way. Suppose $\zeta \coloneqq \max_{j\in [m]} \mathcal{G}_j(\rvw_j^*)$. For any $i\in [m]$, if $\mathcal{G}_i(\rvw_i^*)< \zeta$, then
\begin{equation*}
\lim_{t\rightarrow\infty} \frac{\|\rvw_i(t)\|_2}{\sqrt{\sum_{i=1}^m\|\rvw_i(t)\|_2^2}} \leq \lim_{t\rightarrow\infty} \frac{\kappa_2e^{2t\mathcal{G}_i(\rvw_i^*)}}{\kappa_1e^{2t\zeta}} = 0.
\end{equation*}
Else, if $\mathcal{G}_i(\rvw_i^*)=\zeta$, then
\begin{equation*}
\lim_{t\rightarrow\infty} \frac{\|\rvw_i(t)\|_2}{\sqrt{\sum_{i=1}^m\|\rvw_i(t)\|_2^2}} \geq \lim_{t\rightarrow\infty} \frac{\kappa_1e^{2t\zeta}}{\sqrt{m}\kappa_2e^{2t\zeta}} > 0.
\end{equation*}
We next show that \cref{2hm_nm_wi} is true. For all $i\in [m]$, since $\rvw_i^*$ is a second-order KKT point, from \citet[Lemma 24]{kumar_escape}, there exists a $\gamma\in (0,1)$ such that if $\rvs_{i0}^\top\rvw_i^* > 1-\gamma$ and $\|\rvs_{i0}\|_2 =1$, then the solution $\rvs_i(t)$ of 
\begin{equation*}
\dot{\rvs}_i = \nabla \mathcal{G}_i(\rvs_i), \rvs_i(0) = \rvs_{i0},
\end{equation*}
satisfies $\|\rvs_i(t)\|_2 = \beta_i(t)e^{2t\mathcal{G}_i(\rvw_i^*)}$, for all $t\geq 0$. Here, $\beta_i(t) \in [\eta_1,\eta_2]$, for all $t\geq 0$ and $i\in [m]$, where $\eta_2\geq\eta_1>0$. \\
Now, we may assume that there exists a sufficiently large time $T^*$ such that
\begin{equation*}
\frac{\rvw_i(t)^\top\rvw_i^*}{\|\rvw_i(t)\|_2} > 1-\gamma, \text{for all }t\geq T^*\text{ and for all }i\in [m].
\end{equation*}
Suppose $\widetilde{\rvs}_i(t)$ denotes the solution of 
\begin{equation*}
\dot{\rvs}_i = \nabla \mathcal{G}_i(\rvs_i), \rvs_i(0) = \rvw_i(T^*)/\|\rvw_i(T^*)\|_2,
\end{equation*}
then $\|\widetilde{\rvs}_i(t)\|_2 = \beta_i(t)e^{2t\mathcal{G}_i(\rvw_i^*)}$, for all $t\geq 0$, and $\beta_i(t) \in [\eta_1,\eta_2]$, for all $t\geq 0$ and $i\in [m]$, where $\eta_2\geq\eta_1>0$. Also, using \Cref{traj_init_eq}, we have
\begin{equation*}
\widetilde{\rvs}_i(t) = \frac{\rvw_i(t+T^*)}{\|\rvw_i(T^*)\|_2}, \text{ for all }t\geq 0.
\end{equation*} 
Therefore, $\|\rvw_i(t)\|_2 = \alpha_i(t)e^{2\mathcal{G}_i(\rvw_i^*)t}$ and  $\alpha_i(t)\in [\theta_1\eta_1,\theta_2\eta_2]$, for all $t\geq T^*$, where $\theta_1=\min_{i\in [m]}\|\rvw_i(T^*)\|_2 $ and $\theta_2=\max_{i\in [m]}\|\rvw_i(T^*)\|_2$. \hfill \ensuremath{\blacksquare}\\
\textbf{Proof of Lemma \ref{4hm_gf}}: For all $i\in [m]$, since $\mathcal{G}_i(\rvw_{i0})>0$, from \citet[Lemma 13]{early_dc}, we get $\|\rvw_i(t)\|_2 \geq \|\rvw_{i0}\|_2 = 1>0, $ for all $t\geq 0$. Next, note that
\begin{equation*}
\frac{d}{dt}\left(\frac{\mathcal{G}_i(\rvw_i)}{\|\rvw_i\|_2^L}\right) = \nabla  \mathcal{G}_i(\rvw_i)^\top\left(\mathbf{I}-\frac{\rvw_i\rvw_i^\top}{\|\rvw_i\|_2^2}\right)\frac{\nabla  \mathcal{G}_i(\rvw_i)}{\|\rvw_i\|_2^L} \geq 0,
\end{equation*}
which implies $\mathcal{G}_i(\rvw_i(t)/\|\rvw_i(t)\|_2)$ increases with time. Hence, since $\|\rvw_{i0}\|_2 = 1$, we get that
\begin{align}
&\frac{1}{2}\frac{d\|\rvw_i(t)\|_2^2}{dt} = L\mathcal{G}_i(\rvw_i) = L\|\rvw_i\|_2^L\mathcal{G}_i(\rvw_{i}/\|\rvw_{i}\|_2)  \geq L\|\rvw_i\|_2^L\mathcal{G}_i(\rvw_{i0}/\|\rvw_{i0}\|_2)\geq 0\label{Lhm_nm_bd}
\end{align}
implies
\begin{align*}
&\frac{d\|\rvw_i\|_2}{dt} \geq L\|\rvw_i\|_2^{L-1}\mathcal{G}_i(\rvw_{i0}).
\end{align*}
Taking $\|\rvw_i\|_2^{L-1}$ to the LHS and integrating from $0$ to $t\in [0,T_i]$, we get
\begin{equation*}
\frac{1}{L-2}\left(\frac{1}{\|\rvw_i(0)\|_2^{L-2}} - \frac{1}{\|\rvw_i(t)\|_2^{L-2}}\right) \geq L\mathcal{G}_i(\rvw_{i0})t.
\end{equation*}
Using $\|\rvw_i(0)\|_2= \|\rvw_{i0}\|_2=1$, and simplifying the above equation we get
\begin{equation}
\|\rvw_i(t)\|_2^{L-2} \geq \frac{	1}{1- tL(L-2)\mathcal{G}_i(\rvw_{i0})}.
\label{eq_bd_wt}
\end{equation}
Since $\lim_{t\rightarrow T_i} {\|\rvw_i(t)\|_2} = \infty$, \cref{eq_bd_wt} implies $T_i\leq 1/(L(L-2)\mathcal{G}_i(\rvw_{i0}))$. Next, since $\mathcal{G}_i(\rvw_i(t)/\|\rvw_i(t)\|_2)$ increases with time and $\lim_{t\rightarrow T_i} \mathcal{G}_i(\rvw_i(t)/\|\rvw_i(t)\|_2) = \mathcal{G}_i(\rvw_i^*)$, we get
\begin{equation}
\frac{1}{2}\frac{d\|\rvw_i(t)\|_2^2}{dt} = L\mathcal{G}_i(\rvw_i) \leq L\|\rvw_i\|_2^L\mathcal{G}_i(\rvw_i^*) ,
\end{equation}
Simplifying the above equation similarly gives us
\begin{equation*}
\|\rvw_i(t)\|_2^{L-2} \leq \frac{	1}{1- tL(L-2)\mathcal{G}_i(\rvw_{i}^*)}.
\end{equation*}
Since $\lim_{t\rightarrow T_i} {\|\rvw_i(t)\|_2} = \infty$, the above equation implies $T_i\geq 1/(L(L-2)\mathcal{G}_i(\rvw_{i}^*))$.\\
Next, if $T_i = T^*$, then $\lim_{t\rightarrow T_i} {\|\rvw_i(t)\|_2} = \infty$, and if $T_i > T^*$, then $\lim_{t\rightarrow T_i} {\|\rvw_i(t)\|_2} < \infty$. Since $T_i=T^*$, for some $i\in [m]$, we get
\begin{equation*}
\lim_{t\rightarrow T_i} \sum_{i=1}^m{\|\rvw_i(t)\|_2^2} = \infty.
\end{equation*}
Finally, from the proof of  \citet[Lemma 2]{early_dc}, we have
\begin{equation*}
\|\rvw_i(t)\|_2 = \frac{\alpha_i(t)}{(T_i-t)^{1/(L-2)}}, \text{ for all } t\in [0,T_i] \text{ for all } i\in [m],
\end{equation*}
where $\alpha_i(t)\in [\kappa_1,\kappa_2]$. Hence, if $T_i > T^*$, then $\|\rvw_i(T^*)\|_2 < \infty$, which implies
\begin{equation*}
\lim_{t\rightarrow T^*} \frac{\|\rvw_i(t)\|_2}{\sqrt{\sum_{j=1}^m\|\rvw_j(t)\|_2^2}} = \lim_{t\rightarrow T^*} \frac{\|\rvw_i(T^*)\|_2}{\sqrt{\sum_{j=1}^m\|\rvw_j(t)\|_2^2}} = 0.
\end{equation*}
Next, if $T_i = T^*$, then
\begin{equation*}
\lim_{t\rightarrow T^*} \frac{\|\rvw_i(t)\|_2}{\sqrt{\sum_{j=1}^m\|\rvw_j(t)\|_2^2}} = \lim_{t\rightarrow T^*}\frac{\alpha_i(t)}{\sqrt{\sum_{j=1}^m\alpha_j^2(t)\left(\frac{T^*-t}{T_j-t}\right)^{2/(L-2)}}} \geq \frac{\kappa_1}{\sqrt{m}\kappa_2}> 0.
\end{equation*}
This completes the proof \hfill \ensuremath{\blacksquare}
\subsection{Equivalence between OMP and NP for Diagonal Linear Networks}
\label{omp_np_eq}
In this subsection, we show that for two-layer diagonal linear networks, the NP algorithm is equivalent to OMP, under certain assumptions. For an input $\rvx\in \sR^d$, the network output is $\mathcal{H}(\rvx;\rvv,\rvu) = \rvx^\top(\rvv\odot\rvu),$ where $\rvu, \rvv \in \sR^d$ are the weights of first and second layer, respectively. Thus, there are $d$ number of neurons, where $u_j$ and $v_j$, the $j$th entry of $\rvu$ and $\rvv$, represent the incoming and outgoing weights of the $j$th neuron. Suppose $\{\rvx_i,y_i\}_{i=1}^n$ is the training dataset, and let $\rmX = \left[\rvx_1,\cdots,\rvx_n\right]^\top\in \sR^{n\times d}$, $\rvy = \left[y_1,\cdots,y_n\right]^\top\in \sR^n$. The training loss can be written as
\begin{equation*}
\mathcal{L}(\rvv,\rvu) = \frac{1}{2}\sum_{i=1}^n(\rvx_i^\top(\rvv\odot\rvu)-y_i)^2 = \frac{1}{2}\left\|\rmX(\rvv\odot\rvu)-\rvy\right\|_2^2.
\end{equation*}
Since the $j$th neuron only takes the $j$th coordinate as input, the neurons here can not be exchanged without changing the network's output, unlike feed-forward networks.  Consequently, the constrained NCF would be different for each neuron, even though they are in the same layer. We now examine the iterations of NP algorithm for this network.\\
\textbf{First iteration: }We begin by maximizing the constrained NCF for each neuron:
\begin{equation*}
\max_{v_j^2+u_j^2 = 1} v_ju_j\rmX[:,j]^\top\rvy, \text{ for all }j\in [n].
\end{equation*}
Each problem admits a closed-form expression of its global maximum: $|\rmX[:,j]^\top\rvy|/	{2}$, for $(\overline{v}_j,\overline{u}_j) = (1/\sqrt{2}, \text{sign}(\rmX[:,j]^\top\rvy)/\sqrt{2})$. Let ${j}_1 = \arg\max_{j\in [n]}|\rmX[:,j]^\top\rvy|$, then the most dominant KKT point corresponds to the ${j}_1$th neuron. Thus, we minimize the following training loss via gradient descent:
\begin{equation*}
\mathcal{L}(u_{{j}_1},v_{{j}_1}) = \frac{1}{2}\|\rmX[:,{j}_1] v_{{j}_1}u_{{j}_1}-\rvy\|_2^2,
\end{equation*}
with initialization $(v_{{j}_1}(0),u_{{j}_1}(0)) = (\delta/\sqrt{2},\delta\text{sign}(\rmX[:,j_1]^\top\rvy)/\sqrt{2})$. 

In Lemma \ref{gm_gd} (below), we  prove that for sufficiently small $\delta$ and step-size, gradient descent converges to the global minimum. Importantly, aside from the origin and global minimum, there are no spurious stationary points. Moreover, the global minimum of the above problem is same as the global minimum of the following convex problem:
\begin{equation*}
g(\beta) = \frac{1}{2}\|\rmX[:,{j}_1] \beta-\rvy\|_2^2,
\end{equation*}
If $\beta^*$ is its optimal solution, then the residual $\overline{\rvy} = \rvy - \rmX[:,{j}_1] \beta^*$. In summary, maximizing the constrained NCF for each neuron is equivalent to identifying which column of $\rmX$ has highest absolute correlation with $\rvy$. Then, that column is used to find the best fit for the labels $\rvy$. These are exactly the steps of the first iteration of OMP, \cite{omp_org}.\\
\textbf{Second iteration: } Let $(v_{j_1}^*,u_{j_1}^*)$ denote the weights after the first iteration. We next maximize the constrained NCF for each remaining neuron:
\begin{equation*}
\max_{v_j^2+u_j^2 = 1} v_ju_j\rmX[:,j]^\top\overline{\rvy}, \text{ for all }j\in [n]\symbol{92}\{j_1\}.
\end{equation*}
Again, the global maximum is $|\rmX[:,j]^\top\overline{\rvy}|/\sqrt{2}$, for $(\overline{v}_j,\overline{u}_j) = (1/\sqrt{2}, \text{sign}(\rmX[:,j]^\top\overline{\rvy})/\sqrt{2})$. Let ${j}_2 = \arg\max_{j\in [n]\symbol{92}\{j_1\}}|\rmX[:,j]^\top\overline{\rvy}|$, then the most dominant KKT point corresponds to the ${j}_2$th neuron. Thus, we minimize the following training loss via gradient descent:
\begin{equation*}
\mathcal{L}(u_{{j}_1},v_{{j}_1},u_{{j}_2},v_{{j}_2}) = \frac{1}{2}\|\rmX[:,{j}_1] v_{{j}_1}u_{{j}_1}+\rmX[:,{j}_2] v_{{j}_2}u_{{j}_2}-\rvy\|_2^2,
\end{equation*}
with initialization $(v_{{j}_2}(0),u_{{j}_2}(0)) = (\delta/\sqrt{2},\delta\text{sign}(\rmX[:,j_2]^\top\overline{\rvy})/\sqrt{2})$ and $(v_{{j}_1}(0),u_{{j}_1}(0)) = (v_{j_1}^*,u_{j_1}^*)$. However, unlike the first iteration, here additional stationary points appear, making global convergence of gradient descent to global minimum harder to prove. Nevertheless, if we assume convergence to the global minimum, then this step is equivalent to finding the the global minimum of 
\begin{equation*}
g({\beta_1,\beta_2}) = \frac{1}{2}\|\rmX[:,{j}_1] \beta_1+\rmX[:,{j}_2] \beta_2-\rvy\|_2^2.
\end{equation*}
Thus, in the second iteration as well, maximizing the constrained NCF leads to identifying which remaining column of $\rmX$ has highest absolute correlation with the residual error. Then, this new column, along with the previous one, is used to find the best fit for the labels $\rvy$. This is exactly same as the second iteration of the OMP algorithm.\\
The subsequent iterations are same as second iteration, and matches with the iterations of OMP, under the same assumption: the gradient descent iterates converge to the global minimum. In summary, for two-layer diagonal linear networks, the NP algorithm is equivalent to OMP, provided gradient descent reaches the global minimum at each iteration.\\
\textbf{Experiment: }We next empirically demonstrate that the NP algorithm for two-layer diagonal networks does match with the OMP solution and is different from the minimum $\ell_1$-norm solution. Suppose
\begin{equation*}
\rmX =  \begin{bmatrix}
1 & 0 & -0.1 \\
0 & 1 & 1+\frac{0.2}{3} \\
\end{bmatrix} \text{ and } \rvb =  \begin{bmatrix}
1 \\ 2
\end{bmatrix}.
\end{equation*}
\textbf{Minimizing  $\ell_1$-norm. }Consider the following optimization problem:
\begin{equation*}
\min_\rvz \|\rvz\|_1, \text{ such that } \rmX\rvz = \rvb.
\end{equation*}
The solution of the above optimization problem is $\rvz^* = [1,2,0]^\top$, where $\|\rvz^*\|_1 = 3$.\\
\textbf{OMP. }In the first iteration of OMP, the third column has maximum absolute correlation with $\rvb$, which is used to find the best fit for $\rvb$. This produces $\rvz_1 = [0,0,1.772].$

In the second iteration, the first column has maximum absolute correlation with the residual error, which is used, along with the third column, to find the best fit for $\rvb$. This produces $\rvz_2 = [1.1875,0,1.875],$ which exactly fits $\rvb$, and the algorithm stops here. Thus, OMP finds a sparse solution different from the minimum $\ell_1$-norm solution.\\
\textbf{NP. }In the first iteration, third neuron has the most dominant KKT point. We minimize
\begin{equation*}
\mathcal{L}(u_3,v_3) = \frac{1}{2}\|\rmX[:,3] v_3u_3-\rvb\|_2^2,
\end{equation*} 
via gradient descent with step-size $0.001$ and initialization $(v_3(0),u_3(0)) = (\delta/\sqrt{2},\delta/\sqrt{2})$, where $\delta=0.01$. The gradient descent converges to $(v_3^*,u_3^*) = (1.33,1.33)$, implying $v_3^*u_3^* = 1.77$. Hence, the output of the first iteration is same as the first iteration of OMP.

In the second iteration, the first neuron has the most dominant KKT point. We minimize
\begin{equation*}
\mathcal{L}(u_1,v_1,u_3,v_3) = \frac{1}{2}\|\rmX[:,1] v_1u_1+\rmX[:,3] v_3u_3-\rvb\|_2^2,
\end{equation*} 
via gradient descent with step-size $0.001$ and initialization $(v_3(0),u_3(0)) = (1.33,1.33),$ $(v_1(0),u_1(0)) = (\delta/\sqrt{2},\delta/\sqrt{2})$, where $\delta=0.01$. The gradient descent converges to $(v_3^*,u_3^*) = (1.37,1.37), (v_1^*,u_1^*) = (1.09,1.09)$, implying $v_3^*u_3^* = 1.876, v_1^*u_1^* = 1.188$. Hence, the output of the second iteration is also same as the second iteration of OMP. Therefore, the output of the NP algorithm is same as OMP, and is different from the minimum $\ell_1$-norm solution. 
\begin{lemma}\label{gm_gd}
	Consider minimizing the following optimization problem via gradient descent:
	\begin{equation*}
	\mathcal{L}(u,v) \coloneqq \frac{1}{2}\|\rvx vu-\rvy\|_2^2,
	\end{equation*}
	where $\rvx, \rvy\in \sR^d$ and, without loss of generality, let $\|\rvx\|_2=1$. Suppose the initial weights are $(v(0),u(0)) = (\delta/\sqrt{2},\delta\text{sign}(\rvx^\top\rvy)/\sqrt{2})$. Then, for sufficiently small $\delta>0$ and step-size, the gradient descent converges to the global minimum, that is, $v(\infty)u(\infty) = \rvx^\top\rvy$.
\end{lemma}
\begin{proof}
	We first show that the stationary points of $\mathcal{L}(u,v)$ are either the origin or the global minimum. At any stationary point $(\overline{u},\overline{v})$, we have
	\begin{align*}
	&{0} = \nabla_{v}\mathcal{L}(\overline{u},\overline{v}) = \overline{u}\rvx^\top(\rvx \overline{v}\overline{u}-\rvy) = \overline{u}( \overline{v}\overline{u}-\rvx^\top\rvy),\\
	&{0} = \nabla_{u}\mathcal{L}(\overline{u},\overline{v}) = \overline{v}\rvx^\top(\rvx \overline{v}\overline{u}-\rvy) = \overline{v}( \overline{v}\overline{u}-\rvx^\top\rvy).
	\end{align*}
	The origin is clearly a stationary point. If either $\overline{v}$ or $\overline{u}$ is non-zero, then $\overline{v}\overline{u}=\rvx^\top\rvy$, that is, the stationary point is a global minimum. \\
	The next set of inequalities show that for all sufficiently small $\delta>0$, the loss at initialization is strictly smaller than at the origin. 
	\begin{align*}
	\mathcal{L}(u(0),v(0)) &= \frac{1}{2}\left(u(0)^2v(0)^2 - 2u(0)v(0)\rvx^\top\rvy + \|\rvy\|_2^2\right)\\
	&=  \frac{1}{2}\left(\frac{\delta^4}{4} - {\delta^2}|\rvx^\top\rvy| + \|\rvy\|_2^2\right)\\
	&\leq \frac{1}{2}\left( - \frac{\delta^2}{2}|\rvx^\top\rvy| + \|\rvy\|_2^2\right) \leq \|\rvy\|_2^2/2 = \mathcal{L}(0,0),
	\end{align*}
	where the first inequality is true for all sufficiently small $\delta>0$. 
	
	We next show that for sufficiently small $\delta>0$ and step-size, gradient descent converges to a stationary point while the loss decreases at each step. Since under these conditions convergence to the origin is impossible, the iterates must approach the global minimum.
	
	We first prove some key properties of gradient descent iterates written below.
	\begin{align*}
	&v(t+1) = v(t) - \eta u(t)\rvx^\top(\rvx v(t)u(t)-\rvy) = v(t)-\eta u(t)(v(t)u(t)-\rvx^\top\rvy),\\
	&u(t+1) = u(t) - \eta v(t)\rvx^\top(\rvx v(t)u(t)-\rvy) = u(t)-\eta v(t)(v(t)u(t)-\rvx^\top\rvy).
	\end{align*}
	If $\rvx^\top\rvy>0$, then $v(t)=u(t)$, for all $t\geq 0$. This can be shown via induction. It is trivially true for $t=0$. Suppose it holds for some $t\geq 0$. Then,
	\begin{equation*}
	v(t+1) = v(t)-\eta u(t)(v(t)u(t)-\rvx^\top\rvy)=u(t)-\eta v(t)(v(t)u(t)-\rvx^\top\rvy) = u(t+1),
	\end{equation*}
	where the second equality uses $v(t)=u(t)$. This proves the claim. If $\rvx^\top\rvy<0$, then $v(t)=-u(t)$, for all $t\geq 0$. This can be shown similarly via induction. 
	
	We next show that $|v(t)|\leq \sqrt{|\rvx^\top\rvy|}$, if $\eta\in (0,1/(2|\rvx^\top\rvy|))$ and $\delta\in (0,\sqrt{2|\rvx^\top\rvy|})$. If $\rvx^\top\rvy>0$,  then $v(t)=u(t)$, implying 
	\begin{equation*}
	v(t+1) = v(t)-\eta v(t)(v^2(t)-\rvx^\top\rvy) = v(t)(1-\eta(v^2(t)-\rvx^\top\rvy)).
	\end{equation*}
	Let $h(v) \coloneqq v(1-\eta(v^2-\rvx^\top\rvy)).$ Then, $h(0) = 0$ and $h\left(\sqrt{\rvx^\top\rvy}\right) = \sqrt{\rvx^\top\rvy}$. Also, for all $v\in [0,\sqrt{\rvx^\top\rvy}]$ and $\eta\in (0,1/(2|\rvx^\top\rvy|))$,
	\begin{equation*}
	h'(v) = (1-\eta(v^2-\rvx^\top\rvy)) - 2\eta v^2 = 1- 3\eta v^2 + \eta\rvx^\top\rvy \geq 1-2\eta\rvx^\top\rvy> 0.
	\end{equation*}
	Thus, if $v\in [0,\sqrt{\rvx^\top\rvy}]$, then $h(v)\in [0,\sqrt{\rvx^\top\rvy}]$. Since $v(t+1) = h(v(t))$ and $v(0) = \delta/\sqrt{2}< \sqrt{\rvx^\top\rvy}$, we get $|v(t)|\leq \sqrt{|\rvx^\top\rvy|}$, for all $t\geq 0$.\\
	Next, if  $\rvx^\top\rvy<0$,  then $v(t)=-u(t)$ and 
	\begin{equation*}
	v(t+1) = v(t)+\eta v(t)(-v^2(t)-\rvx^\top\rvy) = v(t)(1-\eta(v^2(t)+\rvx^\top\rvy)).
	\end{equation*}
	Let $g(v) \coloneqq v(1-\eta(v^2+\rvx^\top\rvy)).$ Then, $g(0) = 0$ and $g\left(\sqrt{-\rvx^\top\rvy}\right) = \sqrt{-\rvx^\top\rvy}$. Also, for all $v\in [0,\sqrt{-\rvx^\top\rvy}]$ and $\eta\in (0,1/(2|\rvx^\top\rvy|))$,
	\begin{equation*}
	g'(v) = (1-\eta(v^2+\rvx^\top\rvy)) - 2\eta v^2 = 1- 3\eta v^2 - \eta\rvx^\top\rvy \geq 1+2\eta\rvx^\top\rvy\geq 0.
	\end{equation*}
	 Thus, if $v\in [0,\sqrt{-\rvx^\top\rvy}]$, then $g(v)\in [0,\sqrt{-\rvx^\top\rvy}]$. Since $v(t+1) = g(v(t))$ and $v(0) = \delta/\sqrt{2}\leq \sqrt{-\rvx^\top\rvy}$, we get $|v(t)|\leq \sqrt{|\rvx^\top\rvy|}$, for all $t\geq 0$.
	
	Since the gradient descent iterates remain bounded, the loss $\mathcal{L}(v,u)$ has Lipschitz gradient along the entire trajectory. Consequently, as shown in \citet[Chapter 1.2.3]{nesterov_cvx}, gradient descent with a sufficiently small step size converges to a stationary point while the loss decreases at each iteration. Moreover, for all sufficiently small $\delta$, the loss at initialization is strictly lower than the loss at the origin. Therefore, gradient descent cannot converge to the origin and must instead converge to the global minimum. This completes the proof.
\end{proof}
\subsection{Impact of Rescaling the Weights}\label{scale_wt}
Consider training a three-layer ReLU network with the NP algorithm. Let $(\overline{\rmW}_1,\overline{\rmW}_2,\overline{\rmW}_3)$ be the set of learned weights at the end of some iteration, and let $\overline{\rvy}$ be the corresponding residual error. To add a neuron in the next iteration, the following two objectives are maximized via projected gradient ascent:
\begin{align*}
&\max_{\|\rvb_3\|_2^2+\|\rva_2\|_2^2=1}\sum_{i=1}^n\overline{y}_i\rvb_3\sigma(\rva_2^\top\sigma(\overline{\rmW}_1\rvx_i)), \text{ and }\\ &\max_{\|\rvb_2\|_2^2+\|\rva_1\|_2^2=1}\sum_{i=1}^n\overline{y}_i\overline{\rmW}_3 \text{diag}(\sigma'(\overline{\rmW}_2\sigma(\overline{\rmW}_1\rvx_i)))\rvb_2\sigma(\rva_1^\top\rvx_i).
\end{align*}
Although $(\overline{\rmW}_1,\overline{\rmW}_2,\overline{\rmW}_3)$ is a stationary point of the training loss, it is not isolated: symmetries present in the network generate continuous manifolds of stationary points. For example, for any $c>0$, $(c\overline{\rmW}_1,\overline{\rmW}_2,\overline{\rmW}_3/c)$ is also a stationary point of the training loss. 

Now, suppose instead of $(\overline{\rmW}_1,\overline{\rmW}_2,\overline{\rmW}_3)$, the set of learned weights are $(c\overline{\rmW}_1,\overline{\rmW}_2,\overline{\rmW}_3/c)$, for some $c>0$. We next describe the impact of this rescaling of the weights on the addition of the neuron in the next iteration. Note that, the network output does not change because $\sigma(\overline{\rmW}_3/c\sigma(\overline{\rmW}_2\sigma(c\overline{\rmW}_1\rvx))) = \sigma(\overline{\rmW}_3\sigma(\overline{\rmW}_2\sigma(\overline{\rmW}_1\rvx))), $ where the equality follows from $1$-positive homogeneity of the ReLU activation. Hence, the residual error is unchanged by this rescaling. To add a neuron in this case, the following two function will be maximized via projected gradient ascent:
\begin{align*}
&\max_{\|\rvb_3\|_2^2+\|\rva_2\|_2^2=1}c\sum_{i=1}^n\overline{y}_i\rvb_3\sigma(\rva_2^\top\sigma(\overline{\rmW}_1\rvx_i)), \text{ and } \\ &\max_{\|\rvb_2\|_2^2+\|\rva_1\|_2^2=1}\frac{1}{c}\sum_{i=1}^n\overline{y}_i\overline{\rmW}_3 \text{diag}(\sigma'(\overline{\rmW}_2\sigma(\overline{\rmW}_1\rvx_i)))\rvb_2\sigma(\rva_1^\top\rvx_i),
\end{align*}
where we used $0$-homogeneity of $\sigma'(x)$. Thus, one objective is multiplied by $c$ and the other by $1/c$. If $c$ deviates sufficiently from $1$, the most dominant KKT point in this case can be different from the previous case, and consequently a different neuron will be added. In short, although rescaling preserves network output and stationarity, it can alter which neuron the NP algorithm chooses to add.

The above discussion raises a natural question: what is the appropriate scaling of the weights? To answer this, we turn to gradient flow dynamics. For ReLU  activation and small initialization, the norm of incoming and outgoing weights of each hidden neuron	 remains nearly balanced throughout training \citep{du_bal}. In fact, if weights are balanced at initialization, they stay balanced during training. Consequently, an appropriate scaling is that the weights of each hidden neuron are balanced. However, gradient descent with large step-size and/or large number of iterations can produce imbalance. For such cases, we can apply the algorithm of \citet{saul_bal} to restore balance without altering the network's output.

The NP algorithm also implicitly encourages balancedness in the weights. In the first iteration, the weights are initialized along a KKT point of a constrained NCF, which are balanced (Lemma \ref{bal_r1_kkt}). Thus, for small step-size, the weights can be expected to remain balanced during optimization via gradient descent. In later iterations, weights of the newly added neurons are along a KKT point of a constrained NCF, which are also balanced. Thus, if the weights at the end of the previous iteration were balanced and for small step-size, the weights can be expected to remain balanced during optimization via gradient descent. Nevertheless, if the weights become unbalanced due to some reason, applying the algorithm from \citet{saul_bal} at the end of each NP iteration can restore balance. Also, we balance the weights after every iteration of NP for the PVR task; in other experiments we found that weights typically remain nearly balanced throughout training.

Although our discussion focused on three-layer ReLU networks, the arguments extend to deeper networks. For activations of the form $\sigma(x) = \max(x,\alpha x)^p$ with $p\geq 2$, the only quantitative change is that the norm of incoming weights should be $\sqrt{p}$ times the outgoing weights, rather than being equal. 
\bibliography{sample}

\end{document}

%% file: math_commands.tex
\usepackage{amsmath,amsfonts,bm}

\def\eqref#1{equation~\ref{#1}}

\def\1{\bm{1}}

\def\rva{{\mathbf{a}}}
\def\rvb{{\mathbf{b}}}

\def\rve{{\mathbf{e}}}

\def\rvg{{\mathbf{g}}}
\def\rvh{{\mathbf{h}}}
\def\rvu{{\mathbf{i}}}

\def\rvn{{\mathbf{n}}}

\def\rvp{{\mathbf{p}}}
\def\rvq{{\mathbf{q}}}
\def\rvr{{\mathbf{r}}}
\def\rvs{{\mathbf{s}}}
\def\rvt{{\mathbf{t}}}
\def\rvu{{\mathbf{u}}}
\def\rvv{{\mathbf{v}}}
\def\rvw{{\mathbf{w}}}
\def\rvx{{\mathbf{x}}}
\def\rvy{{\mathbf{y}}}
\def\rvz{{\mathbf{z}}}

\def\rmA{{\mathbf{A}}}
\def\rmB{{\mathbf{B}}}
\def\rmC{{\mathbf{C}}}

\def\rmN{{\mathbf{N}}}

\def\rmS{{\mathbf{S}}}

\def\rmW{{\mathbf{W}}}
\def\rmX{{\mathbf{X}}}
\def\rmY{{\mathbf{Y}}}

\DeclareMathAlphabet{\mathsfit}{\encodingdefault}{\sfdefault}{m}{sl}
\SetMathAlphabet{\mathsfit}{bold}{\encodingdefault}{\sfdefault}{bx}{n}

\def\gH{{\mathcal{H}}}

\def\sN{{\mathbb{N}}}

\def\sR{{\mathbb{R}}}
\def\sS{{\mathbb{S}}}

\DeclareMathOperator*{\argmax}{arg\,max}